\documentclass[11pt,twoside]{article}

\usepackage{fullpage}

\usepackage{epsf}
\usepackage{fancyhdr}
\usepackage{graphics}
\usepackage{graphicx}
\usepackage{psfrag}
\usepackage{microtype}
\usepackage{subfigure}
\usepackage{algorithmic}

\usepackage[linesnumbered,ruled]{algorithm2e}% http://ctan.org/pkg/algorithm2e
\DontPrintSemicolon
\usepackage{color}

\usepackage{amsthm}

\usepackage{amsfonts}
\usepackage{amsmath}
\usepackage{amssymb,bbm}
\usepackage{wrapfig}
\usepackage{tikz}
\usetikzlibrary{positioning}

\usepackage{natbib}

\usepackage{url}% for url's in bib
\usepackage[colorlinks,linkcolor=magenta,citecolor=blue, pagebackref=true,backref=true]{hyperref}
\renewcommand*{\backrefalt}[4]{%
    \ifcase #1 \footnotesize{(Not cited.)}%
    \or        \footnotesize{(Cited on page~#2.)}%
    \else      \footnotesize{(Cited on pages~#2.)}%
    \fi}

\long\def\comment#1{}
\usepackage{nicefrac}
\usepackage[normalem]{ulem}
\usepackage{chngpage}

 \usepackage{tabularx}%

\usepackage{enumitem}
\usepackage{booktabs}
\usepackage{caption}

\usepackage{mathtools}

\usepackage{fullpage}

\newtheorem{theorem}{Theorem}[section]

\newtheorem{lemma}[theorem]{Lemma}
\newtheorem{proposition}[theorem]{Proposition}

\newtheorem{remark}[theorem]{Remark}
\newtheorem{assumption}[theorem]{Assumption}

\usepackage{bm}
\usepackage{makecell}

\usepackage{bigints}

\newcommand{\E}{\mathbb{E}}
\newcommand{\R}{\mathbb{R}}
\newcommand{\Pro}{\mathbb{P}}
\newcommand{\I}{\mathbb{I}}
\newcommand{\N}{\mathbb{N}}

\newcommand{\calA}{\mathcal{A}}

\newcommand{\calD}{\mathcal{D}}

\newcommand{\calF}{\mathcal{F}}

\newcommand{\calO}{\mathcal{O}}
\newcommand{\calS}{\mathcal{S}}

\newcommand{\calN}{\mathcal{N}}
\newcommand{\calL}{\mathcal{L}}
\newcommand{\calU}{\mathcal{U}}
\newcommand{\calX}{\mathcal{X}}

\newcommand{\bI}{\mathbf{I}}

\newcommand{\ba}{\mathbf{a}}
\newcommand{\bb}{\mathbf{b}}

\newcommand{\bbf}{\mathbf{f}}

\newcommand{\bs}{\mathbf{s}}
\newcommand{\bg}{\mathbf{g}}

\newcommand{\bp}{\mathbf{p}}

\newcommand{\bw}{\mathbf{w}}
\newcommand{\bx}{\mathbf{x}}

\newcommand{\bz}{\mathbf{z}}

\newcommand{\dd}{\mathrm{d}}

\newcommand{\DIV}{\mathbf{div}}

\usepackage{adjustbox}

\newcommand{\hatpi}{{\color{blue}\hat{\pi}}}
\newcommand{\tildepi}{{\color{orange}\tilde{\pi}}}

\newcommand{\barpi}{{\color{red}{\bar{\pi}}}}

\newcommand{\bara}{{\color{red}{\bar{\ba}}}}

\newcommand{\hata}{{\color{blue}\hat{\ba}}}
\newcommand{\tildea}{{\color{orange}\tilde{\ba}}}

\usepackage{adjustbox}

\newcommand{\btheta}{\bm{\theta}}

\newcommand{\bepsilon}{\bm{\epsilon}}
\newcommand{\bphi}{\bm{\phi}}
\newcommand{\bpsi}{\bm{\psi}}

\newcommand{\bnabla}{\bm{\nabla}}

\renewcommand{\part}[2]{\dfrac{\partial #1}{\partial #2}}

\usepackage{enumitem}

\newcommand{\KL}{\mathrm{KL}}
\newcommand{\FI}{\mathrm{FI}}

\usepackage{extarrows}
\usepackage{caption}
\usepackage{subfigure}
\usepackage{graphicx}
\usepackage{sidecap}
\usepackage{float}
\usepackage{framed}
\usepackage{tabularx}
\usepackage{twoopt}
\usepackage{adjustbox}

\usepackage{lscape}

\usepackage{cancel}

\usepackage{tcolorbox}
\begin{document}

%%%%%%% TITLE PAGE %%%%%%%%%%%%%%%%%%%%%%%%%%%%%%%%%%%%%%%%%%%%%%%%%%%
%\\[.2cm]
\begin{center}

{\bf{\LARGE{Policy Representation via Diffusion Probability Model for Reinforcement Learning}}}
\footnote{* L.Yang and Z.Huang share equal contributions.}
\\
\vspace*{.2in}
{\large{
\begin{tabular}{ccc}
\textbf{Long Yang}$^{1,*}$, \textbf{Zhixiong Huang} $^{2,*}$, \textbf{Fenghao Lei}$^2$, \textbf{Yucun Zhong}$^3$, \textbf{Yiming Yang}$^4$, 
\\\textbf{Cong Fang}$^1$, \textbf{Shiting Wen}$^5$, \textbf{Binbin Zhou}$^6$, \textbf{Zhouchen Lin}$^1$
 \end{tabular}
}}

%\vspace*{.2in}

\begin{tabular}{c}
$^1$School of Artificial Intelligence, Peking University, Beijing, China\\
  $^2$College of Computer Science and Technology, Zhejiang University, China \\
  $^3$MOE Frontiers Science Center for Brain and Brain-Machine Integration \& College of\\ Computer Science, Zhejiang University, China. \\
  $^4$Institute of Automation Chinese Academy of Sciences Beijing, China\\
   $^5$School of Computer and Data Engineering, NingboTech University, China\\
  $^6$College of Computer Science and Technology, Hangzhou City University, China\\
  \texttt{\{yanglong001,fangcong,zlin\}@pku.edu.cn},\\
    \texttt{\{zx.huang,lfh,yucunzhong\}@zju.edu.cn},   \texttt{\{wensht\}@nit.zju.edu.cn}\\
    \texttt{yangyiming2019@ia.ac.cn},\texttt{bbzhou@hzcu.edu.cn}
\end{tabular}

\begin{center}
%Project Homepage:\url{xxxx}\\
Code:\url{https://github.com/BellmanTimeHut/DIPO}
\end{center}

\vspace*{.2in}

\today

\vspace*{.2in}

\begin{abstract} 
Popular reinforcement learning (RL) algorithms tend to produce a unimodal policy distribution, which weakens the expressiveness of complicated policy and decays the ability of exploration. The diffusion probability model is powerful to learn complicated multimodal distributions, which has shown promising and potential applications to RL.

In this paper, we formally build a theoretical foundation of policy representation via the diffusion probability model and provide practical implementations of diffusion policy for online model-free RL. Concretely, we character diffusion policy as a stochastic process induced by stochastic differential equations, which is a new approach to representing a policy.
Then we present a convergence guarantee for diffusion policy, which provides a theory to understand the multimodality of diffusion policy.
Furthermore, we propose the DIPO, which implements model-free online RL with \textbf{DI}ffusion \textbf{PO}licy.
To the best of our knowledge, DIPO is the first algorithm to solve model-free online RL problems with the diffusion model.
Finally, extensive empirical results show the effectiveness and superiority of DIPO on the standard continuous control Mujoco benchmark.
\end{abstract}
\end{center}

\newpage

\setcounter{tocdepth}{2}
\tableofcontents

\addtocounter{page}{-1}
\thispagestyle{empty}
\newpage

\clearpage

\section{Introduction}
%Recent three years have witnessed the impressive performance of diffusion probability model \citep{ho2020denoising,song2020score} on different tasks (e.g., image synthesis \citep{dhariwal2021diffusion}, natural language processing \citep{austin2021structured}, molecular graph modeling \citep{jing2022torsional}).
%Diffusion models (or diffusion policy) have also been applied to reinforcement learning (RL) (e.g., planing \citep{janner2022diffuser}, imitation learning \citep{reuss2023goal}, or offline RL \citep{ajay2023conditional}).

Existing policy representations (e.g., Gaussian distribution) for reinforcement learning (RL) tend to output a unimodal distribution over the action space, 
which may be trapped in a locally optimal solution due to its limited expressiveness of complex distribution and may result in poor performance.
%which weakens the expressiveness of complicated policy and decays the agent's ability to explore the environment.
Diffusion probability model \citep{sohl2015deep,ho2020denoising,song2020score} is powerful to learn complicated multimodal distributions, which has been applied to RL tasks (e.g., \citep{ajay2023conditional,reuss2023goal,chi2023diffusion}).

Although the diffusion model (or diffusion policy) shows its promising and potential applications to RL tasks, previous works are all empirical or only consider offline RL settings. This raises some fundamental questions: 
How to character diffusion policy? How to show the expressiveness of diffusion policy? How to design a diffusion policy for online model-free RL?
Those are the focuses of this paper.

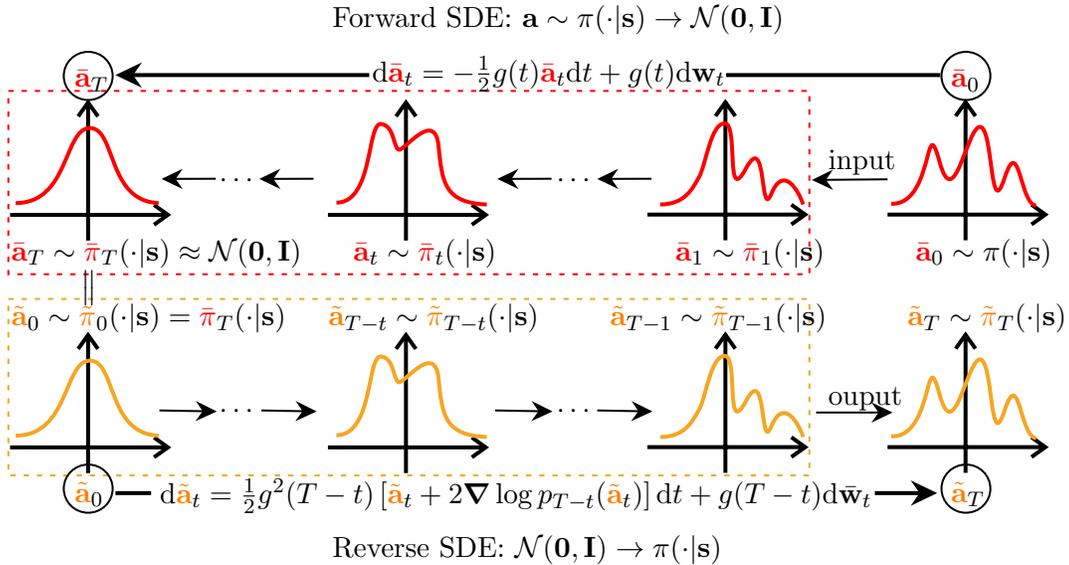
\begin{figure}[h]
\begin{center}
\tikzset{every picture/.style={line width=0.75pt}} %set default line width to 0.75pt        

\begin{tikzpicture}[x=0.75pt,y=0.75pt,yscale=-1,xscale=1]
%uncomment if require: \path (0,883); %set diagram left start at 0, and has height of 883

%Shape: Axis 2D [id:dp6270187211188827] 
\draw [color={rgb, 255:red, 0; green, 0; blue, 0 }  ,draw opacity=1 ][line width=1.5]  (209.16,159.55) -- (288.76,159.55)(248.52,102.85) -- (248.52,172.53) (281.76,154.55) -- (288.76,159.55) -- (281.76,164.55) (243.52,109.85) -- (248.52,102.85) -- (253.52,109.85)  ;
%Shape: Axis 2D [id:dp8146357314759392] 
\draw [color={rgb, 255:red, 0; green, 0; blue, 0 }  ,draw opacity=1 ][line width=1.5]  (370.28,159.41) -- (449.88,159.41)(409.63,102.71) -- (409.63,172.4) (442.88,154.41) -- (449.88,159.41) -- (442.88,164.41) (404.63,109.71) -- (409.63,102.71) -- (414.63,109.71)  ;
%Shape: Axis 2D [id:dp9852137035996071] 
\draw [color={rgb, 255:red, 0; green, 0; blue, 0 }  ,draw opacity=1 ][line width=1.5]  (48.86,159.55) -- (128.46,159.55)(88.21,102.85) -- (88.21,172.53) (121.46,154.55) -- (128.46,159.55) -- (121.46,164.55) (83.21,109.85) -- (88.21,102.85) -- (93.21,109.85)  ;
%Curve Lines [id:da42435973004455607] 
\draw [color={rgb, 255:red, 255; green, 0; blue, 0 }  ,draw opacity=1 ][line width=1.5]    (51.67,153.61) .. controls (78.83,153.61) and (75.85,115.88) .. (89.13,115.77) .. controls (93.53,115.73) and (95.94,119.77) .. (98.15,125.23) .. controls (102.61,136.23) and (106.25,152.99) .. (123.78,153.61) ;
%Shape: Axis 2D [id:dp3202385311347322] 
\draw [color={rgb, 255:red, 0; green, 0; blue, 0 }  ,draw opacity=1 ][line width=1.5]  (491.4,159.41) -- (571,159.41)(530.75,102.71) -- (530.75,172.4) (564,154.41) -- (571,159.41) -- (564,164.41) (525.75,109.71) -- (530.75,102.71) -- (535.75,109.71)  ;
%Curve Lines [id:da18024268291571022] 
\draw [color={rgb, 255:red, 255; green, 0; blue, 0 }  ,draw opacity=1 ][line width=1.5]    (493.27,153.94) .. controls (505.22,154.45) and (509.1,139.33) .. (511.47,130.51) .. controls (513.36,123.42) and (514.29,120.38) .. (517.62,132.71) .. controls (525.11,160.4) and (531.65,114.09) .. (538.22,115.17) .. controls (544.8,116.25) and (542.91,164.09) .. (550.4,141.02) .. controls (554.48,128.43) and (556.4,133.55) .. (558.64,140.89) .. controls (560.5,147) and (562.58,154.65) .. (566.32,154.86) ;
%Curve Lines [id:da47056633980008233] 
\draw [color={rgb, 255:red, 255; green, 0; blue, 0 }  ,draw opacity=1 ][line width=1.5]    (376.72,154.32) .. controls (389.27,154.32) and (393.05,146.33) .. (396.43,140.33) .. controls (399.81,134.34) and (402.15,113.43) .. (409.29,113.37) .. controls (416.43,113.31) and (409.29,156.2) .. (420.01,136.11) .. controls (430.73,116.01) and (426.81,163.45) .. (432.87,148.79) .. controls (438.93,134.14) and (446.68,147.51) .. (448.82,154.32) ;
%Curve Lines [id:da5376313160299797] 
\draw [color={rgb, 255:red, 255; green, 0; blue, 0 }  ,draw opacity=1 ][line width=1.5]    (213.45,153.39) .. controls (232.18,153.39) and (228.7,112.98) .. (235.93,113.71) .. controls (243.15,114.43) and (242.07,127.19) .. (245.29,126.63) .. controls (248.51,126.06) and (263.18,106.58) .. (264.96,124.78) .. controls (266.73,142.98) and (267.8,153.55) .. (289.31,154.32) ;
%Straight Lines [id:da9972978030515713] 
\draw [line width=0.75]    (321.93,141.33) -- (297.6,141.41) ;
\draw [shift={(294.6,141.42)}, rotate = 360] [fill={rgb, 255:red, 0; green, 0; blue, 0 }  ][line width=0.08]  [draw opacity=0] (10.72,-5.15) -- (0,0) -- (10.72,5.15) -- (7.12,0) -- cycle    ;
%Straight Lines [id:da8174584806730354] 
\draw [line width=0.75]    (372.31,141.33) -- (347.98,141.41) ;
\draw [shift={(344.98,141.42)}, rotate = 359.82] [fill={rgb, 255:red, 0; green, 0; blue, 0 }  ][line width=0.08]  [draw opacity=0] (10.72,-5.15) -- (0,0) -- (10.72,5.15) -- (7.12,0) -- cycle    ;
%Straight Lines [id:da6392217455830231] 
\draw [line width=0.75]    (491.41,141.61) -- (456.23,141.92) ;
\draw [shift={(453.23,141.95)}, rotate = 359.5] [fill={rgb, 255:red, 0; green, 0; blue, 0 }  ][line width=0.08]  [draw opacity=0] (10.72,-5.15) -- (0,0) -- (10.72,5.15) -- (7.12,0) -- cycle    ;
%Straight Lines [id:da8063309819824793] 
\draw [line width=0.75]    (152.58,141.33) -- (128.25,141.41) ;
\draw [shift={(125.25,141.42)}, rotate = 360] [fill={rgb, 255:red, 0; green, 0; blue, 0 }  ][line width=0.08]  [draw opacity=0] (10.72,-5.15) -- (0,0) -- (10.72,5.15) -- (7.12,0) -- cycle    ;
%Straight Lines [id:da8624779894337169] 
\draw [line width=0.75]    (201.89,141.33) -- (177.55,141.41) ;
\draw [shift={(174.55,141.42)}, rotate = 359.82] [fill={rgb, 255:red, 0; green, 0; blue, 0 }  ][line width=0.08]  [draw opacity=0] (10.72,-5.15) -- (0,0) -- (10.72,5.15) -- (7.12,0) -- cycle    ;
%Shape: Axis 2D [id:dp7141862514204103] 
\draw [color={rgb, 255:red, 0; green, 0; blue, 0 }  ,draw opacity=1 ][line width=1.5]  (209.16,276.92) -- (288.76,276.92)(248.52,220.22) -- (248.52,289.91) (281.76,271.92) -- (288.76,276.92) -- (281.76,281.92) (243.52,227.22) -- (248.52,220.22) -- (253.52,227.22)  ;
%Shape: Axis 2D [id:dp2978091496324189] 
\draw [color={rgb, 255:red, 0; green, 0; blue, 0 }  ,draw opacity=1 ][line width=1.5]  (370.28,276.79) -- (449.88,276.79)(409.63,220.09) -- (409.63,289.77) (442.88,271.79) -- (449.88,276.79) -- (442.88,281.79) (404.63,227.09) -- (409.63,220.09) -- (414.63,227.09)  ;
%Shape: Axis 2D [id:dp592826894095515] 
\draw [color={rgb, 255:red, 0; green, 0; blue, 0 }  ,draw opacity=1 ][line width=1.5]  (48.86,276.92) -- (128.46,276.92)(88.21,220.22) -- (88.21,289.91) (121.46,271.92) -- (128.46,276.92) -- (121.46,281.92) (83.21,227.22) -- (88.21,220.22) -- (93.21,227.22)  ;
%Curve Lines [id:da5438521284508069] 
\draw [color={rgb, 255:red, 245; green, 166; blue, 35 }  ,draw opacity=1 ][line width=1.5]    (51.67,270.99) .. controls (78.83,270.99) and (75.85,233.26) .. (89.13,233.15) .. controls (93.53,233.11) and (95.94,237.15) .. (98.15,242.6) .. controls (102.61,253.61) and (106.25,270.37) .. (123.78,270.99) ;
%Shape: Axis 2D [id:dp30547348539750274] 
\draw [color={rgb, 255:red, 0; green, 0; blue, 0 }  ,draw opacity=1 ][line width=1.5]  (491.4,276.79) -- (571,276.79)(530.75,220.09) -- (530.75,289.77) (564,271.79) -- (571,276.79) -- (564,281.79) (525.75,227.09) -- (530.75,220.09) -- (535.75,227.09)  ;
%Curve Lines [id:da32445962262464234] 
\draw [color={rgb, 255:red, 245; green, 166; blue, 35 }  ,draw opacity=1 ][line width=1.5]    (493.27,271.31) .. controls (505.22,271.83) and (509.1,256.71) .. (511.47,247.88) .. controls (513.36,240.8) and (514.29,237.76) .. (517.62,250.09) .. controls (525.11,277.77) and (531.65,231.47) .. (538.22,232.55) .. controls (544.8,233.63) and (542.91,281.47) .. (550.4,258.39) .. controls (554.48,245.81) and (556.4,250.93) .. (558.64,258.26) .. controls (560.5,264.38) and (562.58,272.03) .. (566.32,272.24) ;
%Curve Lines [id:da9958854782418936] 
\draw [color={rgb, 255:red, 245; green, 166; blue, 35 }  ,draw opacity=1 ][line width=1.5]    (376.72,271.69) .. controls (389.27,271.69) and (393.05,263.71) .. (396.43,257.71) .. controls (399.81,251.72) and (402.15,230.81) .. (409.29,230.75) .. controls (416.43,230.69) and (409.29,273.57) .. (420.01,253.48) .. controls (430.73,233.39) and (426.81,280.83) .. (432.87,266.17) .. controls (438.93,251.52) and (446.68,264.88) .. (448.82,271.69) ;
%Curve Lines [id:da7533392222377482] 
\draw [color={rgb, 255:red, 245; green, 166; blue, 35 }  ,draw opacity=1 ][line width=1.5]    (213.45,270.77) .. controls (232.18,270.77) and (228.7,230.36) .. (235.93,231.08) .. controls (243.15,231.8) and (242.07,244.57) .. (245.29,244) .. controls (248.51,243.44) and (263.18,223.96) .. (264.96,242.16) .. controls (266.73,260.36) and (267.8,270.93) .. (289.31,271.69) ;
%Straight Lines [id:da6262629876601493] 
\draw [line width=0.75]    (455.38,259.32) -- (490.97,259.32) ;
\draw [shift={(493.97,259.32)}, rotate = 180] [fill={rgb, 255:red, 0; green, 0; blue, 0 }  ][line width=0.08]  [draw opacity=0] (10.72,-5.15) -- (0,0) -- (10.72,5.15) -- (7.12,0) -- cycle    ;

%Shape: Rectangle [id:dp8006413610970371] 
\draw  [color={rgb, 255:red, 245; green, 166; blue, 35 }  ,draw opacity=1 ][dash pattern={on 2pt off 3pt}][line width=0.75]  (48,202) -- (452,202) -- (452,291.06) -- (48,291.06) -- cycle ;

%Shape: Rectangle [id:dp9945466327593062] 
\draw  [color={rgb, 255:red, 255; green, 0; blue, 0 }  ,draw opacity=1 ][dash pattern={on 2pt off 3pt}][line width=0.75]  (48,97) -- (452,97) -- (452,189.52) -- (48,189.52) -- cycle ;
%Straight Lines [id:da7378196233307015] 
\draw [line width=0.75]    (124.17,258.27) -- (150.11,258.74) ;
\draw [shift={(153.11,258.79)}, rotate = 181.05] [fill={rgb, 255:red, 0; green, 0; blue, 0 }  ][line width=0.08]  [draw opacity=0] (10.72,-5.15) -- (0,0) -- (10.72,5.15) -- (7.12,0) -- cycle    ;
%Straight Lines [id:da5461255973158965] 
\draw [line width=0.75]    (175.62,258.27) -- (200.49,258.74) ;
\draw [shift={(203.49,258.79)}, rotate = 181.09] [fill={rgb, 255:red, 0; green, 0; blue, 0 }  ][line width=0.08]  [draw opacity=0] (10.72,-5.15) -- (0,0) -- (10.72,5.15) -- (7.12,0) -- cycle    ;
%Straight Lines [id:da34794784731385375] 
\draw [line width=0.75]    (293.53,258.27) -- (319.47,258.74) ;
\draw [shift={(322.47,258.79)}, rotate = 181.05] [fill={rgb, 255:red, 0; green, 0; blue, 0 }  ][line width=0.08]  [draw opacity=0] (10.72,-5.15) -- (0,0) -- (10.72,5.15) -- (7.12,0) -- cycle    ;
%Straight Lines [id:da354597502290946] 
\draw [line width=0.75]    (344.98,258.27) -- (369.85,258.74) ;
\draw [shift={(372.85,258.79)}, rotate = 181.09] [fill={rgb, 255:red, 0; green, 0; blue, 0 }  ][line width=0.08]  [draw opacity=0] (10.72,-5.15) -- (0,0) -- (10.72,5.15) -- (7.12,0) -- cycle    ;
%Shape: Ellipse [id:dp7306683187793959] 
\draw   (75.94,89.43) .. controls (75.94,82.71) and (81.46,77.27) .. (88.27,77.27) .. controls (95.07,77.27) and (100.59,82.71) .. (100.59,89.43) .. controls (100.59,96.14) and (95.07,101.59) .. (88.27,101.59) .. controls (81.46,101.59) and (75.94,96.14) .. (75.94,89.43) -- cycle ;
%Shape: Ellipse [id:dp5380525349605647] 
\draw   (518.62,297.75) .. controls (518.62,291.03) and (524.14,285.59) .. (530.94,285.59) .. controls (537.75,285.59) and (543.27,291.03) .. (543.27,297.75) .. controls (543.27,304.46) and (537.75,309.91) .. (530.94,309.91) .. controls (524.14,309.91) and (518.62,304.46) .. (518.62,297.75) -- cycle ;
%Shape: Ellipse [id:dp3870867189792493] 
\draw   (75.94,297.75) .. controls (75.94,291.03) and (81.46,285.59) .. (88.27,285.59) .. controls (95.07,285.59) and (100.59,291.03) .. (100.59,297.75) .. controls (100.59,304.46) and (95.07,309.91) .. (88.27,309.91) .. controls (81.46,309.91) and (75.94,304.46) .. (75.94,297.75) -- cycle ;
%Shape: Ellipse [id:dp0914774041355857] 
\draw   (518.62,89.43) .. controls (518.62,82.71) and (524.14,77.27) .. (530.94,77.27) .. controls (537.75,77.27) and (543.27,82.71) .. (543.27,89.43) .. controls (543.27,96.14) and (537.75,101.59) .. (530.94,101.59) .. controls (524.14,101.59) and (518.62,96.14) .. (518.62,89.43) -- cycle ;
%Straight Lines [id:da3774712921260066] 

\draw [line width=1.5]    (110,89)--  (230,89) ;
\draw [shift={(102,89)}, rotate = 360] [fill={rgb, 255:red, 0; green, 0; blue, 0 }  ][line width=0.08]  [draw opacity=0] (13.4,-6.43) -- (0,0) -- (13.4,6.44) -- (8.9,0) -- cycle    ;
%Straight Lines [id:da30653890375683446] 
\draw [line width=1.5]    (410,89) -- (518,89) ;
%Straight Lines [id:da10816676775166267] 
\draw [line width=1.5]    (102,300) -- (120,300) ;
%Straight Lines [id:da7815392922195064] 
\draw [line width=1.5]    (485,300) -- (512,300) ;
\draw [shift={(516.47,300)}, rotate = 180] [fill={rgb, 255:red, 0; green, 0; blue, 0 }  ][line width=0.08]  [draw opacity=0] (13.4,-6.43) -- (0,0) -- (13.4,6.44) -- (8.9,0) -- cycle    ;

% Text Node
\draw (522,85) node [anchor=north west][inner sep=0.75pt]    {$\bara_{0}$};
\draw (80,85) node [anchor=north west][inner sep=0.75pt]    {$\bara_{T}$};
\draw (83,186) node [anchor=north west][inner sep=0.75pt]    {$\big\|$};
\draw (80,293) node [anchor=north west][inner sep=0.75pt]    {$\tildea_{0}$};
\draw (522,293) node [anchor=north west][inner sep=0.75pt]    {$\tildea_{T}$};
% Text Node
\draw (383,170) node [anchor=north west][inner sep=0.75pt]    {$\bara_{1}\sim\barpi_{1}(\cdot|\bs)$};
% Text Node
\draw (220.56,170) node [anchor=north west][inner sep=0.75pt]    {$\bara_{t}\sim\barpi_{t}(\cdot|\bs)$};
% Text Node
\draw (48,170) node [anchor=north west][inner sep=0.75pt]    {$\bara_{T}\sim\barpi_{T}(\cdot|\bs)\approx\calN(\bm{0},\bI)$};
% Text Node
\draw (504,170) node [anchor=north west][inner sep=0.75pt]    {$\bara_{0}\sim\pi(\cdot|\bs)$};
\draw (460,125) node [anchor=north west][inner sep=0.75pt]    {input};

\draw (460,245) node [anchor=north west][inner sep=0.75pt]    {ouput};
% Text Node
\draw (323,137.5) node [anchor=north west][inner sep=0.75pt]    {$\cdots $};
% Text Node
\draw (153,137.5) node [anchor=north west][inner sep=0.75pt]    {$\cdots$};
% Text Node
\draw (153,254) node [anchor=north west][inner sep=0.75pt]    {$\cdots $};
% Text Node
\draw (48,203) node [anchor=north west][inner sep=0.75pt]    {$\tildea_{0} \sim \tildepi_{0}(\cdot|\bs)=\barpi_{T}(\cdot|\bs)$};
% Text Node
\draw (208,203) node [anchor=north west][inner sep=0.75pt]    {$\tildea_{T-t} \sim \tildepi_{T-t}(\cdot|\bs)$};
% Text Node
\draw (350,203) node [anchor=north west][inner sep=0.75pt]    {$\tildea_{T-1} \sim \tildepi_{T-1}(\cdot|\bs)$};
% Text Node
\draw (500,203) node [anchor=north west][inner sep=0.75pt]    {$\tildea_{T} \sim \tildepi_{T}(\cdot|\bs)$};
% Text Node
\draw (323,254) node [anchor=north west][inner sep=0.75pt]    {$\cdots $};
% Text Node
\draw (230,78) node [anchor=north west][inner sep=0.75pt]    {$\dd \bara_t=-\frac{1}{2}g(t)\bara_t\dd t+g(t)\dd \bw_{t}$};
\draw (210,52) node [anchor=north west][inner sep=0.75pt]    {Forward SDE:~$\ba\sim\pi(\cdot|\bs)\rightarrow \calN(\bm{0},\bI)$};
% Text Node
\draw (123,291) node [anchor=north west][inner sep=0.75pt]    {$\dd \tildea_{t}=\frac{1}{2}g^2(T-t)\left[\tildea_{t}+2\bnabla \log p_{T-t}( \tildea_{t})\right]\dd t+g(T-t)\dd \bar{\bw}_{t}$};

\draw (210,320) node [anchor=north west][inner sep=0.75pt]    {Reverse SDE: $\calN(\bm{0},\bI)\rightarrow\pi(\cdot|\bs)$};
\end{tikzpicture}
\end{center}
\caption{Diffusion Policy: Policy Representation via Stochastic Process. For a given state $\bs$, the forward stochastic process $\{\bara_{t}|\bs\}$ maps the input $\bara_{0}=:\ba\sim\pi(\cdot|\bs)$ to be a noise; then we recover the input by the stochastic process $\{\tildea_{t}|\bs\}$ that reverses the reversed SDE if we know the score function $\bnabla \log p_{t}(\cdot)$, where $p_{t}(\cdot)$ is the probability distribution of the forward process, i.e., $p_{t}(\cdot)=\barpi_{t}(\cdot|\bs)$.
}
\label{fig:diffusion-policy}
\end{figure}

\subsection{Our Main Work}

In this paper, we mainly consider diffusion policy from the next three aspects.

\textbf{Charactering Diffusion Policy as Stochastic Process.}  We formulate diffusion policy as a stochastic process that involves two processes induced by stochastic differential equations (SDE), see Figure \ref{fig:diffusion-policy}, where the forward process disturbs the input policy $\pi$ to noise, then the reverse process infers the policy $\pi$ according to a corresponding reverse SDE. Although this view is inspired by the score-based generative model \citep{song2020score}, we provide a brand new approach to represent a policy: via a stochastic process induced by SDE, neither via value function nor parametric function. 
Under this framework, the diffusion policy is flexible to generate actions according to numerical SDE solvers.

\textbf{Convergence Analysis of Diffusion Policy.}
Under mild conditions, Theorem \ref{finite-time-diffusion-policy} presents a theoretical convergence guarantee for diffusion policy. 
The result shows that if the score estimator is sufficiently accurate, then diffusion policy efficiently infers the actions from any realistic policy that generates the training data.
It is noteworthy that Theorem \ref{finite-time-diffusion-policy} also shows that diffusion policy is powerful to represent a multimodal distribution, which leads to sufficient exploration and better reward performance, Section \ref{sec:policy-representation4rl} and Appendix \ref{app-sec-multi-goal} provide more discussions with numerical verifications for this view.

\textbf{Diffusion Policy for Model-free Online RL.} Recall the standard model-free online RL framework, see Figure \ref{fig:framework-standard-rl}, where the policy improvement produces a new policy $\pi^{'}\succeq\pi$ according to the data $\calD$.
However, Theorem \ref{finite-time-diffusion-policy} illustrates that the diffusion policy only fits the distribution of the policy $\pi$ but does not improve the policy $\pi$.
We can not embed the diffusion policy into the standard RL training framework, i.e., the policy improvement in Figure \ref{fig:framework-standard-rl} can not be naively replaced by diffusion policy. To apply diffusion policy to model-free online RL task, we propose the DIPO algorithm, see Figure \ref{fig:framework}.
The proposed DIPO considers a novel way for policy improvement, we call it \textbf{action gradient} that updates each $\ba_t\in\calD$ along the gradient field (over the action space) of state-action value:
\[
\ba_{t}\leftarrow\ba_{t}+\eta \nabla_{\ba} Q_{\pi}( \bs_t,\ba_t),
\]
where for a given state $\bs$, $Q_{\pi}(\bs,\ba)$ measures the reward performance over the action space $\calA$. 
Thus, DIPO improves the policy according to the actions toward to better reward performance. To the best of our knowledge, this paper first presents the idea of action gradient, which provides an efficient way to make it possible to design a diffusion policy for online RL.

\subsection{Paper Organization}

Section \ref{background-rl} presents the background of reinforcement learning.
Section \ref{sec:policy-representation4rl} presents our motivation from the view of policy representation.
Section \ref{sec-diffusion-policy} presents the theory of diffusion policy.
Section \ref{sec:score-matching-implementation} presents the practical implementation of diffusion policy for model-free online reinforcement learning.
Section \ref{section-ex} presents the experiment results.

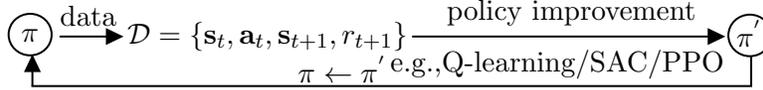
\begin{figure}[t!]

\begin{center}

\tikzset{every picture/.style={line width=0.75pt}} %set default line width to 0.75pt        

\begin{tikzpicture}[x=0.75pt,y=0.75pt,yscale=-1,xscale=1]
%uncomment if require: \path (0,663); %set diagram left start at 0, and has height of 663
%Shape: Ellipse [id:dp7766501195054403] 
\draw   (70.5,504.01) .. controls (70.5,498.22) and (75.36,493.53) .. (81.35,493.53) .. controls (87.35,493.53) and (92.21,498.22) .. (92.21,504.01) .. controls (92.21,509.81) and (87.35,514.5) .. (81.35,514.5) .. controls (75.36,514.5) and (70.5,509.81) .. (70.5,504.01) -- cycle ;

%Straight Lines [id:da8106841766142445] 
\draw    (96.21,505) -- (125.5,505) ;
\draw [shift={(128.5,505)}, rotate = 180] [fill={rgb, 255:red, 0; green, 0; blue, 0 }  ][line width=0.08]  [draw opacity=0] (8.93,-4.29) -- (0,0) -- (8.93,4.29) -- cycle    ;
%%Straight Lines [id:da8566270476191015] 
\draw    (274,505) -- (428,505) ;
\draw [shift={(432,505)}, rotate = 180] [fill={rgb, 255:red, 0; green, 0; blue, 0 }  ][line width=0.08]  [draw opacity=0] (8.93,-4.29) -- (0,0) -- (8.93,4.29) -- cycle    ;
%%Straight Lines [id:da9541379772233695] 
\draw    (82.5,515) -- (82.5,530) ;
\draw [shift={(82.5,514)}, rotate = 90] [fill={rgb, 255:red, 0; green, 0; blue, 0 }  ][line width=0.08]  [draw opacity=0] (8.93,-4.29) -- (0,0) -- (8.93,4.29) -- cycle    ;
%%Shape: Ellipse [id:dp29648346000001313] 
%\draw   (580.5,505.01) .. controls (580.5,499.22) and (585.36,494.53) .. (591.35,494.53) .. controls (597.35,494.53) and (602.21,499.22) .. (602.21,505.01) .. controls (602.21,510.81) and (597.35,515.5) .. (591.35,515.5) .. controls (585.36,515.5) and (580.5,510.81) .. (580.5,505.01) -- cycle ;
%%Straight Lines [id:da8171531955779341] 
\draw    (445,515) -- (445,530) ;
%%Straight Lines [id:da1036721662550848] 
%\draw    (485,505) --(575,505) ;
%\draw [shift={(580,505)}, rotate = 180] [fill={rgb, 255:red, 0; green, 0; blue, 0 }  ][line width=0.08]  [draw opacity=0] (8.93,-4.29) -- (0,0) -- (8.93,4.29) -- cycle    ;
%%Shape: Square [id:dp28021857806977524] 
%\draw  [fill={rgb, 255:red, 184; green, 233; blue, 134 }  ,fill opacity=1 ][dash pattern={on 0.84pt off 2.51pt}] (180,496) -- (195,496) -- (195,515) -- (180,515) -- cycle ;
%%Straight Lines [id:da6858294893863679] 
\draw    (82,530) -- (445.5,530) ;

%Shape: Ellipse [id:dp6509604643176914] 
\draw   (433.5,504.44) .. controls (433.5,498.65) and (438.36,493.95) .. (444.35,493.95) .. controls (450.35,493.95) and (455.21,498.65) .. (455.21,504.44) .. controls (455.21,510.23) and (450.35,514.92) .. (444.35,514.92) .. controls (438.36,514.92) and (433.5,510.23) .. (433.5,504.44) -- cycle ;
%%Shape: Square [id:dp38437508186810276] 
%\draw  [fill={rgb, 255:red, 184; green, 233; blue, 134 }  ,fill opacity=1 ][dash pattern={on 0.84pt off 2.51pt}] (406,496) -- (421,496) -- (421,515) -- (406,515) -- cycle ;

%%Straight Lines [id:da9911890981161683] 
%\draw [color={rgb, 255:red, 65; green, 117; blue, 5 }  ,draw opacity=1 ]   (187.5,480) -- (187.5,495) ;
%%Straight Lines [id:da08405590615466552] 
%\draw [color={rgb, 255:red, 65; green, 117; blue, 5 }  ,draw opacity=1 ]   (187,480)-- (414,480)  ;
%%Straight Lines [id:da365571057609656] 
%\draw [color={rgb, 255:red, 65; green, 117; blue, 5 }  ,draw opacity=1 ]   (413.5,480) -- (413.5,495) ;
%\draw [shift={(413.5,500)}, rotate = 270] [fill={rgb, 255:red, 65; green, 117; blue, 5 }  ,fill opacity=1 ][line width=0.08]  [draw opacity=0] (8.93,-4.29) -- (0,0) -- (8.93,4.29) -- cycle    ;

%Shape: Rectangle [id:dp7711164261075109] 
%\draw  [dash pattern={on 0.84pt off 2.51pt}] (68,458) -- (602,458) -- (602,550) -- (68,550) -- cycle ;
%\draw  (68,458) -- (602,458) -- (602,550) -- (68,550) -- cycle ;

% Text Node
\draw (75.36,500) node [anchor=north west][inner sep=0.75pt]    {$\pi $};
% Text Node
\draw (129,496) node [anchor=north west][inner sep=0.75pt]    {$\calD=\{\bs_{t},\ba_{t} ,\bs_{t+1} ,r_{t+1}\}$};
% Text Node
\draw (95,488) node [anchor=north west][inner sep=0.75pt]   [align=left] {data};
% Text Node
\draw (290,485) node [anchor=north west][inner sep=0.75pt]   [align=left] {policy improvement};
\draw (261.66,510) node [anchor=north west][inner sep=0.75pt]   [align=left] {e.g.,Q-learning/SAC/PPO};
% Text Node
\draw (215,513) node [anchor=north west][inner sep=0.75pt]    {$\pi \gets \pi ^{'}$};
% Text Node
\draw (436,495) node [anchor=north west][inner sep=0.75pt]    {$\pi^{'}$};
% Text Node
%\draw (484,493) node [anchor=north west][inner sep=0.75pt]   [align=left] {diffusion model};
%\draw (484,504) node [anchor=north west][inner sep=0.75pt]   [align=left] {~~Algorithm \ref{algo:diffusion-policy-general-case}};
%% Text Node
%\draw (350,495) node [anchor=north west][inner sep=0.75pt]    {$\calD^{'}=\{\bs_{t},\ba_{t} ,\bs_{t+1} ,r_{t+1}\}$};
% Text Node

\end{tikzpicture}

\end{center}

\caption{Standard Training Framework for Model-free Online RL.}
\label{fig:framework-standard-rl}

\end{figure}

\begin{figure}[t!]

\begin{center}

\tikzset{every picture/.style={line width=0.75pt}} %set default line width to 0.75pt        

\begin{tikzpicture}[x=0.75pt,y=0.75pt,yscale=-1,xscale=1]
%uncomment if require: \path (0,663); %set diagram left start at 0, and has height of 663
%Shape: Ellipse [id:dp7766501195054403] 
\draw   (70.5,504.01) .. controls (70.5,498.22) and (75.36,493.53) .. (81.35,493.53) .. controls (87.35,493.53) and (92.21,498.22) .. (92.21,504.01) .. controls (92.21,509.81) and (87.35,514.5) .. (81.35,514.5) .. controls (75.36,514.5) and (70.5,509.81) .. (70.5,504.01) -- cycle ;

%Straight Lines [id:da8106841766142445] 
\draw    (96.21,505) -- (125.5,505) ;
\draw [shift={(128.5,505)}, rotate = 180] [fill={rgb, 255:red, 0; green, 0; blue, 0 }  ][line width=0.08]  [draw opacity=0] (8.93,-4.29) -- (0,0) -- (8.93,4.29) -- cycle    ;
%%Straight Lines [id:da8566270476191015] 
\draw    (263,505) -- (345,505) ;
\draw [shift={(350,505)}, rotate = 180] [fill={rgb, 255:red, 0; green, 0; blue, 0 }  ][line width=0.08]  [draw opacity=0] (8.93,-4.29) -- (0,0) -- (8.93,4.29) -- cycle    ;
%%Straight Lines [id:da9541379772233695] 
\draw    (82.5,515) -- (82.5,530) ;
\draw [shift={(82.5,514)}, rotate = 90] [fill={rgb, 255:red, 0; green, 0; blue, 0 }  ][line width=0.08]  [draw opacity=0] (8.93,-4.29) -- (0,0) -- (8.93,4.29) -- cycle    ;
%Shape: Ellipse [id:dp29648346000001313] 
\draw   (527.5,505.01) .. controls (527.5,499.22) and (532.36,494.53) .. (538.35,494.53) .. controls (544.35,494.53) and (549.21,499.22) .. (549.21,505.01) .. controls (549.21,510.81) and (544.35,515.5) .. (538.35,515.5) .. controls (532.36,515.5) and (527.5,510.81) .. (527.5,505.01) -- cycle ;

%%Straight Lines [id:da8171531955779341] 

\draw    (538,515) -- (538,530) ;
%%Straight Lines [id:da1036721662550848] 

\draw    (425,505) --(520,505) ;
\draw [shift={(525,505)}, rotate = 180] [fill={rgb, 255:red, 0; green, 0; blue, 0 }  ][line width=0.08]  [draw opacity=0] (8.93,-4.29) -- (0,0) -- (8.93,4.29) -- cycle    ;
%%Shape: Square [id:dp28021857806977524] 
\draw  [fill={rgb, 255:red, 184; green, 233; blue, 134 }  ,fill opacity=1 ][dash pattern={on 2pt off 3pt}] (180,496) -- (195,496) -- (195,515) -- (180,515) -- cycle ;
%%Straight Lines [id:da6858294893863679] 

\draw    (82,530) -- (538.5,530) ;
%Shape: Square [id:dp38437508186810276] 
\draw  [fill={rgb, 255:red, 184; green, 233; blue, 134 }  ,fill opacity=1 ][dash pattern={on 2pt off 3pt}] (406,496) -- (421,496) -- (421,515) -- (406,515) -- cycle ;

%Straight Lines [id:da9911890981161683] 
\draw [color={rgb, 255:red, 65; green, 117; blue, 5 }  ,draw opacity=1 ]   (187.5,478) -- (187.5,495) ;
%Straight Lines [id:da08405590615466552] 

\draw [color={rgb, 255:red, 65; green, 117; blue, 5 }  ,draw opacity=1 ]   (187,478)-- (414,478)  ;
%Straight Lines [id:da365571057609656] 
\draw [color={rgb, 255:red, 65; green, 117; blue, 5 }  ,draw opacity=1 ]   (413.5,478) -- (413.5,495) ;
\draw [shift={(413.5,500)}, rotate = 270] [fill={rgb, 255:red, 65; green, 117; blue, 5 }  ,fill opacity=1 ][line width=0.08]  [draw opacity=0] (8.93,-4.29) -- (0,0) -- (8.93,4.29) -- cycle    ;

%Shape: Rectangle [id:dp7711164261075109] 
%\draw  [dash pattern={on 0.84pt off 2.51pt}] (68,458) -- (602,458) -- (602,550) -- (68,550) -- cycle ;
%\draw  (68,458) -- (602,458) -- (602,550) -- (68,550) -- cycle ;

% Text Node
\draw (75.36,500) node [anchor=north west][inner sep=0.75pt]    {$\pi $};
% Text Node
\draw (124,496) node [anchor=north west][inner sep=0.75pt]    {$\calD=\{\bs_{t},\ba_{t} ,\bs_{t+1} ,r_{t+1}\}$};
% Text Node
\draw (95,483) node [anchor=north west][inner sep=0.75pt]   [align=left] {data};
% Text Node
\draw (261.66,483) node [anchor=north west][inner sep=0.75pt]   [align=left] {action gradient};
% Text Node
\draw (300,513) node [anchor=north west][inner sep=0.75pt]    {$\pi \gets \pi ^{'}$};
% Text Node
\draw (530,495) node [anchor=north west][inner sep=0.75pt]    {$\pi^{'}$};
% Text Node
\draw (426,483) node [anchor=north west][inner sep=0.75pt]   [align=left] {diffusion policy};
%\draw (484,504) node [anchor=north west][inner sep=0.75pt]   [align=left] {~~Algorithm \ref{algo:diffusion-policy-general-case}};
% Text Node
\draw (346,495) node [anchor=north west][inner sep=0.75pt]    {$\calD^{'}=\{\bs_{t},\ba_{t}\}$};
% Text Node
\draw (240,460) node [anchor=north west][inner sep=0.75pt]    {$ \ba_{t}+\eta \nabla_{\ba} Q_{\pi}( \bs_t,\ba_t)\rightarrow\ba_{t}$};
\end{tikzpicture}
\end{center}
\caption{Framework of DIPO: Implementation for Model-free Online RL with \textbf{DI}ffusion \textbf{PO}licy.}
\label{fig:framework}
\end{figure}
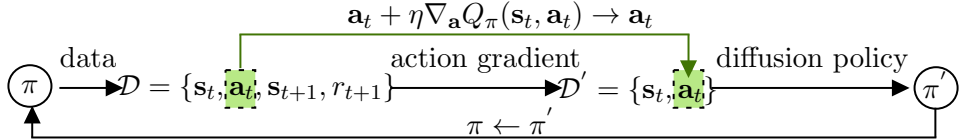

\section{Reinforcement Learning}

\label{background-rl}

Reinforcement learning (RL)\citep{sutton2018reinforcement} is formulated as \emph{Markov decision process} $\mathcal{M}=(\mathcal{S},\mathcal{A},\Pro(\cdot),{r},\gamma,d_{0})$,
where $\mathcal{S}$ is the state space; $\mathcal{A}\subset\R^{p}$ is the continuous action space;
$\Pro(\bs^{'}|\bs,\ba)$ is the probability of state transition from $\bs$ to $\bs^{'}$ after playing $\ba$;
$r(\bs'|\bs,\ba)$ denotes the reward that the agent observes when the state transition from $\bs$ to $\bs^{'}$ after playing $\ba$;
$\gamma\in(0,1)$ is the discounted factor, and $d_{0}(\cdot)$ is the initial state distribution.
A policy $\pi$ is a probability distribution defined on $\mathcal{S}\times\mathcal{A}$, and $\pi(\ba|\bs)$ denotes the probability of playing $\ba$ in state $\bs$.
Let $\{\bs_{t},\ba_{t},\bs_{t+1},r(\bs_{t+1}|\bs_{t},\ba_{t})\}_{t\ge0}\sim\pi$ be the trajectory sampled by the policy $\pi$, where $\bs_{0}\sim d_{0}(\cdot)$, $\ba_{t}\sim\pi(\cdot|\bs_{t})$, $\bs_{t+1}\sim \Pro(\cdot|\bs_t,\ba_t)$.
The goal of RL is to find a policy $\pi$ such that
$
\pi_{\star}=:\arg\max_{\pi} \E_{\pi}\left[\sum_{t=0}^{\infty}\gamma^t r(\bs_{t+1}|\bs_t,\ba_t)\right].
$

\section{Motivation: A View from Policy Representation}

\label{sec:policy-representation4rl}

In this section, we clarify our motivation from the view of policy representation: diffusion model is powerful to policy representation, which leads to sufficient exploration and better reward performance.

\subsection{Policy Representation for Reinforcement Learning}

Value function and parametric function based are the main two approaches to represent policies, while diffusion policy expresses a policy via a stochastic process (shown in Figure \ref{fig:diffusion-policy}) that is essentially difficult to the previous representation.
In this section, we will clarify this view.
Additionally, we will provide an empirical verification with a numerical experiment.

\subsubsection{Policy Representation via Value Function} A typical way to represent policy is \emph{$\epsilon$-greedy policy} \citep{sutton1998reinforcement} or \emph{energy-based policy} \citep{sallans2004reinforcement,peters2010relative},
\begin{flalign}
\label{eps-greedy-policy}
\pi(\ba|\bs)=
\begin{cases}
\arg\max_{\ba^{'}\in\calA} Q_{\pi}(\bs,\ba^{'}) & \text{w.p.~} 1-\epsilon; \\
\text{randomly~play~}\ba\in\calA
    & \text{w.p.~} \epsilon;\\
\end{cases}
~\text{or}~
\pi(\ba|\bs)=\dfrac{\exp\left\{Q_{\pi}(\bs,\ba)\right\}}{Z_{\pi}(\bs)},
\end{flalign}
where \[Q_{\pi}(\bs,\ba)=:\E_{\pi}\left[\sum_{t=0}^{\infty}\gamma^t r(\bs_{t+1}|\bs_t,\ba_t)|\bs_{0}=\bs,\ba_{0}=\ba\right],\]
the normalization term $Z_{\pi}(\bs)=\int_{\R^{p}}\exp\left\{Q_{\pi}(\bs,\ba)\right\}\dd \ba$, and ``w.p.'' is short for ``with probability''.
The representation (\ref{eps-greedy-policy}) illustrates a connection between policy and value function, which is widely used in \emph{value-based methods}
(e.g., SASRA \citep{rummery1994line}, 
Q-Learning \citep{watkins1989learning},
DQN \citep{mnih2015human}) and \emph{energy-based methods} (e.g., 
SQL \citep{schulman2017equivalence,haarnoja2017reinforcement,haarnoja2018composable}, 
SAC \citep{haarnoja2018soft}).

\subsubsection{Policy Representation via Parametric Function} Instead of consulting a value function, the parametric policy is to represent a policy by a parametric function (e.g., neural networks), denoted as $\pi_{\btheta}$, where $\btheta$ is the parameter. Policy gradient theorem \citep{sutton1999policy,silver2014deterministic} plays a center role to learn $\btheta$, which is fundamental in modern RL (e.g., TRPO \citep{schulman2015trust}, DDPG \citep{lillicrap2015continuous}, PPO \citep{schulman2017proximal}, IMPALA \citep{espeholt2018impala}, et al).

\subsubsection{Policy Representation via Stochastic Process} It is different from both value-based and parametric policy representation; the diffusion policy (see Figure \ref{fig:diffusion-policy}) generates an action via a stochastic process, which is a fresh view for the RL community. 
The diffusion model with RL first appears in \citep{janner2022diffuser}, where it proposes the \emph{diffuser} that plans by iteratively refining trajectories, which is an essential offline RL method. \cite{ajay2023conditional,reuss2023goal} model a policy as a return conditional diffusion model, \cite{chen2023offline,wang2023diffusion,chi2023diffusion} consider to generate actions via diffusion model. The above methods are all to solve offline RL problems. To the best of our knowledge, our proposed method is the first diffusion approach to online model-free reinforcement learning.

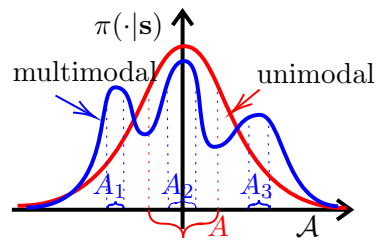
\begin{wrapfigure}{r}{0.35\linewidth}

\tikzset{every picture/.style={line width=0.75pt}} %set default line width to 0.75pt        

\begin{tikzpicture}[x=0.75pt,y=0.75pt,yscale=-1,xscale=1]
%uncomment if require: \path (0,401); %set diagram left start at 0, and has height of 401

%Shape: Brace [id:dp27472427301551283] 
\draw  [color={rgb, 255:red, 0; green, 0; blue, 255}  ,draw opacity=1 ][line width=0.75]  (381.25,105.49) .. controls (381.25,104.41) and (380.71,103.87) .. (379.63,103.87) -- (379.63,103.87) .. controls (378.08,103.87) and (377.31,103.33) .. (377.31,102.25) .. controls (377.31,103.33) and (376.54,103.87) .. (375,103.87)(375.69,103.87) -- (375,103.87) .. controls (373.92,103.87) and (373.38,104.41) .. (373.38,105.49) ;
%Shape: Axis 2D [id:dp13653545899016417] 
\draw [line width=1.5]  (325,106.38) -- (496,106.38)(411.09,6.22) -- (411.09,116.11) (489,101.38) -- (496,106.38) -- (489,111.38) (406.09,13.22) -- (411.09,6.22) -- (416.09,13.22)  ;
%Curve Lines [id:da14185136657551034] 
\draw [color={rgb, 255:red, 255; green, 0; blue, 0 }  ,draw opacity=1 ][line width=1.5]    (327.9,104.36) .. controls (389.49,104.94) and (390.21,24.12) .. (411.95,23.83) .. controls (418.9,23.73) and (423.56,32.06) .. (428.67,43.43) .. controls (439.53,67.6) and (452.43,105.52) .. (493.83,105.52) ;
%Straight Lines [id:da5828449641627231] 
\draw [color={rgb, 255:red, 255; green, 0; blue, 0 }  ,draw opacity=1 ][line width=0.4]  [dash pattern={on 0.84pt off 2.51pt}]  (429.03,106.04) -- (428.67,43.43) ;
%Straight Lines [id:da10426766913061725] 
\draw [color={rgb, 255:red, 255; green, 0; blue, 0 }  ,draw opacity=1 ] [line width=0.4]  [dash pattern={on 0.84pt off 2.51pt}]  (394.09,105.31) -- (394.09,42.92) ;
%Shape: Brace [id:dp04558722170515028] 
\draw  [color={rgb, 255:red, 255; green, 0; blue, 0 }  ,draw opacity=1 ] (394.09,106.41) .. controls (394.09,111.08) and (396.42,113.41) .. (401.09,113.41) -- (401.41,113.41) .. controls (408.08,113.41) and (411.41,115.74) .. (411.41,120.41) .. controls (411.41,115.74) and (414.74,113.41) .. (421.41,113.41)(418.41,113.41) -- (422.03,113.41) .. controls (426.7,113.41) and (429.03,111.08) .. (429.03,106.41) ;
%Curve Lines [id:da7818724832157147] 
\draw [color={rgb, 255:red, 0; green, 0; blue, 255}  ,draw opacity=1 ][line width=1.5]    (332.25,105.23) .. controls (378.62,105.23) and (366.3,44.61) .. (377.17,44.61) .. controls (388.04,44.61) and (383.69,68.86) .. (392.39,68.28) .. controls (401.08,67.7) and (401.08,31.33) .. (411.95,31.33) .. controls (422.82,31.33) and (416.77,74.96) .. (427.17,74.63) .. controls (437.56,74.3) and (437.31,59.04) .. (449.63,58.47) .. controls (461.94,57.89) and (461.94,103.5) .. (489.48,104.65) ;
%Straight Lines [id:da5268049108319026] 
\draw [color={rgb, 255:red, 0; green, 0; blue, 255}  ,draw opacity=1 ][line width=0.4]  [dash pattern={on 0.84pt off 2.51pt}]  (444.56,105.23) -- (444.56,59.62) ;
%Straight Lines [id:da07577964608612908] 
\draw [color={rgb, 255:red, 0; green, 0; blue, 255}  ,draw opacity=1 ][line width=0.4]  [dash pattern={on 0.84pt off 2.51pt}]  (455.42,105.23) -- (454.13,67.12) -- (455,61.22) ;
%Straight Lines [id:da7762847222310381] 
\draw [color={rgb, 255:red, 0; green, 0; blue, 255}  ,draw opacity=1 ][line width=0.4]  [dash pattern={on 0.84pt off 2.51pt}]  (403.98,105.8) -- (403.25,43.46) ;
%Straight Lines [id:da13385093641407164] 
\draw [color={rgb, 255:red, 0; green, 0; blue, 255}  ,draw opacity=1 ][line width=0.4]  [dash pattern={on 0.84pt off 2.51pt}]  (417.75,105.8) -- (417.75,44.61) ;
%Straight Lines [id:da23886267032587205] 
\draw [color={rgb, 255:red, 0; green, 0; blue, 255}  ,draw opacity=1 ][line width=0.4]  [dash pattern={on 0.84pt off 2.51pt}]  (372.82,105.8) -- (372.82,51.25) ;
%Straight Lines [id:da9222769770048751] 
\draw [color={rgb, 255:red, 0; green, 0; blue, 255}  ,draw opacity=1 ][line width=0.4]  [dash pattern={on 0.84pt off 2.51pt}]  (381.52,105.8) -- (381.25,46.97) ;
%Shape: Brace [id:dp10967017644559274] 
\draw  [color={rgb, 255:red, 0; green, 0; blue, 255}  ,draw opacity=1 ][line width=0.4]  (417.81,105.74) .. controls (417.81,103.89) and (416.88,102.96) .. (415.03,102.96) -- (415.03,102.96) .. controls (412.38,102.96) and (411.06,102.03) .. (411.06,100.18) .. controls (411.06,102.03) and (409.74,102.96) .. (407.09,102.96)(408.28,102.96) -- (407.09,102.96) .. controls (405.24,102.96) and (404.31,103.89) .. (404.31,105.74) ;
%Shape: Brace [id:dp27776204037161345] 
\draw  [color={rgb, 255:red, 0; green, 0; blue, 255}  ,draw opacity=1 ][line width=0.75]  (454.94,105.99) .. controls (454.94,104.52) and (454.21,103.79) .. (452.74,103.79) -- (452.74,103.79) .. controls (450.64,103.79) and (449.59,103.06) .. (449.59,101.59) .. controls (449.59,103.06) and (448.54,103.79) .. (446.45,103.79)(447.39,103.79) -- (446.45,103.79) .. controls (444.98,103.79) and (444.25,104.52) .. (444.25,105.99) ;
%Straight Lines [id:da0010357199615511625] 
%\draw [color={rgb, 255:red, 255; green, 0; blue, 0 }  ,draw opacity=1 ][fill={rgb, 255:red, 255; green, 0; blue, 0 }  ,fill opacity=1 ][line width=0.75]    (427,33) --(447,33) ;
%Straight Lines [id:da9631480767353613] 
%\draw [color={rgb, 255:red, 0; green, 0; blue, 255}  ,draw opacity=1 ]   (426,50) -- (446,50) ;

% Text Node
\draw (438.34,87) node [anchor=north west][inner sep=0.75pt]      {{\color{blue}{$A_3$}}};
% Text Node
\draw (364.09,87) node [anchor=north west][inner sep=0.75pt]      {{\color{blue}{$A_1$}}};
% Text Node
\draw (398.41,87) node [anchor=north west][inner sep=0.75pt]     {{\color{blue}{$A_2$}}};
% Text Node
\draw (422,108) node [anchor=north west][inner sep=0.75pt]    {{\color{red}{$A$}}};
% Text Node
\draw (365,5) node [anchor=north west][inner sep=0.75pt]    {$\pi ( \cdot |\bs)$};
% Text Node
\draw (466.63,108) node [anchor=north west][inner sep=0.75pt]    {$\calA$};
% Text Node
\draw (447,31) node [anchor=north west][inner sep=0.75pt]   [align=left] {unimodal};

%Straight Lines [id:da026142587587722588] 
\draw [color={rgb, 255:red, 0; green,0; blue, 255}  ,draw opacity=1 ]   (347,48.72) -- (365.25,58.75) ;
\draw [shift={(367,59.72)}, rotate = 208.81] [color={rgb, 255:red, 0; green, 0; blue, 255 }  ,draw opacity=1 ][line width=0.75]    (10.93,-3.29) .. controls (6.95,-1.4) and (3.31,-0.3) .. (0,0) .. controls (3.31,0.3) and (6.95,1.4) .. (10.93,3.29)   ;

%Straight Lines [id:da48091193037991675] 
\draw [color={rgb, 255:red, 255; green, 0; blue, 0 }  ,draw opacity=1 ]   (450,45.72) -- (437.22,56.17) ;
\draw [shift={(435.67,57.43)}, rotate = 320.75] [color={rgb, 255:red, 255; green, 0; blue, 0}  ,draw opacity=1 ][line width=0.75]    (10.93,-3.29) .. controls (6.95,-1.4) and (3.31,-0.3) .. (0,0) .. controls (3.31,0.3) and (6.95,1.4) .. (10.93,3.29)   ;

% Text Node
\draw (324,30) node [anchor=north west][inner sep=0.75pt]   [align=left] {multimodal};

\end{tikzpicture}

\caption{Unimodal Distribution vs Multimodal Distribution.}
    \vspace*{-12pt}
    \label{unimodal-vs-multimodal-policy}
\end{wrapfigure}

\begin{figure*}[t!]
    \centering
    {\includegraphics[width=3.8cm,height=3.8cm]{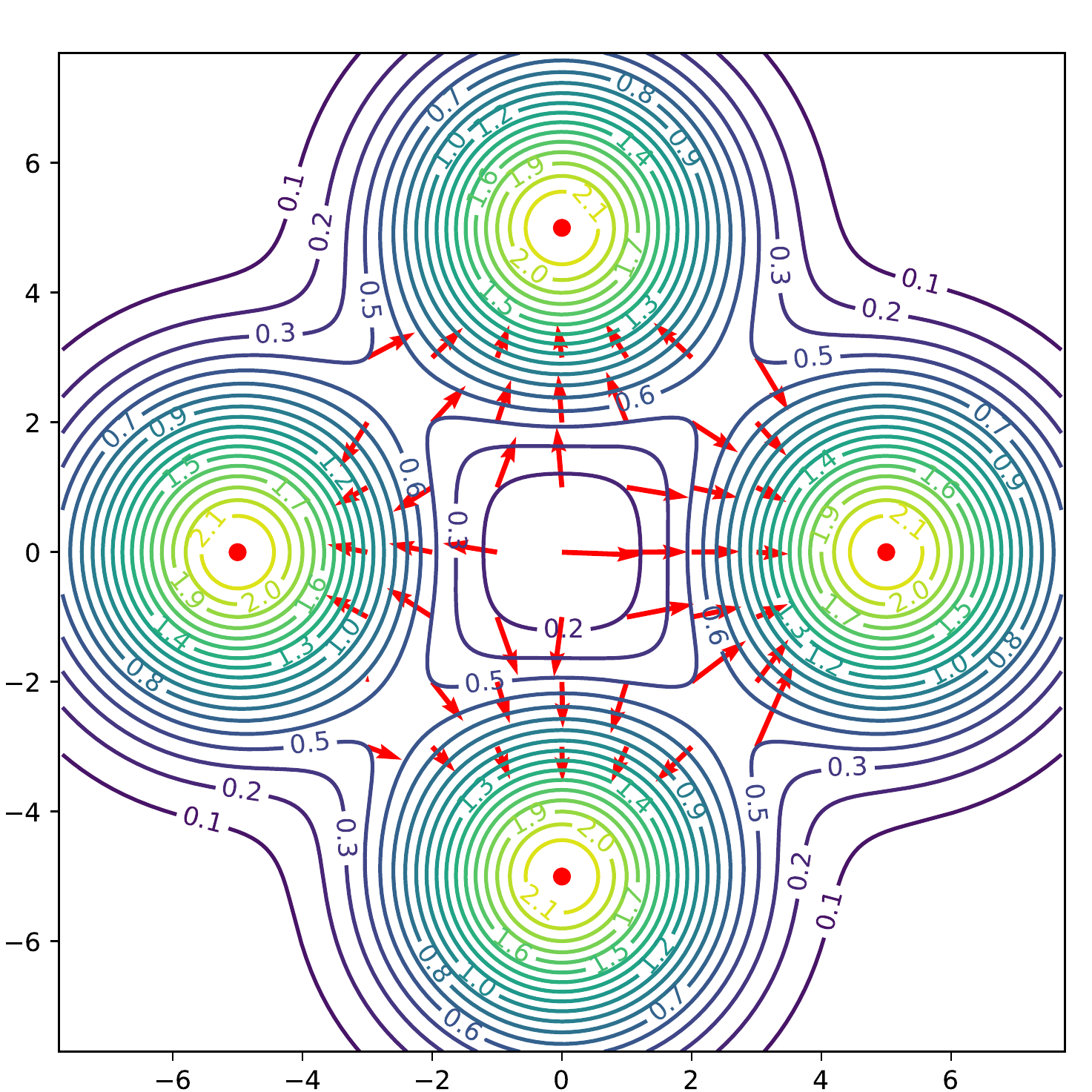}}
    {\includegraphics[width=3.8cm,height=3.8cm]{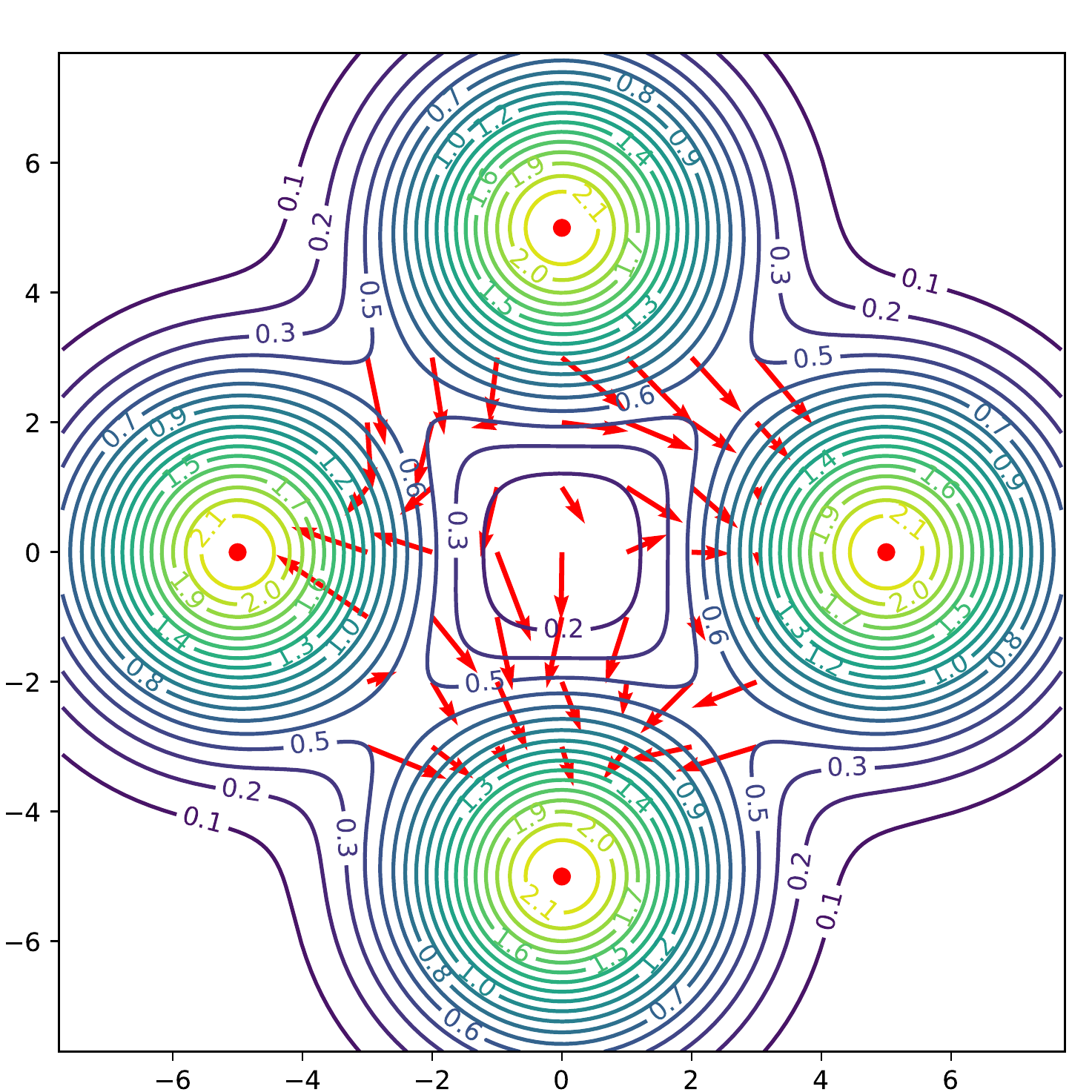}}
            {\includegraphics[width=3.8cm,height=3.8cm]{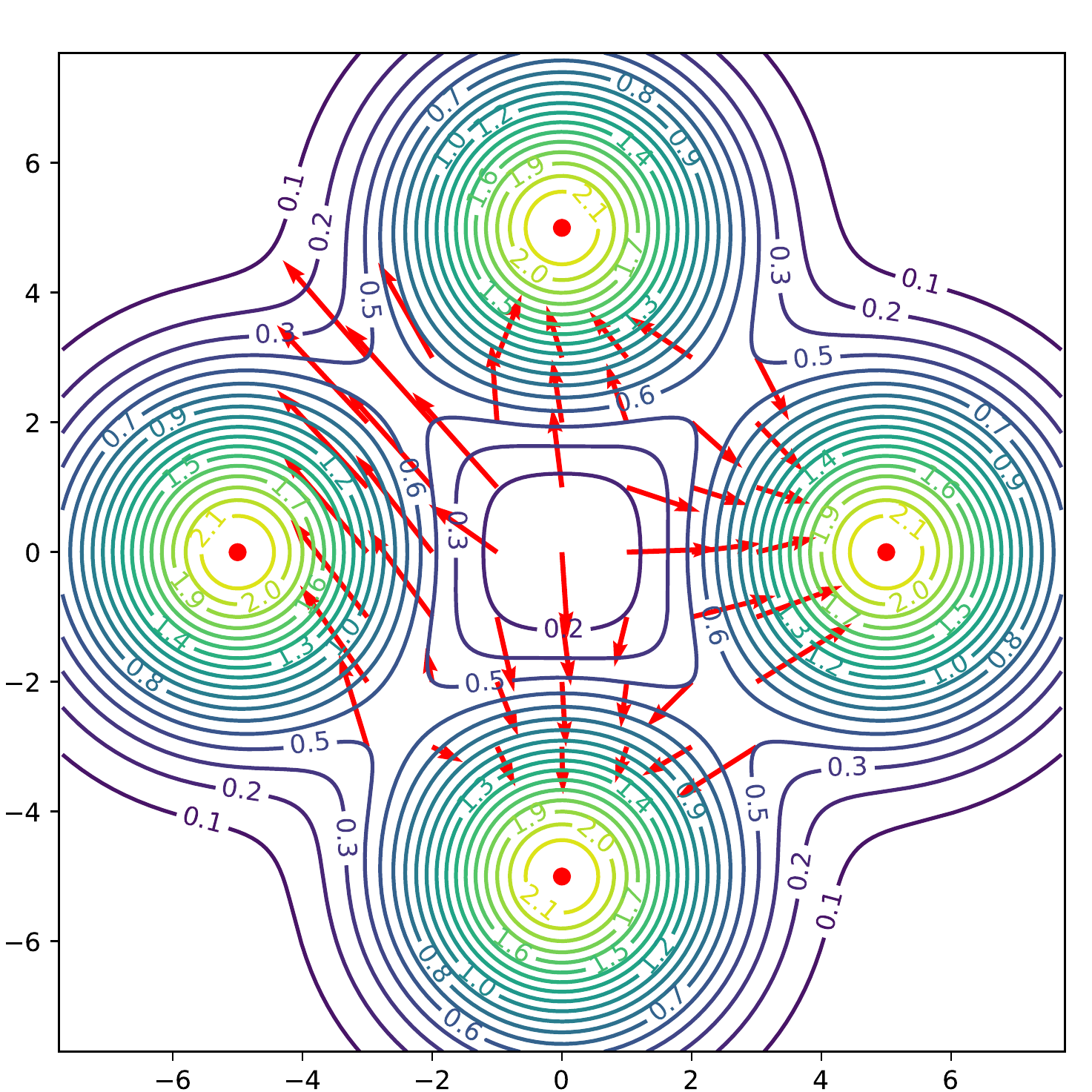}}
    {\includegraphics[width=3.8cm,height=3.8cm]{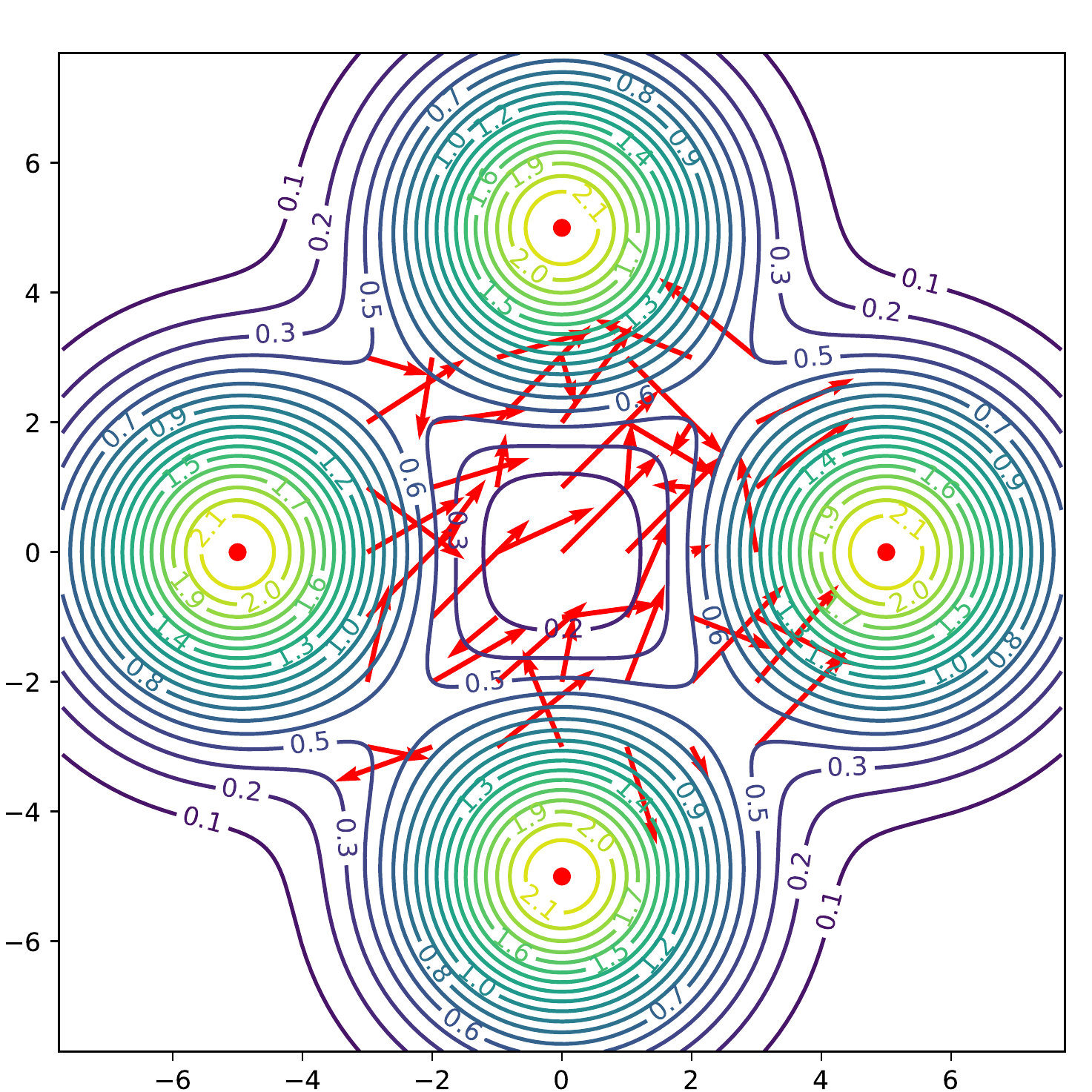}}
      \subfigure[Diffusion Policy]
        {\includegraphics[width=3.8cm,height=3.8cm]{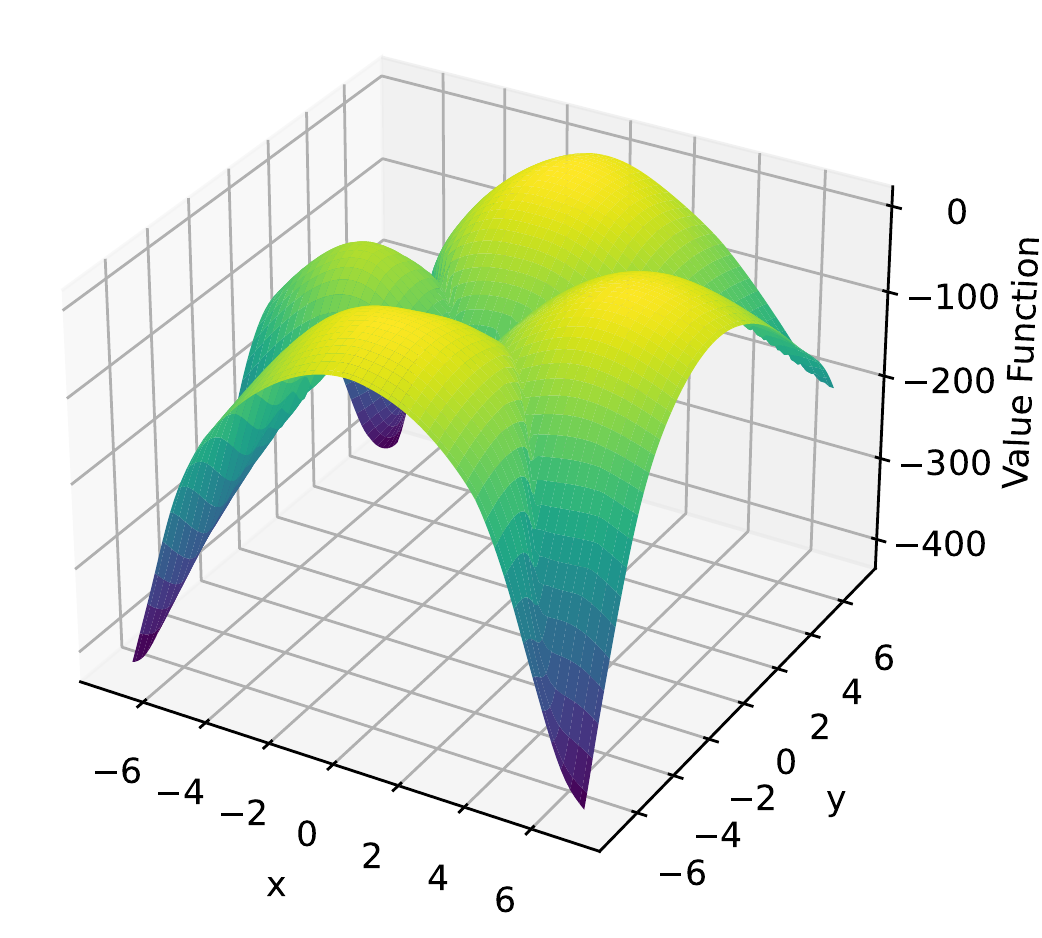}}
    \subfigure[SAC]
    {\includegraphics[width=3.8cm,height=3.8cm]{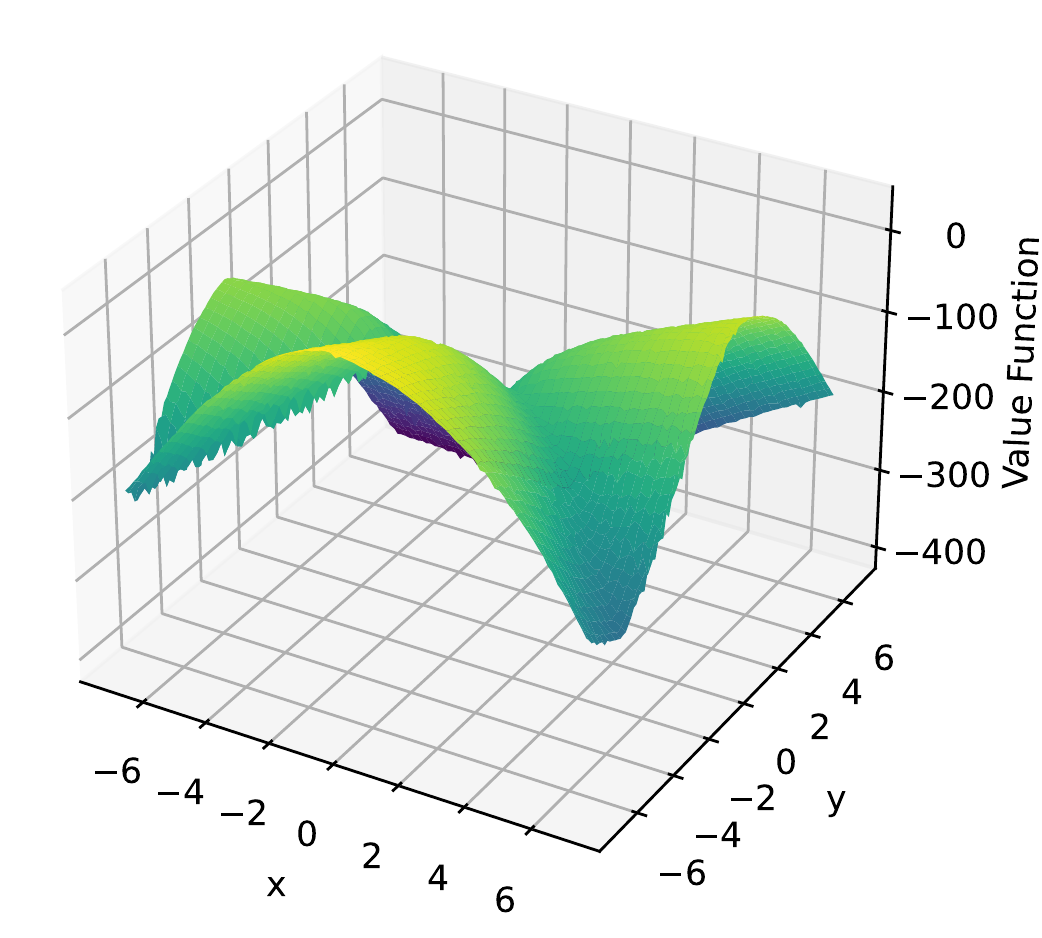}}
     \subfigure[TD3]
    {\includegraphics[width=3.8cm,height=3.8cm]{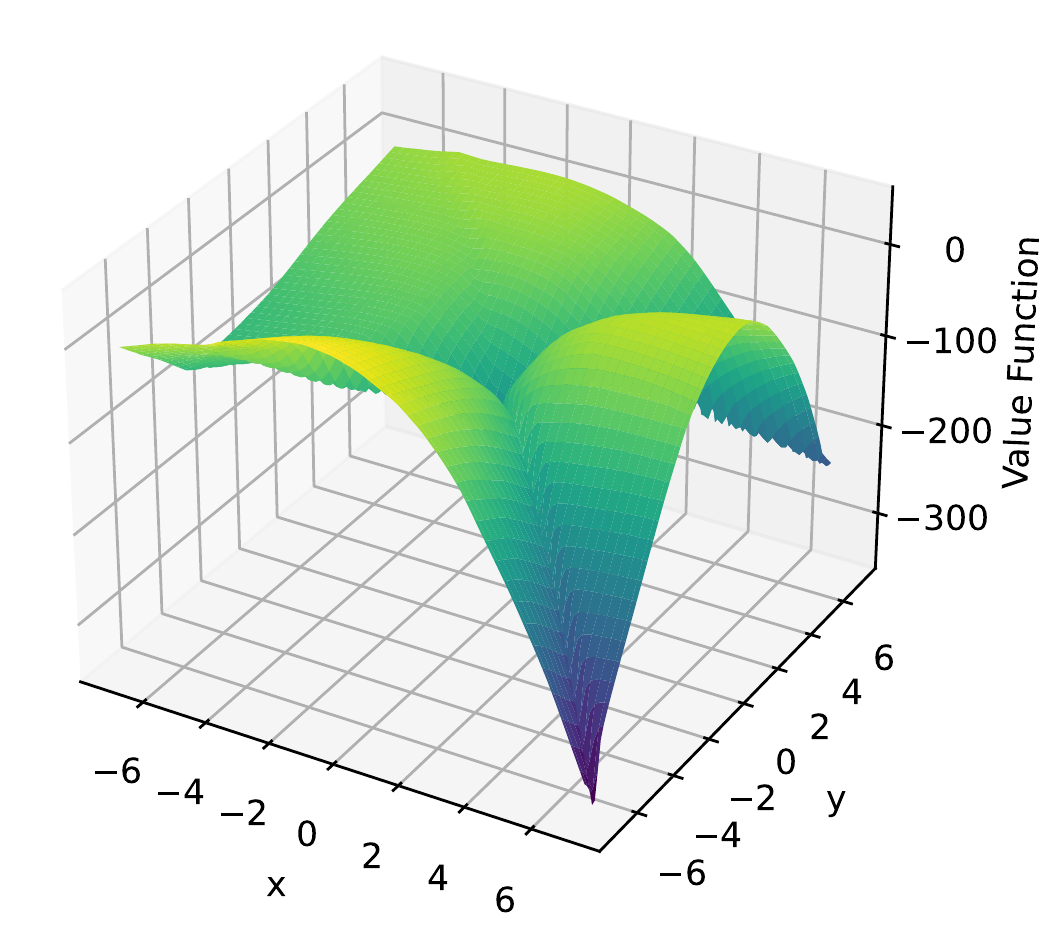}}
    \subfigure[PPO]
    {\includegraphics[width=3.8cm,height=3.8cm]{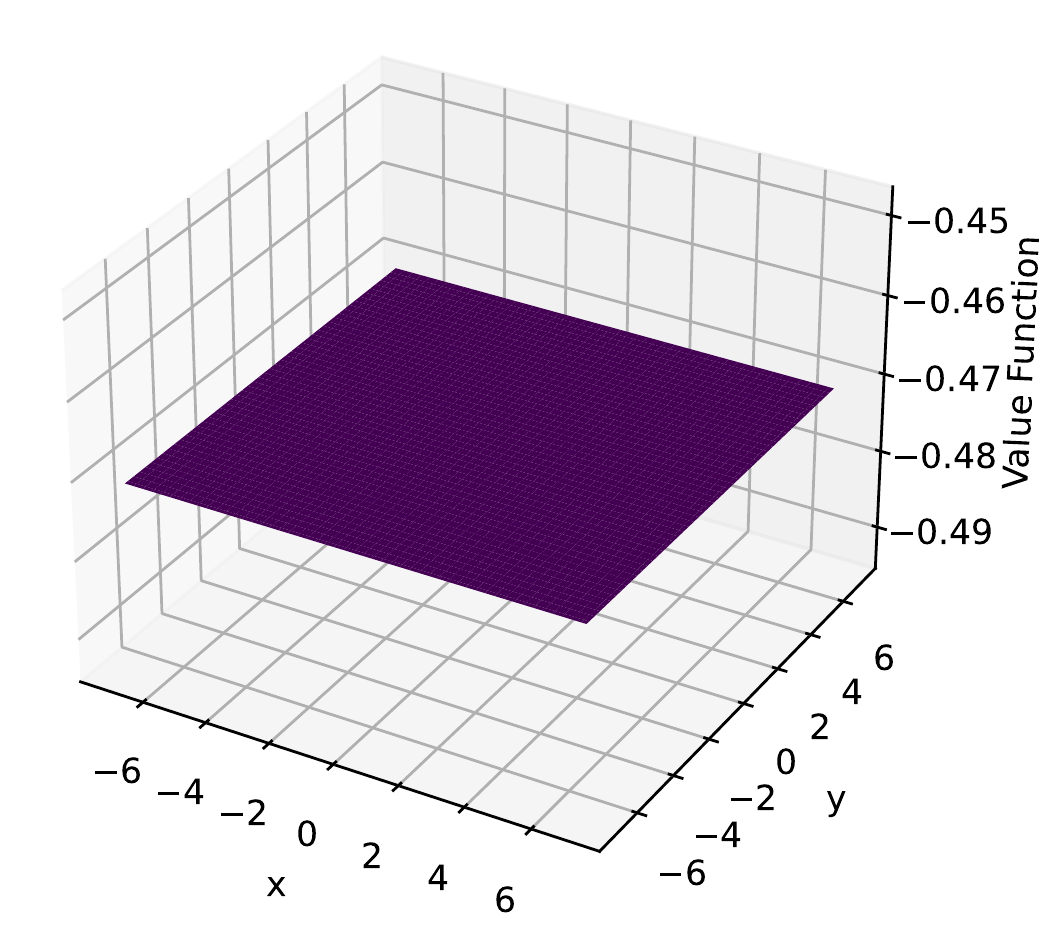}}
    \caption
    {Policy representation comparison of different policies on multimodal environment.
    %and the optimal policy is to go to one of the goal positions randomly.
    %, and we have shown the solid red lines as the final action learned by different algorithms with 1000 iterations.
    }
    \label{fig:comparsion-multi-goal}
\end{figure*}

\subsection{Diffusion Model is Powerful to Policy Representation}
\label{diff-me-po-re}

This section presents the diffusion model is powerful to represent a complex policy distribution.
by two following aspects: 
1) fitting a multimodal policy distribution is efficient for exploration;
2) empirical verification with a numerical experiment.
%
%
%The Gaussian probability density over with mean and standard deviation given by parametric functions depending on the state is a widely used way to produce a policy parameterization in RL \citep[Chapter 13.6]{sutton2018reinforcement}. 

The Gaussian policy is widely used in RL, which is a unimodal distribution, and it plays actions around the region of its mean center with a higher probability, i.e., the red region {\color{red}{$A$}} in Figure \ref{unimodal-vs-multimodal-policy}.
The unimodal policy weakens the expressiveness of complicated policy and decays the agent's ability to explore the environment.
%may be trapped in a locally optimal solution due to its limited expressiveness of complex distribution and may result in poor performance.
While for a multimodal policy, it plays actions among the different regions: {\color{blue}{$A_1$}}$\cup${\color{blue}{$A_2$}}$\cup${\color{blue}{$A_3$}}.
Compared to the unimodal policy, the multimodal policy is powerful to explore the unknown world, making the agent understand the environment efficiently and make a more reasonable decision.

We compare the ability of policy representation among SAC, TD3 \citep{fujimoto2018addressing}, PPO and diffusion policy on the “multi-goal” environment \citep{haarnoja2017reinforcement} (see Figure \ref{fig:comparsion-multi-goal}),
%We design a simple “multi-goal” environment according to the \emph{Didactic Example} \citep{haarnoja2017reinforcement}, in which the agent is a 2D point mass tries to reach one of four symmetrically placed goals.
%The goal positions are symmetrically distributed at the four points $(0,5)$, $(0,-5)$, $(5,0)$ and $(-5,0)$. 
%The optimal policy is to randomly go to one of the four-goal positions, %we have shown the contour of the reward curve in Figure \ref{fig:comparsion-multi-goal}.
%and a reasonable policy should be able to take actions uniformly to those four goal positions with the same probability due to the symmetry of the reward curve, which characters the capacity of exploration of a policy to understand the environment. 
where the $x$-axis and $y$-axis are 2D states, the four red dots denote the states of the goal at $(0,5)$, $(0,-5)$, $(5,0)$ and $(-5,0)$ symmetrically. 
%The reward is defined as a mixture of Gaussians, with means placed at the goal positions.
%The optimal policy is to go to an arbitrary goal randomly, 
%and the optimal maximum entropy policy should be able to choose each of the four goals at random
A reasonable policy should be able to take actions uniformly to those four goal positions with the same probability, 
which characters the capacity of exploration of a policy to understand the environment. 
%\textbf{Results Report.} Due to the limitation of space, in this section we only show partial results in Figure \ref{fig:comparsion-multi-goal}, and the additional comparisons and discussions are shown in Appendix \ref{sec:app-experiments-01}.
In Figure \ref{fig:comparsion-multi-goal}, the red arrowheads represent the directions of actions, and the length of the red arrowheads represents the size of the actions.
Results show that diffusion policy accurately captures a multimodal distribution landscape, while both SAC, TD3, and PPO are not well suited to capture such a multimodality.
From the distribution of action direction and length, we also know the diffusion policy keeps a more gradual and steady action size than the SAC, TD3, and PPO to fit the multimodal distribution. For more details about 2D/3D plots, environment, comparisons, and discussions, please refer to Appendix \ref{app-sec-multi-goal}.

%
%Empirical results show that the proposed diffusion policy equips a more powerful representation of multimodal policy, in which the agent is a 2D point mass tries to reach one of four symmetrically placed goals.
%The goal positions are symmetrically distributed at the four points $(0,5)$, $(0,-5)$, $(5,0)$ and $(-5,0)$. 
%The optimal policy is to go to one of the goal positions randomly, we have shown the contour of the reward curves in Figure \ref{fig:comparsion-multi-goal}.

\section{Diffusion Policy}
\label{sec-diffusion-policy}

In this section, we present the details of diffusion policy from the following three aspects: its stochastic dynamic equation (shown in Figure \ref{fig:diffusion-policy}), discretization implementation, and finite-time analysis of its performance for the policy representation.

\subsection{Stochastic Dynamics of Diffusion Policy}

Recall Figure \ref{fig:diffusion-policy}, we know diffusion policy contains two processes: forward process and reverse process. We present its dynamic in this section.

\textbf{Forward Process.} To simplify the expression, we only consider $g(t)=\sqrt{2}$, which is parallel to the general setting in Figure \ref{fig:diffusion-policy}.
For any given state $\bs$, the forward process produces a sequence $\{{(\bara_t|\bs)}\}_{t=0:T}$ that starting with $\bara_{0}\sim\pi(\cdot|\bs)$, 
and it follows the Ornstein-Uhlenbeck process (also known as Ornstein-Uhlenbeck SDE),
%the next SDE,
%\begin{flalign}
%%\label{def:diffusion-policy-sde-forward-process}
%\dd \bara_t=\bbf(\bara_t,t)\dd t+g(t)\dd \bw_{t}.
%\end{flalign}
%In this paper, we only consider the choice $g(t)=\sqrt{2}$, and $\bbf(\bx,t)=-\bx$.
%Then the forward process is reduced to 
%the next Ornstein-Uhlenbeck process \citep{ornstein1919brownian,uhlenbeck1930theory},
\begin{flalign}
\label{def:diffusion-policy-sde-forward-process-01}
\dd {\bara_t}=-{\bara_t}\dd t+\sqrt{2}\dd \bw_{t}.
\end{flalign}
Let $\bara_t\sim \barpi_t(\cdot|\bs)$ be the evolution distribution along the Ornstein-Uhlenbeck flow (\ref{def:diffusion-policy-sde-forward-process-01}).
According to Proposition \ref{solution4linearsdes} (see Appendix \ref{app-sec-transition-probability4ou}), we know the conditional distribution of $\bara_t|\bara_0$ is Gaussian,
\begin{flalign}
\label{forward-process-kernel}
\bara_t|\bara_0\sim\calN\left(\mathrm{e}^{-t}\bara_0,\left(1-\mathrm{e}^{-2t}\right)\bI\right).
\end{flalign}
That implies the forward process (\ref{def:diffusion-policy-sde-forward-process-01}) transforms policy $\pi(\cdot|\bs)$ to the Gaussian noise $\calN(\bm{0},\bI)$. 
%For practice, the forward process (\ref{def:diffusion-policy-sde-forward-process-01}) will terminate close to the standard Gaussian distribution with a large time horizon $T$ such that $\bara_{T}\sim \barpi_{T}(\cdot|\bs)\approx\calN(\bm{0},\bI)$.

\textbf{Reverse Process.} For any given state $\bs$, if we reverse the stochastic process $\{{(\bara_{t}|\bs)}\}_{t=0:T}$,
%defined in the forward process (\ref{def:diffusion-policy-sde-forward-process-01}), 
then we obtain a process that transforms noise into the policy $\pi(\cdot|\bs)$. Concretely, we model the policy as the process $\{(\tildea_t|\bs)\}_{t=0:T}$ according to the next Ornstein-Uhlenbeck process (running forward in time),
%Furthermore, we set
%\[
%{\color{orange}{\tilde{\ba}_{t}}}={\color{red}{\bar{\ba}_{T-t}}},~~\mathrm{for~all}~t\in[0,T].
%\]
%then the process ${\color{orange}\{{\tilde{\ba}_{t}}\}_{t=0}^{T}}$ follows the next SDE,
\begin{flalign}
\label{def:diffusion-policy-sde-reverse-process}
\dd \tildea_{t}=\left(\tildea_{t}+2\bnabla \log p_{T-t}(\tildea_{t})\right)\dd t+\sqrt{2}\dd \bw_{t},
\end{flalign}
where $p_{t}(\cdot)$ is the probability density function of $\barpi_t(\cdot|\bs)$.
Furthermore, according to
\citep{anderson1982reverse}, with an initial action $\tildea_{0}\sim\barpi_{T}(\cdot|\bs)$,
the reverse process $\{\tildea_{t}\}_{t=0:T}$ shares the same distribution as the time-reversed version of the forward process $\{\bara_{T-t}\}_{t=0:T}$. That also implies for all $t=0,1,\cdots,T$,%for any given $\bs\in\calS$, we know,
\begin{flalign}
\label{relation-policy-revers-forward}
\tildepi_{t}(\cdot|\bs)=\barpi_{T-t}(\cdot|\bs),~~~\text{if}~\tildea_{0}\sim\barpi_{T}(\cdot|\bs).
\end{flalign}
%Then we know $\log p_{T-t}(\tildea_{t})=\log \tildepi_{t}(\tildea_{t}|\bs)$, then we rewrite (\ref{def:diffusion-policy-sde-reverse-process}) as follows,
%\begin{flalign}
%\label{def:diffusion-policy-sde-reverse-process-reversion}
%\dd \tildea_{t}=\left(\tildea_{t}+2\bnabla \log \tildepi_{t}(\tildea_{t}|\bs)\right)\dd t+\sqrt{2}\dd \bw_{t}.
%\end{flalign}

\textbf{Score Matching.} %The process starts from the standard Gaussian distribution ${\color{orange}{\tilde{\ba}_{0}}}\sim\calN(\mathbf{0},\bI)$ to  ${\color{orange}{\tilde{\ba}_{T}}}$, and let ${\color{orange}{\tilde{\ba}_{t}}}\sim {\color{orange}{\tilde{\pi}_{t}}}(\cdot|\bs)$ be the evolution along the Ornstein-Uhlenbeck flow (\ref{def:diffusion-policy-sde-forward-process}).
The score function $\bnabla \log p_{T-t}(\cdot)$ defined in (\ref{def:diffusion-policy-sde-reverse-process}) is not explicit, we consider an estimator ${\bm{\hat \mathbf{S}}}(\cdot,\bs,T-t)$ to approximate the score function at a given state $\bs$. We consider the next problem,
\begin{flalign}
\label{score-matching-01}
{\bm{\hat \mathbf{S}}}(\cdot,\bs,T-t)=:&\arg\min_{\hat{\bm{s}}(\cdot)\in\calF}\E_{\ba\sim\tildepi_{t}(\cdot|\bs)}\left[
\big\|
\hat{\bm{s}}(\ba,\bs,t)-\bnabla \log p_{T-t}(\ba)
\big\|_2^{2}
\right]\\
\label{score-matching}
\overset{(\ref{relation-policy-revers-forward})}=&\arg\min_{\hat{\bm{s}}(\cdot)\in\calF}\E_{\ba\sim\tildepi_{t}(\cdot|\bs)}\left[
\big\|
\hat{\bm{s}}(\ba,\bs,t)-\bnabla \log \tildepi_{t}(\ba|\bs)
\big\|_2^{2}
\right],
\end{flalign}
where $\calF$ is the collection of function approximators (e.g., neural networks). 
We will provide the detailed implementations with a parametric function approximation later; please refer to Section \ref{sec:score-matching-implementation} or Appendix \ref{app-sec:loss}.

\subsection{Exponential Integrator Discretization for Diffusion Policy}

In this section, we consider the implementation of the reverse process (\ref{def:diffusion-policy-sde-reverse-process}) with exponential integrator discretization \citep{zhang2022fastsampling,lee2023convergence}.
%For any given state $\bs$, we construct a process $\{(\hata_t|\bs)\}_{t=0:T}$ that runs starting from $\hata_0\sim\calN(\bm{0},\bI)$.
Let $h>0$ be the step-size, assume reverse length $K=\frac{T}{h}\in\N$, and $t_{k}=:hk$, $k=0,1,\cdots,K$.
Then we give a partition on the interval $[0,T]$ as follows, $0=t_{0}<t_{1}<\cdots<t_{k}<t_{k+1}<\cdots<t_{K}=T$.
Furthermore, we take the discretization to the reverse process (\ref{def:diffusion-policy-sde-reverse-process}) according to the following equation,
\begin{flalign}
\label{def:diffusion-policy-sde-reverse-process-s}
\dd \hata_t=\big(\hata_t+2 \bm{\hat \mathbf{S}}(\hata_{t_k},\bs,T-t_{k})\big)\dd t+\sqrt{2}\dd \bw_{t}, ~t\in[t_{k},t_{k+1}],
\end{flalign}
where it runs initially from $\hata_0\sim\calN(\bm{0},\bI)$.
%We return the final solution $\pi(\cdot|\bs)=:{\color{orange}{\tilde{\ba}_{T}}}$ as the action the agent will play, and we call it as \emph{diffusion policy}.
By It\^{o} integration to the two sizes of (\ref{def:diffusion-policy-sde-reverse-process-s}) on the $k$-th interval $[t_{k},t_{k+1}]$, we obtain the exact solution of the SDE (\ref{def:diffusion-policy-sde-reverse-process-s}), for each $k=0,1,2,\cdots,K-1$,
%\begin{flalign}
%\label{discretization-implementation}
%\hata_{t_{k+1}}-\hata_{t_{k}}=&\left(\mathrm{e}^{t_{k+1}-t_{k}}-1\right)\left(\hata_{t_{k}}+2 \bm{\hat \mathbf{S}}(\hata_{t_k},\bs,T-t_{k})\right)+\sqrt{2}\int_{t_k}^{t_{k+1}}\mathrm{e}^{t^{'}-t_{k}}\dd \bw_{t^{'}},
%\end{flalign}
%which implies the following iteration,
\begin{flalign}
\label{iteration-exponential-integrator-discretization}
\hata_{t_{k+1}}=&\mathrm{e}^{h}\hata_{t_{k}}+\left(\mathrm{e}^{h}-1\right)\big(\hata_{t_{k}}+2 \bm{\hat \mathbf{S}}(\hata_{t_k},\bs,T-t_{k})\big)+\sqrt{\mathrm{e}^{2h}-1}\bz_{k},~\bz_{k}\sim\calN(\bm{0},\bI).
\end{flalign}
For the derivation from the SDE (\ref{def:diffusion-policy-sde-reverse-process-s}) to the iteraion (\ref{iteration-exponential-integrator-discretization}), please refer to Appendix \ref{app-sec:exponential-iIntegrator}, and
we have shown the implementation in Algorithm \ref{algo:diffusion-policy-general-case}.

\begin{algorithm}[t]
    \caption{Diffusion Policy with Exponential Integrator Discretization to Approximate $\pi(\cdot|\bs)$}
    \label{algo:diffusion-policy-general-case}
    \begin{algorithmic}[1]
         \STATE \textbf{input:} state $\bs$, horizon $T$, reverse length $K$, step-size $h=\frac{T}{K}$, score estimators $\hat{\mathbf{S}}(\cdot,\bs,T-t)$;
          \STATE \textbf{initialization:} a random action $\hata_{0}\sim\calN(\bm{0},\bI)$;
              \STATE\textbf{for }{$k=0,1,\cdots,K-1$} \textbf{do}
              %\STATE $\bz_{k}\sim\calN(\bm{0},\bI)$, if $k>1$; else $\bz_{k}=0$;
              %\STATE ${\color{red}{\tilde{\ba}_{k-1}}}\gets\dfrac{1}{\sqrt{\alpha_k}}\left({\color{red}{\tilde{\ba}_k}}-\dfrac{\beta_k}{\sqrt{1-\bar{\alpha}_k}}\bepsilon_{\bphi}({\color{red}{\tilde{\ba}_{k}}},\bs,k)\right)+\sigma_k\bz_k;$
               \STATE ~~~~~a random $\bz_{k}\sim\calN(\bm{0},\bI)$, set $t_k=hk$;
             \STATE ~~~~~$\hata_{t_{k+1}}=\mathrm{e}^{h}\hata_{t_{k}}+\left(\mathrm{e}^{h}-1\right)\left(\hata_{t_{k}}+2 \bm{\hat \mathbf{S}}(\hata_{t_k},\bs,T-t_{k})\right)+\sqrt{\mathrm{e}^{2h}-1}\bz_{k}$;
              % \ENDFOR  
              \STATE \textbf{output:} $\hata_{t_{K}}$;
     \end{algorithmic}
\end{algorithm}

%\begin{lemma}
%The time derivative of KL-divergence between the distribution $\hatpi_t(\cdot|\bs)$ and $\tildepi_{t}(\cdot|\bs)$ can be further decomposed  as follows,
%\begin{flalign}
%&\dfrac{\dd}{\dd t}\KL\big(\hatpi_t(\cdot|\bs)\|\tildepi_t(\cdot|\bs)\big)=-\FI \big(\hatpi_t(\cdot|\bs)\|\tildepi_t(\cdot|\bs)\big)\\
%&~~~~~~~~~~+2\bigintssss_{\R^{p}}\bigintssss_{\R^{p}}\rho_{0,t}({\hata_{0}},\ba|\bs)\left\langle\bnabla\log \dfrac{\hatpi_t(\ba|\bs)}{\tildepi_{t}(\ba|\bs)},\bm{\hat \mathbf{S}}({\hata_{0}},\bs,T)-\bnabla\log\tildepi_t(\ba|\bs)\right\rangle\dd\ba\dd {\hata_{0}}.
%\end{flalign}
%\end{lemma}

\subsection{Convergence Analysis of Diffusion Policy}

In this section, we present the convergence analysis of diffusion policy, we need the following notations and assumptions before we further analyze.
Let $\rho(\bx)$ and $\mu(\bx)$ be two smooth probability density functions on the space $\R^{p}$, 
the Kullback–Leibler (KL) divergence and relative Fisher information (FI) from $\mu(\bx)$ to $\rho(\bx)$ are defined as follows,
\[
\KL(\rho\|\mu)=\bigintsss_{\R^{p}}\rho(\bx)\log\dfrac{\rho(\bx)}{\mu(\bx)}\dd \bx, ~\FI(\rho\|\mu)=\bigintsss_{\R^{p}}\rho(\bx)\left\|\bnabla\log\left(\dfrac{\rho(\bx)}{\mu(\bx)}\right)\right\|_{2}^{2}\dd \bx.
\]
\begin{assumption}
[Lipschitz Score Estimator and Policy]
\label{assumption-score}
The score estimator is $L_s$-Lipschitz over action space $\calA$, and the policy $\pi(\cdot|\bs)$ is $L_{p}$-Lipschitz over action space $\calA$, i.e., for any $\ba$, $\ba^{'}\in\calA$, the following holds,
\[
\|\bm{\hat \mathbf{S}}(\ba,\bs,t)-\bm{\hat \mathbf{S}}(\ba^{'},\bs,t)\|\leq L_{s}\|\ba-\ba^{'}\|, ~\|\bnabla\log\pi(\ba|\bs)-\bnabla\log\pi(\ba^{'}|\bs)\|\leq L_{p}\|\ba-\ba^{'}\|.
\]
\end{assumption}
%\begin{assumption}
%\label{assumption-lipschitz-policy}
%The policy $\pi(\cdot|\bs)$ is $L_{p}$-Lipschitz over the action space $\calA$, i.e., for any $\ba$, $\ba^{'}\in\calA$,
%\[
%\|\bnabla\pi(\ba|\bs)-\bnabla\pi(\ba^{'}|\bs)\|\leq L_{p}\|\ba-\ba^{'}\|.
%\]
%\end{assumption}
%Under Assumption \ref{assumption-score}, according to \citep[Lemma 13]{chen2022improved}, we know $\bnabla\log {\color{orange}{\tilde{\pi}_{t}}}(\cdot|\bs)$ is $L_{p}L_{t}$--Lipschitz, where $0<L_{t}<\frac{1}{2L_{p}(1-\mathrm{e}^{-2t})}$.
\begin{assumption}
[Policy with $\nu$-LSI Setting]
\label{assumption-policy-calss}
The policy $\pi(\cdot|\bs)$ satisfies $\nu$-Log-Sobolev inequality (LSI) that defined as follows, 
there exists constant $\nu>0$, for any probability distribution $\mu(\bx)$ such that
\[
\KL(\mu\|\pi)\leq\frac{1}{2\nu}\FI(\mu\|\pi).
\]
\end{assumption}

Assumption \ref{assumption-score} is a standard setting for Langevin-based algorithms (e.g., \citep{wibisono2022convergence,vempala2019rapid}), and we extend it with RL notations.
Assumption \ref{assumption-policy-calss} presents the policy distribution class that we are concerned, 
which contains many complex distributions that are not restricted to be log-concave, e.g. any slightly smoothed bound distribution admits the condition (see \citep[Proposition 1]{ma2019sampling}).
%which contains many complex distributions, e.g., both any slightly smoothed bound distribution \citep{ma2019sampling} and the non-log-concave distribution admit the condition \citep{ledoux1999concentration}.

\begin{theorem}[Finite-time Analysis of Diffusion Policy]
\label{finite-time-diffusion-policy}
For a given state $\bs$, let $\{\barpi_{t}(\cdot|\bs)\}_{t=0:T}$ and $\{\tildepi_{t}(\cdot|\bs)\}_{t=0:T}$ be the distributions along the flow (\ref{def:diffusion-policy-sde-forward-process-01}) and (\ref{def:diffusion-policy-sde-reverse-process}) correspondingly, where $\{\barpi_{t}(\cdot|\bs)\}_{t=0:T}$ starts at $\barpi_{0}(\cdot|\bs)=\pi(\cdot|\bs)$ and 
$\{\tildepi_{t}(\cdot|\bs)\}_{t=0:T}$ starts at $\tildepi_{0}(\cdot|\bs)=\barpi_{T}(\cdot|\bs)$.
Let $\hatpi_{k}(\cdot|\bs)$ be the distribution of the iteration (\ref{iteration-exponential-integrator-discretization}) at the $k$-th time $t_{k}=hk$, i.e.,
$\hata_{t_k}\sim \hatpi_{k}(\cdot|\bs)$ denotes the diffusion policy (see Algorithm \ref{algo:diffusion-policy-general-case}) at the time $t_{k}=hk$.
Let $\{\hatpi_{k}(\cdot|\bs)\}_{k=0:K}$ be starting at $\hatpi_{0}(\cdot|\bs)=\calN(\bm{0},\bI)$, under Assumption \ref{assumption-score} and \ref{assumption-policy-calss}, let the reverse length
$
K\ge T\cdot\max\left\{\tau^{-1}_{0},\mathtt{T}^{-1}_{0},12L_{s},\nu\right\}$, where constants $\tau_{0}$ and $\mathtt{T}_{0}$ will be special later.
Then the KL-divergence between diffusion policy $\hatpi_{K}(\cdot|\bs)$ and input policy $\pi(\cdot|\bs)$ is upper-bounded as follows,
\begin{flalign}
\nonumber
&\KL\big(\hatpi_{K}(\cdot|\bs)\|\pi(\cdot|\bs)\big)\leq\underbrace{\mathrm{e}^{-\frac{9}{4}\nu hK}\KL \big(\calN(\bm{0},\bI)\|\pi(\cdot|\bs)\big)}_{\mathrm{convergence~of~forward~process}~(\ref{def:diffusion-policy-sde-forward-process-01})}+\underbrace{\left(64pL_{s}\sqrt{{5}/{\nu}}\right)h}_{\mathrm{errors~from~discretization}~(\ref{iteration-exponential-integrator-discretization})}+\dfrac{20}{3}\epsilon_{\mathrm{score}},\\
\nonumber
&\emph{where}~\epsilon_{\mathrm{score}}=\underbrace{\sup_{(k,t)\in[K]\times[t_{k},t_{k+1}]}\left\{\log\E_{\ba\sim\tildepi_t(\cdot|\bs)}\left[\exp\left\|\bm{\hat \mathbf{S}}(\ba,\bs,T-hk)-\bnabla\log\tildepi_t(\ba|\bs)\right\|_{2}^{2}\right]\right\}}_{\mathrm{errors~from~score~matching}~(\ref{score-matching})}.
\end{flalign}
\end{theorem}
Theorem \ref{finite-time-diffusion-policy} illustrates that the errors involve the following three terms.
The first error involves $\KL \big(\calN(\bm{0},\bI)\|\pi(\cdot|\bs)\big)$ that represents how close the distribution of the input policy $\pi$ is to the standard Gaussian noise, which is bounded by the exponential convergence rate of 
Ornstein-Uhlenbeck flow (\ref{def:diffusion-policy-sde-forward-process-01}) \citep{bakry2014analysis,wibisono2018convexity,chen2023sampling}.
The second error is sourced from exponential integrator discretization implementation (\ref{iteration-exponential-integrator-discretization}), which scales with the discretization step-size $h$. The discretization error term implies a first-order convergence rate with respect to the discretization step-size $h$ and scales polynomially on other parameters.
The third error is sourced from score matching (\ref{score-matching}), which represents how close the score estimator $\hat{\mathbf{S}}$ is to the score function $\bnabla\log p_{T-t}(\cdot)$ defined in (\ref{def:diffusion-policy-sde-reverse-process}).
That implies for the practical implementation, the error from score matching could be sufficiently small if we find a good score estimator $\hat{\mathbf{S}}$.

Furthermore, for any $\epsilon>0$, if we find a good score estimator that makes the score matching error satisfy $\epsilon_{\mathrm{score}}<\frac{1}{20}\epsilon$, the step-size $h=\calO\left(\frac{\epsilon\sqrt{\nu}}{pL_{s}}\right)$, and reverse length $K=\frac{9}{4\nu h}\log\frac{3\KL (\calN(\bm{0},\bI)\|\pi(\cdot|\bs))}{\epsilon}$, then Theorem \ref{finite-time-diffusion-policy} implies the output of diffusion policy $(\hatpi_{K}(\cdot|\bs)$ makes a sufficient close to the input policy $\pi(\cdot|\bs)$ with the measurement by $\KL(\hatpi_{K}(\cdot|\bs)\|\pi(\cdot|\bs))\leq\epsilon$.

\section{DIPO: Implementation of Diffusion Policy for Model-Free Online RL }

\label{sec:score-matching-implementation}

In this section, we present the details of DIPO, which is an implementation of  \textbf{DI}ffusion \textbf{PO}licy for model-free reinforcement learning. According to Theorem \ref{finite-time-diffusion-policy}, diffusion policy only fits the current policy $\pi$ that generates the training data (denoted as $\calD$), but it does not improve the policy $\pi$. It is different from traditional policy-based RL algorithms, we can not improve a policy according to the policy gradient theorem since diffusion policy is not a parametric function but learns a policy via a stochastic process. Thus, we need a new way to implement policy improvement, which is nontrivial.
We have presented the framework of DIPO in Figure \ref{fig:framework}, and shown the key steps of DIPO in Algorithm \ref{algo:diffusion-free-based-rl}. For the detailed implementation, please refer to Algorithm \ref{app-algo-diffusion-free-based-rl} (see Appendix \ref{sec:app-details-of-implementation}).
%Now, we present the main steps as follows.
%Before clarifying our insights to policy improvement with diffusion policy for model-free RL, we need to present the details of training the diffusion policy.

%\subsection{Implementation for Score Matching of DIPO and Playing Action}

\subsection{Training Loss of DIPO}

It is intractable to directly apply the formulation (\ref{score-matching}) to estimate the score function since $\bnabla \log p_{t}(\cdot)=\bnabla \log \barpi_{t}(\cdot|\bs)$ is unknown, which is sourced from the initial distribution $\bara_{0}\sim\pi(\cdot|\bs)$ is unknown. According to denoising score matching \citep{vincent2011connection,hyvarinen2005estimation}, a practical way is to solve the next optimization problem (\ref{denosing-score-matching}). For any given $\bs\in\calS$,
\begin{flalign}
\label{denosing-score-matching}
\min_{\bphi}\calL(\bphi)=
\min_{\hat{\bm{s}}_{\bphi}\in\calF}
\int_{0}^{T}
\omega (t)
\E_{\bara_{0}\sim\pi(\cdot|\bs)}
\E_{\bara_t|\bara_0}
\left[
\big\|
\hat{\bm{s}}_{\bphi}(\bara_t,\bs,t)-\bnabla\log \varphi_{t}(\bara_t|\bara_0)
\big\|_2^{2}
\right]\dd t,
\end{flalign}
where $\omega(t):[0,T]\rightarrow \R_{+}$ is a positive weighting function; $\varphi_{t}(\bara_t|\bara_0)=\calN\left(\mathrm{e}^{-t}\bara_0,\left(1-\mathrm{e}^{-2t}\right)\bI\right)$ denotes the transition kernel of the forward process (\ref{forward-process-kernel}); $\E_{\bara_t|\bara_0}[\cdot]$ denotes the expectation with respect to $\varphi_{t}(\bara_t|\bara_0)$; and $\bphi$ is the parameter needed to be learned. 
Then, according to Theorem \ref{loss-diff-policy} (see Appendix \ref{app-sec:loss}), we rewrite the objective (\ref{denosing-score-matching}) as follows,
\begin{flalign}
\label{loss-diffusion-action}
\calL(\bphi)&=\E_{k\sim\calU(\{1,2,\cdots,K\}),\bz\sim\calN(\bm{0},\bI),(\bs,\ba)\sim\calD}\left[\|\bz-\bepsilon_{\bphi}\left(\sqrt{\bar\alpha_k}\ba+\sqrt{1-\bar{\alpha}_k}\bz,\bs,k\right)\|_{2}^{2}\right],
\end{flalign}
where $\calU(\cdot)$ denotes uniform distribution,
\[\bm{\epsilon}_{\bphi}\left(\cdot,\cdot,k\right)=-\sqrt{1-\bar{\alpha}_{k}}\hat{\bm{s}}_{\bphi}\left(\cdot,\cdot,T-t_k\right),\] and $\bar{\alpha}_k$ will be special.
The objective (\ref{loss-diffusion-action}) provides a way to learn $\bphi$ from samples; see line 14-16 in Algorithm \ref{algo:diffusion-free-based-rl}.

\subsection{Playing Action of DIPO}
Replacing the score estimator $\hat{\mathbf{S}}$ (defined in Algorithm \ref{algo:diffusion-policy-general-case}) according to $\hat{\bm{\epsilon}}_{\bphi}$, after some algebras (see Appendix \ref{app-derivation-of-actions}), we rewrite diffusion policy (i.e., Algorithm \ref{algo:diffusion-policy-general-case}) as follows,
\begin{flalign}
\label{dipo-actions}
\hata_{k+1}=\dfrac{1}{\sqrt{\alpha_{k}}}\left(\hata_{k}-\dfrac{1-\alpha_{k}}{\sqrt{1-\bar{\alpha}_{k}}}\bm{\bepsilon_{\bphi}}(\hata_{k},\bs,k)\right)
+\sqrt{\dfrac{1-\alpha_{k}}{\alpha_{k}}}\bz_{k},
\end{flalign}
where $k=0,1,\cdots,K-1$ runs forward in time, the noise $\bz_{k}\sim\calN(\bm{0},\bI)$. The agent plays the last (output) action $\hata_{K}$.

\begin{algorithm}[t]
    \caption{ (DIPO) Model-Free Reinforcement Learning with \textbf{DI}ffusion \textbf{PO}licy}
    \label{algo:diffusion-free-based-rl}
    \begin{algorithmic}[1]
         \STATE \textbf{initialize} ${\bphi}$, critic network $Q_{\bpsi}$; $\{\alpha_i\}_{i=0}^{K}$; $\bar{\alpha}_{k}=\prod_{i=1}^{k}\alpha_{i}$; step-size $\eta$;
        \REPEAT
           % \STATE  {\color{blue}{{{\texttt{\#update~experience}}}}}
               \STATE dataset $\calD\gets\emptyset$; initialize $\bs_{0}\sim d_{0}(\cdot)$; 
               \STATE {\color{lightgray}{{{\texttt{\#update~experience}}}}}
                 \STATE \textbf{for} {$t=0,1,\cdots,T$}~\textbf{do} 
               \STATE ~~~~play $\ba_{t}$ follows (\ref{dipo-actions});
                $\bs_{t+1}\sim \Pro(\cdot|\bs_{t},\ba_{t})$;
                $\calD\gets\calD\cup\{\bs_{t},\ba_{t},\bs_{t+1},r(\bs_{t+1}|\bs_{t},\ba_{t})\}$;     
   \STATE   {\color{lightgray}{{{\texttt{\#update~value~function}}}}}
       \STATE \textbf{repeat} $N$ times
     \STATE ~~~~sample $(\bs_{t},\ba_{t},\bs_{t+1},r(\bs_{t+1}|\bs_{t},\ba_{t}))\sim\calD$ i.i.d;
    take gradient descent on $\ell_{\mathrm{Q}}(\bpsi)$ (\ref{loss-q});
     \STATE  {\color{lightgray}{{{\texttt{\#action gradient}}}}}
          \STATE \textbf{for} {$t=0,1,\cdots,T$}~\textbf{do} 
         \STATE~~~~replace each action $\ba_{t}\in\calD$ follows 
         $\ba_{t}\gets\ba_{t}+\eta\nabla_{\ba}Q_{\bpsi}(\bs_t,\ba)|_{\ba=\ba_t};$
     \STATE  {\color{lightgray}{{{\texttt{\#update~policy}}}}}
          \STATE\textbf{repeat} $N$ times 
            \STATE~~~~sample $(\bs,\ba)$ from $\calD$ i.i.d, sample index $k\sim \calU(\{1,\cdots,K\})$, $\bz\sim \calN(\bm{0},\bI)$;
       \STATE~~~~take gradient decent on the loss $\ell_{\mathrm{d}}(\bphi)=\|\bz-\bepsilon_{\bphi}(\sqrt{\bar\alpha_k}\ba+\sqrt{1-\bar{\alpha}_k}\bz,\bs,k)\|_{2}^{2}$; 
           \UNTIL{the policy performs well in the environment.}
     \end{algorithmic}
\end{algorithm}

\subsection{Policy Improvement of DIPO}

According to (\ref{loss-diffusion-action}), we know that only the state-action pairs $(\bs,\ba)\in\calD$ are used to learn a policy.
That inspires us that if we design a method that transforms a given pair $(\bs,\ba)\in\calD$ to be a  ``better'' pair, then we use the ``better'' pair to learn a new diffusion policy $\pi^{'}$,  then $\pi^{'}\succeq\pi$. 
About  ``better'' state-action pair should maintain a higher reward performance than the originally given pair $(\bs,\ba)\in\calD$. 
We break our key idea into two steps: 
\textbf{1)} first, we regard the reward performance as a function with respect to actions, $J_{\pi}(\ba)=\E_{\bs\sim d_{0}(\cdot)}[Q_{\pi}(\bs,\ba)]$, which quantifies how the action $\ba$ affects the performance; 
\textbf{2)} then, we update all the actions $\ba\in\calD$ through the direction $\bnabla_{\ba}J_{\pi}(\ba)$ by gradient ascent method:
%\begin{flalign}
%\label{def:objetive-action}
%J_{\pi}(\ba)=\E_{\bs\sim d_{0}(\cdot)}[Q_{\pi}(\bs,\ba)],
%\end{flalign}
%Recall the policy $\pi_{\bphi}$ is a deterministic policy, 
%We apply gradient ascent to the objective $J_{\pi}(\ba)$ (\ref{def:objetive-action}), for each $\ba\in\calD$, 
%we improve the action as follows,
\begin{flalign}
\label{improve-action}
\ba\gets\ba+\eta\bnabla_{\ba}J_{\pi}(\ba)=\ba+\eta\E_{\bs\sim d_{0}(\cdot)}[\bnabla_{\ba}Q_{\pi}(\bs,\ba)]
,
\end{flalign}
where $\eta>0$ is step-size, and we call $\nabla_{\ba}J_{\pi}(\ba)$ as \textbf{action gradient}.
%\begin{flalign}
%\label{action-gradient}
%\nabla_{\ba}J_{\pi}(\ba)=:\nabla_{\ba}\int_{\bs\in\calS} \rho_{0}(\bs)Q_{\pi}(\bs,\ba)\dd \bs
%=\int_{\bs\in\calS} \rho_{0}(\bs)\nabla_{\ba}Q_{\pi}(\bs,\ba)\dd \bs.
%\end{flalign}
To implement (\ref{improve-action}) from samples, we need a neural network $Q_{\bpsi}$ to estimate $Q_{\pi}$.
Recall $\{\bs_{t},\ba_{t},\bs_{t+1},r(\bs_{t+1}|\bs_{t},\ba_{t})\}_{t\ge0}\sim\pi$, we train the parameter $\bpsi$ by minimizing the following Bellman residual error,
\begin{flalign}
\label{loss-q}
\ell_{\mathrm{Q}}(\bpsi)=\big(r(\bs_{t+1}|\bs_t,\ba_t)+\gamma Q_{\bpsi}(\bs_{t+1},\ba_{t+1})-Q_{\bpsi}(\bs_t,\ba_t)\big)^2.
\end{flalign}
Finally, we consider each pair $(\bs_{t},\ba_{t})\in\calD$, and replace the action $\ba_{t}\in\calD$ as follows,
\begin{flalign}
\label{def:improve-action}
\ba_{t}\gets\ba_{t}+\eta\bnabla_{\ba}Q_{\bpsi}(\bs_t,\ba)|_{\ba=\ba_t}.
\end{flalign}

%which regards the performance as a function with respect to action, and improve the performance by gradient ascent methods.
%and improve the performance by gradient ascent methods.
%update the actions (among $\calD$) through the direction of \emph{action gradient}, which regards the performance as a function with respect to action, and improve the performance by gradient ascent methods.

%
%\begin{itemize}
%\item Collecting data $\calD=\{\bs_{t},\ba_{t},\bs_{t+1},r(\bs_{t+1}|\bs_{t},\ba_{t})\}_{t\ge0}$ according to diffusion policy, where each action $\ba_{t}$ is the output of Algorithm \ref{algo:diffusion-policy-general-case}.
%\item Policy Improvement
%\end{itemize}

%According to Theorem \ref{finite-time-diffusion-policy}, diffusion policy only fits the current policy $\pi$ that generates the training data $\calD$, but it does not improve the current policy $\pi$.
%In this section, we present the approach to improve the current policy. Our key idea is to update the actions (among $\calD$) through the direction of \emph{action gradient}, which regards the performance as a function with respect to action, and improve the performance by gradient ascent methods.

\section{Related Work}
\label{app-related-work}

Due to the diffusion model being a fast-growing field, this section only presents the work that relates to reinforcement learning, a recent work \citep{yang2022diffusion} provides a comprehensive survey on the diffusion model.
In this section, first, we review recent advances in diffusion models with reinforcement learning. 
Then, we review the generative models for reinforcement learning.

\subsection{Diffusion Models for Reinforcement Learning}

The diffusion model with RL first appears in \citep{janner2022diffuser}, where it proposes the diffuser that plans by iteratively refining trajectories, which is an essential offline RL method. Later \cite{ajay2023conditional} model a policy as a return conditional diffusion model, \cite{chen2023offline,wang2023diffusion,chi2023diffusion} consider to generate actions via diffusion model. SE(3)-diffusion fields \citep{urain2023se} consider learning data-driven SE(3) cost functions as diffusion models.  
\cite{pearce2023imitating} model the imitating human behavior with diffusion models.
\cite{reuss2023goal} propose score-based diffusion policies for the goal-conditioned imitation learning problems.
ReorientDiff \citep{mishra2023reorientdiff} presents a reorientation planning method that utilizes a diffusion model-based approach. 
StructDiffusion \citep{liu2022structdiffusion} is an object-centric transformer with a diffusion model,  based on high-level language goals, which constructs structures out of a single
RGB-D image.
\cite{brehmer2023edgi} propose an equivariant diffuser for generating interactions (EDGI), which trains a diffusion model on an offline trajectory dataset, where EDGI learns a world model and planning in it as a conditional generative modeling problem follows the diffuser \citep{janner2022diffuser}.
DALL-E-Bot \citep{kapelyukh2022dall} explores the web-scale image diffusion models for robotics.
AdaptDiffuser \citep{liang2023adaptdiffuser} is an evolutionary planning algorithm with diffusion, which is adapted to unseen tasks.

The above methods are all to solve offline RL problems, to the best of our knowledge, the proposed DIPO is the first diffusion approach to solve online model-free RL problems.
The action gradient plays a critical way to implement DIPO, which never appears in existing RL literature.
In fact, the proposed DIPO shown in Figure \ref{fig:framework} is a general training framework for RL, where we can replace the diffusion policy with any function fitter (e.g., MLP or VAE).

\subsection{Generative Models for Policy Learning}

In this section, we mainly review the generative models, including VAE \citep{kingma2013auto}, GAN \citep{goodfellow2020generative}, Flow  \citep{rezende2015variational},
and GFlowNet \citep{bengio2021flow,bengio2021gflownet} for policy learning.
Generative models are mainly used in cloning diverse behaviors \citep{pomerleau1988alvinn}, imitation learning \citep{osa2018algorithmic}, goal-conditioned imitation learning \citep{argall2009survey}, or offline RL \citep{levine2020offline}, a recent work \citep{yang2023foundation} provides a foundation presentation for the generative models for policy learning.

 \textbf{VAE for Policy Learning.} \cite{lynch2020learning,ajay2021opal} have directly applied  auto-encoding variational Bayes (VAE) \citep{kingma2013auto} and VQ-VAE \citep{van2017neural}
 model behavioral priors. 
 \cite{mandlekar2020learning} design the low-level policy that is conditioned on latent from the CVAE.
\cite{pertsch2021accelerating} joint the representation of skill embedding and skill prior via a deep latent variable model.
 \cite{mees2022matters,rosete2023latent} consider seq2seq CVAE \citep{lynch2020learning,wang2022hierarchical} to model of conditioning the action decoder on the latent plan allows the policy to use the entirety of its capacity for learning unimodal behavior.

\textbf{GAN for Imitation Learning.} GAIL \citep{ho2016generative} considers the Generative Adversarial Networks (GANs)  \citep{goodfellow2020generative} to imitation learning.
These methods consist of a generator and a discriminator, where the generator policy learns to imitate the experts'  behaviors, and the discriminator distinguishes between real and fake trajectories, which models the imitation learning as a distribution matching problem between the expert policy's state-action distribution and the agent’s policy \citep{fulearning,wang2021learning}. For several advanced results and applications, please refer to \citep{chen2023gail,deka2023arc,rafailov2023model,taranovic2023adversarial}.

\textbf{Flow and GFlowNet Model for Policy Learning.} \cite{singhparrot2020} consider normalizing flows \citep{rezende2015variational} for the multi-task RL tasks.
\cite{li2023diverse} propose diverse policy optimization, which consider the GFlowNet \citep{bengio2021flow,bengio2021gflownet} for the structured action spaces.
\cite{licflownets} propose CFlowNets that combines GFlowNet with continuous control.
Stochastic GFlowNet \citep{pan2023stochastic} learns a model of the environment to capture the stochasticity of state transitions. \cite{malkintrajectory2022} consider training a GFlowNet with trajectory balance.

\textbf{Other Methods.} Decision Transformer (DT) \citep{chen2021decision} model the offline RL tasks as a conditional sequence problem, which does not learn a policy follows the traditional methods (e.g., \cite{sutton1988learning,sutton1998reinforcement}).
Those methods with DT belong to the task-agnostic behavior learning methods, which is an active direction in policy learning (e,g.,  \citep{cui2023play, rt12022arxiv,zheng2022online,konan2023contrastive,kimpreference2023}).
Energy-based models \citep{lecun2006tutorial} are also modeled as conditional policies  \citep{florence2022implicit} or applied to inverse RL \citep{liu2021energy}.
Autoregressive model \citep{vaswani2017attention, brown2020language} represents the policy as the distribution of action, where it considers the distribution of the whole trajectory \citep{reed2022a, shafiullah2022behavior}.

\section{Experiments}
\label{section-ex}

In this section, we aim to cover the following three issues: How does DIPO compare to the widely used RL algorithms (SAC, PPO, and TD3) on the standard continuous control benchmark?
How to show and illustrate the empirical results? 
How does the diffusion model compare to VAE \citep{kingma2013auto} and multilayer perceptron (MLP) for learning distribution?
How to choose the reverse length $K$ of DIPO for the reverse inference?

\subsection{Comparative Evaluation and Illustration}
\label{ex-comparative-eva-ill}

We provide an evaluation on MuJoCo tasks \citep{todorov2012mujoco}.
Figure \ref{fig:comparsion-mujoco} shows the reward curves for SAC, PPO, TD3, and DIPO on MuJoCo tasks. 
To demonstrate the robustness of the proposed DIPO, we train DIPO with the same hyperparameters for all those 5 tasks, where we provide the hyperparameters in Table \ref{tab:hyperparameters-mujoco}, see Appendix \ref{app-sec:hyper-parameters4mujoco}. For each algorithm, we plot the average return of 5 independent trials as the solid curve and plot the standard deviation across 5 same seeds as the transparent shaded region. We evaluate all the methods with $10^{6}$ iterations. Results show that the proposed DIPO achieves the best score across all those 5 tasks, and DIPO learns much faster than SAC, PPO, and TD3 on the tasks of Ant-3v and Walker2d-3v. 
Although the asymptotic reward performance of DIPO is similar to baseline algorithms on other 3 tasks, the proposed DIPO achieves better performance at the initial iterations, 
we will try to illustrate some insights for such empirical results of HalfCheetah-v3 in Figure \ref{state-HalfCheetah}, for more discussions, see Appendix \ref{app-sec-experiment-dipo}.

\begin{figure*}[t!]
    \centering
    \subfigure[Ant-v3]
    {\includegraphics[width=3.0cm,height=2.5cm]{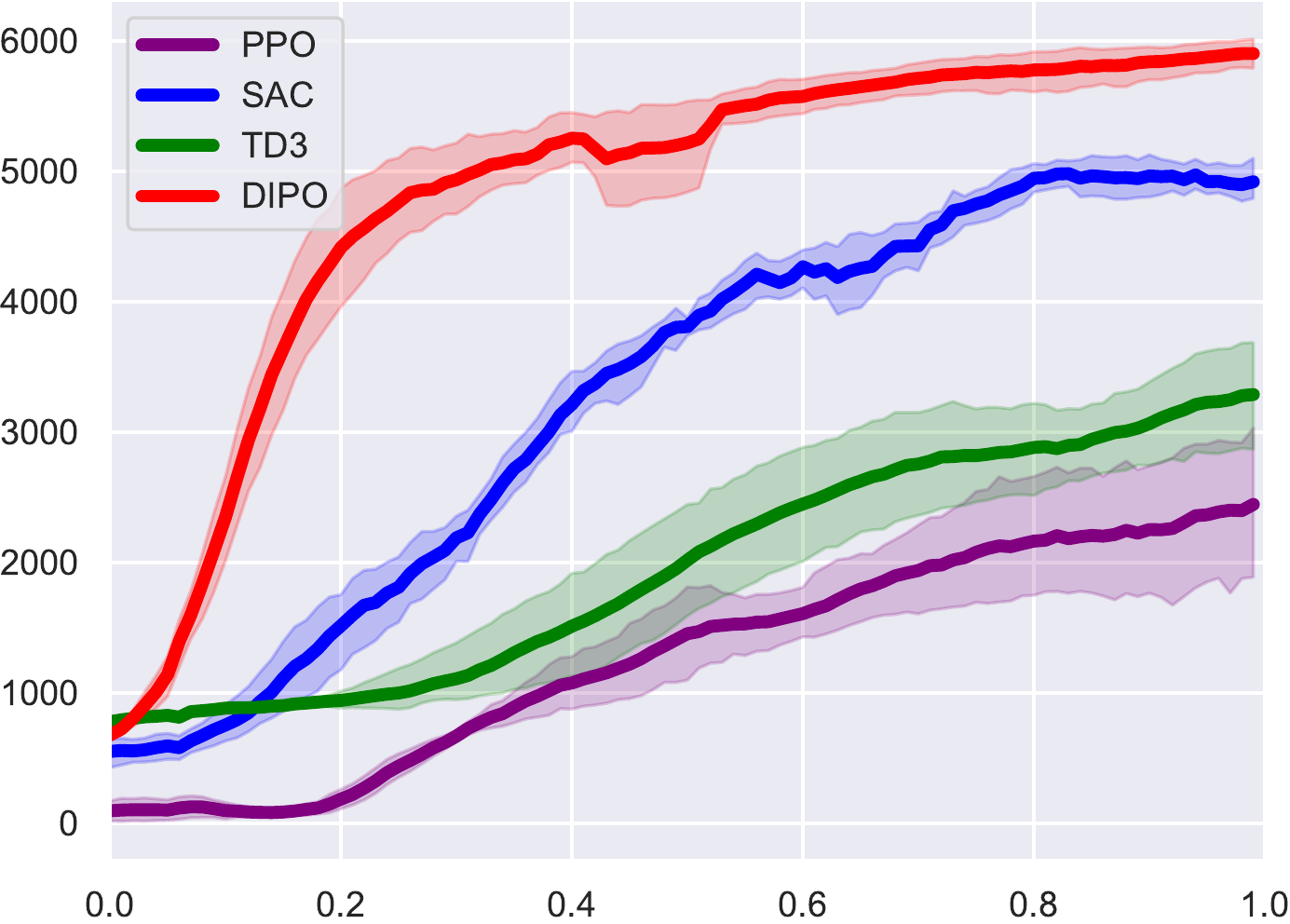}}
      \subfigure[HalfCheetah-v3]
     {\includegraphics[width=3.0cm,height=2.5cm]{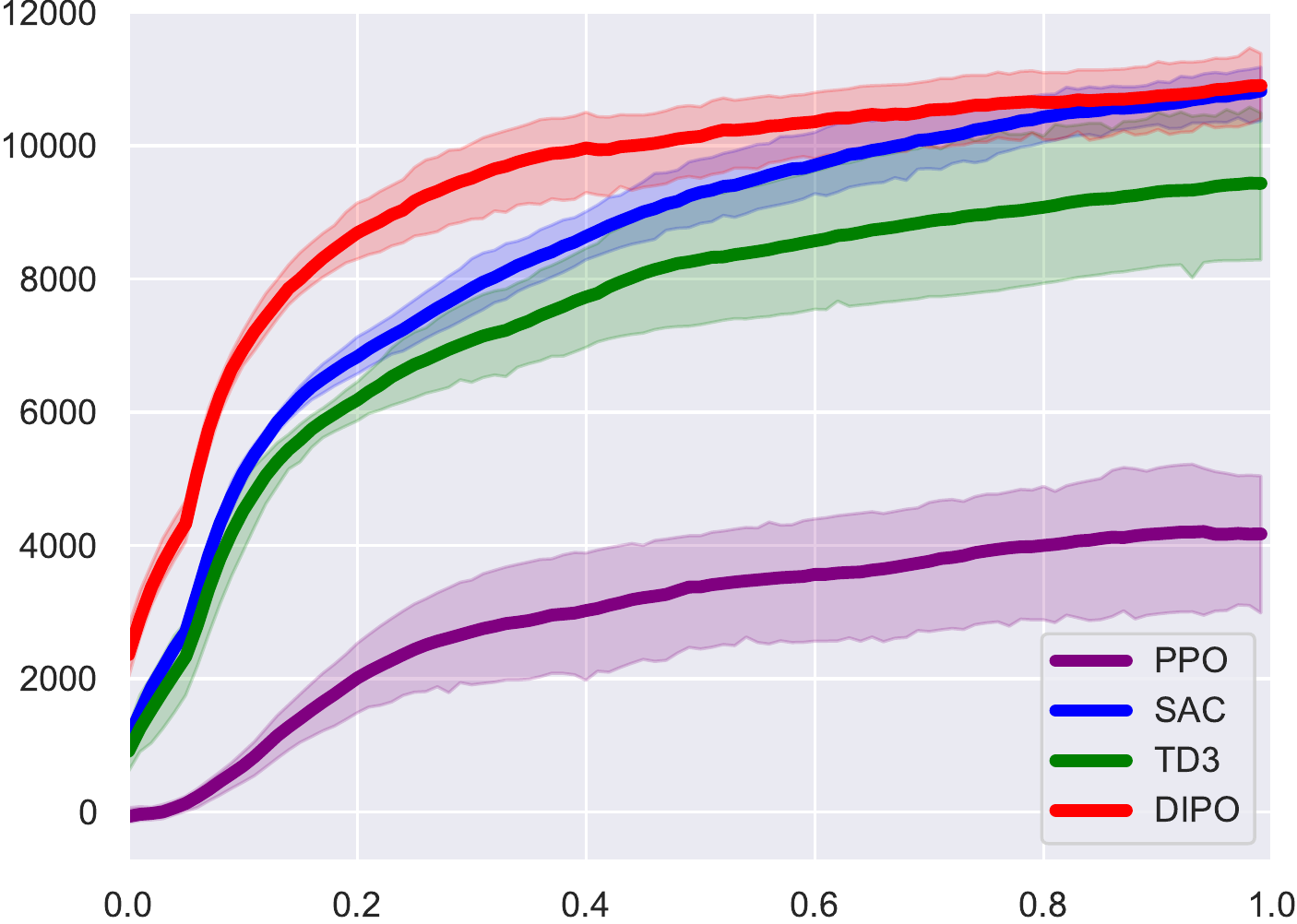}}
       \subfigure[Hopper-v3]
      {\includegraphics[width=3.0cm,height=2.5cm]{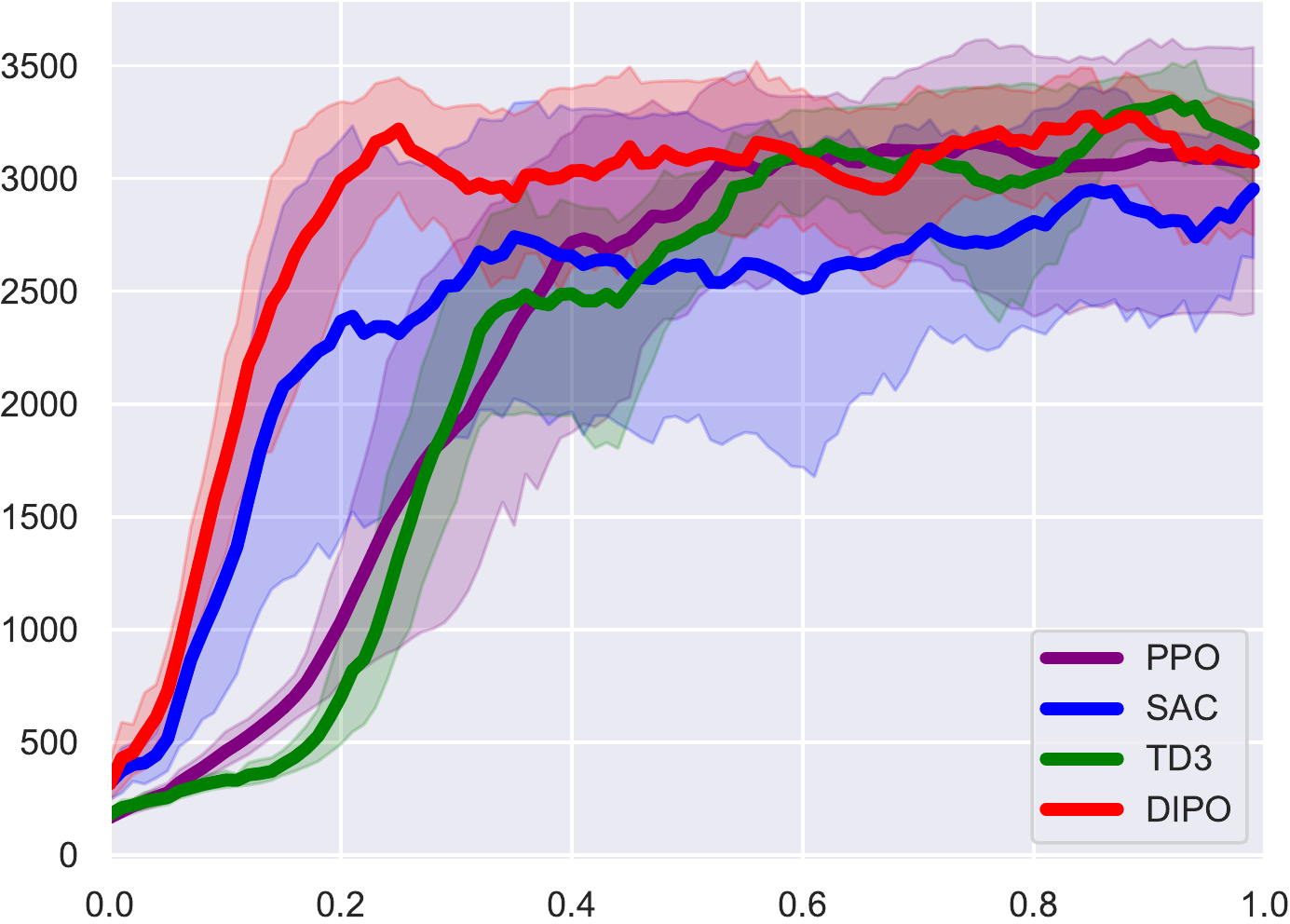}}
        \subfigure[Humanoid-v3]
       {\includegraphics[width=3.0cm,height=2.5cm]{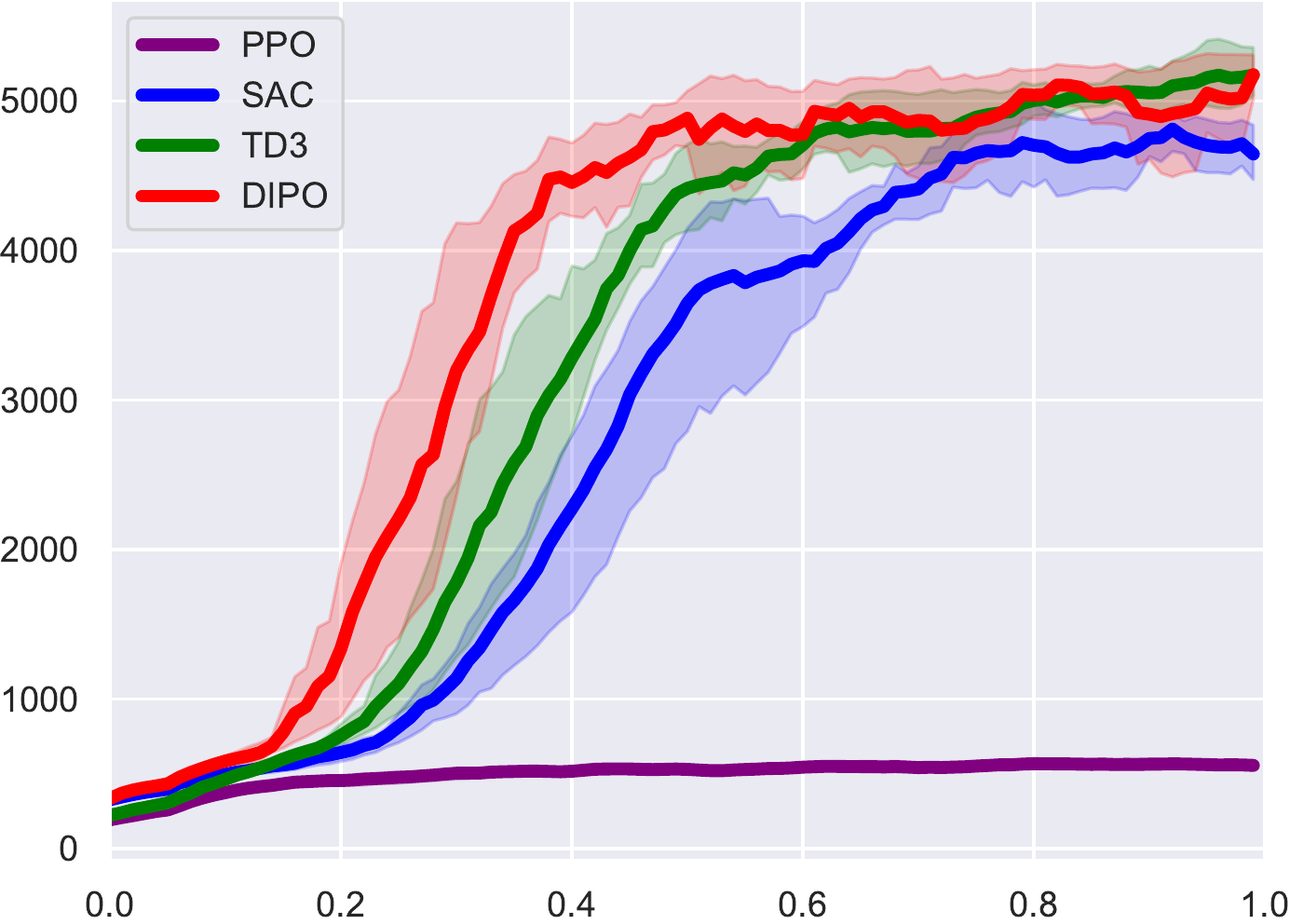}}
         \subfigure[Walker2d-v3]
        {\includegraphics[width=3.0cm,height=2.5cm]{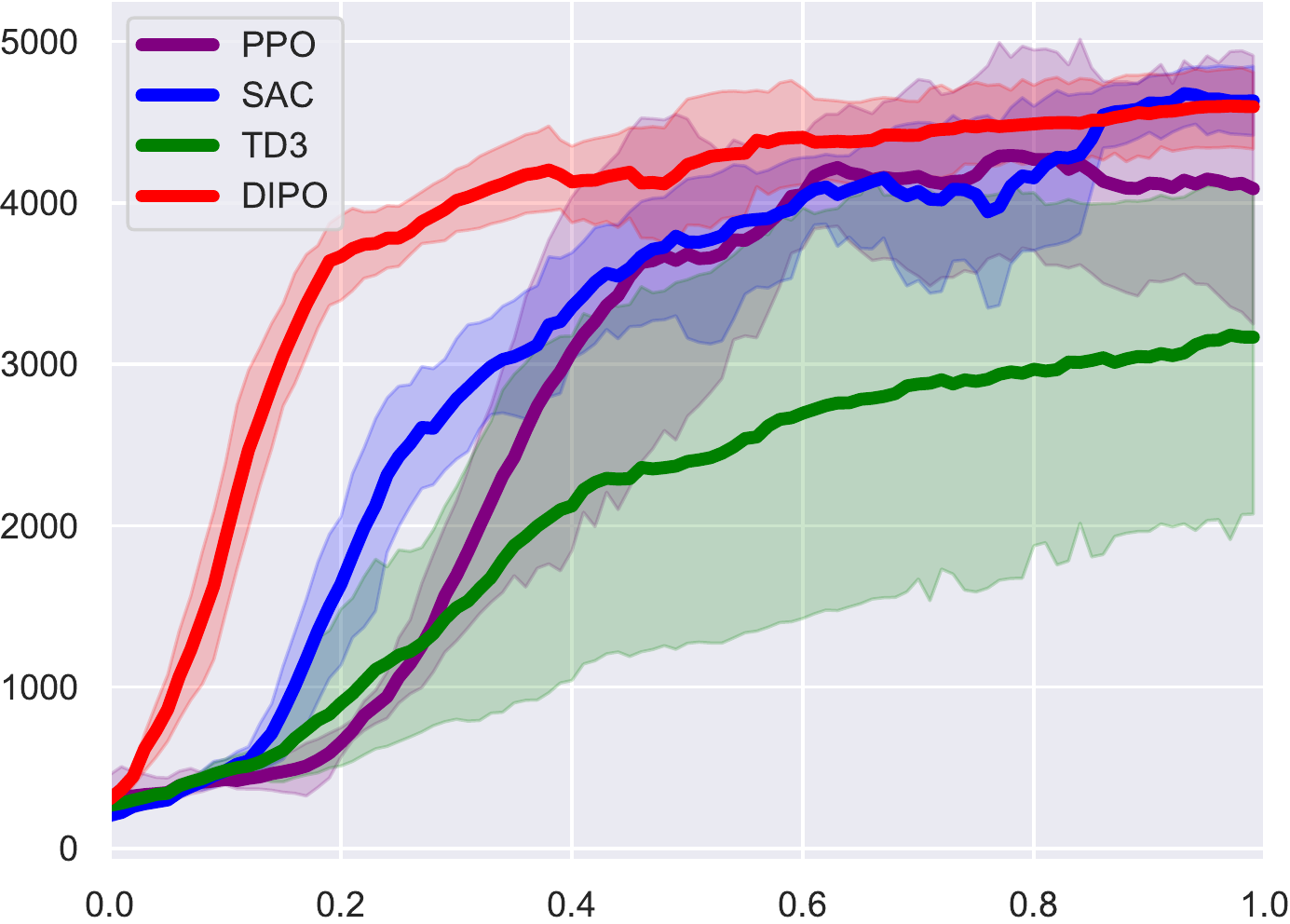}}
              \caption
    { Average performances on MuJoCo Gym environments with $\pm$ std shaded, where the horizontal axis of coordinate denotes the iterations $(\times 10^6)$, the plots smoothed with a window of 10.
    %and the optimal policy is to go to one of the goal positions randomly.
    %, and we have shown the solid red lines as the final action learned by different algorithms with 1000 iterations.
    }
    \label{fig:comparsion-mujoco}
\end{figure*}
\begin{figure*}[t!]
    \centering
   % \subfigure[$10^5$ iterations]
    {\includegraphics[width=1.4cm,height=1.4cm]{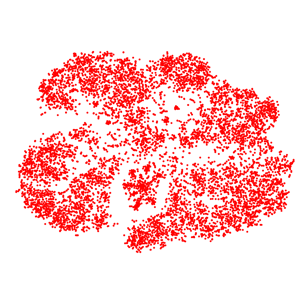}}
     %\subfigure[$2\times10^5$ iterations]
   {\includegraphics[width=1.4cm,height=1.4cm]{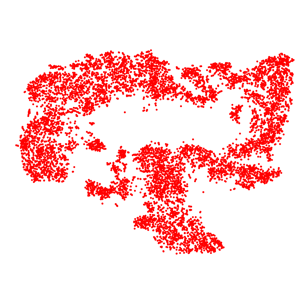}}
     % \subfigure[$3\times10^5$ iterations]
{\includegraphics[width=1.4cm,height=1.4cm]{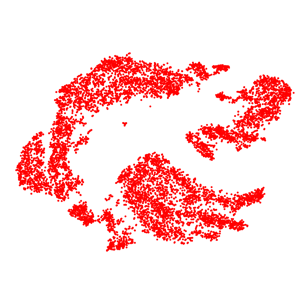}}
     % \subfigure[$4\times10^5$ iterations]
    {\includegraphics[width=1.4cm,height=1.4cm]{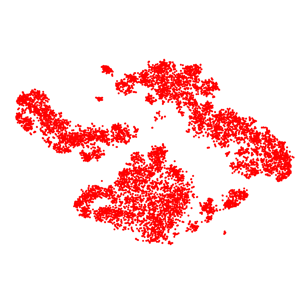}}
     %\subfigure[$5\times10^5$ iterations]
{\includegraphics[width=1.4cm,height=1.4cm]{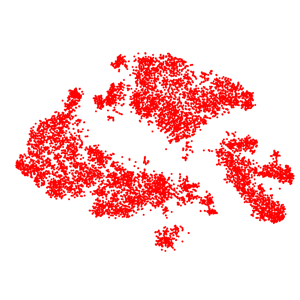}}
   % \subfigure[$10^5$ iterations]
    {\includegraphics[width=1.4cm,height=1.4cm]{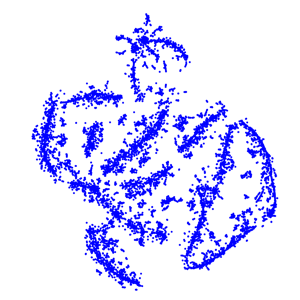}}
    % \subfigure[$2\times10^5$ iterations]
   {\includegraphics[width=1.4cm,height=1.4cm]{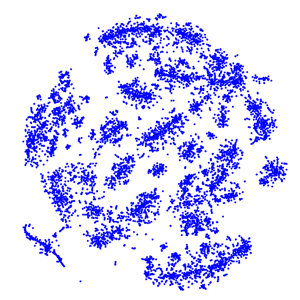}}
    %  \subfigure[$3\times10^5$ iterations]
{\includegraphics[width=1.4cm,height=1.4cm]{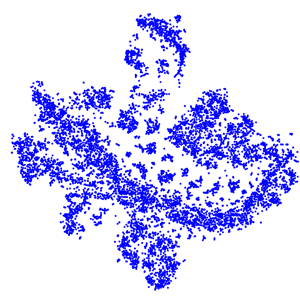}}
     % \subfigure[$4\times10^5$ iterations]
    {\includegraphics[width=1.4cm,height=1.4cm]{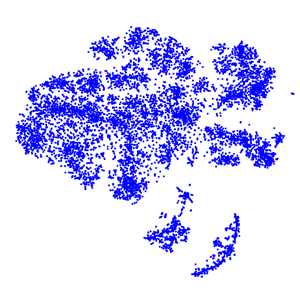}}
   %  \subfigure[$5\times10^5$ iterations]
{\includegraphics[width=1.4cm,height=1.4cm]{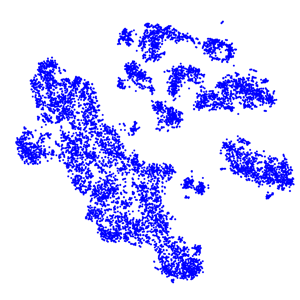}}
 % \subfigure[$10^5$ iterations]
    {\includegraphics[width=1.4cm,height=1.4cm]{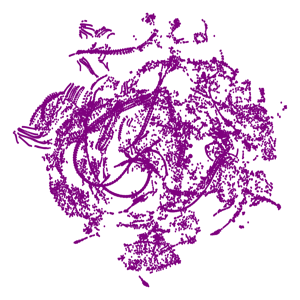}}
%     \subfigure[$2\times10^5$ iterations]
   {\includegraphics[width=1.4cm,height=1.4cm]{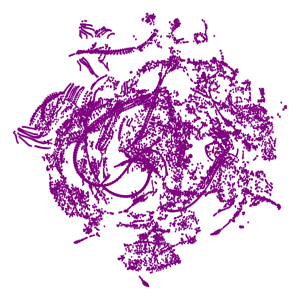}}
%      \subfigure[$3\times10^5$ iterations]
{\includegraphics[width=1.4cm,height=1.4cm]{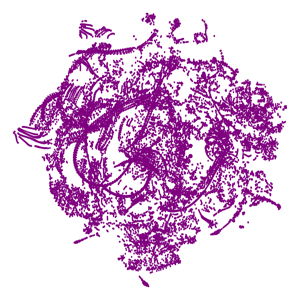}}
%      \subfigure[$4\times10^5$ iterations]
{\includegraphics[width=1.4cm,height=1.4cm]{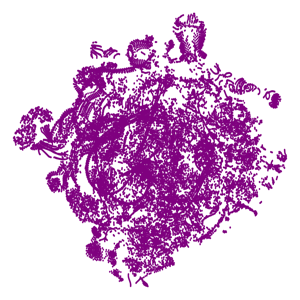}}
%     \subfigure[$5\times10^5$ iterations]
{\includegraphics[width=1.4cm,height=1.4cm]{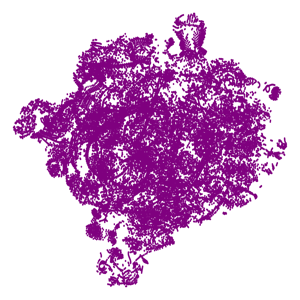}}
% \subfigure[$2$E$5$ iterations]
    {\includegraphics[width=1.4cm,height=1.4cm]{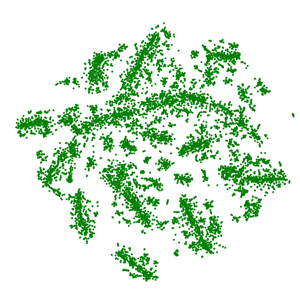}}
    % \subfigure[$4$E$5$  iterations]
   {\includegraphics[width=1.4cm,height=1.4cm]{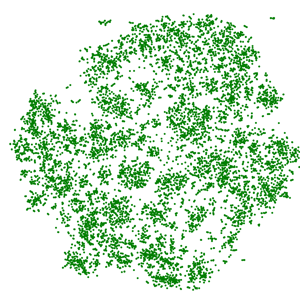}}
      %\subfigure[$6$E$5$  iterations]
{\includegraphics[width=1.4cm,height=1.4cm]{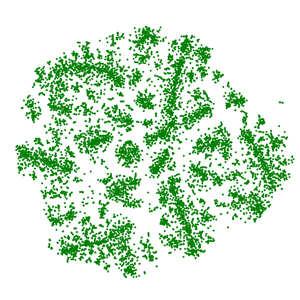}}
     % \subfigure[$8$E$5$  iterations]
    {\includegraphics[width=1.4cm,height=1.4cm]{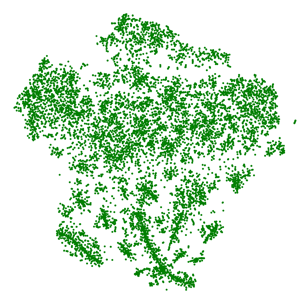}}
     %\subfigure[E6 iterations]
{\includegraphics[width=1.4cm,height=1.4cm]{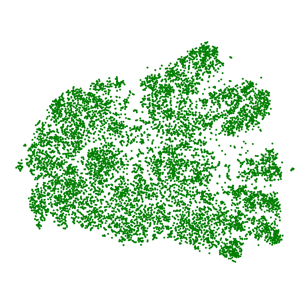}}
{\includegraphics[width=9cm,height=0.38cm]{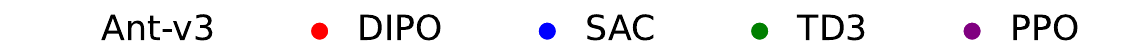}}
   \caption
  {State-visiting visualization by each algorithm on the Ant-v3 task, where states get dimension reduction by t-SNE. The points with different colors represent the states visited by the policy with the style. The distance between points represents the difference between states. }
 \label{state-ant}
\end{figure*}
\begin{figure*}[t!]
    \centering
   % \subfigure[$10^5$ iterations]
    {\includegraphics[width=1.4cm,height=1.4cm]{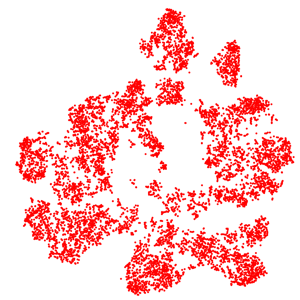}}
     %\subfigure[$2\times10^5$ iterations]
   {\includegraphics[width=1.4cm,height=1.4cm]{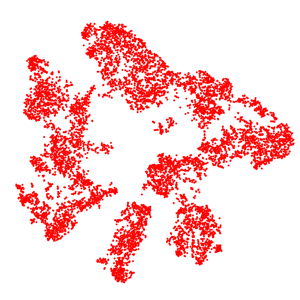}}
     % \subfigure[$3\times10^5$ iterations]
{\includegraphics[width=1.4cm,height=1.4cm]{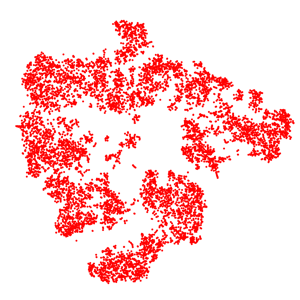}}
     % \subfigure[$4\times10^5$ iterations]
    {\includegraphics[width=1.4cm,height=1.4cm]{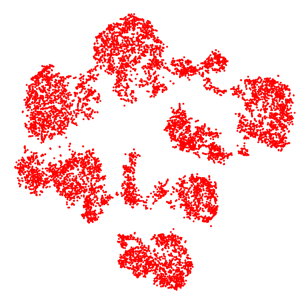}}
     %\subfigure[$5\times10^5$ iterations]
{\includegraphics[width=1.4cm,height=1.4cm]{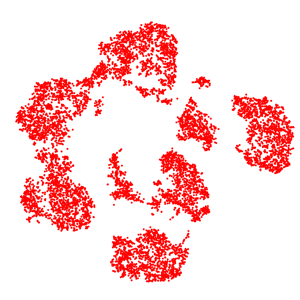}}
   % \subfigure[$10^5$ iterations]
    {\includegraphics[width=1.4cm,height=1.4cm]{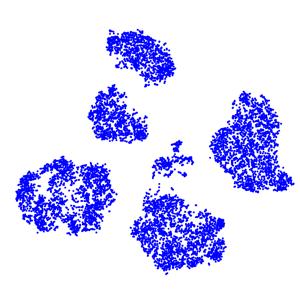}}
    % \subfigure[$2\times10^5$ iterations]
   {\includegraphics[width=1.4cm,height=1.4cm]{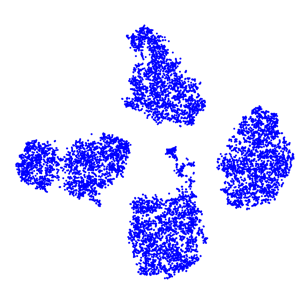}}
    %  \subfigure[$3\times10^5$ iterations]
{\includegraphics[width=1.4cm,height=1.4cm]{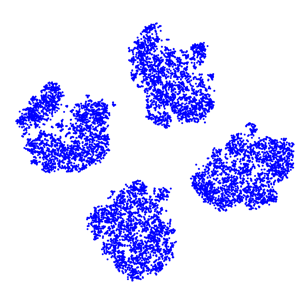}}
     % \subfigure[$4\times10^5$ iterations]
    {\includegraphics[width=1.4cm,height=1.4cm]{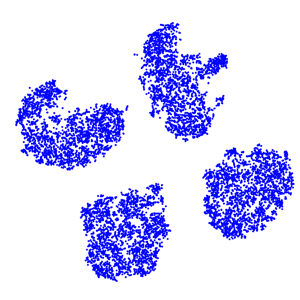}}
   %  \subfigure[$5\times10^5$ iterations]
{\includegraphics[width=1.4cm,height=1.4cm]{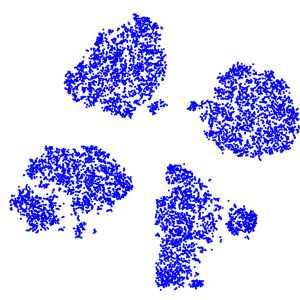}}
    {\includegraphics[width=9cm,height=0.38cm]{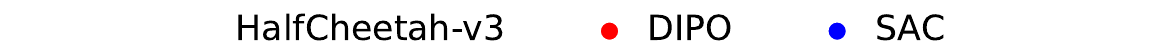}}
    \caption
    {State-visiting visualization for comparison between DIPO and SAC on HalfCheetah-v3.}
     \label{state-HalfCheetah}
         \vspace*{-10pt}
\end{figure*}

\subsection{State-Visiting Visualization}

From Figure \ref{fig:comparsion-mujoco}, we also know that DIPO achieves the best initial reward performance among all the 5 tasks, a more intuitive illustration has been shown in Figure \ref{state-ant} and \ref{state-HalfCheetah}, where we only consider Ant-v3 and HalfCheetah-v3; for more discussions and observations, see Appendix \ref{sec-app-state-visiting}.
We show the state-visiting region to compare both the exploration and final reward performance, where we use the same t-SNE \citep{van2008visualizing} to transfer the high-dimensional states visited by all the methods for 2D visualization.
Results of Figure \ref{state-ant} show that the DIPO explores a wider range of state-visiting, covering TD3, SAC, and PPO.
Furthermore, from  Figure \ref{state-ant}, we also know DIPO achieves a more dense state-visiting at the final period, which is a reasonable result since after sufficient training, the agent identifies and avoids the "bad" states, and plays actions transfer to "good" states.
On the contrary, PPO shows an aimless exploration in the Ant-v3 task, which partially explains why PPO is not so good in the Ant-v3 task.

From Figure \ref{state-HalfCheetah} we know, at the initial time, DIPO covers more regions than SAC in the HalfCheetah-v3, which results in DIPO obtaining a better reward performance than SAC. This result coincides with the results of Figure \ref{fig:comparsion-multi-goal}, which demonstrates that DIPO is efficient for exploration, which leads DIPO to better reward performance. While we also know that SAC starts with a narrow state visit that is similar to the final state visit, and SAC performs with the same reward performance with DIPO at the final, which implies SAC runs around the "good" region at the beginning although SAC performs a relatively worse initial reward performance than DIPO.
Thus, the result of Figure \ref{state-HalfCheetah} partially explains why DIPO performs better than SAC at the initial iterations but performs with same performance with SAC at the final for the HalfCheetah-v3 task.

\begin{figure*}[t!]
    \centering
   \subfigure[Ant-v3]
    {\includegraphics[width=3cm,height=3cm]{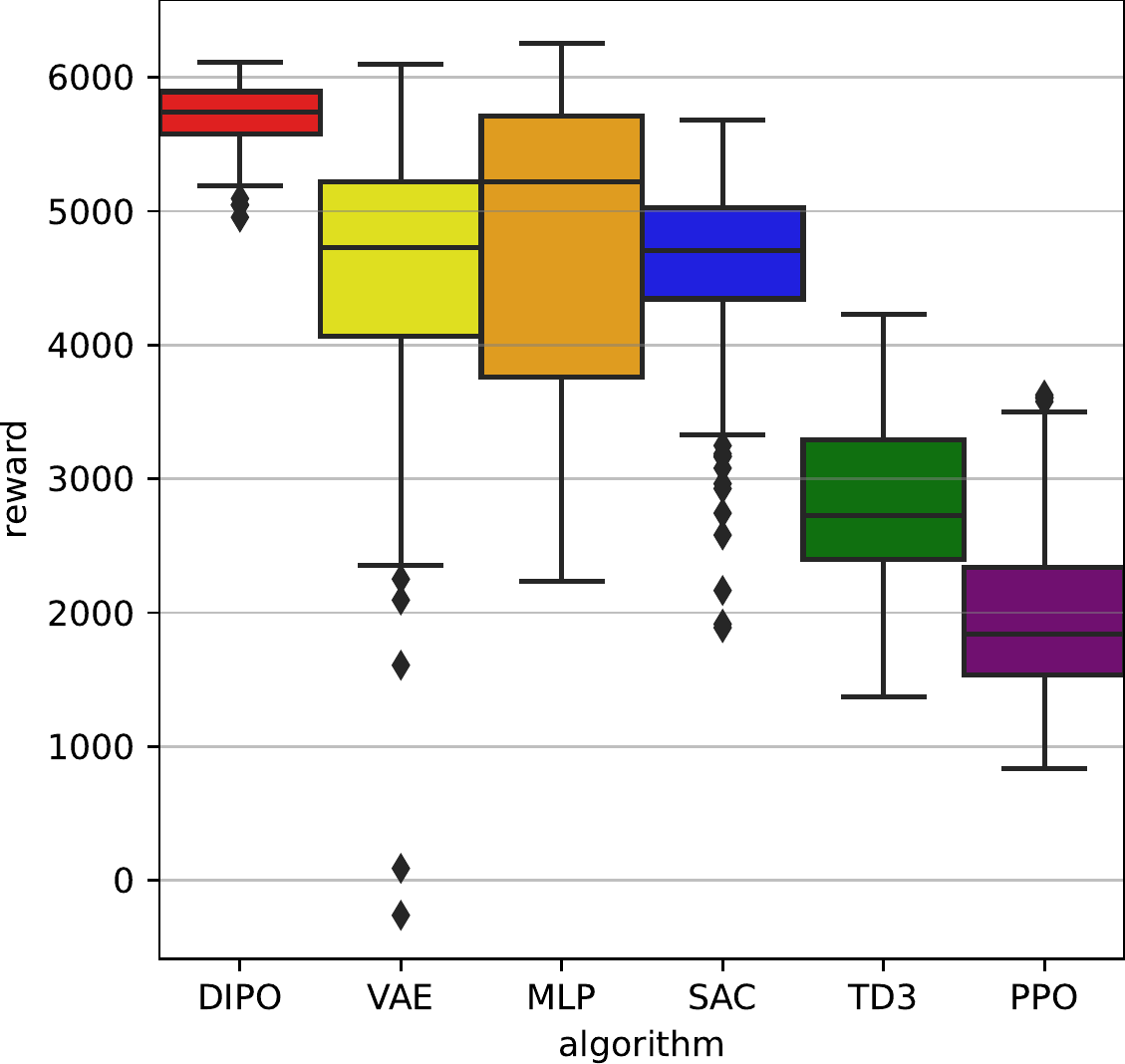}}
    \subfigure[HalfCheetah-v3]
   {\includegraphics[width=3cm,height=3cm]{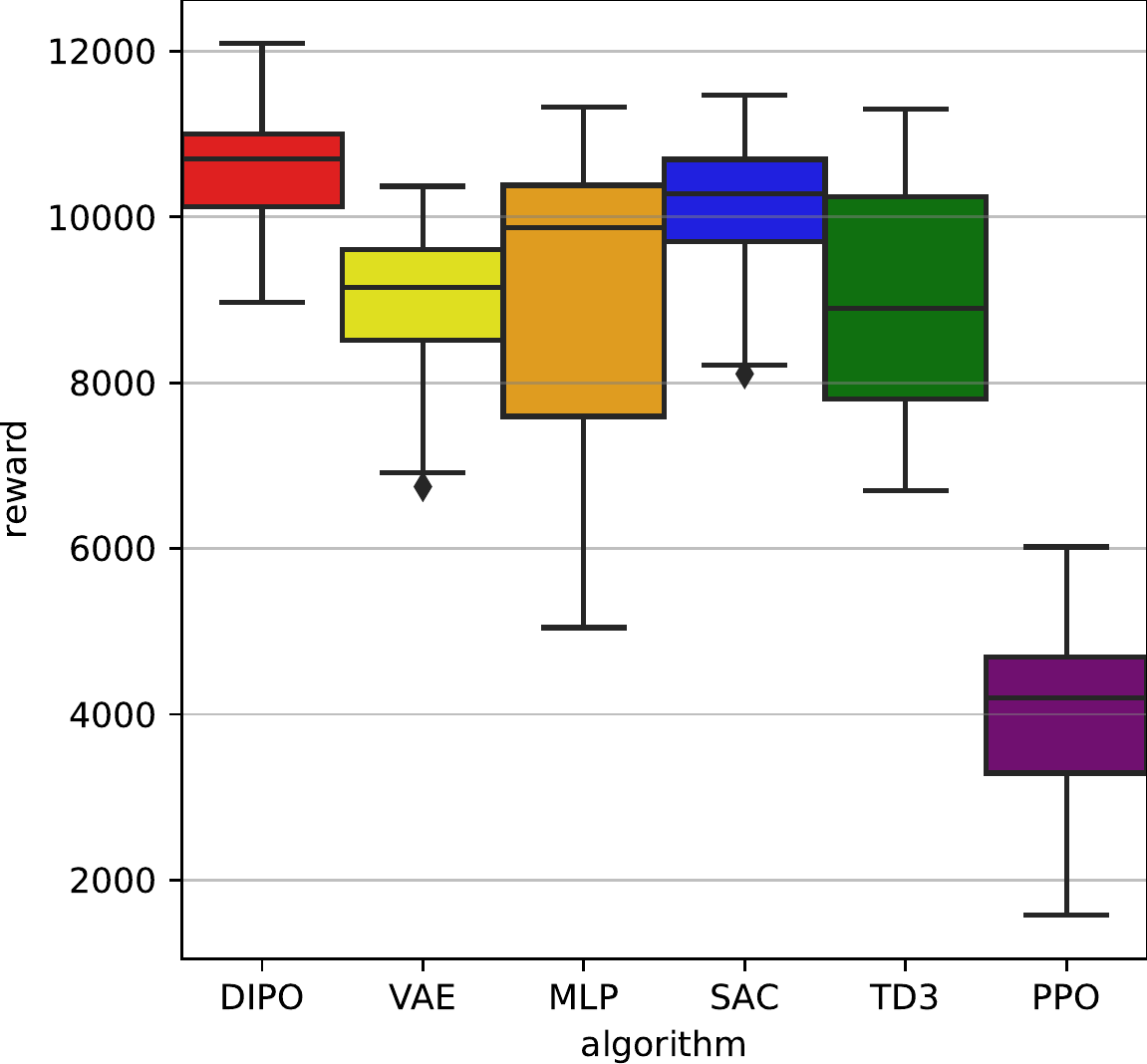}}
     \subfigure[Hopper-v3]
{\includegraphics[width=3cm,height=3cm]{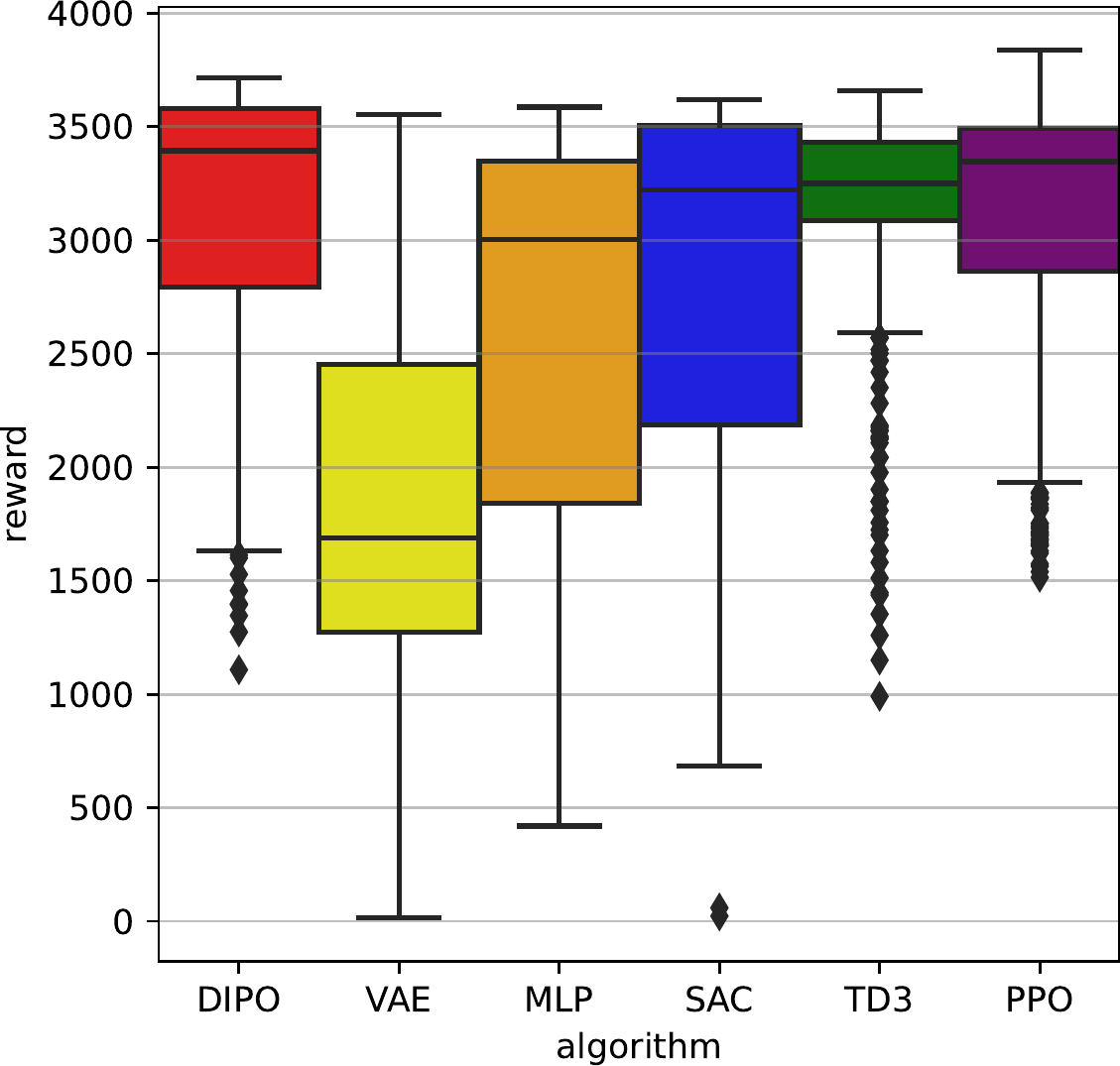}}
      \subfigure[Humanoid-v3]
    {\includegraphics[width=3cm,height=3cm]{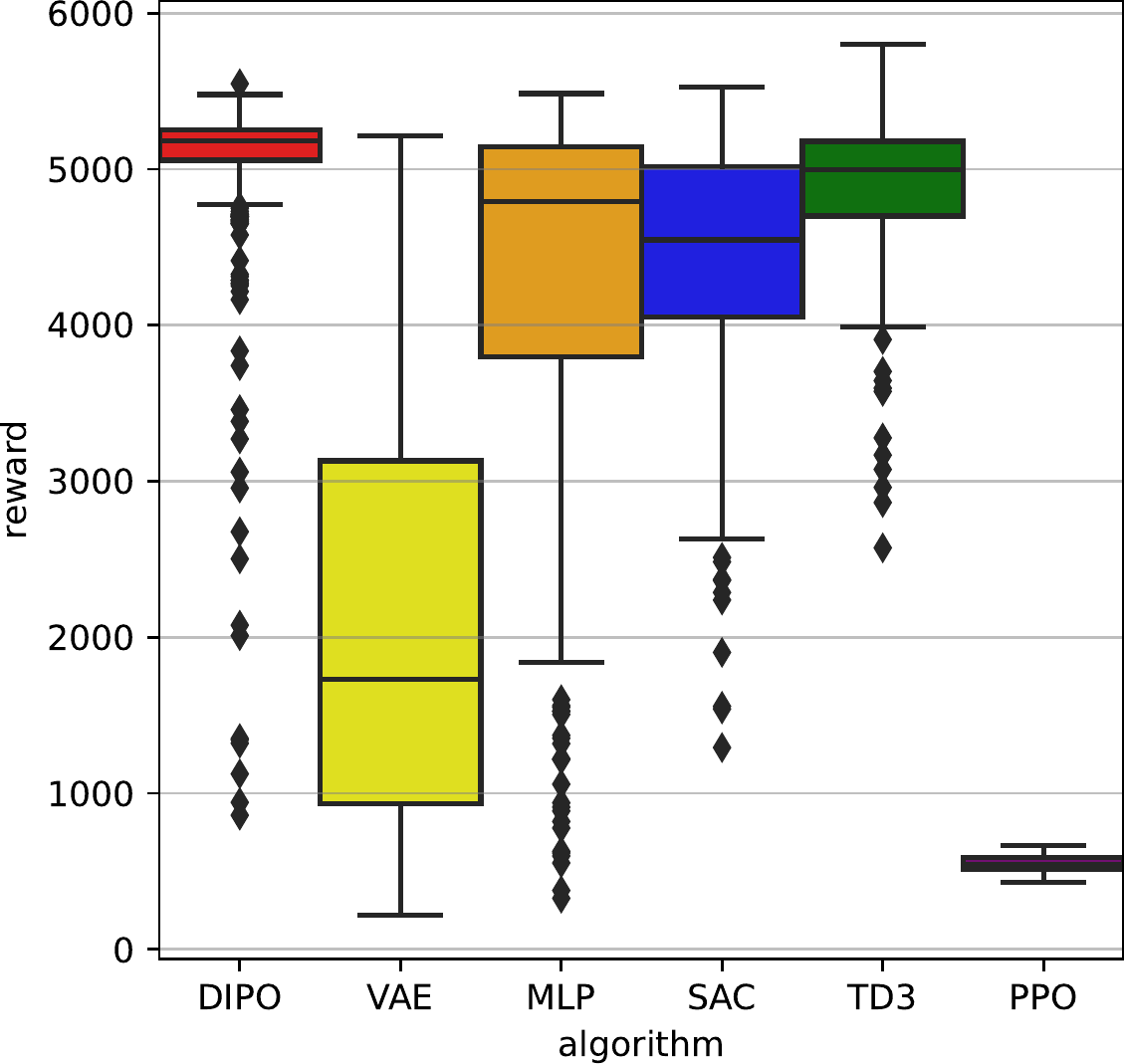}}
     \subfigure[Walker2d-v3]
{\includegraphics[width=3cm,height=3cm]{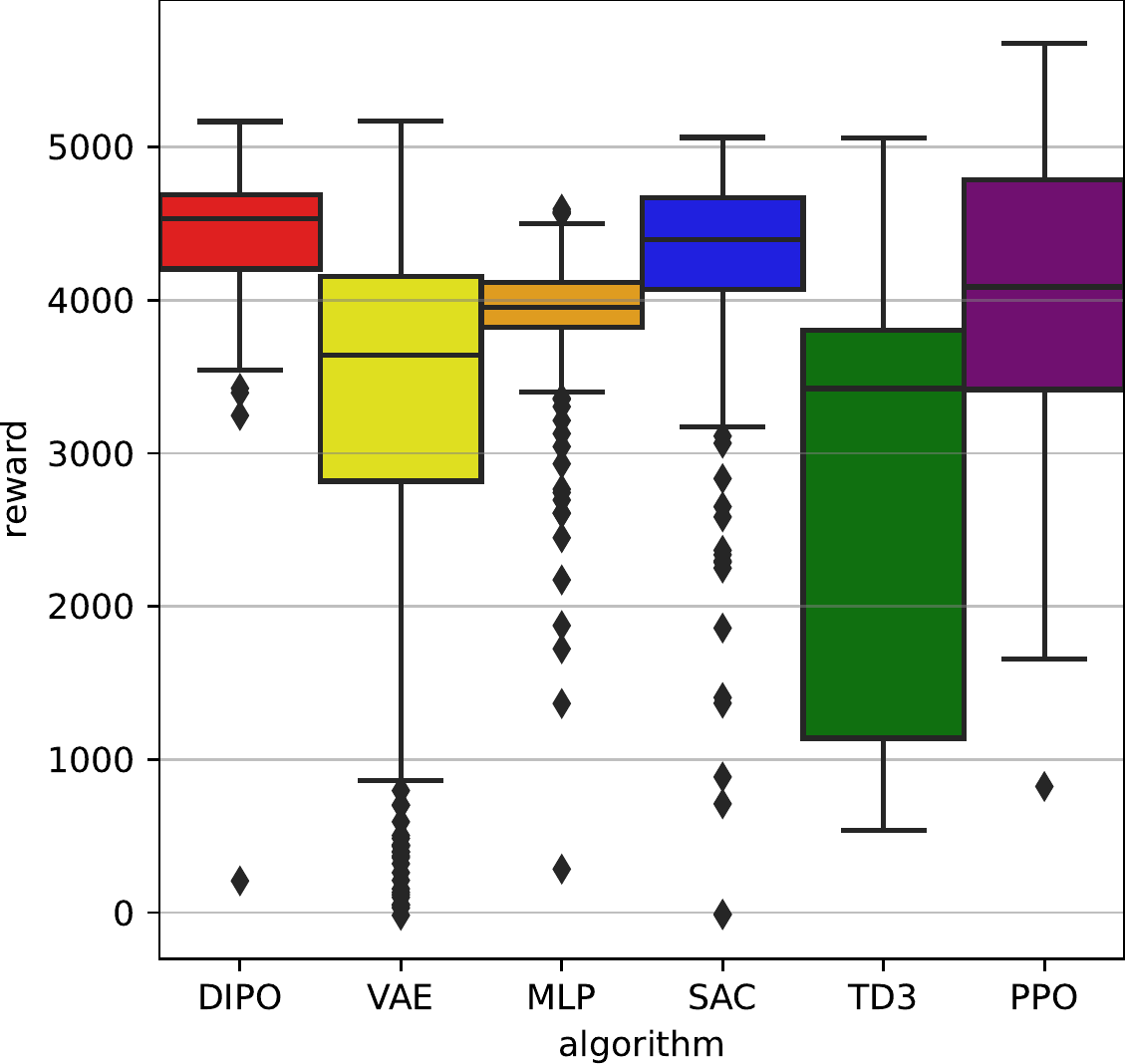}}
    \caption
    {Reward Performance Comparison to VAE and MLP with DIPO, SAC, PPO and TD3.}
    \label{dipo-vae-mlp}
    \vspace*{-10pt}
\end{figure*}
\begin{figure*}[t!]
    \centering
   \subfigure[Ant-v3]
    {\includegraphics[width=3.0cm,height=2.5cm]{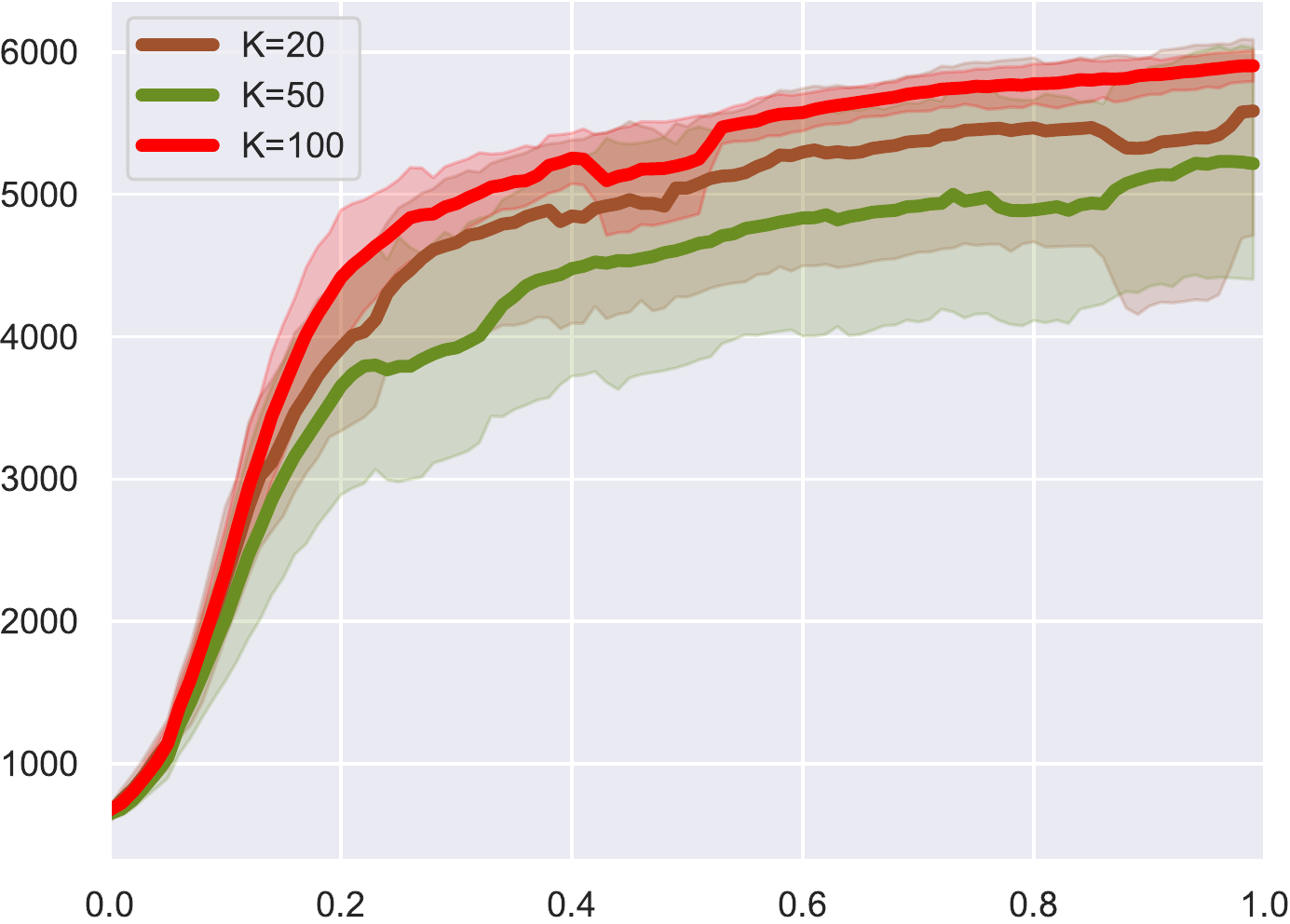}}
    \subfigure[HalfCheetah-v3]
   {\includegraphics[width=3.0cm,height=2.5cm]{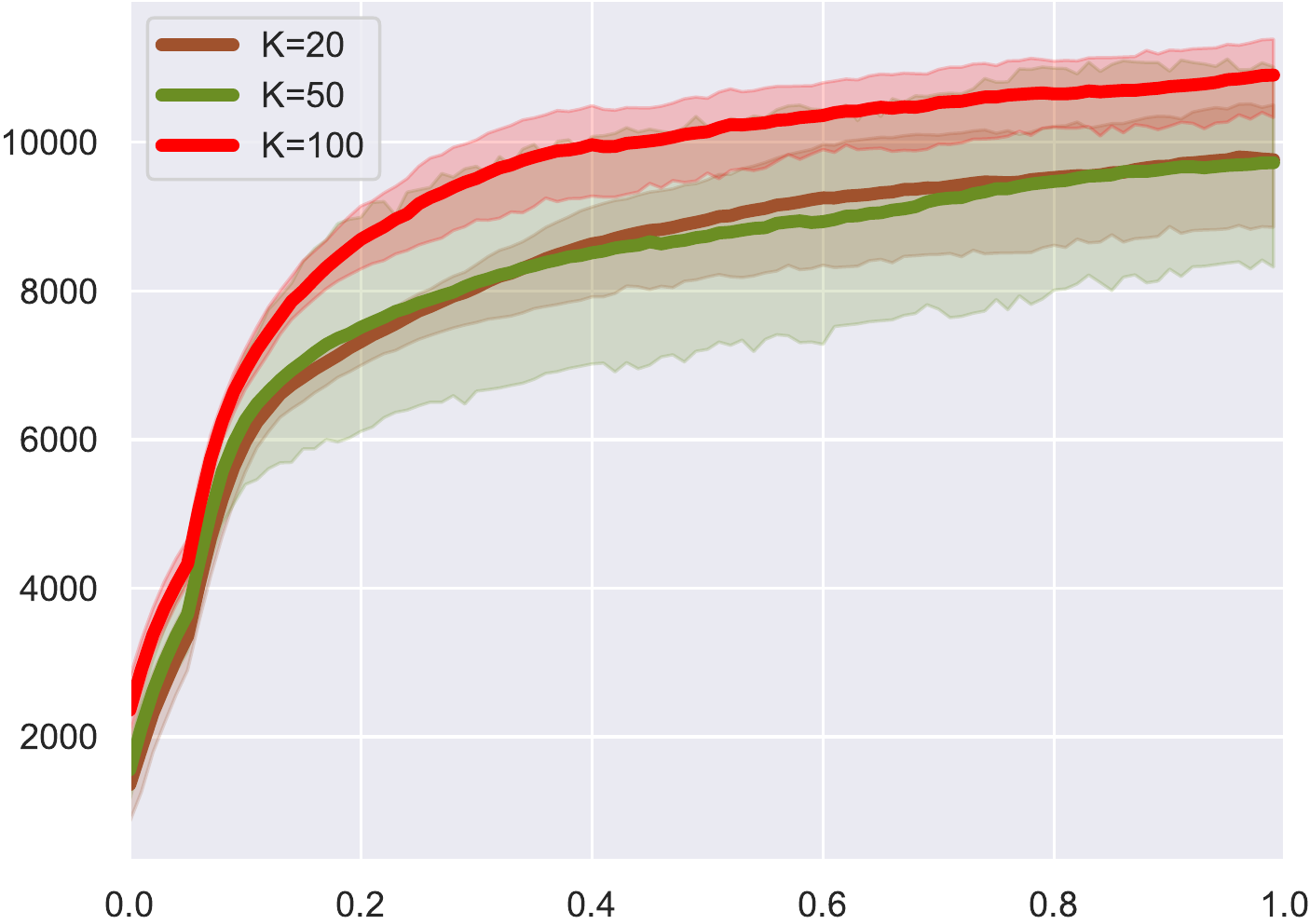}}
     \subfigure[Hopper-v3]
{\includegraphics[width=3.0cm,height=2.5cm]{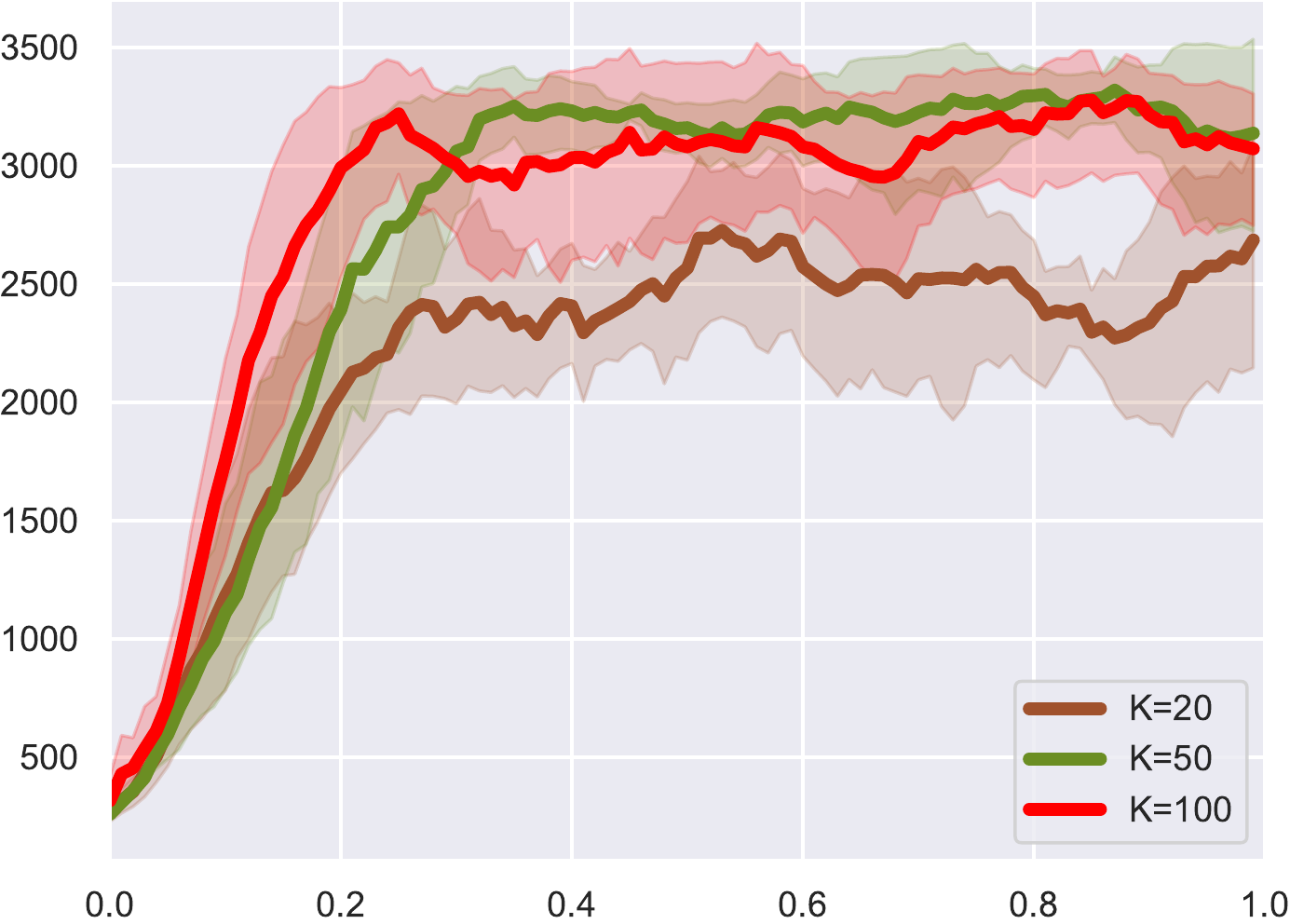}}
      \subfigure[Humanoid-v3]
    {\includegraphics[width=3.0cm,height=2.5cm]{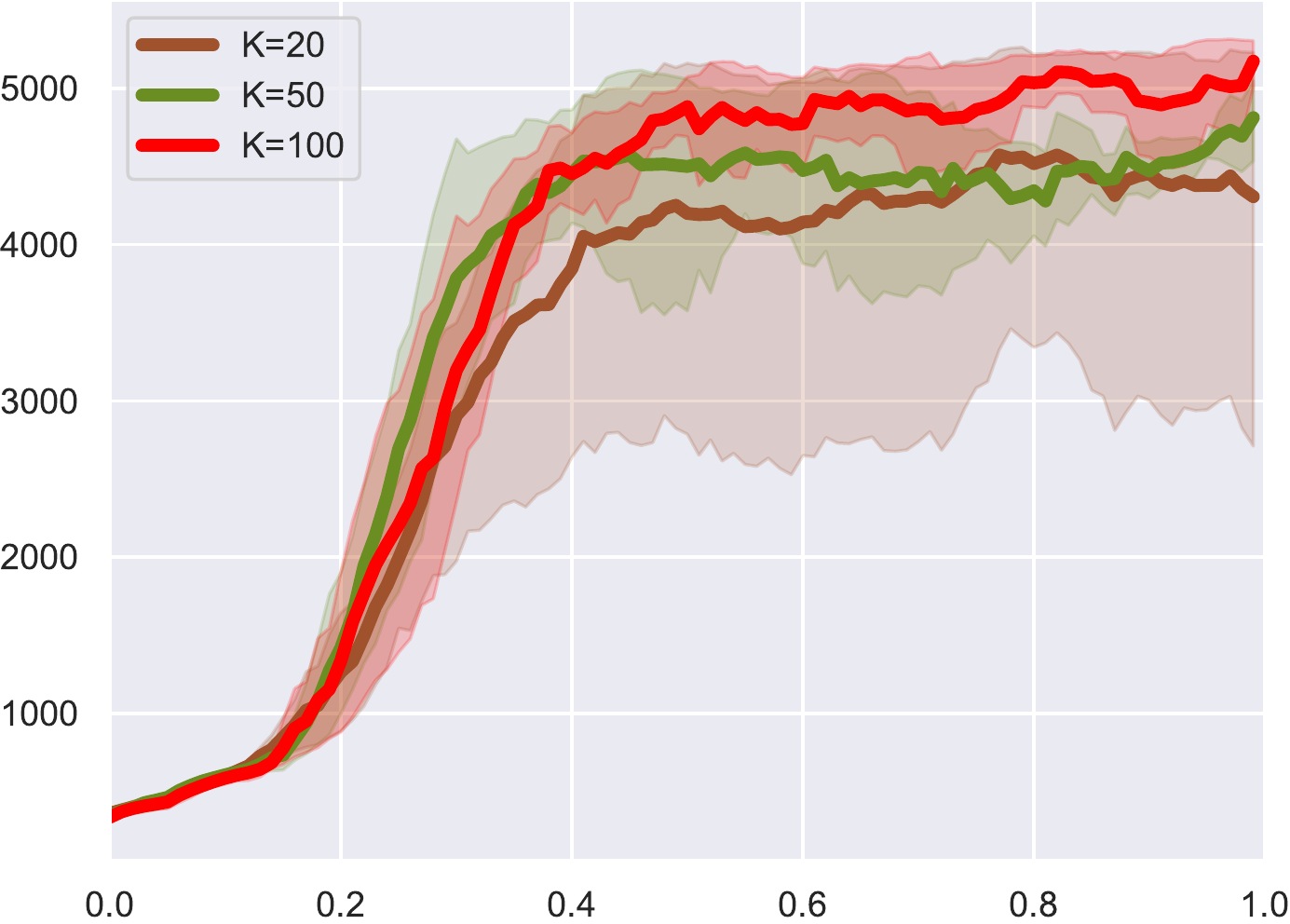}}
     \subfigure[Walker2d-v3]
{\includegraphics[width=3.0cm,height=2.5cm]{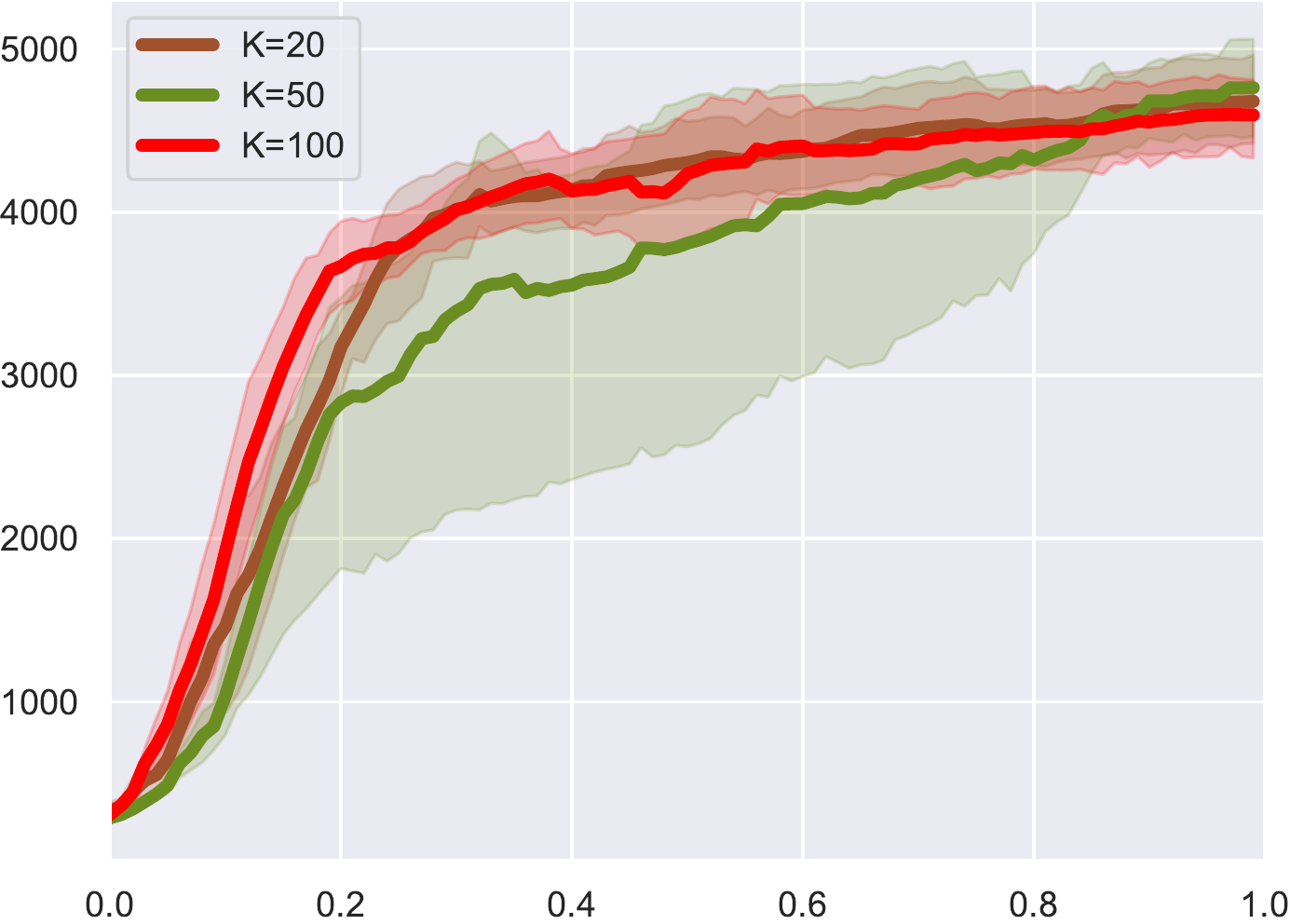}}
    \caption
    {Learning curves with different reverse lengths $K=20,50,100$.
    }
    \label{dipo-k-ablation}
    \vspace*{-10pt}
\end{figure*}
\begin{table*}[t!]
 \centering
 % \begin{footnotesize}
 \vskip 0.1in
 \begin{adjustbox}{width=1\textwidth}
 \begin{tabular}{|c|c|c|c|c|c|c|c}
  \hline 
Reverse length         & Ant-v3 & HalfCheetah-v3 & Hopper-v3  &Humanoid-v3 &Walker2d-v3  \\
  \hline 
 $K=100$           & $\textbf{5622.30}\pm\textbf{487.09}$ & $\textbf{10472.31}\pm\textbf{654.96}$ & $3123.14\pm636.23$ & $\textbf{4878.41}\pm\textbf{822.03}$ & $\textbf{4409.18}\pm\textbf{469.06}$  \\
  \hline 
 $K=50$         & $4877.41\pm1010.35$     & $9198.20\pm1738.25$ & $\textbf{3214.83}\pm\textbf{491.15}$ &$4513.39\pm1075.94$ & $4199.34\pm1062.31$ \\
 \hline 
 $K=20$          & $5288.77\pm970.35$ & $9343.69\pm 986.82$ & $2511.63\pm 837.03$& $4294.79\pm1583.48$ & $\textbf{4467.20}\pm\textbf{368.13}$ \\
  \hline 
 \end{tabular}
 \end{adjustbox}
  
%  \caption{Bootstrap mean with 100 bootstrap samples of reward/cost return after training on robot with speed limit environments. Cost thresholds are in brackets under the environment names.}
  \caption{Average return over final 6E5 iterations with different reverse lengths $K=20,50,100$, and maximum value is bolded for each task.}
 \label{tab:mujoco-k-reverselength}
 % \end{footnotesize}
     \vspace*{-10pt}
\end{table*}

\subsection{Ablation Study}

In this section, we consider the ablation study to compare the diffusion model with VAE and MLP for policy learning, and show a trade-off on the reverse length $K$ for reverse inference.

\subsubsection{Comparison to VAE and MLP}
Both VAE and MLP are widely used to learn distribution in machine learning, a fundamental question is: 
why must we consider the diffusion model to learn a policy distribution? what the reward performance is if we use VAE and MLP to model a policy distribution?
We show the answer in Figure \ref{dipo-vae-mlp}, where the VAE (or MLP) is the result we replace the diffusion policy of DIPO (see Figure \ref{fig:framework}) with VAE (or MLP), i.e.,
we consider VAE (or MLP)+action gradient for the tasks. 
Results show that the diffusion model is more powerful than VAE and MLP for learning a distribution.
This implies the diffusion model is an expressive and flexible family to model a distribution, which is also consistent with the field of the generative model.

\subsubsection{Comparison with Different Reverse Lengths} Reverse length $K$ is an important parameter for the diffusion model, which not only affects the reward performance but also affects the training time, we show the results in Figure \ref{dipo-k-ablation} and Table \ref{tab:mujoco-k-reverselength}.
The results show that the reverse time $K=100$ returns a better reward performance than other cases (except Hopper-v3 task).
Longer reverse length consumes more reverse time for inference, we hope to use less time for reverse time for action inference. 
However, a short reverse length $K=20$ decays the reward performance among (except Walker2d-v3 task), which implies a trade-off between reward performance and reverse length $K$. In practice, we set $K=100$ throughout this paper.

\section{Conlusion}

We have formally built a theoretical foundation of diffusion policy, which shows a policy representation via the diffusion probability model and which is a new way to represent a policy via a stochastic process. Then, we have shown a convergence analysis for diffusion policy, which provides a theory to understand diffusion policy.
Furthermore, we have proposed an implementation for model-free online RL with a diffusion policy, named DIPO. 
Finally, extensive empirical results show the effectiveness of DIPO among the Mujoco tasks.

%%% References
\bibliographystyle{plainnat}
\bibliography{ref}

\clearpage

\appendix

\section{Review on Notations}

This section reviews some notations and integration by parts formula, using the notations consistent with \citep{vempala2019rapid}.

Given a smooth function $f:\R^{n} \rightarrow\R$, its \textbf{gradient} $\bnabla f:\R^{n} \rightarrow\R^{n}$ is the vector of partial derivatives:
\[
\bnabla f(\bx) = \left(\part{f(\bx)}{x_1}, \dots, \part{f(\bx)}{x_n} \right)^{\top}.
\]
The {\bf Hessian} $\bnabla^2 f \colon \R^n \to \R^{n \times n}$ is the matrix of second partial derivatives:
\[
\bnabla^2 f(\bx) = \left(\part{^2f(\bx)}{x_i x_j} \right)_{1 \le i,j \le n}.
\]
The {\bf Laplacian} $\Delta f \colon \R^n \to \R$ is the trace of its Hessian:
\[
\Delta f(\bx) = \mathrm{Tr}\left(\bnabla^2 f(\bx)\right) = \sum_{i=1}^n \part{^2 f(\bx)}{x_i^2}.
\]
Given a smooth vector field $\bp = (p_1,\dots,p_n) \colon \R^n \to \R^n$, its {\bf divergence} is $\DIV \cdot \bp \colon \R^n \to \R$:
\begin{flalign}
\label{app-eq-49}
(\DIV \cdot \bp)(\bx) = \sum_{i=1}^n \part{p_i(\bx)}{x_i}.
\end{flalign}
When the variable of a function is clear and without causing ambiguity, we also denote the above notation as follows,
\begin{flalign}
\label{app-eq-50}
\DIV \cdot (\bp(\bx)) =:
(\DIV \cdot \bp)(\bx) = \sum_{i=1}^n \part{p_i(\bx)}{x_i}.
\end{flalign}
In particular, the divergence of the gradient is the Laplacian:
\begin{flalign}
\label{laplacian-operator}
(\DIV \cdot \bnabla f)(\bx) = \sum_{i=1}^n \part{^2 f(\bx)}{x_i^2} = \Delta f(\bx).
\end{flalign}
Let $\rho(\bx)$ and $\mu(\bx)$ be two smooth probability density functions on the space $\R^{p}$, 
the Kullback–Leibler (KL) divergence and relative Fisher information (FI) from $\mu(\bx)$ to $\rho(\bx)$ are defined as follows,
\begin{flalign}
\KL(\rho\|\mu)=\bigintsss_{\R^{p}}\rho(\bx)\log\dfrac{\rho(\bx)}{\mu(\bx)}\dd \bx, ~\FI(\rho\|\mu)=\bigintsss_{\R^{p}}\rho(\bx)\left\|\bnabla\log\dfrac{\rho(\bx)}{\mu(\bx)}\right\|_{2}^{2}\dd \bx.
\end{flalign}

Before we further analyze, we need the integration by parts formula.
For any function $f \colon \R^p \to \R$ and vector field $\bm{v} \colon \R^p \to \R^p$ with sufficiently fast decay at infinity,
we have the following {\bf integration by parts} formula:
\begin{flalign}
\label{integration-by-parts-formula}
\bigintsss_{\R^p} \langle \bm{v}(\bx), \nabla f(\bx) \rangle \dd\bx = -\bigintsss_{\R^p} f(\bx) (\DIV \cdot \bm{v})(\bx) \dd\bx.
\end{flalign}

\section{Auxiliary Results}

\subsection{Diffusion Probability Model (DPM).} 

This section reviews some basic background about the diffusion probability model (DPM).
For a given (but unknown) $p$-dimensional data distribution $q(\bx_0)$, 
DPM \citep{sohl2015deep,ho2020denoising,song2020score} is a latent variable generative model that learns a parametric model to approximate the distribution $q(\bx_0)$.
To simplify the presentation in this section, we only focus on the continuous-time diffusion \citep{song2020score}. 
The mechanism of DPM contains two processes, \emph{forward process} and \emph{reverse process}; we present them as follows.

\textbf{Forward Process}. The forward process produces a sequence $\{\bx_{t}\}_{t=0:T}$ that perturbs the initial $\bx_{0}\sim q(\cdot)$ into a Gaussian noise, which follows the next stochastic differential equation (SDE),
\begin{flalign}
\label{def:sde-forward-process}
\dd \bx_{t}=\bbf(\bx_{t},t)\dd t+g(t)\dd \bw_{t},
\end{flalign}
where $\bbf(\cdot,\cdot)$ is the drift term, $g(\cdot)$ is the diffusion term, and $\bw_{t}$ is the standard Wiener process.

\textbf{Reverse Process.} According to \cite{anderson1982reverse,haussmann1986time}, there exists a corresponding reverse SDE that exactly coincides with the solution of the forward SDE (\ref{def:sde-forward-process}):
\begin{flalign}
\label{def:sde-reverse-process}
\dd \bx_{t}=\left[\bbf(\bx_{t},t)-g^{2}(t)\bnabla \log p_{t}(\bx_{t})\right]\dd \bar{t}+g(t)\dd \bar{\bw}_{t},
\end{flalign}
where $\dd\bar{t}$ is the backward time differential, $\dd\bar{\bw}_{t}$ is a standard Wiener process flowing backward in time, and $p_{t}(\bx_{t})$is the marginal probability distribution of the random variable $\bx_{t}$ at time $t$.

Once the \emph{score function} $\bnabla \log p_{t}(\bx_{t})$ is known for each time $t$, we can derive the reverse diffusion process from SDE (\ref{def:sde-reverse-process}) and simulate it to sample from $q(\bx_{0})$ \citep{song2020score}.

\subsection{Transition Probability for Ornstein-Uhlenbeck Process}
\label{app-sec-transition-probability4ou}

\begin{proposition}
\label{solution4linearsdes}
Consider the next SDEs,
\[
\dd \bx_{t}=-\dfrac{1}{2}\beta(t)\bx_{t}\dd t+g(t)\dd \bw_{t},
\]
where $\beta(\cdot)$ and $g(\cdot)$ are real-valued functions. Then, for a given $\bx_{0}$, the conditional distribution of $\bx_{t}|\bx_{0}$ is a Gaussian distribution, i.e.,
\[
\bx_{t}|\bx_{0}\sim\calN\left(\bx_{t}|\bm{\mu}_t,\bm{\Sigma}_{t}\right),
\]
where
\begin{flalign}
\nonumber
\bm{\mu}_{t}=&\exp\left\{-\dfrac{1}{2}\int_{0}^{t}\beta(s)\dd s\right\}\bx_{0},\\
\nonumber
\bm{\Sigma}_{t}=&\exp\left\{-\int_{0}^{t}\beta(s)\dd s\right\}\left(\int_{0}^{t}\exp\left\{\int_{0}^{\tau}\beta(s)\dd s\right\}g^{2}(\tau)\dd \tau\right)\bI.
\end{flalign}
\end{proposition}
\begin{proof}
See \citep[A.1]{kim2022soft-truncation} or \citep[Chapter 6.1]{sarkka2019applied}.
\end{proof}

\subsection{Exponential Integrator Discretization}
\label{app-sec:exponential-iIntegrator}

The next Proposition \ref{solution4-exponential-iIntegrator} provides a fundamental way for us to derivate the exponential integrator discretization (\ref{iteration-exponential-integrator-discretization}). 

\begin{proposition}
\label{solution4-exponential-iIntegrator}
For a given state $\bs$, we consider the following continuous time process, for $t\in[t_{k},t_{k+1}]$,
 \begin{flalign}
\dd \bx_t=\left(\bx_t+2  \bm{\hat \mathbf{S}}(\bx_{t_k},\bs,T-t_{k})\right)\dd t+\sqrt{2}\dd \bw_{t}.
\end{flalign}
Then with It\^{o} integration \citep{chung1990introduction},
\[
\bx_{t}-\bx_{t_{k}}=\left(\mathrm{e}^{t-t_{k}}-1\right)\left(\bx_{t_{k}}+2  \bm{\hat \mathbf{S}}(\bx_{t_k},\bs,T-t_{k})\right)+\sqrt{2}\int_{t_k}^{t}\mathrm{e}^{t^{'}-t_{k}}\dd \bw_{t^{'}},
\]
where $t\in[t_{k},t_{k+1}]$, and $t_{k}=hk$.
\end{proposition}
\begin{proof}
See \citep[Chapter 6.1]{sarkka2019applied}.
\end{proof}
Recall SDE (\ref{def:diffusion-policy-sde-reverse-process-s}), according to Proposition \ref{solution4-exponential-iIntegrator}, we know the next SDE  
\begin{flalign}
\dd \hata_t=\left(\hata_t+2 \bm{\hat \mathbf{S}}(\hata_{t_k},\bs,T-t_{k})\right)\dd t+\sqrt{2}\dd \bw_{t}, ~t\in[t_{k},t_{k+1}]
\end{flalign}
formulates the exponential integrator discretization as follows,
\begin{flalign}
\label{app-eq-62}
\hata_{t_{k+1}}-\hata_{t_{k}}=&\left(\mathrm{e}^{t_{k+1}-t_{k}}-1\right)\left(\hata_{t_{k}}+2 \bm{\hat \mathbf{S}}(\hata_{t_k},\bs,T-t_{k})\right)+\sqrt{2}\int_{t_k}^{t_{k+1}}\mathrm{e}^{t^{'}-t_{k}}\dd \bw_{t^{'}}\\
 \label{app-eq-34}
=&\left(\mathrm{e}^{h}-1\right)\left(\hata_{t_{k}}+2 \bm{\hat \mathbf{S}}(\hata_{t_k},\bs,T-t_{k})\right)+\sqrt{\mathrm{e}^{2h}-1}\bz,
\end{flalign}
where last equation holds due to Wiener process is a stationary process with independent increments, i.e., the following holds,
\begin{flalign}
\nonumber
\sqrt{2}\int_{t_k}^{t_{k+1}}\mathrm{e}^{t^{'}-t_{k}}\dd \bw_{t^{'}}
=&\sqrt{2}\int_{0}^{t_{k+1}}\mathrm{e}^{t^{'}-t_{k}}\dd \bw_{t^{'}}-\sqrt{2}\int_{0}^{t_{k}}\mathrm{e}^{t^{'}-t_{k}}\dd \bw_{t^{'}}\\
\nonumber
=&\sqrt{\mathrm{e}^{2h}-1}\bz,
\end{flalign}
where $\bz\sim\calN(\bm{0},\bI)$.

Then, we rewrite (\ref{app-eq-34}) as follows,
\begin{flalign}
\nonumber
\hata_{t_{k+1}}=&\mathrm{e}^{h}\hata_{t_{k}}+\left(\mathrm{e}^{h}-1\right)\left(\hata_{t_{k}}+2 \bm{\hat \mathbf{S}}(\hata_{t_k},\bs,T-t_{k})\right)+\sqrt{\mathrm{e}^{2h}-1}\bz,
\end{flalign}
which concludes the iteration defined in (\ref{iteration-exponential-integrator-discretization}).

\subsection{Fokker–Planck Equation}

\label{app-fokker–planck-equation}

The Fokker–Planck equation is named after Adriaan Fokker and Max Planck, who described it in 1914 and 1917 \citep{fokker1914mittlere,planck1917satz}.
 It is also known as the Kolmogorov forward equation, after Andrey Kolmogorov, who independently discovered it in 1931 \citep{kolmogoroff1931analytischen}.
 For more history and background about Fokker–Planck equation, please refer to \citep{risken1996fokker} or \url{https://en.wikipedia.org/wiki/Fokker\%E2\%80\%93Planck_equation#cite_note-1}.

For an It\^{o} process driven by the standard Wiener process 
$
\bw_{t}$ and described by the stochastic differential equation
\begin{flalign}
{\displaystyle \dd\mathbf {x} _{t}={\boldsymbol {\mu }}(\mathbf {x} _{t},t)\dd t+{\boldsymbol {\Sigma }}(\mathbf {x} _{t},t)\dd\mathbf {w} _{t},}
\end{flalign}
where $\mathbf {x} _{t}$ and ${\boldsymbol {\mu }}(\mathbf {x} _{t},t)$ are $N$-dimensional random vectors, 
${\boldsymbol {\Sigma }}(\mathbf {x} _{t},t)$ is an ${\displaystyle n\times m}$ matrix and 
$\mathbf {w} _{t}$ is an $m$-dimensional standard Wiener process, the probability density 
$p(\mathbf {x} ,t)$ for $\mathbf {x} _{t}$ satisfies the Fokker–Planck equation
\begin{flalign}
{\displaystyle {\frac {\partial p(\mathbf {x} ,t)}{\partial t}}=-\sum _{i=1}^{N}{\frac {\partial }{\partial x_{i}}}\left[\mu _{i}(\mathbf {x} ,t)p(\mathbf {x} ,t)\right]+\sum _{i=1}^{N}\sum _{j=1}^{N}{\frac {\partial ^{2}}{\partial x_{i}\,\partial x_{j}}}\left[D_{ij}(\mathbf {x} ,t)p(\mathbf {x} ,t)\right],}
\end{flalign}
with drift vector 
$
{\boldsymbol {\mu }}=(\mu _{1},\ldots ,\mu _{N}) 
$
and diffusion tensor 
$
{\textstyle \mathbf {D} (\bx,t)={\dfrac {1}{2}}{\boldsymbol {\Sigma \Sigma }}^{\top}}$, i.e.
\[
{\displaystyle D_{ij}(\mathbf {x} ,t)={\dfrac {1}{2}}\sum _{k=1}^{M}\sigma _{ik}(\mathbf {x} ,t)\sigma _{jk}(\mathbf {x} ,t),}
\]
and $\sigma_{ij}$ denotes the $(i,j)$-th element of the matrix $\bm\Sigma$.

\subsection{Donsker-Varadhan Representation for KL-divergence}
\label{sec-re-kl}

\begin{proposition}[\citep{donsker1983asymptotic}]
 Let $\rho,\mu$ be  two probability distributions on the measure space $(\calX,\calF)$, where $\calX\in\R^{p}$. Then 
 \[
 \KL(\rho\|\mu)=\sup_{f:\calX\rightarrow\R}\left\{\int_{\R^{p}}\rho(\bx)f(\bx)\dd\bx-\log\int_{\R^{p}}\mu(\bx)\exp(f(\bx))\dd\bx\right\}.
 \]
\end{proposition}
The Donsker-Varadhan representation for KL-divergence implies for any $f(\cdot)$,
\begin{flalign}
\label{donsker-varadhan-representation}
 \KL(\rho\|\mu)\ge\int_{\R^{p}}\rho(\bx)f(\bx)\dd\bx-\log\int_{\R^{p}}\mu(\bx)\exp(f(\bx))\dd\bx,
 \end{flalign}
which is useful later.

\subsection{Some Basic Results for Diffusion Policy}

\begin{proposition}
\label{app-propo-02}
Let $\pi(\cdot|\bs)$ satisfy $\nu$-Log-Sobolev inequality (LSI) (see Assumption \ref{assumption-policy-calss}),
the initial random action ${\color{red}{\bar{\ba}_{0}}}\sim\pi(\cdot|\bs)$, and ${\color{red}{\bar{\ba}_{t}}}$ evolves according to
the following Ornstein-Uhlenbeck process,
\begin{flalign}
\label{def:diffusion-policy-sde-forward-process-02}
\dd {\color{red}{\bar{\ba}_{t}}}=-{\color{red}{\bar{\ba}_{t}}}\dd t+\sqrt{2}\dd \bw_{t}.
\end{flalign}
Let ${\color{red}{\bar{\ba}_{t}}}\sim {\color{red}{\bar{\pi}_{t}}}(\cdot|\bs)$ be the evolution along the Ornstein-Uhlenbeck flow (\ref{def:diffusion-policy-sde-forward-process-02}), then 
${\color{red}{\bar{\pi}_{t}}}(\cdot|\bs)$ is $\nu_t$-LSI, where \[\nu_t=\dfrac{\nu}{\nu+(1-\nu)\mathrm{e}^{-2t}}.\]
\end{proposition}
\begin{proof}
See \citep[Lemma 6]{wibisono2022convergence}.
 \end{proof}

\begin{proposition} 
\label{app-propo-04}
Under Assumption \ref{assumption-score}, then $\bnabla\log {\color{orange}{\tilde{\pi}_{t}}}(\cdot|\bs)$ is $L_{p}\mathrm{e}^{t}$-Lipschitz on the time interval $[0,\mathtt{T}_{0}]$, where the policy ${\color{orange}{\tilde{\pi}_{t}}}(\cdot|\bs)$ is the evolution along the flow (\ref{def:diffusion-policy-sde-reverse-process}),
and the time $\mathtt{T}_{0}$ is defined as follows,
\[
\mathtt{T}_{0}=:\sup_{t\ge0}\left\{t:1-\mathrm{e}^{-2t}\leq \frac{\mathrm{e}^{-t}}{L_p}\right\}.
\]
\end{proposition}
\begin{proof}
\citep[Lemma 13]{chen2022improved}.
\end{proof}
The positive scalar $\mathtt{T}_{0}$ is well-defined, i.e., $\mathtt{T}_{0}$ always exists. In fact, let $1-\mathrm{e}^{-2t}\leq \frac{\mathrm{e}^{-t}}{L_p}\ge0$, then the following holds,
\[
\mathrm{e}^{-t}\ge\sqrt{\dfrac{1}{L^{2}_{p}}+4
}-\dfrac{1}{L_{p}},
\]
then 
\[
\mathtt{T}_{0}=\log\left(\dfrac{1}{4}\left(
\sqrt{\dfrac{1}{L^{2}_{p}}+4
}+\dfrac{1}{L_{p}}
\right)\right).
\]

\begin{proposition}
\label{app-propo-03}
 \emph{(\cite[Lemma 10]{vempala2019rapid})}
Let $\rho(\bx)$ be a probability distribution function on $\R^{p}$, and let $f(\bx)=-\log \rho(\bx)$ be a $L$-smooth, i.e., there exists a positive constant $L$ such that $-L\bI\preceq\bnabla^{2} f(\bx)\preceq L\bI$ for all $\bx\in\R^{p}$.
Furthermore, let $\rho(\bx)$ satisfy the LSI condition with constant $\nu>0$, i.e., for any probability distribution $\mu(\bx)$,
$
\KL(\mu\|\rho)\leq\dfrac{1}{2\nu}\FI(\mu\|\rho).
$
Then for any distribution $\mu(\bx)$, the following equation holds,
\[
\E_{\bx\sim\mu(\cdot)}\left[\|\bnabla\log \rho(\bx)\|_2^{2}\right]=\int_{\R^{p}}\mu(\bx)\left\|\bnabla\log \rho(\bx)\right\|_2^{2}\dd\bx\leq\dfrac{4L^2}{\nu}\KL(\mu\|\rho)+2pL.
\]
\end{proposition}

\clearpage

\section{Implementation Details of DIPO}

\label{sec:app-details-of-implementation}

In this section, we provide all the details of our implementation for DIPO.

\begin{algorithm}[H]
    \caption{(DIPO): Model-Free Learning with \textbf{Di}ffusion \textbf{Po}licy}
    \label{app-algo-diffusion-free-based-rl}
    \begin{algorithmic}[1]
         \STATE Initialize parameter ${\bphi}$, critic networks $Q_{\bpsi}$, target networks $Q_{\bpsi^{'}}$, length $K$; 
           \STATE Initialize $\{\beta_{i}\}_{i=1}^{K}$; $\alpha_i=:1-\beta_i,\bar{\alpha}_k=:\prod_{i=1}^{k}\alpha_i,~\sigma_{k}=:\sqrt{\dfrac{1-\bar{\alpha}_{k-1}}{1-\bar{\alpha}_{k}}\beta_k}$;
         \STATE Initialize $\bpsi^{'}\gets\bpsi$, $\bphi^{'}\gets\bphi$;
        \REPEAT
              \STATE  {\color{brown}{\underline{{\texttt{\#update~experience~with~diffusion~policy}}}}}
              \STATE dataset $\calD_{\mathrm{env}}\gets\emptyset$; initial state $\bs_{0}\sim d_{0}(\cdot)$;
               \FOR{$t=0,1,\cdots,T$}
               \STATE  initial $\hata_{K}\sim\calN(\bm{0},\bI)$;
            \FOR{$k=K,\cdots,1$}
              \STATE $\bz_{k}\sim\calN(\bm{0},\bI)$, if $k>1$; else $\bz_{k}=0$;
              \STATE $\hata_{k-1}\gets\dfrac{1}{\sqrt{\alpha_k}}\Big(\hata_{k}-\dfrac{\beta_k}{\sqrt{1-\bar{\alpha}_k}}\bepsilon_{\bphi}(\hata_{k},\bs,k)\Big)+\sigma_k\bz_k;$
               \ENDFOR                  
               \STATE $\ba_{t}\gets\hata_{0}$; $\bs_{t+1}\sim \Pro(\cdot|\bs_{t},\ba_{t})$; $\calD_{\mathrm{env}}\gets\calD_{\mathrm{env}}\cup\{\bs_{t},\ba_{t},\bs_{t+1},r(\bs_{t+1}|\bs_{t},\ba_{t})\}$;
                \ENDFOR
         \STATE  {\color{brown}{\underline{{\texttt{\#update~value~function}}}}}
         \FOR {each mini-batch data}
          \STATE sample mini-batch $\calD$ from $\calD_{\mathrm{env}}$ with size $N$, $\calD=\{\bs_{j},\ba_{j},\bs_{j+1},r(\bs_{j+1}|\bs_{j},\ba_{j})\}_{j=1}^{N}$;
       \STATE  take gradient descent as follows \[\bpsi\gets\bpsi-\eta_{\psi}\bnabla_{\bpsi}\dfrac{1}{N}\sum_{j=1}^{N}\Big(r(\bs_{j+1}|\bs_j,\ba_j)+\gamma Q_{\bpsi^{'}}(\bs_{j+1},\ba_{j+1})-Q_{\bpsi}(\bs_j,\ba_j)\Big)^2;\]
         \ENDFOR
         \STATE  {\color{brown}{\underline{{\texttt{\#improve~experience~through~action}}}}}
         \FOR{$t=0,1,\cdots,T$}
         \STATE replace the action $\ba_{t}\in\calD_{\mathrm{env}}$ as follows 
         \[\ba_{t}\gets\ba_{t}+\eta_{a}\bnabla_{\ba}Q_{\bpsi}(\bs_{t},\ba)\big|_{\ba=\ba_t};\]
         \ENDFOR
         \STATE  {\color{brown}{\underline{{\texttt{\#update~diffusion~policy}}}}}
               \FOR{each pair}
        \STATE sample a pair $(\bs,\ba)\sim\calD_{\mathrm{env}}$ uniformly; $k\sim \mathrm{Uniform}(\{1,\cdots,K\})$; $\bz\sim \calN(\bm{0},\bI)$;
        \STATE take gradient descent as follows \[\bphi\gets\bphi-\eta_{\phi}\bnabla_{\bphi}\left\|\bz-\bepsilon_{\bphi}\left(\sqrt{\bar\alpha_k}\ba+\sqrt{1-\bar{\alpha}_k}\bz,\bs,k\right)\right\|_{2}^{2};\]
          \ENDFOR
          \STATE soft update $\bpsi^{'}\gets\rho\bpsi^{'}+(1-\rho)\bpsi$;
           \STATE soft update $\bphi^{'}\gets\rho\bphi^{'}+(1-\rho)\bphi$;
           \UNTIL{the policy performs well in the real environment.}
     \end{algorithmic}
\end{algorithm}

\subsection{DIPO: Model-Free Learning with Diffusion Policy}

Our source code follows the Algorithm \ref{app-algo-diffusion-free-based-rl}.

\subsection{Loss Function of DIPO}
\label{app-sec:loss}

In this section, we provide the details of {\color{brown}{\underline{{\texttt{\#update~diffusion~policy}}}}} presented in Algorithm \ref{app-algo-diffusion-free-based-rl}.
We present the derivation of the loss of score matching (\ref{denosing-score-matching}) and present the details of updating the diffusion from samples.
First, the next Theorem \ref{loss-diff-policy} shows an equivalent version of the loss defined in (\ref{denosing-score-matching}), then we present the learning details from samples.

\subsubsection{Conditional Sampling Version of Score Matching}

\begin{theorem}
\label{loss-diff-policy}
For give a partition on the interval $[0,T]$, $0=t_{0}<t_{1}<\cdots<t_{k}<t_{k+1}<\cdots<t_{K}=T$, let $\alpha_{0}=\mathrm{e}^{-2T}$,
$\alpha_{k}=\mathrm{e}^{2(-t_{k+1}+t_{k})},$ and $\bar{\alpha}_{k-1}=\prod_{k^{'}=0}^{k}\alpha_{k^{'}}$.
Setting $\omega(t)$ according to the next (\ref{app-weight-setting}), then the objective (\ref{denosing-score-matching}) follows the next expectation version,
\begin{flalign}
\label{pro-loss-dipo}
\calL(\bphi)=\E_{k\sim\calU([K]),\bz_{k}\sim\calN(\bm{0},\bI),\bara_{0}\sim\pi(\cdot|\bs)}\left[\bz_{k}-\bm{\epsilon}_{\bphi}\left(\sqrt{\bar{\alpha}_{k}}\bara_0+\sqrt{1-\bar{\alpha}_{k}}\bz_{{k}},\bs,k\right)\right],
\end{flalign}
where $[K]=:\{1,2,\cdots,K\}$, $\calU(\cdot)$ denotes uniform distribution, the parametic funciton \[\bm{\epsilon}_{\bphi}(\cdot,\cdot,\cdot):\calA\times\calS\times [K]\rightarrow\R^{p}\] shares the parameter $\bphi$ according to: \[\bm{\epsilon}_{\bphi}\left(\cdot,\cdot,k\right)
=
-\sqrt{1-\bar{\alpha}_{k}}\hat{\bm{s}}_{\bphi}\left(\cdot,\cdot,T-t_k\right).\]
\end{theorem}

\begin{proof}  According to (\ref{forward-process-kernel}), and Proposition \ref{solution4linearsdes}, we know $ \varphi_{t}(\bara_t|\bara_0)=\calN\left(\mathrm{e}^{-t}\bara_0,\left(1-\mathrm{e}^{-2t}\right)\bI\right)$, then 
\begin{flalign}
\label{expression-score-function}
\bnabla\log \varphi_{t}(\bara_t|\bara_0)=-\dfrac{\bara_t-\mathrm{e}^{-t}\bara_{0}}{1-2\mathrm{e}^{-t}}=-\dfrac{\bz_{t}}{\sqrt{1-2\mathrm{e}^{-t}}},
\end{flalign}
where $\bz_{t}\sim\calN(\bm{0},\bI)$,

Let $\sigma_{t}=\sqrt{1-\mathrm{e}^{-2t}}$, according to (\ref{forward-process-kernel}), we know
\begin{flalign}
\label{app-forward-process-kernel}
\bara_t =\mathrm{e}^{-t}\bara_0+\left(\sqrt{1-\mathrm{e}^{-2t}}\right)\bz_{t}=
\mathrm{e}^{-t}\bara_0+\sigma_{t}\bz_{t},
\end{flalign}
where $\bz_{t}\sim\calN(\bm{0},\bI)$.

Recall (\ref{denosing-score-matching}), we obtain
 \begin{flalign}
\nonumber
\calL(\bphi)=&\bigintsss_{0}^{T}
\omega (t)
\E_{\bara_{0}\sim\pi(\cdot|\bs)}
\E_{\bara_t|\bara_0}
\left[
\left\|
\hat{\bm{s}}_{\bphi}(\bara_t,\bs,t)-\bnabla\log \varphi_{t}(\bara_t|\bara_0)
\right\|_2^{2}
\right]\dd t\\
\label{app-eq--58}
\overset{(\ref{expression-score-function})}=&\bigintsss_{0}^{T}\dfrac{\omega (t)}{\sigma_{t}^{2}}
\E_{\bara_{0}\sim\pi(\cdot|\bs)}
\E_{\bara_t|\bara_0}
\left[
\left\|
\sigma_{t}\hat{\bm{s}}_{\bphi}(\bara_t,\bs,t)+\bz_{t}
\right\|_2^{2}
\right]\dd t\\
\overset{t\leftarrow T-t}=&\bigintsss_{0}^{T}\dfrac{\omega (T-t)}{\sigma_{T-t}^{2}}
\E_{\bara_{0}\sim\pi(\cdot|\bs)}
\E_{\bara_{T-t}|\bara_0}
\left[
\left\|
\sigma_{T-t}\hat{\bm{s}}_{\bphi}(\bara_{T-t},\bs,T-t)+\bz_{T-t}
\right\|_2^{2}
\right]\dd t.
\end{flalign}

Furthermore, we define an indicator function $\I_{t^{'}}(t)$ as follows,
\[
\I_{t^{'}}(t)=:\left\{
\begin{array}{rcl}
1 ,   & {\text{if}~t^{'}= t;}\\
0   ,     &  {\text{if}~t^{'}\ne t.}
\end{array} \right. 
\]
Let the weighting function be defined as follows, for any $t\in[0,T]$,
 \begin{flalign}
\label{app-weight-setting}
\omega(t)=\dfrac{1}{K}\sum_{k=1}^{K}\left(1-\mathrm{e}^{-2(T-t)}\right)\I_{T-t_{k}}(t),
\end{flalign}
where we give a partition on the interval $[0,T]$ as follows, 
\[0=t_{0}<t_{1}<\cdots<t_{k}<t_{k+1}<\cdots<t_{K}=T.\]
Then, we rewrite (\ref{app-eq--58}) as follows,
\begin{flalign}
\label{app-eq-59}
\calL(\bphi)=\dfrac{1}{K}\sum_{k=1}^{K}
\E_{\bara_{0}\sim\pi(\cdot|\bs)}
\E_{\bara_{T-t_k}|\bara_0}
\left[
\left\|
\sigma_{T-t_{k}}\hat{\bm{s}}_{\bphi}(\bara_{T-t_k},\bs,T-t_k)+\bz_{T-t_k}
\right\|_2^{2}
\right].
\end{flalign}

We consider the next term contained in (\ref{app-eq-59})
\begin{flalign}
\nonumber
\hat{\bm{s}}_{\bphi}\left(\bara_{T-t_k},\bs,T-t_k\right)\overset{(\ref{app-forward-process-kernel})}=&\hat{\bm{s}}_{\bphi}\left(\mathrm{e}^{-(T-t_{k})}\bara_0+\sigma_{T-t_{k}}\bz_{T-t_{k}},\bs,T-t_k\right)\\
\nonumber
=&\hat{\bm{s}}_{\bphi}\left(\mathrm{e}^{-(T-t_{k})}\bara_0+\sqrt{1-\mathrm{e}^{-2(T- t_{k})}}\bz_{T-t_{k}},\bs,T-t_k\right),
\end{flalign}
where $\bz_{T-t_{k}}\sim\calN(\bm{0},\bI)$,
then obtain 
\begin{flalign}
\nonumber
&\E_{\bara_{T- t_k}|\bara_0}\left[\left\|\hat{\bm{s}}_{\bphi}(\bara_{T-t_k},\bs,T-t_k)+\bz_{T-t_k}\right\|_2^{2}\right]\\
\nonumber
=&\E_{\bz_{T-t_{k}}\sim\calN(\bm{0},\bI)}\left[\sigma_{T-t_{k}}\hat{\bm{s}}_{\bphi}\left(\mathrm{e}^{-(T-t_{k})}\bara_0+\sqrt{1-\mathrm{e}^{-2(T- t_{k})}}\bz_{T-t_{k}},\bs,T-t_k\right)\right].
\end{flalign}

Now, we rewrite (\ref{app-eq-59}) as the next expectation version,
\begin{flalign}
\label{app-eq-60}
\calL(\bphi)=\E_{k\sim\calU([K]),\bz_{t_k}\sim\calN(\bm{0},\bI),\bara_{0}\sim\pi(\cdot|\bs)}\left[\sigma_{T-t_{k}}\hat{\bm{s}}_{\bphi}\left(\mathrm{e}^{-(T-t_{k})}\bara_0+\sqrt{1-\mathrm{e}^{-2(T- t_{k})}}\bz_{T-t_{k}},\bs,T-t_k\right)\right]
\end{flalign}
For $k=0,1,\cdots,K$, and $\alpha_{0}=\mathrm{e}^{-2T}$ and
\begin{flalign}
\label{app-eq-63}
\alpha_{k}=\mathrm{e}^{2(-t_{k+1}+t_{k})}.
\end{flalign}
Then we obtain
\[
\bar{\alpha}_{k}=\prod_{k^{'}=0}^{k-1}\alpha_{k^{'}}=\mathrm{e}^{-2(T- t_{k})}.
\]
With those notations, we rewrite (\ref{app-eq-60}) as follows, 
\begin{flalign}
\label{app-eq-61}
\calL(\bphi)=\E_{k\sim\calU([K]),\bz_{k}\sim\calN(\bm{0},\bI),\bara_{0}\sim\pi(\cdot|\bs)}
\left[\sqrt{1-\bar{\alpha}_{k}}\hat{\bm{s}}_{\bphi}\left(\sqrt{\bar{\alpha}_{k}}\bara_0+\sqrt{1-\bar{\alpha}_{k}}\bz_{{k}},\bs,T-t_k\right)+\bz_{k}\right].
\end{flalign}
Finally, we define a function $\bm{\epsilon}_{\bphi}(\cdot,\cdot,\cdot):\calS\times\calA\times [K]\rightarrow\R^{p}$, and 
\begin{flalign}
\label{app-eq-64}
\bm{\epsilon}_{\bphi}\left(\sqrt{\bar{\alpha}_{k}}\bara_0+\sqrt{1-\bar{\alpha}_{k}}\bz_{{k}},\bs,k\right)
=:
-\sqrt{1-\bar{\alpha}_{k}}\hat{\bm{s}}_{\bphi}\left(\sqrt{\bar{\alpha}_{k}}\bara_0+\sqrt{1-\bar{\alpha}_{k}}\bz_{{k}},\bs,T-t_k\right),
\end{flalign}
i,e. we estimate the score function via an estimator $\epsilon_{\bphi}$ as follows,
\begin{flalign}
\label{app-eq-68}
\hat{\bm{s}}_{\bphi}\left(\sqrt{\bar{\alpha}_{k}}\bara_0+\sqrt{1-\bar{\alpha}_{k}}\bz_{{k}},\bs,T-t_k\right)=-\dfrac{\bm{\epsilon}_{\bphi}\left(\sqrt{\bar{\alpha}_{k}}\bara_0+\sqrt{1-\bar{\alpha}_{k}}\bz_{{k}},\bs,k\right)}{\sqrt{1-\bar{\alpha}_{k}}}.
\end{flalign}
Then we rewrite (\ref{app-eq-61}) as follows,
\[
\calL(\bphi)=\E_{k\sim\calU([K]),\bz_{k}\sim\calN(\bm{0},\bI),\bara_{0}\sim\pi(\cdot|\bs)}\left[\bz_{k}-\bm{\epsilon}_{\bphi}\left(\sqrt{\bar{\alpha}_{k}}\bara_0+\sqrt{1-\bar{\alpha}_{k}}\bz_{{k}},\bs,k\right)\right].
\]
This concludes the proof.
\end{proof}

\subsubsection{Learning from Samples}

According to the expectation version of loss (\ref{pro-loss-dipo}), we know, for each pair $(\bs,\ba)$ sampled from experience memory, let $k\sim \mathrm{Uniform}(\{1,\cdots,K\})$ and $\bz\sim \calN(\bm{0},\bI)$, the following empirical loss
 \[\ell_{\mathrm{d}}(\bphi)=\|\bz-\bepsilon_{\bphi}\left(\sqrt{\bar\alpha_k}\ba+\sqrt{1-\bar{\alpha}_k}\bz,\bs,k\right)\|_{2}^{2}\]
is a unbiased estimator of $\calL(\bphi)$ defined in (\ref{pro-loss-dipo}).

Finally, we learn the parameter $\bphi$ by minimizing the empirical loss $\ell_{\mathrm{d}}(\bphi)$ according to gradient decent method:
 \[\bphi\gets\bphi-\eta_{\phi}\bnabla_{\bphi}\left\|\bz-\bepsilon_{\bphi}\left(\sqrt{\bar\alpha_k}\ba+\sqrt{1-\bar{\alpha}_k}\bz,\bs,k\right)\right\|_{2}^{2},\]
where $\bepsilon_{\bphi}$ is the step-size.
For the implementation, see lines 25-28 in Algorithm \ref{app-algo-diffusion-free-based-rl}.

\subsection{Playing Actions of DIPO}

\label{app-derivation-of-actions}

\begin{algorithm}[t]
    \caption{Diffusion Policy (A Backward Version \citep{ho2020denoising})}
    \label{algo:diffusion-policy-ddpm}
    \begin{algorithmic}[1]
         \STATE input state $\bs$; parameter $\bphi$; reverse length $K$;
           \STATE initialize $\{\beta_{i}\}_{i=1}^{K}$; $\alpha_i=:1-\beta_i,\bar{\alpha}_k=:\prod_{i=1}^{k}\alpha_i,~\sigma_{k}=:\sqrt{\dfrac{1-\bar{\alpha}_{k-1}}{1-\bar{\alpha}_{k}}\beta_k}$;
         \STATE  initial $\hata_{K}\sim\calN(\bm{0},\bI)$;
             \FOR{$k=K,\cdots,1$}
              \STATE $\bz_{k}\sim\calN(\bm{0},\bI)$, if $k>1$; else $\bz_{k}=0$;
              \STATE $\hata_{k-1}\gets\dfrac{1}{\sqrt{\alpha_k}}\Big(\hata_{k}-\dfrac{\beta_k}{\sqrt{1-\bar{\alpha}_k}}\bepsilon_{\bphi}(\hata_{k},\bs,k)\Big)+\sigma_k\bz_k;$
               \ENDFOR  
              \STATE return $\hata_{0}$
     \end{algorithmic}
\end{algorithm}

In this section, we present all the details of {\color{brown}{\underline{{\texttt{\#update~experience~with~diffusion~policy}}}}} presented in Algorithm \ref{app-algo-diffusion-free-based-rl}.

Let $\beta_{k}=1-\alpha_{k}$, then according to Taylar formualtion, we know
\begin{flalign}
\label{app-eq-65}
 \sqrt{\alpha_k}=1-\dfrac{1}{2}\beta_{k}+o(\beta_{k}).
 \end{flalign}
Recall the exponential integrator discretization (\ref{app-eq-62}), we know
\begin{flalign}
\nonumber
\hata_{t_{k+1}}-\hata_{t_{k}}=&\left(\mathrm{e}^{t_{k+1}-t_{k}}-1\right)\left(\hata_{t_{k}}+2 \bm{\hat \bm{s}_{\bphi}}(\hata_{t_k},\bs,T-t_{k})\right)+\sqrt{2}\int_{t_k}^{t_{k+1}}\mathrm{e}^{t^{'}-t_{k}}\dd \bw_{t^{'}},
\end{flalign}
which implies
\begingroup
\allowdisplaybreaks
\begin{flalign}
\nonumber
\hata_{t_{k+1}}=&\hata_{t_{k}}+\left(\mathrm{e}^{t_{k+1}-t_{k}}-1\right)\left(\hata_{t_{k}}+2 \bm{\hat \bm{s}_{\bphi}}(\hata_{t_k},\bs,T-t_{k})\right)+\sqrt{\mathrm{e}^{2(t_{k+1}-t_{k})}-1}\bz_{t_k}\\
\nonumber
\overset{(\ref{app-eq-63})}=&\hata_{t_{k}}+\left(\dfrac{1}{\sqrt{\alpha_{k}}}-1\right)\left(\hata_{t_{k}}
+2 \bm{\hat \bm{s}_{\bphi}}(\hata_{t_k},\bs,T-t_{k})\right)+\sqrt{\dfrac{1-\alpha_{k}}{\alpha_{k}}}\bz_{t_k}\\
\nonumber
\overset{(\ref{app-eq-68})}=&\dfrac{1}{\sqrt{\alpha_{k}}}\hata_{t_{k}}-2\left(\dfrac{1}{\sqrt{\alpha_{k}}}-1\right) \dfrac{1}{\sqrt{1-\bar{\alpha}_{k}}}\bm{ \bepsilon_{\bphi}}(\hata_{t_k},\bs,k)+\sqrt{\dfrac{1-\alpha_{k}}{\alpha_{k}}}\bz_{t_k}\\
\label{app-eq-66}
=&\dfrac{1}{\sqrt{\alpha_{k}}}\hata_{t_{k}}-\dfrac{\beta_{k}}{\sqrt{\alpha_{k}}}\cdot\dfrac{1}{\sqrt{1-\bar{\alpha}_{k}}}\bm{ \bepsilon_{\bphi}}(\hata_{t_k},\bs,k)
+\sqrt{\dfrac{1-\alpha_{k}}{\alpha_{k}}}\bz_{t_k},
\end{flalign}
\endgroup
where $\bz_{t_k}\sim\calN(\bm{0},\bI)$, Eq.(\ref{app-eq-66}) holds since we use the fact (\ref{app-eq-65}), which implies
\[
2\left(\dfrac{1}{\sqrt{\alpha_{k}}}-1\right)=2\left(\dfrac{1-\sqrt{\alpha_{k}}}{\sqrt{\alpha_{k}}}\right)=\dfrac{\beta_{k}}{\sqrt{\alpha_{k}}}+o\left(\dfrac{\beta_{k}}{\sqrt{\alpha_{k}}}\right).
\]
To simplify the expression, we rewrite (\ref{app-eq-66}) as follows,
\begin{flalign}
\label{app-eq-67}
\hata_{k+1}=&\dfrac{1}{\sqrt{\alpha_{k}}}\hata_{k}-\dfrac{\beta_{k}}{\sqrt{\alpha_{k}}}\cdot\dfrac{1}{\sqrt{1-\bar{\alpha}_{k}}}\bm{\bepsilon_{\bphi}}(\hata_{k},\bs,k)
+\sqrt{\dfrac{1-\alpha_{k}}{\alpha_{k}}}\bz_{k}\\
=&\dfrac{1}{\sqrt{\alpha_{k}}}\left(\hata_{k}-\dfrac{\beta_{k}}{\sqrt{1-\bar{\alpha}_{k}}}\bm{\bepsilon_{\bphi}}(\hata_{k},\bs,k)\right)
+\sqrt{\dfrac{1-\alpha_{k}}{\alpha_{k}}}\bz_{k}\\
\label{app-eq-69}
=&\dfrac{1}{\sqrt{\alpha_{k}}}\left(\hata_{k}-\dfrac{1-\alpha_{k}}{\sqrt{1-\bar{\alpha}_{k}}}\bm{\bepsilon_{\bphi}}(\hata_{k},\bs,k)\right)
+\sqrt{\dfrac{1-\alpha_{k}}{\alpha_{k}}}\bz_{k},
\end{flalign}
where $k=0,1,\cdots,K-1$ runs forward in time, $\bz_{k}\sim\calN(\bm{0},\bI)$. The agent plays the last action $\hata_{K}$.

Since we consider the SDE of the reverse process (\ref{def:diffusion-policy-sde-reverse-process}) that runs forward in time, while most diffusion probability model literature (e.g., \citep{ho2020denoising,song2020score}) consider the backward version for sampling.
To coordinate the relationship between the two versions, we also present the backward version in Algorithm \ref{algo:diffusion-policy-ddpm}, which is essentially identical to the iteration (\ref{app-eq-69}) but rewritten in the running in backward time version.

\clearpage

\section{Time Derivative of KL Divergence Between Difuffusion Policy and True Reverse Process}

In this section, we provide the time derivative of KL divergence between diffusion policy (Algorithm \ref{algo:diffusion-policy-general-case}) and true reverse process (defined in (\ref{def:diffusion-policy-sde-reverse-process})).

\subsection{Time Derivative of KL Divergence at Reverse Time $k=0$}

In this section, we consider the case $k=0$ of diffusion policy (see Algorithm \ref{algo:diffusion-policy-general-case} or the iteration (\ref{iteration-exponential-integrator-discretization})).
If $k=0$, then for $0\leq t\leq h$, the SDE (\ref{def:diffusion-policy-sde-reverse-process-s}) is reduced as follows,
\begin{flalign}
\label{def:diffusion-policy-sde-reverse-process-s-01}
\dd \hata_t=\left(\hata_t+2 \bm{\hat \mathbf{S}}({\hata_{0}},\bs,T)\right)\dd t+\sqrt{2}\dd \bw_{t},~t\in [0,h],
\end{flalign}
where $\bw_{t}$ is the standard Wiener process starting at $\bw_{0}=\bm{0}$.

Let the action $\hata_{t}\sim\hatpi_{t}(\cdot|\bs)$ follows the process (\ref{def:diffusion-policy-sde-reverse-process-s-01}).
The next Proposition \ref{primal-bound-01} considers the distribution difference between the diffusion policy $\hatpi_{t}(\cdot|\bs)$ and the true distribution of backward process (\ref{def:diffusion-policy-sde-reverse-process}) $\tildepi_{t}(\cdot|\bs)$ on the time interval $t\in[0,h]$.

%\textbf{Proposition} \ref{primal-bound-01}.
%\emph{
%Under Assumption \ref{assumption-score} and \ref{assumption-policy-calss}, let $\tildepi_{t}(\cdot|\bs)$ be the distribution at time $t$ with the process (\ref{def:diffusion-policy-sde-reverse-process}), and let $\hatpi_{t}(\cdot|\bs)$ be the distribution at time $t$ with the process (\ref{def:diffusion-policy-sde-reverse-process-s-01}).
%Let
%\begin{flalign}
% &~~~~~~~~~~~~~~~~~~~~~~~~~\tau_{0}=:\sup\left\{t:t\mathrm{e}^{t}\leq\dfrac{\sqrt{5\nu}}{96L_{s}L_{p}}\right\}, \tau=:\min\left\{ \tau_{0},\dfrac{1}{12L_{s}}\right\},\\
% &\epsilon=:\sup_{(k,t)\in[K]\times[kh,(k+1)h]}\left\{\log\E_{\ba\sim\tildepi_t(\cdot|\bs)}\left[\exp\left\|\bm{\hat \mathbf{S}}(\ba,\bs,T-hk)-\bnabla\log\tildepi_t(\ba|\bs)\right\|_{2}^{2}\right]\right\},
%\end{flalign}
%and $0\leq t\leq h\leq \tau$, then the following equation holds,
% \begin{flalign}
%\nonumber
%\dfrac{\dd}{\dd t}\KL\big(\hatpi_t(\cdot|\bs)\|\tildepi_t(\cdot|\bs)\big)\leq
%-\dfrac{\nu}{4}\KL\left(\hatpi_{t}(\cdot|\bs)\|\tildepi_t(\cdot|\bs)\right)
%+\dfrac{5}{4}\nu\epsilon_{\mathrm{score}}
%+12pL_{s}\sqrt{5\nu}t.
%\end{flalign}
%}

\begin{proposition}
\label{primal-bound-01}
Under Assumption  \ref{assumption-score} and \ref{assumption-policy-calss}, let $\tildepi_{t}(\cdot|\bs)$ be the distribution at time $t$ with the process (\ref{def:diffusion-policy-sde-reverse-process}), and let $\hatpi_{t}(\cdot|\bs)$ be the distribution at time $t$ with the process (\ref{def:diffusion-policy-sde-reverse-process-s-01}).
Let
\begin{flalign}
 &~~~~~~~~~~~~~~~~~~~~~~~~~\tau_{0}=:\sup\left\{t:t\mathrm{e}^{t}\leq\dfrac{\sqrt{5\nu}}{96L_{s}L_{p}}\right\}, \tau=:\min\left\{ \tau_{0},\dfrac{1}{12L_{s}}\right\},\\
 &\epsilon_{\mathrm{score}}=:\sup_{(k,t)\in[K]\times[t_{k},t_{k+1}]}\left\{\log\E_{\ba\sim\tildepi_t(\cdot|\bs)}\left[\exp\left\|\bm{\hat \mathbf{S}}(\ba,\bs,T-hk)-\bnabla\log\tildepi_t(\ba|\bs)\right\|_{2}^{2}\right]\right\},
\end{flalign}
and $0\leq t\leq h\leq \tau$, then the following equation holds,
 \begin{flalign}
 \label{propo-01}
\dfrac{\dd}{\dd t}\KL\big(\hatpi_t(\cdot|\bs)\|\tildepi_t(\cdot|\bs)\big)\leq
-\dfrac{\nu}{4}\KL\left(\hatpi_{t}(\cdot|\bs)\|\tildepi_t(\cdot|\bs)\right)
+\dfrac{5}{4}\nu\epsilon_{\mathrm{score}}
+12pL_{s}\sqrt{5\nu}t.
\end{flalign}
\end{proposition}

Before we show the details of the proof, we need to define some notations, which is useful later.
Let 
\begin{flalign}
\label{app-eq-48}
\hatpi_{t}({\color{blue}{\hat{\ba}}}|\bs)=:p({\hata_{t}}={\color{blue}{\hat{\ba}}}|\bs,t)
\end{flalign}
denote the distribution of the action ${\hata_{t}}=\hata$ be played at time $t$ along the process (\ref{def:diffusion-policy-sde-reverse-process-s-01}), where $t\in[0,h]$.
For each $t>0$, let $\rho_{0,t}({\hata_{0}},{\hata_{t}}|\bs)$ denote the joint distribution of $({\hata_{0}},{\hata_{t}})$ conditional on the state $\bs$,  which can be written in terms of the conditionals and marginals as follows,
\[\rho_{0|t}({\hata_{0}}|{\hata_{t}},\bs)=\dfrac{\rho_{0,t}({\hata_{0}},{\hata_{t}}|\bs)}{p({\hata_{t}}={\color{blue}{\hat{\ba}}}|\bs,t)}
=\dfrac{\rho_{0,t}({\hata_{0}},{\hata_{t}}|\bs)}{\hatpi_{t}({\hata_{t}}|\bs)}.
\]

\subsection{Auxiliary Results For Reverse Time $k=0$}

\begin{lemma}
\label{lem-01}
Let $\hatpi_{t}({\color{blue}{\hat{\ba}}}|\bs)$ be the distribution at time $t$ along interpolation SDE (\ref{def:diffusion-policy-sde-reverse-process-s-01}), where $\hatpi_{t}({\color{blue}{\hat{\ba}}}|\bs)$ is short for $p({\hata_{t}}={\color{blue}{\hat{\ba}}}|\bs,t)$, which is  the distribution of the action ${\hata_{t}}=\hata$ be played at time $t$ alongs the process (\ref{def:diffusion-policy-sde-reverse-process-s-01}) among the time $t\in[0,h]$.
 Then its derivation with respect to time satisfies 
\begin{flalign}
\label{app-eq-51}
\dfrac{\partial}{\partial t}\hatpi_{t}({\color{blue}{\hat{\ba}}}|\bs)
=-\hatpi_{t}({\color{blue}{\hat{\ba}}}|\bs)\DIV \cdot 
\left({\color{blue}{\hat{\ba}}}+2\E_{{\hata_{0}}\sim\rho_{0|t}(\cdot|{\color{blue}{\hat{\ba}}},\bs) }\big[\bm{\hat \mathbf{S}}({\hata_{0}},\bs,T)\big|{\hata_{t}}={\color{blue}{\hat{\ba}}}\big]
\right)+\Delta \hatpi_{t}({\color{blue}{\hat{\ba}}}|\bs).
\end{flalign}
\end{lemma}

Before we show the details of the proof, we need to clear the divergence term $\DIV$. In this section, all the notation is defined according to (\ref{app-eq-50}), and its value is at the point $\hata$.

For example, in Eq.(\ref{app-eq-51}), the divergence term $\DIV$ is defined as follows,
\begin{flalign}
\nonumber
&\DIV \cdot \left({\color{blue}{\hat{\ba}}}+2\E_{{\hata_{0}}\sim\rho_{0|t}(\cdot|{\color{blue}{\hat{\ba}}},\bs) }\big[\bm{\hat \mathbf{S}}({\hata_{0}},\bs,T)\big|{\hata_{t}}={\color{blue}{\hat{\ba}}}\big]\right)=(\DIV \cdot\mathbf{p})(\hata),\\
&\mathbf{p}(\ba)=\ba+2\E_{{\hata_{0}}\sim\rho_{0|t}(\cdot|\ba,\bs) }\big[\bm{\hat \mathbf{S}}({\hata_{0}},\bs,T)\big|{\hata_{t}}=\ba\big].
\end{flalign}
For example, in Eq.(\ref{app-eq-52}), the divergence term $\DIV$ is defined as follows,
\begin{flalign}
\nonumber
&\DIV\cdot \Big(p({\color{blue}{\hat{\ba}}}|{\hata_{0}},\bs,t)\left({\color{blue}{\hat{\ba}}}+2 \bm{\hat \mathbf{S}}({\hata_{0}},\bs,T)\right)\Big)=(\DIV \cdot\mathbf{p})(\hata),\\
&\bp(\ba)=p(\ba|{\hata_{0}},\bs,t)\left(\ba+2 \bm{\hat \mathbf{S}}({\hata_{0}},\bs,T)\right).
\end{flalign}
Similar definitions are parallel in Eq.(\ref{app-eq-01}), from Eq.(\ref{app-eq-53}) to Eq.(\ref{app-eq-03}).

\begin{proof}
First, for a given state $\bs$, conditioning on the initial action ${\hata_{0}}$, we introduce a notation
\begin{flalign}
p(\cdot|{\hata_{0}},\bs,t): \R^{p}\rightarrow [0,1],
\end{flalign}
and each \[p({\color{blue}{\hat{\ba}}}|{\hata_{0}},\bs,t)=:p({\hata_{t}}={\color{blue}{\hat{\ba}}}|{\hata_{0}},\bs,t)\]
that denotes the conditional probability distribution starting from ${\hata_{0}}$ to the action ${\hata_{t}}={\color{blue}{\hat{\ba}}}$ at time $t$ under the state $\bs$.
Besides, we also know,
\begin{flalign}
\label{relationship-01}
\hatpi_{t}(\cdot|\bs)=\E_{{\hata_{0}}\sim\calN(\bm{0},\bI)}[p(\cdot|{\hata_{0}},\bs,t)]
=\int_{\R^{p}} \rho_{0}({\hata_{0}})p(\cdot|{\hata_{0}},\bs,t)\dd {\hata_{0}},
\end{flalign}
where $\rho_{0}(\cdot)=\calN(\bm{0},\bI)$ is the initial action distribution for reverse process.

For each $t>0$, let $\rho_{0,t}({\hata_{0}},{\hata_{t}}|\bs)$ denote the joint distribution of $({\hata_{0}},{\hata_{t}})$ conditional on the state $\bs$,  which can be written in terms of the conditionals and marginals as follows,
\[\rho_{0,t}({\hata_{0}},{\hata_{t}}|\bs)
=p({\hata_{0}}|\bs)\rho_{t|0}({\hata_{t}}|{\hata_{0}},\bs)
=p({\hata_{t}}|\bs)\rho_{0|t}({\hata_{0}}|{\hata_{t}},\bs).
\]
Then we obtain the Fokker–Planck equation for the distribution $p(\cdot|{\hata_{0}},\bs,t)$ as follows,
\begin{flalign}
\label{app-eq-52}
\dfrac{\partial}{\partial t}p({\color{blue}{\hat{\ba}}}|{\hata_{0}},\bs,t)=-\DIV\cdot \Big(p({\color{blue}{\hat{\ba}}}|{\hata_{0}},\bs,t)\left({\color{blue}{\hat{\ba}}}+2 \bm{\hat \mathbf{S}}({\hata_{0}},\bs,T)\right)\Big)+\Delta p({\color{blue}{\hat{\ba}}}|{\hata_{0}},\bs,t),
\end{flalign}
where the $\DIV$ term is defined according to (\ref{app-eq-49}) and (\ref{app-eq-50}) if $\bp(\ba)=p(\ba|{\hata_{0}},\bs,t)\left(\ba+2 \bm{\hat \mathbf{S}}({\hata_{0}},\bs,T)\right)$.

Furthermore, according to (\ref{relationship-01}), we know 
\begingroup
\allowdisplaybreaks
\begin{flalign}
\dfrac{\partial}{\partial t}\hatpi_{t}({\color{blue}{\hat{\ba}}}|\bs)=&\dfrac{\partial}{\partial t}\int_{\R^{p}} \rho_{0}({\hata_{0}})p({\color{blue}{\hat{\ba}}}|{\hata_{0}},\bs,t)\dd {\hata_{0}}
=\int_{\R^{p}} \rho_{0}({\hata_{0}})\dfrac{\partial}{\partial t} p({\color{blue}{\hat{\ba}}}|{\hata_{0}},\bs,t)\dd {\hata_{0}}\\
\label{app-eq-04}
=&\int_{\R^{p}} \rho_{0}({\hata_{0}})
\left(
-\DIV\cdot \Big(p({\color{blue}{\hat{\ba}}}|{\hata_{0}},\bs,t)\left({\color{blue}{\hat{\ba}}}+2 \bm{\hat \mathbf{S}}({\hata_{0}},\bs,T)\right)\Big)+\Delta p({\color{blue}{\hat{\ba}}}|{\hata_{0}},\bs,t)
\right)
\dd {\hata_{0}}\\
\label{app-01}
=&-\hatpi_{t}({\color{blue}{\hat{\ba}}}|\bs)\DIV\cdot{\color{blue}{\hat{\ba}}}
-2\DIV \cdot 
\left(
\hatpi_{t}({\color{blue}{\hat{\ba}}}|\bs)\E_{{\hata_{0}}\sim\rho_{0|t}(\cdot|{\color{blue}{\hat{\ba}}},\bs) }\big[\bm{\hat \mathbf{S}}({\hata_{0}},\bs,T)\big|{\hata_{t}}={\color{blue}{\hat{\ba}}}\big]
\right)+\Delta \hatpi_{t}({\color{blue}{\hat{\ba}}}|\bs)
\\
=&-\hatpi_{t}({\color{blue}{\hat{\ba}}}|\bs)\DIV \cdot 
\left({\color{blue}{\hat{\ba}}}+2\E_{{\hata_{0}}\sim\rho_{0|t}(\cdot|{\color{blue}{\hat{\ba}}},\bs) }\big[\bm{\hat \mathbf{S}}({\hata_{0}},\bs,T)\big|{\hata_{t}}={\color{blue}{\hat{\ba}}}\big]
\right)+\Delta \hatpi_{t}({\color{blue}{\hat{\ba}}}|\bs),
\end{flalign}
\endgroup
where Eq.(\ref{app-01}) holds since: with the definition of $\hatpi_{t}({\color{blue}{\hat{\ba}}}|\bs)=:p({\color{blue}{\hat{\ba}}}|\bs,t)$, we obtain 
\begin{flalign}
\label{app-eq-01}
\int_{\R^{p}} \rho_{0}({\hata_{0}})
\left(
-\DIV\cdot \Big(p({\color{blue}{\hat{\ba}}}|{\hata_{0}},\bs,t){\color{blue}{\hat{\ba}}}\Big)
\right)
\dd {\hata_{0}}
=-\hatpi_{t}({\color{blue}{\hat{\ba}}}|\bs)\DIV\cdot{\color{blue}{\hat{\ba}}};
\end{flalign}
recall 
\begin{flalign}
\hatpi_{t}({\color{blue}{\hat{\ba}}}|\bs)=:p({\hata_{t}}={\color{blue}{\hat{\ba}}}|\bs,t),
\end{flalign}
we know
\begin{flalign}
&\rho_{0}({\hata_{0}}) p({\color{blue}{\hat{\ba}}}|{\hata_{0}},\bs,t)=p({\color{blue}{\hat{\ba}}},{\hata_{0}}|\bs,t), \tag*{$\blacktriangleright$ Bayes' theorem}\\
\label{app-eq-02}
p({\color{blue}{\hat{\ba}}},{\hata_{0}}|\bs,t)=&p({\color{blue}{\hat{\ba}}}|\bs,t)
p({\hata_{0}}|{\hata_{t}}={\color{blue}{\hat{\ba}}},\bs,t)=
\hatpi_{t}({\color{blue}{\hat{\ba}}}|\bs)p({\hata_{0}}|{\hata_{t}}={\color{blue}{\hat{\ba}}},\bs,t) ,
\end{flalign}
then we obtain 
\begingroup
\allowdisplaybreaks
\begin{flalign}
\label{app-eq-53}
&-\int_{\R^{p}} \rho_{0}({\hata_{0}})
\DIV\cdot \Big(p({\color{blue}{\hat{\ba}}}|{\hata_{0}},\bs,t) \bm{\hat \mathbf{S}}({\hata_{0}},\bs,T)\Big)
\dd {\hata_{0}}\\
=&-\int_{\R^{p}} \DIV\cdot \Big(p({\color{blue}{\hat{\ba}}},{\hata_{0}}|\bs,t) \bm{\hat \mathbf{S}}({\hata_{0}},\bs,T)\Big)
\dd {\hata_{0}}\\
=&-\int_{\R^{p}} \DIV\cdot \Big(\hatpi_{t}({\color{blue}{\hat{\ba}}}|\bs)p({\hata_{0}}|{\hata_{t}}={\color{blue}{\hat{\ba}}},\bs,t)  \bm{\hat \mathbf{S}}({\hata_{0}},\bs,T)\Big)
\dd {\hata_{0}}  \tag*{$\blacktriangleright$ see Eq.(\ref{app-eq-02})}\\
=&-\DIV \cdot 
\left(
\hatpi_{t}({\color{blue}{\hat{\ba}}}|\bs)\int_{\R^{p}}  p({\hata_{0}}|{\hata_{t}}={\color{blue}{\hat{\ba}}},\bs,t)  \bm{\hat \mathbf{S}}({\hata_{0}},\bs,T)
\dd {\hata_{0}}
\right)\\
\label{app-eq-03}
=&-\DIV \cdot 
\left(
\hatpi_{t}({\color{blue}{\hat{\ba}}}|\bs)\E_{{\hata_{0}}\sim\rho_{0|t}(\cdot|{\color{blue}{\hat{\ba}}},\bs) }\big[\bm{\hat \mathbf{S}}({\hata_{0}},\bs,T)\big|{\hata_{t}}={\color{blue}{\hat{\ba}}}\big]
\right),
\end{flalign}
\endgroup
where the last equation holds since 
\begin{flalign}
\int_{\R^{p}}  p({\hata_{0}}|{\hata_{t}}={\color{blue}{\hat{\ba}}},\bs,t)  \bm{\hat \mathbf{S}}({\hata_{0}},\bs,T)
\dd {\hata_{0}}
=\E_{{\hata_{0}}\sim\rho_{0|t}(\cdot|{\color{blue}{\hat{\ba}}},\bs) }\left[\bm{\hat \mathbf{S}}({\hata_{0}},\bs,T)|{\hata_{t}}={\color{blue}{\hat{\ba}}}\right].
\end{flalign}
Finally, consider (\ref{app-eq-04}) with (\ref{app-eq-01}) and (\ref{app-eq-03}), we conclude the Lemma \ref{lem-01}.
\end{proof}

We consider the time derivative of KL-divergence between the distribution $\hatpi_t(\cdot|\bs)$ and $\tildepi_{t}(\cdot|\bs)$, and decompose it as follows.
\begin{lemma}
\label{app-lem-03}
The time derivative of KL-divergence between the distribution $\hatpi_t(\cdot|\bs)$ and $\tildepi_{t}(\cdot|\bs)$ can be decomposed  as follows,
\begin{flalign}
\dfrac{\dd}{\dd t}\KL\big(\hatpi_t(\cdot|\bs)\|\tildepi_t(\cdot|\bs)\big)=\bigintsss_{\R^{p}}\dfrac{\partial\hatpi_t(\ba|\bs)}{\partial t}\log \dfrac{\hatpi_t(\ba|\bs)}{\tildepi_{t}(\ba|\bs)}\dd \ba
-\bigintsss_{\R^{p}}\dfrac{\hatpi_t(\ba|\bs)}{\tildepi_{t}(\ba|\bs)}\dfrac{\partial\tildepi_t(\ba|\bs)}{\partial t}\dd \ba.
\end{flalign}
\end{lemma}
\begin{proof}
We consider the time derivative of KL-divergence between the distribution $\hatpi_t(\cdot|\bs)$ and $\tildepi_{t}(\cdot|\bs)$, and we know
\begin{flalign}
\nonumber
\dfrac{\dd}{\dd t}\KL\big(\hatpi_t(\cdot|\bs)\|\tildepi_t(\cdot|\bs)\big)=&\dfrac{\dd}{\dd t}\bigintsss_{\R^{p}}\hatpi_t(\ba|\bs)\log \dfrac{\hatpi_t(\ba|\bs)}{\tildepi_{t}(\ba|\bs)}\dd \ba\\
\nonumber
=&\bigintsss_{\R^{p}}\dfrac{\partial\hatpi_t(\ba|\bs)}{\partial t}\log \dfrac{\hatpi_t(\ba|\bs)}{\tildepi_{t}(\ba|\bs)}\dd \ba
+\bigintsss_{\R^{p}}\tildepi_{t}(\ba|\bs)\dfrac{\partial }{\partial t}\left(\dfrac{\hatpi_t(\ba|\bs)}{\tildepi_{t}(\ba|\bs)}\right)\dd\ba\\
\nonumber
=&\bigintsss_{\R^{p}}\dfrac{\partial\hatpi_t(\ba|\bs)}{\partial t}\log \dfrac{\hatpi_t(\ba|\bs)}{\tildepi_{t}(\ba|\bs)}\dd \ba
+\bigintsss_{\R^{p}}\left(\bcancel{\dfrac{\partial\hatpi_t(\ba|\bs)}{\partial t}}- \dfrac{\hatpi_t(\ba|\bs)}{\tildepi_{t}(\ba|\bs)}\dfrac{\partial\tildepi_t(\ba|\bs)}{\partial t}\right)\dd \ba\\
\label{app-eq-05}
=&\bigintsss_{\R^{p}}\dfrac{\partial\hatpi_t(\ba|\bs)}{\partial t}\log \dfrac{\hatpi_t(\ba|\bs)}{\tildepi_{t}(\ba|\bs)}\dd \ba
-\bigintsss_{\R^{p}}\dfrac{\hatpi_t(\ba|\bs)}{\tildepi_{t}(\ba|\bs)}\dfrac{\partial\tildepi_t(\ba|\bs)}{\partial t}\dd \ba,
\end{flalign}
where the last equation holds since 
\[
\bigintsss_{\R^{p}}\dfrac{\partial\hatpi_t(\ba|\bs)}{\partial t}\dd\ba=\dfrac{\dd}{\dd t}\underbrace{\bigintsss_{\R^{p}}\hatpi_t(\ba|\bs)\dd\ba}_{=1}=0.
\]
That concludes the proof.
\end{proof}

The relative entropy and relative Fisher information $\FI \big(\hatpi_t(\cdot|\bs)\|\tildepi_t(\cdot|\bs)\big)$ can be rewritten as follows.
\begin{lemma}
\label{app-lemma-fisher-information}
The relative entropy and relative Fisher information $\FI \big(\hatpi_t(\cdot|\bs)\|\tildepi_t(\cdot|\bs)\big)$ can be rewritten as the following identity,
\begin{flalign}
\nonumber
\FI \big(\hatpi_t(\cdot|\bs)\|\tildepi_t(\cdot|\bs)\big)=\bigintsss_{\R^{p}}\left(\left \langle \bnabla \hatpi_t(\ba|\bs),\bnabla\log \dfrac{\hatpi_t(\ba|\bs)}{\tildepi_{t}(\ba|\bs)}\right\rangle-\left\langle\bnabla\dfrac{\hatpi_t(\ba|\bs)}{\tildepi_{t}(\ba|\bs)},\bnabla\tildepi_t(\ba|\bs)\right\rangle\right)\dd \ba.
\end{flalign}
\end{lemma}
\begin{proof}
We consider the following identity,
\begingroup
\allowdisplaybreaks
\begin{flalign}
\nonumber
&\bigintsss_{\R^{p}}\left(\left\langle\bnabla\dfrac{\hatpi_t(\ba|\bs)}{\tildepi_{t}(\ba|\bs)},\bnabla\tildepi_t(\ba|\bs)\right\rangle-\left \langle \bnabla \hatpi_t(\ba|\bs),\bnabla\log \dfrac{\hatpi_t(\ba|\bs)}{\tildepi_{t}(\ba|\bs)}\right\rangle\right)\dd \ba\\
\nonumber
=&\bigintsss_{\R^{p}}\Bigg(\left\langle\dfrac{\tildepi_{t}(\ba|\bs)\bnabla\hatpi_t(\ba|\bs)-\hatpi_{t}(\ba|\bs)\bnabla\tildepi_t(\ba|\bs)}{\tildepi_{t}(\ba|\bs)},\bnabla\log\tildepi_t(\ba|\bs)\right\rangle-\hatpi_t(\ba|\bs)\left \langle \bnabla\log \hatpi_t(\ba|\bs),\bnabla\log \dfrac{\hatpi_t(\ba|\bs)}{\tildepi_{t}(\ba|\bs)}\right\rangle\Bigg)\dd \ba
\\
\nonumber
=&\bigintsss_{\R^{p}}\hatpi_t(\ba|\bs)\left\langle\bnabla\log \dfrac{\hatpi_t(\ba|\bs)}{\tildepi_{t}(\ba|\bs)},\bnabla\log\tildepi_t(\ba|\bs)\right\rangle\dd\ba-\bigintsss_{\R^{p}}\hatpi_t(\ba|\bs)\left \langle \bnabla\log \hatpi_t(\ba|\bs),\bnabla\log \dfrac{\hatpi_t(\ba|\bs)}{\tildepi_{t}(\ba|\bs)}\right\rangle\dd \ba\\
\nonumber
=&-\bigintsss_{\R^{p}}\hatpi_t(\ba|\bs)\left\langle\bnabla\log \dfrac{\hatpi_t(\ba|\bs)}{\tildepi_{t}(\ba|\bs)},\bnabla\log \dfrac{\hatpi_t(\ba|\bs)}{\tildepi_{t}(\ba|\bs)}\right\rangle\dd\ba\\
\nonumber
=&-\bigintsss_{\R^{p}}\hatpi_t(\ba|\bs)\left\|\bnabla\log \dfrac{\hatpi_t(\ba|\bs)}{\tildepi_{t}(\ba|\bs)}\right\|_{2}^{2}\dd\ba=:-\FI \big(\hatpi_t(\cdot|\bs)\|\tildepi_t(\cdot|\bs)\big),
\end{flalign}
\endgroup
which concludes the proof.
\end{proof}
Lemma \ref{app-lemma-fisher-information} implies the following identity, which is useful later,
\begingroup
\allowdisplaybreaks
\begin{flalign}
\nonumber
&\bigintsss_{\R^{p}}\left(-\left\langle\bnabla\dfrac{\hatpi_t(\ba|\bs)}{\tildepi_{t}(\ba|\bs)},\bnabla\tildepi_t(\ba|\bs)\right\rangle-\left \langle \bnabla \hatpi_t(\ba|\bs),\bnabla\log \dfrac{\hatpi_t(\ba|\bs)}{\tildepi_{t}(\ba|\bs)}\right\rangle\right)\dd \ba\\
\nonumber
=&\bigintsss_{\R^{p}}\left(\left\langle\bnabla\dfrac{\hatpi_t(\ba|\bs)}{\tildepi_{t}(\ba|\bs)},\bnabla\tildepi_t(\ba|\bs)\right\rangle-\left \langle \bnabla \hatpi_t(\ba|\bs),\bnabla\log \dfrac{\hatpi_t(\ba|\bs)}{\tildepi_{t}(\ba|\bs)}\right\rangle-2\left\langle\bnabla\dfrac{\hatpi_t(\ba|\bs)}{\tildepi_{t}(\ba|\bs)},\bnabla\tildepi_t(\ba|\bs)\right\rangle\right)\dd \ba\\
=&-\FI \big(\hatpi_t(\cdot|\bs)\|\tildepi_t(\cdot|\bs)\big)-2\bigintsss_{\R^{p}}\hatpi_t(\ba|\bs)\left\langle\bnabla\log \dfrac{\hatpi_t(\ba|\bs)}{\tildepi_{t}(\ba|\bs)},\bnabla\log\tildepi_t(\ba|\bs)\right\rangle\dd\ba.
\end{flalign}
\endgroup

\begin{lemma}
\label{app-lemma-05}
The time derivative of KL-divergence between the distribution $\hatpi_t(\cdot|\bs)$ and $\tildepi_{t}(\cdot|\bs)$ can be further decomposed  as follows,
\begin{flalign}
&\dfrac{\dd}{\dd t}\KL\big(\hatpi_t(\cdot|\bs)\|\tildepi_t(\cdot|\bs)\big)=-\FI \big(\hatpi_t(\cdot|\bs)\|\tildepi_t(\cdot|\bs)\big)\\
&~~~~~~~~~~+2\bigintsss_{\R^{p}}\bigintsss_{\R^{p}}\rho_{0,t}({\hata_{0}},\ba|\bs)\left\langle\bnabla\log \dfrac{\hatpi_t(\ba|\bs)}{\tildepi_{t}(\ba|\bs)},\bm{\hat \mathbf{S}}({\hata_{0}},\bs,T)-\bnabla\log\tildepi_t(\ba|\bs)\right\rangle\dd\ba\dd {\hata_{0}}.
\end{flalign}
\end{lemma}

\begin{proof}
According to Lemma \ref{app-lem-03}, we need to consider the two terms in (\ref{app-eq-05}) correspondingly.

\textbf{\underline{First term in (\ref{app-eq-05})}.}
Recall Lemma \ref{lem-01}, we know
\begin{flalign}
\nonumber
\dfrac{\partial}{\partial t}\hatpi_t(\ba|\bs)=&-{\hatpi_t}(\ba|\bs)\DIV \cdot 
\left(\ba+2\E_{{\hata_{0}}\sim\rho_{0|t}(\cdot|\ba,\bs) }\big[\bm{\hat \mathbf{S}}({\hata_{0}},\bs,T)\big|{\hata_{t}}={\ba}\big]
\right)+\Delta \hatpi_t(\ba|\bs)\\
\nonumber
\overset{(\ref{laplacian-operator})}=&\DIV \cdot \left(-
\left(\ba+2\E_{{\hata_{0}}\sim\rho_{0|t}(\cdot|\ba,\bs) }\big[\bm{\hat \mathbf{S}}({\hata_{0}},\bs,T)\big|{\hata_{t}}={\ba}\big]
\right){\hatpi_t}(\ba|\bs)+\bnabla \hatpi_t(\ba|\bs)
\right).
\end{flalign}
To short the expression, we define a notation $\bg_{t}(\cdot,\cdot):\calS\times\calA\rightarrow\R^{p}$ as follows,
\[\bg_{t}(\bs,\ba)=:\ba+2\E_{{\hata_{0}}\sim\rho_{0|t}(\cdot|\ba,\bs) }\big[\bm{\hat \mathbf{S}}({\hata_{0}},\bs,T)\big|{\hata_{t}}={\ba}\big],
\]
then we rewrite the distribution at time $t$ along interpolation SDE (\ref{def:diffusion-policy-sde-reverse-process-s-01}) as follows,
\begin{flalign}
\nonumber
\dfrac{\partial}{\partial t}\hatpi_t(\ba|\bs)=\DIV \cdot \Big(-\bg_{t}(\bs,\ba){\hatpi_t}(\ba|\bs)+\bnabla \hatpi_t(\ba|\bs)\Big).
\end{flalign}

We consider the first term in (\ref{app-eq-05}), according to integration by parts formula (\ref{integration-by-parts-formula}), we know
\begin{flalign}
\nonumber
\bigintsss_{\R^{p}}\dfrac{\partial\hatpi_t(\ba|\bs)}{\partial t}\log \dfrac{\hatpi_t(\ba|\bs)}{\tildepi_{t}(\ba|\bs)}\dd \ba
=&\bigintsss_{\R^{p}}\DIV \cdot \Big(-\bg_{t}(\bs,\ba){\hatpi_t}(\ba|\bs)+\bnabla \hatpi_t(\ba|\bs)\Big)\log \dfrac{\hatpi_t(\ba|\bs)}{\tildepi_{t}(\ba|\bs)}\dd \ba\\
\nonumber
\overset{(\ref{integration-by-parts-formula})}=&\bigintsss_{\R^{p}}\left \langle \bg_{t}(\bs,\ba){\hatpi_t}(\ba|\bs)-\bnabla \hatpi_t(\ba|\bs),\bnabla\log \dfrac{\hatpi_t(\ba|\bs)}{\tildepi_{t}(\ba|\bs)}\right\rangle\dd \ba\\
\nonumber
=&\bigintsss_{\R^{p}}{\hatpi_t}(\ba|\bs)\left \langle \bg_{t}(\bs,\ba),\bnabla\log \dfrac{\hatpi_t(\ba|\bs)}{\tildepi_{t}(\ba|\bs)}\right\rangle\dd \ba\\
\label{app-eq-06}
&~~~~~~~~-\bigintsss_{\R^{p}}\left \langle \bnabla \hatpi_t(\ba|\bs),\bnabla\log \dfrac{\hatpi_t(\ba|\bs)}{\tildepi_{t}(\ba|\bs)}\right\rangle\dd \ba.
\end{flalign}

\textbf{\underline{Second term in (\ref{app-eq-05})}.} According to the Kolmogorov backward equation, we know
\begin{flalign}
\dfrac{\partial\tildepi_t(\ba|\bs)}{\partial t}=-\DIV\cdot\left(\tildepi_t(\ba|\bs)\ba\right)-\Delta\tildepi_t(\ba|\bs)
=-\DIV\cdot\Big(\tildepi_t(\ba|\bs)\ba+\bnabla\tildepi_t(\ba|\bs)\Big),
\end{flalign}
then we obtain
\begin{flalign}
\nonumber
\bigintsss_{\R^{p}}\dfrac{\hatpi_t(\ba|\bs)}{\tildepi_{t}(\ba|\bs)}\dfrac{\partial\tildepi_t(\ba|\bs)}{\partial t}\dd \ba=&-\bigintsss_{\R^{p}}\dfrac{\hatpi_t(\ba|\bs)}{\tildepi_{t}(\ba|\bs)}\DIV\cdot\Big(\tildepi_t(\ba|\bs)\ba+\bnabla\tildepi_t(\ba|\bs)\Big)\dd\ba\\
\nonumber
\overset{(\ref{integration-by-parts-formula})}=&\bigintsss_{\R^{p}}\left\langle\bnabla\dfrac{\hatpi_t(\ba|\bs)}{\tildepi_{t}(\ba|\bs)},\tildepi_t(\ba|\bs)\ba+\bnabla\tildepi_t(\ba|\bs)\right\rangle\dd\ba\\
\label{app-eq-07}
=&\bigintsss_{\R^{p}}\tildepi_t(\ba|\bs)\left\langle\bnabla\dfrac{\hatpi_t(\ba|\bs)}{\tildepi_{t}(\ba|\bs)},\ba\right\rangle\dd\ba
+
\bigintsss_{\R^{p}}\left\langle\bnabla\dfrac{\hatpi_t(\ba|\bs)}{\tildepi_{t}(\ba|\bs)},\bnabla\tildepi_t(\ba|\bs)\right\rangle\dd\ba.
\end{flalign}

\textbf{\underline{Time derivative of KL-divergence}.} We consider the next identity,
\begingroup
\allowdisplaybreaks
\begin{flalign}
\nonumber
&\bigintsss_{\R^{p}}\left({\hatpi_t}(\ba|\bs)\left \langle \bg_{t}(\bs,\ba),\bnabla\log \dfrac{\hatpi_t(\ba|\bs)}{\tildepi_{t}(\ba|\bs)}\right\rangle-\tildepi_t(\ba|\bs)\left\langle\bnabla\dfrac{\hatpi_t(\ba|\bs)}{\tildepi_{t}(\ba|\bs)},\ba\right\rangle\right)\dd \ba\\
\nonumber
=&\bigintsss_{\R^{p}}\left({\hatpi_t}(\ba|\bs)\left \langle \bg_{t}(\bs,\ba),\bnabla\log \dfrac{\hatpi_t(\ba|\bs)}{\tildepi_{t}(\ba|\bs)}\right\rangle-\hatpi_t(\ba|\bs)\dfrac{\tildepi_t(\ba|\bs)}{\hatpi_t(\ba|\bs)}\left\langle\bnabla\dfrac{\hatpi_t(\ba|\bs)}{\tildepi_{t}(\ba|\bs)},\ba\right\rangle\right)\dd \ba\\
\nonumber
=&2\bigintsss_{\R^{p}}{\hatpi_t}(\ba|\bs)\left \langle\E_{{\hata_{0}}\sim\rho_{0|t}(\cdot|\ba,\bs) }\big[\bm{\hat \mathbf{S}}({\hata_{0}},\bs,T)\big|{\hata_{t}}={\ba}\big],\bnabla\log \dfrac{\hatpi_t(\ba|\bs)}{\tildepi_{t}(\ba|\bs)}\right\rangle\dd \ba,
\end{flalign}
\endgroup
then according to (\ref{app-eq-05}), and with the results (\ref{app-eq-06}), (\ref{app-eq-07}), we obtain
\begin{flalign}
\label{app-eq-09}
&\dfrac{\dd}{\dd t}\KL\big(\hatpi_t(\cdot|\bs)\|\tildepi_t(\cdot|\bs)\big)=-\FI \big(\hatpi_t(\cdot|\bs)\|\tildepi_t(\cdot|\bs)\big)\\
\nonumber
&~~~~+2\bigintsss_{\R^{p}}\hatpi_t(\ba|\bs)\left\langle\bnabla\log \dfrac{\hatpi_t(\ba|\bs)}{\tildepi_{t}(\ba|\bs)},\E_{{\hata_{0}}\sim\rho_{0|t}(\cdot|\ba,\bs) }\big[\bm{\hat \mathbf{S}}({\hata_{0}},\bs,T)\big|{\hata_{t}}={\ba}\big]-\bnabla\log\tildepi_t(\ba|\bs)\right\rangle\dd\ba.
\end{flalign}
Furthermore, we consider
\begin{flalign}
\nonumber
&\bigintsss_{\R^{p}}\hatpi_t(\ba|\bs)\left\langle\bnabla\log \dfrac{\hatpi_t(\ba|\bs)}{\tildepi_{t}(\ba|\bs)},\E_{{\hata_{0}}\sim\rho_{0|t}(\cdot|\ba,\bs) }\big[\bm{\hat \mathbf{S}}({\hata_{0}},\bs,T)\big|{\hata_{t}}={\ba}\big]-\bnabla\log\tildepi_t(\ba|\bs)\right\rangle\dd\ba\\
\label{app-eq-08}
=&\bigintsss_{\R^{p}}\bigintsss_{\R^{p}}\rho_{0,t}({\hata_{0}},\ba|\bs)\left\langle\bnabla\log \dfrac{\hatpi_t(\ba|\bs)}{\tildepi_{t}(\ba|\bs)},\bm{\hat \mathbf{S}}({\hata_{0}},\bs,T)-\bnabla\log\tildepi_t(\ba|\bs)\right\rangle\dd\ba\dd {\hata_{0}},
\end{flalign}
where Eq.(\ref{app-eq-08}) holds due to $\rho_{0,t}({\hata_{0}},{\hata_{t}}|\bs)$ denotes the joint distribution of $({\hata_{0}},{\hata_{t}})$ conditional on the state $\bs$,  which can be written in terms of the conditionals and marginals as follows,
\begin{flalign}
\label{app-eq-17}
\rho_{0|t}({\hata_{0}}|{\hata_{t}},\bs)=\dfrac{\rho_{0,t}({\hata_{0}},{\hata_{t}}|\bs)}{p_t({\hata_{t}}|\bs)}
=\dfrac{\rho_{0,t}({\hata_{0}},{\hata_{t}}|\bs)}{\hatpi_{t}({\hata_{t}}|\bs)};
\end{flalign}
and in Eq.(\ref{app-eq-08}), we denote ${\hata_{t}}={\ba}$.

Finally, combining (\ref{app-eq-09}) and (\ref{app-eq-08}), we obtain the following equation,
\begin{flalign}
&\dfrac{\dd}{\dd t}\KL\big(\hatpi_t(\cdot|\bs)\|\tildepi_t(\cdot|\bs)\big)=-\FI \big(\hatpi_t(\cdot|\bs)\|\tildepi_t(\cdot|\bs)\big)\\
&~~~~~~~~~~+2\bigintsss_{\R^{p}}\bigintsss_{\R^{p}}\rho_{0,t}({\hata_{0}},\ba|\bs)\left\langle\bnabla\log \dfrac{\hatpi_t(\ba|\bs)}{\tildepi_{t}(\ba|\bs)},\bm{\hat \mathbf{S}}({\hata_{0}},\bs,T)-\bnabla\log\tildepi_t(\ba|\bs)\right\rangle\dd\ba\dd {\hata_{0}},
\end{flalign}
which concludes the proof.
\end{proof}

\begin{lemma}
\label{app-lemma-08}
The time derivative of KL-divergence between the distribution $\hatpi_t(\cdot|\bs)$ and $\tildepi_{t}(\cdot|\bs)$ is bounded as follows,
\begin{flalign}
\nonumber
&\dfrac{\dd}{\dd t}\KL\big(\hatpi_t(\cdot|\bs)\|\tildepi_t(\cdot|\bs)\big)\\
\leq&-\dfrac{3}{4}\FI \big(\hatpi_t(\cdot|\bs)\|\tildepi_t(\cdot|\bs)\big)+4\bigintsss_{\R^{p}}\bigintsss_{\R^{p}}\rho_{0,t}({\hata_{0}},\ba|\bs)\Big\|\bm{\hat \mathbf{S}}({\hata_{0}},\bs,T)-\bnabla\log\tildepi_t(\ba|\bs)\Big\|_{2}^{2}
\dd\ba\dd {\hata_{0}}.
\end{flalign}
\end{lemma}
\begin{proof}
First, we consider
\begingroup
\allowdisplaybreaks
\begin{flalign}
\nonumber
&\bigintsss_{\R^{p}}\bigintsss_{\R^{p}}\rho_{0,t}(\hata_0,\ba|\bs)\left\langle\bnabla\log \dfrac{\hatpi_t(\ba|\bs)}{\tildepi_{t}(\ba|\bs)},\bm{\hat \mathbf{S}}({\hata_{0}},\bs,T)-\bnabla\log\tildepi_t(\ba|\bs)\right\rangle\dd\ba\dd {\hata_{0}}\\
\label{app-eq-13}
\leq&\bigintsss_{\R^{p}}\bigintsss_{\R^{p}}\rho_{0,t}({\hata_{0}},\ba|\bs)\left(2\Big\|\bm{\hat \mathbf{S}}({\hata_{0}},\bs,T)-\bnabla\log\tildepi_t(\ba|\bs)\Big\|_{2}^{2}+\dfrac{1}{8}\left\|\bnabla\log \dfrac{\hatpi_t(\ba|\bs)}{\tildepi_{t}(\ba|\bs)}\right\|_{2}^{2}
\right)
\dd\ba\dd {\hata_{0}}\\
\label{app-eq-14}
=&2\bigintsss_{\R^{p}}\bigintsss_{\R^{p}}\rho_{0,t}({\hata_{0}},\ba|\bs)\Big\|\bm{\hat \mathbf{S}}({\hata_{0}},\bs,T)-\bnabla\log\tildepi_t(\ba|\bs)\Big\|_{2}^{2}
\dd\ba\dd {\hata_{0}}+\dfrac{1}{8}\FI \big(\hatpi_t(\cdot|\bs)\|\tildepi_t(\cdot|\bs)\big),
\end{flalign}
\endgroup
where Eq.(\ref{app-eq-13}) holds since we consider $\langle \ba,\bb\rangle\leq 2\|\ba\|^2+\frac{1}{8}\|\bb\|^2$.

Then, according to Lemma \ref{app-lemma-05}, we obtain
\begin{flalign}
\nonumber
&\dfrac{\dd}{\dd t}\KL\big(\hatpi_t(\cdot|\bs)\|\tildepi_t(\cdot|\bs)\big)=-\FI \big(\hatpi_t(\cdot|\bs)\|\tildepi_t(\cdot|\bs)\big)\\
\nonumber
&~~~~~~~~~~+2\bigintsss_{\R^{p}}\bigintsss_{\R^{p}}\rho_{0,t}({\hata_{0}},\ba|\bs)\left\langle\bnabla\log \dfrac{\hatpi_t(\ba|\bs)}{\tildepi_{t}(\ba|\bs)},\bm{\hat \mathbf{S}}({\hata_{0}},\bs,T)-\bnabla\log\tildepi_t(\ba|\bs)\right\rangle\dd\ba\dd {\hata_{0}}\\
\label{app-eq-15}
\leq&-\dfrac{3}{4}\FI \big(\hatpi_t(\cdot|\bs)\|\tildepi_t(\cdot|\bs)\big)+4\bigintsss_{\R^{p}}\bigintsss_{\R^{p}}\rho_{0,t}({\hata_{0}},\ba|\bs)\Big\|\bm{\hat \mathbf{S}}({\hata_{0}},\bs,T)-\bnabla\log\tildepi_t(\ba|\bs)\Big\|_{2}^{2}
\dd\ba\dd {\hata_{0}},
\end{flalign}
which concludes the proof.
\end{proof}

Before we provide further analysis to show the boundedness of (\ref{app-eq-15})., we need to consider SDE (\ref{def:diffusion-policy-sde-reverse-process-s}).
Let $h>0$ be the step-size, assume $K=\frac{T}{h}\in\N$, and $t_{k}=:hk$, $k=0,1,\cdots,K$.
SDE (\ref{def:diffusion-policy-sde-reverse-process-s}) considers as follows, for $t\in[hk,h(k+1)]$,
\begin{flalign}
\label{def:app-diffusion-policy-sde-reverse-process-s}
\dd \hata_t=\left(\hata_t+2 \bm{\hat \mathbf{S}}(\hata_{t_k},\bs,T-t_{k})\right)\dd t+\sqrt{2}\dd \bw_{t},
\end{flalign}

Recall the SDE (\ref{def:app-diffusion-policy-sde-reverse-process-s}), in this section, we only consider $k=0$, and we obtain the following SDE,
\begin{flalign}
\label{def:diffusion-policy-sde-reverse-process-s-app}
\dd \hata_t=\left(\hata_t+2 \bm{\hat \mathbf{S}}(\hata_0,\bs,T)\right)\dd t+\sqrt{2}\dd \bw_{t},
\end{flalign}
where $\bw_{t}$ is the standard Wiener process starting at $\bw_{0}=\bm{0}$, and $t$ is from $0$ to $h$.

Integration with (\ref{def:diffusion-policy-sde-reverse-process-s-app}), we obtain  
\begin{flalign}
\hata_t-\hata_0=(\mathrm{e}^{t}-1)\left(\hata_{0}+2 \bm{\hat \mathbf{S}}(\hata_0,\bs,T)\right)+\sqrt{2}\int^{t}_{0}\mathrm{e}^{t}\dd \bw_{t},
\end{flalign}
which implies 
\begin{flalign}
\label{app-eq-12}
\hata_t=\mathrm{e}^{t}\hata_0+2(\mathrm{e}^{t}-1) \bm{\hat \mathbf{S}}(\hata_0,\bs,T)+\sqrt{\mathrm{e}^{t}-1}\bz,~\bz\sim\calN(\bm{0},\bI).
\end{flalign}

\begin{lemma}
\label{app-lem-09}
Under Assumption \ref{assumption-score}, for all $0\leq t\leq\frac{1}{12L_{s}}$, then the following holds,
\begin{flalign}
\nonumber
&\bigintsss_{\R^{p}}\bigintsss_{\R^{p}}\rho_{0,t}({\hata_{0}},\ba|\bs)\Big\|\bm{\hat \mathbf{S}}({\hata_{0}},\bs,T)-\bnabla\log\tildepi_t(\ba|\bs)\Big\|_{2}^{2}
\dd\ba\dd {\hata_{0}}\\
\nonumber
\leq&36pt(1+t)L^{2}_{s}+
144t^{2}L^{2}_{s}\bigintsss_{\R^{p}}\hatpi_{t}(\ba|\bs)\left(\Big\|\bm{\hat \mathbf{S}}(\ba,\bs,T)-\bnabla\log\tildepi_t(\ba|\bs)\Big\|_{2}^{2}
+
\Big\|\bnabla\log\tildepi_t(\ba|\bs)\Big\|_{2}^{2}\right)\dd\ba,
\end{flalign}
where $\hata_t$ updated according to (\ref{app-eq-12}).
\end{lemma}
\begin{proof}
See Section \ref{sec-proof-lem-09}.
\end{proof}

\subsection{Proof for Result at Reverse Time $k = 0$}

\begin{proof}
According to the definition of diffusion policy, we know $\tildepi_{t}(\cdot|\bs)=\barpi_{T-t}(\cdot|\bs)$. Then according to Proposition  \ref{app-propo-02}, we know $\tildepi_{t}(\cdot|\bs)$ is $\nu_{T-t}$-LSI, where
\[
\nu_{T-t}=\dfrac{\nu}{\nu+(1-\nu)\mathrm{e}^{-2(T-t)}}.
\]
Since we consider the time-step $0\leq t\leq T$, then
\begin{flalign}
 \label{app-eq-31}
\nu_{T-t}=\dfrac{\nu}{\nu+(1-\nu)\mathrm{e}^{-2(T-t)}}\ge 1,~\forall t\in[0,T].
\end{flalign}

According to Proposition \ref{app-propo-04}, we know under Assumption \ref{assumption-score}, $\bnabla\log {\color{orange}{\tilde{\pi}_{t}}}(\cdot|\bs)$ is $L_{p}\mathrm{e}^{t}$-Lipschitz on the time interval $[0,\mathtt{T}_{0}]$, where
 \[
\mathtt{T}_{0}=:\sup_{t\ge0}\left\{t:1-\mathrm{e}^{-2t}\leq \frac{\mathrm{e}^{t}}{L_p}\right\}.
\]

Then according to Proposition \ref{app-propo-03}, we obtain
\begin{flalign}
\label{app-eq-26}
\bigintsss_{\R^{p}}\hatpi_{t}(\ba|\bs)
\Big\|\bnabla\log\tildepi_t(\ba|\bs)\Big\|_{2}^{2}\dd\ba\leq\dfrac{4L_{p}^2\mathrm{e}^{2t}}{\nu_{T-t}}\KL\left(\hatpi_{t}(\cdot|\bs)\|\tildepi_t(\cdot|\bs)\right)+2pL_{p}\mathrm{e}^{t}.
\end{flalign}

Furthermore, according to Donsker-Varadhan representation (see Section \ref{sec-re-kl}), let \[ f(\ba)=:{\beta_t}\Big\|\bm{\hat \mathbf{S}}(\ba,\bs,T)-\bnabla\log\tildepi_t(\ba|\bs)\Big\|_{2}^{2}
,\] 
the positive constant ${\beta_t}$ will be special later, see Eq.(\ref{app-eq-28}).
With the result (\ref{donsker-varadhan-representation}), we know
\begin{flalign}
\nonumber
 \KL\big(\hatpi_t(\cdot|\bs)\|\tildepi_t(\cdot|\bs)\big)\ge\int_{\R^{p}}\hatpi_t(\ba|\bs)f(\ba)\dd\ba-\log\int_{\R^{p}}\tildepi_t(\ba|\bs)\exp(f(\ba))\dd\ba,
 \end{flalign}
 which implies 
 \begin{flalign}
\nonumber
&\bigintsss_{\R^{p}}\hatpi_t(\ba|\bs)\Big\|\bm{\hat \mathbf{S}}(\ba,\bs,T)-\bnabla\log\tildepi_t(\ba|\bs)\Big\|_{2}^{2}\dd\ba\\
\nonumber
\leq&\dfrac{1}{{\beta_t}}  \KL\big(\hatpi_t(\cdot|\bs)\|\tildepi_t(\cdot|\bs)\big)+\dfrac{1}{{\beta_t}} \log\bigintsss_{\R^{p}}\tildepi_t(\ba|\bs)\exp\left(\Big\|\bm{\hat \mathbf{S}}(\ba,\bs,T)-\bnabla\log\tildepi_t(\ba|\bs)\Big\|_{2}^{2}\right)\dd\ba\\
\label{app-eq-25}
=&\dfrac{1}{{\beta_t}}  \KL\big(\hatpi_t(\cdot|\bs)\|\tildepi_t(\cdot|\bs)\big)+\dfrac{1}{{\beta_t}} \log\E_{\ba\sim\tildepi_t(\cdot|\bs)}\left[\exp\left\|\bm{\hat \mathbf{S}}(\ba,\bs,T)-\bnabla\log\tildepi_t(\ba|\bs)\right\|_{2}^{2}\right].
 \end{flalign}

 Finally, according to Lemma \ref{app-lemma-08}-\ref{app-lem-09}, Eq.(\ref{app-eq-26})-(\ref{app-eq-25}), we obtain 
 \begingroup
\allowdisplaybreaks
 \begin{flalign}
\nonumber
&\dfrac{\dd}{\dd t}\KL\big(\hatpi_t(\cdot|\bs)\|\tildepi_t(\cdot|\bs)\big)\\
\nonumber
\overset{(\ref{app-eq-15})}\leq&-\dfrac{3}{4}\FI \big(\hatpi_t(\cdot|\bs)\|\tildepi_t(\cdot|\bs)\big)+4\bigintsss_{\R^{p}}\bigintsss_{\R^{p}}\rho_{0,t}({\hata_{0}},\ba|\bs)\Big\|\bm{\hat \mathbf{S}}({\hata_{0}},\bs,T)-\bnabla\log\tildepi_t(\ba|\bs)\Big\|_{2}^{2}
\dd\ba\dd {\hata_{0}}\\
\nonumber
\overset{\text{Lemma}~\ref{app-lem-09}}\leq&-\dfrac{3}{4}\FI \big(\hatpi_t(\cdot|\bs)\|\tildepi_t(\cdot|\bs)\big)+576t^{2}L^{2}_{s}\bigintsss_{\R^{p}}\hatpi_{t}(\ba|\bs)\Big\|\bnabla\log\tildepi_t(\ba|\bs)\Big\|_{2}^{2}\dd\ba\\
\nonumber
&~~~~~~~~~~~~~~~~~~~+576t^{2}L^{2}_{s}\bigintsss_{\R^{p}}\hatpi_{t}(\ba|\bs)\Big\|\bm{\hat \mathbf{S}}(\ba,\bs,T)-\bnabla\log\tildepi_t(\ba|\bs)\Big\|_{2}^{2}\dd\ba\\
\nonumber
\leq&-\dfrac{3}{4}\FI \big(\hatpi_t(\cdot|\bs)\|\tildepi_t(\cdot|\bs)\big)+576t^{2}L^{2}_{s}\left(\dfrac{4L_{p}^2\mathrm{e}^{2t}}{\nu_{T-t}}\KL\left(\hatpi_{t}(\cdot|\bs)\|\tildepi_t(\cdot|\bs)\right)+2pL_{p}\mathrm{e}^{t}\right) \tag*{$\blacktriangleright$ due to Eq.(\ref{app-eq-26})}\\
\nonumber
&~~+\dfrac{576t^{2}L^{2}_{s}}{{\beta_t}} \bigg( \KL\big(\hatpi_t(\cdot|\bs)\|\tildepi_t(\cdot|\bs)\big)+ \log\E_{\ba\sim\tildepi_t(\cdot|\bs)}\left[\exp\left\|\bm{\hat \mathbf{S}}(\ba,\bs,T)-\bnabla\log\tildepi_t(\ba|\bs)\right\|_{2}^{2}\right]\bigg)\tag*{$\blacktriangleright$ due to Eq.(\ref{app-eq-25})}\\
\nonumber
=&-\dfrac{3}{4}\FI \big(\hatpi_t(\cdot|\bs)\|\tildepi_t(\cdot|\bs)\big)+576t^{2}L^{2}_{s}\left(\dfrac{4L_{p}^2\mathrm{e}^{2t}}{\nu_{T-t}}+\dfrac{1}{{\beta_t}}\right)\KL\left(\hatpi_{t}(\cdot|\bs)\|\tildepi_t(\cdot|\bs)\right)\\
\nonumber
&~+\dfrac{576t^{2}L^{2}_{s}}{{\beta_t}}\log\E_{\ba\sim\tildepi_t(\cdot|\bs)}\left[\exp\left\|\bm{\hat \mathbf{S}}(\ba,\bs,T)-\bnabla\log\tildepi_t(\ba|\bs)\right\|_{2}^{2}\right]
+1152t^{2}pL^{2}_{s}L_{p}\mathrm{e}^{t}\\
\nonumber
\leq&\left(576t^{2}L^{2}_{s}\left(\dfrac{4L_{p}^2\mathrm{e}^{2t}}{\nu_{T-t}}+\dfrac{1}{{\beta_t}}\right)-\dfrac{3}{2}\nu\right)\KL\left(\hatpi_{t}(\cdot|\bs)\|\tildepi_t(\cdot|\bs)\right)
+\dfrac{576t^{2}L^{2}_{s}}{{\beta_t}}\epsilon_{\mathrm{score}}
+1152t^{2}pL^{2}_{s}L_{p}\mathrm{e}^{t}
\tag*{$\blacktriangleright$ due to Assumption \ref{assumption-policy-calss}}\\
\nonumber
=&\left(576t^{2}L^{2}_{s}\left(4c_{t}+\dfrac{1}{{\beta_t}}\right)-\dfrac{3}{2}\nu\right)\KL\left(\hatpi_{t}(\cdot|\bs)\|\tildepi_t(\cdot|\bs)\right)
+\dfrac{576t^{2}L^{2}_{s}}{{\beta_t}}\epsilon_{\mathrm{score}}
+1152t^{2}pL^{2}_{s}L_{p}\mathrm{e}^{t}
\tag*{$\blacktriangleright$ due to $\frac{L_{p}^2\mathrm{e}^{2t}}{\nu_{T-t}}=:c_{t}$}\\
\label{app-eq-27}
\overset{(\ref{app-eq-28})}=&-\dfrac{\nu}{4}\KL\left(\hatpi_{t}(\cdot|\bs)\|\tildepi_t(\cdot|\bs)\right)
+\dfrac{576t^{2}L^{2}_{s}}{{\beta_t}}\epsilon_{\mathrm{score}}
+1152t^{2}pL^{2}_{s}L_{p}\mathrm{e}^{t}\\
\label{app-eq-29}
\overset{(\ref{app-eq-30})}\leq&-\dfrac{\nu}{4}\KL\left(\hatpi_{t}(\cdot|\bs)\|\tildepi_t(\cdot|\bs)\right)
+\dfrac{5}{4}\nu\epsilon_{\mathrm{score}}
+1152t^{2}pL^{2}_{s}L_{p}\mathrm{e}^{t}
\end{flalign}
 \endgroup
where 
 \begin{flalign}
 \label{app-eq-46}
 \epsilon_{\mathrm{score}}=\sup_{(k,t)\in[K]\times[kh,(k+1)h]}\left\{\log\E_{\ba\sim\tildepi_t(\cdot|\bs)}\left[\exp\left\|\bm{\hat \mathbf{S}}(\ba,\bs,T-hk)-\bnabla\log\tildepi_t(\ba|\bs)\right\|_{2}^{2}\right]\right\};
 \end{flalign}
 Eq.(\ref{app-eq-27}) holds since we set ${\beta_t}$ as follows,
 we set 
$
 576t^{2}L^{2}_{s}\left(4c_{t}+\frac{1}{{\beta_t}}\right)=\frac{5\nu}{4},
$
i.e,
 \begin{flalign}
 \label{app-eq-28}
\dfrac{1}{{\beta_t}}=\dfrac{5\nu}{2304t^{2}L^{2}_{s}}-4c_{t};
 \end{flalign}
where Eq.(\ref{app-eq-29}) holds since 
 \begin{flalign}
 \label{app-eq-30}
\dfrac{576t^{2}L^{2}_{s}}{{\beta_t}}=576t^{2}L^{2}_{s}\left(\dfrac{5\nu}{2304t^{2}L^{2}_{s}}-4c_{t}\right)\leq\dfrac{5\nu}{4}.
 \end{flalign}
 
 Now, we consider the time-step $t$ keeps the constant $\beta_t$ positive, it is sufficient to consider the next condition due to the property (\ref{app-eq-31}),
 \begin{flalign}
 \label{app-eq-32}
\dfrac{5\nu}{2304t^{2}L^{2}_{s}}\ge 4L^{2}_{p}\mathrm{e}^{2t},
 \end{flalign}
which implies $ t\mathrm{e}^{t}\leq\dfrac{\sqrt{5\nu}}{96L_{s}L_{p}}.$

Formally, we define a notation 
\begin{flalign}
 \label{app-eq-43}
 \tau_{0}&=:\sup\left\{t:t\mathrm{e}^{t}\leq\dfrac{\sqrt{5\nu}}{96L_{s}L_{p}}\right\},\\
 \label{app-eq-44}
  \tau&=:\min\left\{ \tau_{0},\mathtt{T}_{0},\dfrac{1}{12L_{s}}\right\}.
\end{flalign}
Then with result of Eq.(\ref{app-eq-32}), if $0\leq t\leq h\leq \tau$, we rewrite Eq.(\ref{app-eq-29}) as follows,
 \begin{flalign}
\nonumber
\dfrac{\dd}{\dd t}\KL\big(\hatpi_t(\cdot|\bs)\|\tildepi_t(\cdot|\bs)\big)\leq&
-\dfrac{\nu}{4}\KL\left(\hatpi_{t}(\cdot|\bs)\|\tildepi_t(\cdot|\bs)\right)
+\dfrac{5}{4}\nu\epsilon_{\mathrm{score}}
+12pL_{s}\sqrt{5\nu}t\\
\nonumber
=&-\dfrac{\nu}{4}\KL\left(\hatpi_{t}(\cdot|\bs)\|\tildepi_t(\cdot|\bs)\right)+12pL_{s}\sqrt{5\nu}t\\
+\dfrac{5}{4}\nu\sup_{(k,t)\in[K]\times[t_{k},t_{k+1}]}&\left\{\log\E_{\ba\sim\tildepi_t(\cdot|\bs)}\left[\exp\left\|\bm{\hat \mathbf{S}}(\ba,\bs,T-hk)-\bnabla\log\tildepi_t(\ba|\bs)\right\|_{2}^{2}\right]\right\},
\end{flalign}
which concludes the proof.
\end{proof}

\begin{remark}
The result of Proposition \ref{app-propo-04} only depends on Assumption \ref{assumption-score}, thus the result (\ref{app-eq-26}) does not depend on additional assumption of the uniform $L$-smooth of $\log\tildepi_{t}$ on the time interval $[0,T]$, e.g., \cite{wibisono2022convergence}. Instead of  the uniform $L$-smooth of $\log\tildepi_{t}$, we consider the $L_{p}\mathrm{e}^{t}$-Lipschitz on the time interval $[0,\mathtt{T}_{0}]$, which is one of the difference between our proof and \citep{wibisono2022convergence}. Although we obtain a similar convergence rate from the view of Langevin-based algorithms, we need a weak condition.
\end{remark}

\subsection{Proof for Result at Arbitrary Reverse Time $k$}

%\textbf{Proposition \ref{primal-bound-02}.}
%\emph{
%Under Assumption \ref{assumption-score} and \ref{assumption-policy-calss}.
%Let $\tildepi_{k}(\cdot|\bs)$ be the distribution at the time $t=hk$ along the process (\ref{def:diffusion-policy-sde-reverse-process}) that starts from $\tildepi_{0}(\cdot|\bs)=\barpi_{T}(\cdot|\bs)$, then $\tildepi_{k}(\cdot|\bs)=\barpi_{T-hk}(\cdot|\bs)$. 
%Let $\hatpi_{k}(\cdot|\bs)$ be the distribution of the iteration (\ref{iteration-exponential-integrator-discretization}) at the $k$-the time $t_{k}=hk$, starting from $\hatpi_{0}(\cdot|\bs)=\calN(\bm{0},\bI)$.
%Let $0<h\leq \tau$, then for all $k=0,1,\cdots,K-1$,
%\begin{flalign}
%\nonumber
%\KL\big(\hatpi_{k+1}(\cdot|\bs)\|\tildepi_{k+1}(\cdot|\bs)\big)
%\leq&\mathrm{e}^{-\frac{1}{4}\nu h}\KL\big(\hatpi_{k}(\cdot|\bs)\|\tildepi_{k}(\cdot|\bs)\big)+\dfrac{5}{4}\nu\epsilon_{\mathrm{score}} h+12pL_{s}\sqrt{5\nu}h^{2}.
%\end{flalign}
%}

\begin{proposition}
\label{primal-bound-02}
Under Assumption \ref{assumption-score} and \ref{assumption-policy-calss}.
Let $\tildepi_{k}(\cdot|\bs)$ be the distribution at the time $t=hk$ along the process (\ref{def:diffusion-policy-sde-reverse-process}) that starts from $\tildepi_{0}(\cdot|\bs)=\barpi_{T}(\cdot|\bs)$, then $\tildepi_{k}(\cdot|\bs)=\barpi_{T-hk}(\cdot|\bs)$. 
Let $\hatpi_{k}(\cdot|\bs)$ be the distribution of the iteration (\ref{iteration-exponential-integrator-discretization}) at the $k$-the time $t_{k}=hk$, starting from $\hatpi_{0}(\cdot|\bs)=\calN(\bm{0},\bI)$.
Let $0<h\leq \tau$, then for all $k=0,1,\cdots,K-1$,
\begin{flalign}
\nonumber
\KL\big(\hatpi_{k+1}(\cdot|\bs)\|\tildepi_{k+1}(\cdot|\bs)\big)
\leq&\mathrm{e}^{-\frac{1}{4}\nu h}\KL\big(\hatpi_{k}(\cdot|\bs)\|\tildepi_{k}(\cdot|\bs)\big)+\dfrac{5}{4}\nu\epsilon_{\mathrm{score}} h+12pL_{s}\sqrt{5\nu}h^{2},
\end{flalign}
where $\tau$ is defined in (\ref{app-eq-44}).
\end{proposition}

\begin{proof}
Recall Proposition \ref{primal-bound-01}, we know for any $0\leq t\leq h\leq \tau$, the following holds
\begin{flalign}
 \label{app-eq-33}
\dfrac{\dd}{\dd t}\KL\big(\hatpi_t(\cdot|\bs)\|\tildepi_t(\cdot|\bs)\big)\leq
-\dfrac{\nu}{4}\KL\left(\hatpi_{t}(\cdot|\bs)\|\tildepi_t(\cdot|\bs)\right)
+\dfrac{5}{4}\nu\epsilon_{\mathrm{score}}
+12pL_{s}\sqrt{5\nu}h,
\end{flalign}
where comparing to (\ref{propo-01}), we use the condition $t\leq h$.

We rewrite (\ref{app-eq-33}) as follows,
\begin{flalign}
\nonumber
\dfrac{\dd}{\dd t}\left(
\mathrm{e}^{\frac{1}{4}\nu t}
\KL\big(\hatpi_t(\cdot|\bs)\|\tildepi_t(\cdot|\bs)\big)
\right)
\leq \mathrm{e}^{\frac{1}{4}\nu t}
\left(
\dfrac{5}{4}\nu\epsilon_{\mathrm{score}}
+12pL_{s}\sqrt{5\nu}h
\right).
\end{flalign}
Then, on the interval $[0,h]$, we obtain 
\begin{flalign}
\nonumber
\bigintsss_{0}^{h}
\dfrac{\dd}{\dd t}\left(
\mathrm{e}^{\frac{1}{4}\nu t}
\KL\big(\hatpi_t(\cdot|\bs)\|\tildepi_t(\cdot|\bs)\big)
\right)\dd t
\leq \bigintsss_{0}^{h} \mathrm{e}^{\frac{1}{4}\nu t}
\left(
\dfrac{5}{4}\nu\epsilon_{\mathrm{score}}
+12pL_{s}\sqrt{5\nu}h
\right)\dd t,
\end{flalign}
which implies 
\begin{flalign}
\nonumber
\mathrm{e}^{\frac{1}{4}\nu h}\KL\big(\hatpi_h(\cdot|\bs)\|\tildepi_h(\cdot|\bs)\big)\leq
\KL\big(\hatpi_{0}(\cdot|\bs)\|\tildepi_{0}(\cdot|\bs)\big)+\dfrac{4}{\nu}\left(\mathrm{e}^{\frac{1}{4}\nu h}-1\right)\left(\dfrac{5}{4}\nu\epsilon_{\mathrm{score}}+12pL_{s}\sqrt{5\nu}h\right).
\end{flalign}
Furthermore, we obtain
\begin{flalign}
\nonumber
\KL\big(\hatpi_h(\cdot|\bs)\|\tildepi_h(\cdot|\bs)\big)
\leq&\mathrm{e}^{-\frac{1}{4}\nu h}\KL\big(\hatpi_{0}(\cdot|\bs)\|\tildepi_{0}(\cdot|\bs)\big)+\dfrac{4}{\nu}\left(1-\mathrm{e}^{-\frac{1}{4}\nu h}\right)\left(\dfrac{5}{4}\nu\epsilon_{\mathrm{score}}+12pL_{s}\sqrt{5\nu}h\right)\\
\label{app-eq-35}
\leq&\mathrm{e}^{-\frac{1}{4}\nu h}\KL\big(\hatpi_{0}(\cdot|\bs)\|\tildepi_{0}(\cdot|\bs)\big)+\dfrac{5}{4}\nu\epsilon_{\mathrm{score}} h+12pL_{s}\sqrt{5\nu}h^{2},
\end{flalign}
where last equation holds since we use $1-\mathrm{e}^{-x}\leq x$, if $x\ge0$.

Recall $\tildepi_{k}(\cdot|\bs)$ is the distribution at the time $t=hk$ along the process (\ref{def:diffusion-policy-sde-reverse-process}) that starts from $\tildepi_{0}(\cdot|\bs)=\barpi_{T}(\cdot|\bs)$, then $\tildepi_{k}(\cdot|\bs)=\barpi_{T-hk}(\cdot|\bs)$. 

Recall $\hatpi_{k}(\cdot|\bs)$ is the distribution of the iteration (\ref{iteration-exponential-integrator-discretization}) at the $k$-the time $t_{k}=hk$, starting from $\hatpi_{0}(\cdot|\bs)=\calN(\bm{0},\bI)$.

According to (\ref{app-eq-35}), we rename the $\tildepi_{0}(\cdot|\bs)$ with $\tildepi_{k}(\cdot|\bs)$, $\tildepi_{h}(\cdot|\bs)$ with $\tildepi_{k+1}(\cdot|\bs)$,
$\hatpi_{0}(\cdot|\bs)$ with $\hatpi_{k}(\cdot|\bs)$ and $\hatpi_{h}(\cdot|\bs)$ with $\hatpi_{k+1}(\cdot|\bs)$, then we obtain 
\begin{flalign}
\nonumber
\KL\big(\hatpi_{k+1}(\cdot|\bs)\|\tildepi_{k+1}(\cdot|\bs)\big)
\leq&\mathrm{e}^{-\frac{1}{4}\nu h}\KL\big(\hatpi_{k}(\cdot|\bs)\|\tildepi_{k}(\cdot|\bs)\big)+\dfrac{5}{4}\nu\epsilon_{\mathrm{score}} h+12pL_{s}\sqrt{5\nu}h^{2},
\end{flalign}
which concludes the result.
\end{proof}

\section{Proof of Theorem \ref{finite-time-diffusion-policy}}

\textbf{Theorem \ref{finite-time-diffusion-policy}} (Finite-time Analysis of Diffusion Policy).
\emph{
For a given state $\bs$, let $\{\barpi_{t}(\cdot|\bs)\}_{t=0:T}$ and $\{\tildepi_{t}(\cdot|\bs)\}_{t=0:T}$ be the distributions along the Ornstein-Uhlenbeck flow (\ref{def:diffusion-policy-sde-forward-process-01}) and (\ref{def:diffusion-policy-sde-reverse-process}) correspondingly, where $\{\barpi_{t}(\cdot|\bs)\}_{t=0:T}$ starts at $\barpi_{0}(\cdot|\bs)=\pi(\cdot|\bs)$ and 
$\{\tildepi_{t}(\cdot|\bs)\}_{t=0:T}$ starts at $\tildepi_{0}(\cdot|\bs)=\barpi_{T}(\cdot|\bs)$.
Let $\hatpi_{k}(\cdot|\bs)$ be the distribution of the exponential integrator discretization iteration (\ref{iteration-exponential-integrator-discretization}) at the $k$-the time $t_{k}=hk$, i.e.,
$\hata_{t_k}\sim \hatpi_{k}(\cdot|\bs)$ denotes the distribution of the diffusion policy (see Algorithms \ref{algo:diffusion-policy-general-case}) at the time $t_{k}=hk$.
Let $\{\hatpi_{k}(\cdot|\bs)\}_{k=0:K}$ be starting at $\hatpi_{0}(\cdot|\bs)=\calN(\bm{0},\bI)$, under Assumption \ref{assumption-score} and \ref{assumption-policy-calss}, let the reverse length $K$ satisfy
\[K\ge
 T\cdot\max\left\{\dfrac{1}{\tau_{0}},\dfrac{1}{\mathtt{T}_{0}},12L_{s},\nu\right\},
\]
where
\[\tau_{0}=:\sup_{t\ge0}\left\{t:t\mathrm{e}^{t}\leq\dfrac{\sqrt{5\nu}}{96L_{s}L_{p}}\right\},\mathtt{T}_{0}=:\sup_{t\ge0}\left\{t:1-\mathrm{e}^{-2t}\leq \frac{\mathrm{e}^{t}}{L_p}\right\}.\]
Then the KL-divergence between the diffusion policy $\hata_{K}\sim\hatpi_{K}(\cdot|\bs)$ and input policy $\pi(\cdot|\bs)$ is upper-bounded as follows,
\begin{flalign}
\nonumber
&\KL\big(\hatpi_{K}(\cdot|\bs)\|\pi(\cdot|\bs)\big)\leq\underbrace{\mathrm{e}^{-\frac{9}{4}\nu hK}\KL \big(\calN(\bm{0},\bI)\|\pi(\cdot|\bs)\big)}_{\mathrm{convergence~of~forward~process}}+\underbrace{64pL_{s}\sqrt{\dfrac{5}{\nu}}\cdot\dfrac{T}{K}}_{\mathrm{errors~from~discretization}}\\
\nonumber
&~~~~~~~~~~~~~~~+\dfrac{20}{3}\underbrace{\sup_{(k,t)\in[K]\times[t_{k},t_{k+1}]}\left\{\log\E_{\ba\sim\tildepi_t(\cdot|\bs)}\left[\exp\left\|\bm{\hat \mathbf{S}}(\ba,\bs,T-hk)-\bnabla\log\tildepi_t(\ba|\bs)\right\|_{2}^{2}\right]\right\}}_{\mathrm{errors~from~score~matching}}.
\end{flalign}
}
\begin{proof}
Recall $\tildepi_{k}(\cdot|\bs)=\barpi_{T-hk}(\cdot|\bs)$, then we know
\begin{flalign}
\label{app-eq-38}
\tildepi_{K}(\cdot|\bs)=\barpi_{T-hK}(\cdot|\bs)=\barpi_{0}(\cdot|\bs)=\pi(\cdot|\bs),
\end{flalign}
then according to Proposition \ref{primal-bound-02}, we know
 \begingroup
\allowdisplaybreaks
\begin{flalign}
\nonumber
\KL\big(\hatpi_{K}(\cdot|\bs)\|\pi(\cdot|\bs)\big)\overset{(\ref{app-eq-38})}=&\KL\big(\hatpi_{K}(\cdot|\bs)\|\tildepi_{K}(\cdot|\bs)\big)\\
\nonumber
\leq&\mathrm{e}^{-\frac{1}{4}\nu K}\KL\big(\hatpi_{0}(\cdot|\bs)\|\tildepi_{0}(\cdot|\bs)\big)+\sum_{j=0}^{K-1}\mathrm{e}^{-\frac{1}{4}\nu h j}\left(\dfrac{5}{4}\nu\epsilon_{\mathrm{score}} h+12pL_{s}\sqrt{5\nu}h^{2}\right)\\
\nonumber
\leq&\mathrm{e}^{-\frac{1}{4}\nu K}\KL\big(\hatpi_{0}(\cdot|\bs)\|\tildepi_{0}(\cdot|\bs)\big)+\dfrac{1}{1-\mathrm{e}^{-\frac{1}{4}\nu h}}\left(\dfrac{5}{4}\nu\epsilon_{\mathrm{score}} h+12pL_{s}\sqrt{5\nu}h^{2}\right)\\
\label{app-eq-36}
\leq&\mathrm{e}^{-\frac{1}{4}\nu K}\KL\big(\hatpi_{0}(\cdot|\bs)\|\tildepi_{0}(\cdot|\bs)\big)+\dfrac{16}{3\nu h}\left(\dfrac{5}{4}\nu\epsilon_{\mathrm{score}} h+12pL_{s}\sqrt{5\nu}h^{2}\right)\\
\label{app-eq-42}
=&\mathrm{e}^{-\frac{1}{4}\nu K}\KL\big(\hatpi_{0}(\cdot|\bs)\|\tildepi_{0}(\cdot|\bs)\big)+\dfrac{20}{3}\epsilon_{\mathrm{score}}+64\sqrt{\dfrac{5}{\nu}}pL_{s}h,
\end{flalign}
 \endgroup
where Eq.(\ref{app-eq-36})  holds since we consider the 
\begin{flalign}
 \label{app-eq-45}
1-\mathrm{e}^{-x}\ge\dfrac{3}{4}x, ~\text{if}~0<x\leq\dfrac{1}{4},
\end{flalign}
and we set the step-size $h$ satisfies the next condition:
\[h\nu\leq 1,~\text{i.e.,}~h\leq\dfrac{1}{\nu}.\]
Let $\xi(\cdot)$ be standard Gaussian distribution on $\R^{p}$, i.e., $\xi(\cdot)\sim\calN(\bm{0},\bI)$, then we obtain the following result: for a given state $\bs$,
 \begingroup
\allowdisplaybreaks
\begin{flalign}
\nonumber
\dfrac{\dd }{\dd t}\KL \left(\xi(\cdot)\|\barpi_{t}(\cdot|\bs)\right)&=\dfrac{\dd }{\dd t}\bigintsss_{\R^{p}}\xi(\ba)\log\dfrac{\xi(\ba)}{\barpi_{t}(\ba|\bs)}\dd \ba\\
\nonumber
&=-\bigintsss_{\R^{p}}\dfrac{\xi(\ba)}{\barpi_{t}(\ba|\bs)}\dfrac{\partial\barpi_{t}(\ba|\bs)}{\partial t}\dd \ba\\
\nonumber
&=-\bigintsss_{\R^{p}}\dfrac{\xi(\ba)}{\barpi_{t}(\ba|\bs)}\left(
\DIV\cdot\left(\barpi_{t}(\ba|\bs)\bnabla\log \dfrac{\barpi_{t}(\ba|\bs)}{\xi(\ba)}
\right)
\right)\dd\ba\tag*{$\blacktriangleright$ Fokker–Planck Equation}\\
\nonumber
&=\bigintsss_{\R^{p}}\left\langle\bnabla\dfrac{\xi(\ba)}{\barpi_{t}(\ba|\bs)},\barpi_{t}(\ba|\bs)\bnabla\log \dfrac{\barpi_{t}(\ba|\bs)}{\xi(\ba)}\right\rangle\dd\ba\tag*{$\blacktriangleright$ Integration by Parts}\\
\nonumber
&=\bigintsss_{\R^{p}}\left\langle\dfrac{\xi(\ba)}{\barpi_{t}(\ba|\bs)}\bnabla\log\dfrac{\xi(\ba)}{\barpi_{t}(\ba|\bs)},\barpi_{t}(\ba|\bs)\bnabla\log \dfrac{\barpi_{t}(\ba|\bs)}{\xi(\ba)}\right\rangle\dd\ba\\
\nonumber
&=\bigintsss_{\R^{p}}\xi(\ba)\left\langle\bnabla\log\dfrac{\xi(\ba)}{\barpi_{t}(\ba|\bs)},\bnabla\log \dfrac{\barpi_{t}(\ba|\bs)}{\xi(\ba)}\right\rangle\dd\ba\\
\nonumber
&=-\bigintsss_{\R^{p}}\xi(\ba)\left\|\bnabla\log\dfrac{\xi(\ba)}{\barpi_{t}(\ba|\bs)}\right\|^{2}_{2}=-\E_{\ba\sim\xi(\cdot)}\left[\left\|\bnabla\log\dfrac{\xi(\ba)}{\barpi_{t}(\ba|\bs)}\right\|^{2}_{2}\right]\\
\nonumber
&=-\FI \left(\xi(\cdot)\|\barpi_{t}(\cdot|\bs)\right)\\
&\leq-2\nu_t\KL\left(\xi(\cdot)\|\barpi_{t}(\cdot|\bs)\right)\tag*{$\blacktriangleright$ Assumption \ref{assumption-policy-calss} and Proposition \ref{app-propo-02}}\\
\nonumber
&=-\dfrac{2\nu}{\nu+(1-\nu)\mathrm{e}^{-2t}}\KL\left(\xi(\cdot)\|\barpi_{t}(\cdot|\bs)\right)\\
\label{app-eq-37}
&\leq-2\nu\KL\left(\xi(\cdot)\|\barpi_{t}(\cdot|\bs)\right),
\end{flalign}
\endgroup
where the last equation holds since $\mathrm{e}^{-t}\leq 1$ with $t\ge0$.

Eq.(\ref{app-eq-37}) implies
\[
\dfrac{\dd}{\dd t}\log\KL \left(\xi(\cdot)\|\barpi_{t}(\cdot|\bs)\right)\leq-2\nu,
\]
integrating both sides of above equation on the interval $[0,T]$, we obtain 
\begin{flalign}
\label{app-eq-40}
\KL \left(\xi(\cdot)\|\barpi_{T}(\cdot|\bs)\right)\leq\mathrm{e}^{-2\nu T}\KL \left(\xi(\cdot)\|\barpi_{0}(\cdot|\bs)\right).
\end{flalign}
According to definition of diffusion policy, since: $\hata_{0}\sim\calN(\bm{0},\bI)$, and $\tildepi_{0}(\cdot|\bs)\overset{(\ref{relation-policy-revers-forward})}=\barpi_{T}(\cdot|\bs)$, then we know 
\begin{flalign}
\label{app-eq-39}
\KL\big(\hatpi_{0}(\cdot|\bs)\|\tildepi_{0}(\cdot|\bs)\big)=\KL\big(\xi(\cdot)\|\barpi_{T}(\cdot|\bs)\big),
\end{flalign}
which implies 
\begin{flalign}
\label{app-eq-41}
\KL\big(\hatpi_{0}(\cdot|\bs)\|\tildepi_{0}(\cdot|\bs)\big)\overset{(\ref{app-eq-39})}=\KL\big(\xi(\cdot)\|\barpi_{T}(\cdot|\bs)\big)\overset{(\ref{app-eq-40})}\leq\mathrm{e}^{-2\nu T}\KL \left(\xi(\cdot)\|\barpi_{0}(\cdot|\bs)\right).
\end{flalign}
Combining (\ref{app-eq-42}) and (\ref{app-eq-41}), we obtain
\begin{flalign}
\nonumber
\KL\big(\hatpi_{K}(\cdot|\bs)\|\pi(\cdot|\bs)\big)\leq&\mathrm{e}^{-\frac{1}{4}\nu hK-T}\KL \left(\xi(\cdot)\|\barpi_{0}(\cdot|\bs)\right)+\dfrac{20}{3}\epsilon_{\mathrm{score}}+64\sqrt{\dfrac{5}{\nu}}pL_{s}h\\
\label{app-eq-47}
\overset{(\ref{app-eq-38})}=&\mathrm{e}^{-\frac{9}{4}\nu hK}\KL \big(\calN(\bm{0},\bI)\|\pi(\cdot|\bs)\big)+\dfrac{20}{3}\epsilon_{\mathrm{score}}+64\sqrt{\dfrac{5}{\nu}}pL_{s}h.
\end{flalign}
Recall the following conditions (\ref{app-eq-43}), (\ref{app-eq-44}), and (\ref{app-eq-45}) on the step-size $h$,
\[
h\leq \min\left\{ \tau_{0},\mathtt{T}_{0},\dfrac{1}{12L_{s}},\dfrac{1}{\nu}\right\},
\]
which implies the reverse length $K$ satisfy the following condition
\[
K=\dfrac{T}{h}\ge T\cdot\max\left\{\dfrac{1}{\tau_{0}},\dfrac{1}{\mathtt{T}_{0}},12L_{s},\nu\right\}.
\]
Finally, recall the definition of $\epsilon$ (\ref{app-eq-46}), we rewrite (\ref{app-eq-47}) as follows
\begin{flalign}
\nonumber
&\KL\big(\hatpi_{K}(\cdot|\bs)\|\pi(\cdot|\bs)\big)\leq\mathrm{e}^{-\frac{9}{4}\nu hK}\KL \big(\calN(\bm{0},\bI)\|\pi(\cdot|\bs)\big)+64pL_{s}\sqrt{\dfrac{5}{\nu}}\cdot\dfrac{T}{K}\\
\nonumber
&~~~~~~~~~~~~~~~~~~~~~~~~+\dfrac{20}{3}\sup_{(k,t)\in[K]\times[t_{k},t_{k+1}]}\left\{\log\E_{\ba\sim\tildepi_t(\cdot|\bs)}\left[\exp\left\|\bm{\hat \mathbf{S}}(\ba,\bs,T-hk)-\bnabla\log\tildepi_t(\ba|\bs)\right\|_{2}^{2}\right]\right\},
\end{flalign}
which concludes the proof.
\end{proof}

\section{Additional Details}
\subsection{Proof of Lemma \ref{app-lem-06}}

\begin{lemma} 
\label{app-lem-06}
Under Assumption \ref{assumption-score}, for all $0\leq t^{'}\leq T$, if $t\leq\frac{1}{12L_{s}}$, then for any given state $\bs$
\begin{flalign}
\nonumber
\left\|\bm{\hat \mathbf{S}}\left(\hata_t,\bs,t^{'}\right)-\bm{\hat \mathbf{S}}\left(\hata_0,\bs,t^{'}\right)\right\|\leq 3L_{s}t\left\|\hata_{0}\right\|+6L_{s}t\left\|\bm{\hat \mathbf{S}}\left(\hata_t,\bs,t^{'}\right)\right\|+3L_{s}\sqrt{t}\|\bz\|,
\end{flalign}
and
\begin{flalign}
\label{app-eq-19}
\left\|\bm{\hat \mathbf{S}}\left(\hata_t,\bs,t^{'}\right)-\bm{\hat \mathbf{S}}\left(\hata_0,\bs,t^{'}\right)\right\|^{2}_{2}\leq 36L^{2}_{s}t^{2}\left\|\hata_{0}\right\|_{2}^{2}+72L^{2}_{s}t^{2}\left\|\bm{\hat \mathbf{S}}\left(\hata_t,\bs,t^{'}\right)\right\|_{2}^{2}+36L^{2}_{s}t\|\bz\|^{2}_{2},
\end{flalign}
where $\hata_t$ updated according to (\ref{app-eq-12}).
\end{lemma}

\begin{proof} (of Lemma \ref{app-lem-06}).
First, we consider 
\begin{flalign}
\nonumber
\left\|\bm{\hat \mathbf{S}}\left(\hata_t,\bs,t^{'}\right)-\bm{\hat \mathbf{S}}\left(\hata_0,\bs,t^{'}\right)\right\|
\leq& L_{s}\left\|\hata_t-\hata_{0}\right\|\\
\nonumber
=&\left\|(\mathrm{e}^{t}-1)\hata_{0}+2(\mathrm{e}^{t}-1) \bm{\hat \mathbf{S}}(\hata_0,\bs,t^{'})+\sqrt{\mathrm{e}^{t}-1}\bz\right\|\\
\label{app-eq-11}
\leq&2L_{s}t\left\|\hata_{0}\right\|_{2}^{2}+4L_{s}t\left\|\bm{\hat \mathbf{S}}(\hata_0,\bs,t^{'})\right\|+2L_{s}\sqrt{t}\|\bz\|,
\end{flalign}
where the last equation holds due to $\mathrm{e}^{t}-1\leq 2t$.

Furthermore, we consider the case with $t\leq\frac{1}{12L_s}$, then we obtain the boundedness of the term 
\begin{flalign}
\nonumber
\left\|\bm{\hat \mathbf{S}}\left(\hata_0,\bs,t^{'}\right)\right\|\leq& \left\|\bm{\hat \mathbf{S}}\left(\hata_{t},\bs,t^{'}\right)\right\|+L_{s}\left\|\hata_t-\hata_{0}\right\|\\
\nonumber
\leq& \left\|\bm{\hat \mathbf{S}}\left(\hata_{t},\bs,t^{'}\right)\right\|+2L_{s}t\left\|\hata_{0}\right\|+4L_{s}t\left\|\bm{\hat \mathbf{S}}(\hata_0,\bs,t^{'})\right\|+2L_{s}\sqrt{t}\|\bz\|\\
\nonumber
\leq& \left\|\bm{\hat \mathbf{S}}\left(\hata_{t},\bs,t^{'}\right)\right\|+2L_{s}t\left\|\hata_{0}\right\|+\dfrac{1}{3}\left\|\bm{\hat \mathbf{S}}(\hata_0,\bs,t^{'})\right\|+2L_{s}\sqrt{t}\|\bz\|,
\end{flalign}
which implies
\begin{flalign}
\label{app-eq-10}
\left\|\bm{\hat \mathbf{S}}\left(\hata_0,\bs,t^{'}\right)\right\|\leq \dfrac{3}{2}\left\|\bm{\hat \mathbf{S}}\left(\hata_{t},\bs,t^{'}\right)\right\|+3L_{s}t\left\|\hata_{0}\right\|_{2}^{2}+3L_{s}\sqrt{t}\|\bz\|.
\end{flalign}
Taking Eq.(\ref{app-eq-10}) into Eq.(\ref{app-eq-11}), and with $t\leq\frac{1}{12L_s}$, we obtain
\begin{flalign}
\nonumber
\left\|\bm{\hat \mathbf{S}}\left(\hata_t,\bs,t^{'}\right)-\bm{\hat \mathbf{S}}\left(\hata_0,\bs,t^{'}\right)\right\|\leq 3L_{s}t\left\|\hata_{0}\right\|+6L_{s}t\left\|\bm{\hat \mathbf{S}}\left(\hata_t,\bs,t^{'}\right)\right\|+3L_{s}\sqrt{t}\|\bz\|.
\end{flalign}
Finally, we know
\begin{flalign}
\left\|\bm{\hat \mathbf{S}}\left(\hata_t,\bs,t^{'}\right)-\bm{\hat \mathbf{S}}\left(\hata_0,\bs,t^{'}\right)\right\|^{2}_{2}\leq 36L^{2}_{s}t^{2}\left\|\hata_{0}\right\|_{2}^{2}+72L^{2}_{s}t^{2}\left\|\bm{\hat \mathbf{S}}\left(\hata_{t},\bs,t^{'}\right)\right\|_{2}^{2}+36L^{2}_{s}t\|\bz\|^{2}_{2},
\end{flalign}
which concludes the proof of Lemma \ref{app-lem-06}.
\end{proof}

\subsection{Proof of Lemma \ref{app-lem-09}}
\label{sec-proof-lem-09}
\begin{proof}(Lemma \ref{app-lem-09})
Recall the update rule of $\hata_t$ (\ref{app-eq-12}), 
\[
\hata_t=\mathrm{e}^{t}\hata_0+2(\mathrm{e}^{t}-1) \bm{\hat \mathbf{S}}(\hata_0,\bs,T)+\sqrt{\mathrm{e}^{t}-1}\bz,~\bz\sim\calN(\bm{0},\bI).
\]
To simplify the expression, in this section, we introduce the following notation
\begin{flalign}
\label{app-eq-21}
\bz\sim \rho_{z}(\cdot),~\text{where}~\rho_{z}(\cdot)=\calN(\bm{0},\bI).
\end{flalign}
According to the definition of $\rho_{0,t}({\hata_{0}},{\hata_{t}}|\bs)$ (\ref{app-eq-17}), we denote
$\hata_t=\ba$, then we know,
\begin{flalign}
\nonumber
&\bigintsss_{\R^{p}}\bigintsss_{\R^{p}}\rho_{0,t}({\hata_{0}},\ba|\bs)\Big\|\bm{\hat \mathbf{S}}({\hata_{0}},\bs,T)-\bnabla\log\tildepi_t(\ba|\bs)\Big\|_{2}^{2}
\dd\ba\dd {\hata_{0}}\\
\nonumber
\leq&2\bigintsss_{\R^{p}}\bigintsss_{\R^{p}}\rho_{0,t}({\hata_{0}},\ba|\bs)\left(\Big\|\bm{\hat \mathbf{S}}({\hata_{0}},\bs,T)-\bm{\hat \mathbf{S}}({\hata_{t}},\bs,T)\Big\|_{2}^{2}
+
\Big\|\bm{\hat \mathbf{S}}({\hata_{t}},\bs,T)-\bnabla\log\tildepi_t(\ba|\bs)\Big\|_{2}^{2}
\right)
\dd\ba\dd {\hata_{0}}\\
\nonumber
=&2\bigintsss_{\R^{p}}\bigintsss_{\R^{p}}\rho_{0,t}({\hata_{0}},\ba|\bs)\left(\Big\|\bm{\hat \mathbf{S}}({\hata_{0}},\bs,T)-\bm{\hat \mathbf{S}}(\ba,\bs,T)\Big\|_{2}^{2}
+
\Big\|\bm{\hat \mathbf{S}}(\ba,\bs,T)-\bnabla\log\tildepi_t(\ba|\bs)\Big\|_{2}^{2}
\right)
\dd\ba\dd {\hata_{0}}
\end{flalign}
Recall Lemma \ref{app-lem-06}, we know
\begingroup
\allowdisplaybreaks
\begin{flalign}
\nonumber
&\bigintsss_{\R^{p}}\bigintsss_{\R^{p}}\rho_{0,t}({\hata_{0}},\ba|\bs)\Big\|\bm{\hat \mathbf{S}}({\hata_{0}},\bs,T)-\bm{\hat \mathbf{S}}(\ba,\bs,T)\Big\|_{2}^{2}\dd\ba\dd {\hata_{0}}\\
\nonumber
\overset{(\ref{app-eq-19})}\leq&\bigintsss_{\R^{p}}\bigintsss_{\R^{p}}\bigintsss_{\R^{p}}\rho_{0,t}({\hata_{0}},\ba|\bs)\rho_{z}(\bz)\left(36L^{2}_{s}t^{2}\left\|\hata_{0}\right\|_{2}^{2}+72L^{2}_{s}t^{2}\left\|\bm{\hat \mathbf{S}}\left(\ba, \bs,T\right)\right\|_{2}^{2}+36L^{2}_{s}t\|\bz\|^{2}_{2}\right)\dd\ba\dd {\hata_{0}}\dd\bz\\
\nonumber
=&\bigintsss_{\R^{p}}\bigintsss_{\R^{p}}\bigintsss_{\R^{p}}\rho_{0,t}({\hata_{0}},\ba|\bs)\rho_{z}(\bz)\left(36L^{2}_{s}t^{2}\left\|\hata_{0}\right\|_{2}^{2}+72L^{2}_{s}t^{2}\left\|\bm{\hat \mathbf{S}}\left(\ba,\bs,T\right)\right\|_{2}^{2}\right)\dd\ba\dd {\hata_{0}}\dd\bz\\
\nonumber
&~~~~~~~~+36L^{2}_{s}t\bigintsss_{\R^{p}}\bigintsss_{\R^{p}}\bigintsss_{\R^{p}}\rho_{0,t}({\hata_{0}},\ba|\bs)\rho_{z}(\bz)\|\bz\|^{2}_{2}\dd\ba\dd {\hata_{0}}\dd\bz\\
\label{app-eq-20}
=&36L^{2}_{s}t^{2}\bigintsss_{\R^{p}}\hatpi_{0}({\hata_{0}}|\bs)\left\|\hata_{0}\right\|_{2}^{2}\dd {\hata_{0}}
+72L^{2}_{s}t^{2}\bigintsss_{\R^{p}}\hatpi_{t}(\ba|\bs)\left\|\bm{\hat \mathbf{S}}\left(\ba,\bs,T\right)\right\|_{2}^{2}\dd \ba
+36L^{2}_{s}pt\\
\label{app-eq-23}
=&36L^{2}_{s}pt^{2}+
72L^{2}_{s}t^{2}\bigintsss_{\R^{p}}\hatpi_{t}(\ba|\bs)\left\|\bm{\hat \mathbf{S}}\left(\ba,\bs,T\right)\right\|_{2}^{2}\dd {\ba}
+36L^{2}_{s}pt\\
\label{app-eq-24}
\leq&36pt(1+t)L^{2}_{s}+
144t^{2}L^{2}_{s}\bigintsss_{\R^{p}}\hatpi_{t}(\ba|\bs)\left(\Big\|\bm{\hat \mathbf{S}}(\ba,\bs,T)-\bnabla\log\tildepi_t(\ba|\bs)\Big\|_{2}^{2}
+
\Big\|\bnabla\log\tildepi_t(\ba|\bs)\Big\|_{2}^{2}
\right)\dd {\ba},
\end{flalign}
\endgroup
where the first term in Eq.(\ref{app-eq-20}) holds since: 
\[
\int_{\R^{p}}\int_{\R^{p}}\rho_{0,t}({\hata_{0}},\ba|\bs)\rho_{z}(\bz)\dd\ba\dd\bz=\hatpi_{0}({\hata_{0}}|\bs);
\]
the second term in Eq.(\ref{app-eq-20}) holds since: 
\[
\int_{\R^{p}}\int_{\R^{p}}\int_{\R^{p}}\rho_{0,t}({\hata_{0}},\ba|\bs)\rho_{z}(\bz)\dd {\hata_{0}}\dd\bz=\hatpi_{t}(\ba|\bs);
\]
the third term in Eq.(\ref{app-eq-20}) holds since: $\bz\sim\calN(\bm{0},\bI)$, then $\|\bz\|_{2}^{2}\sim\chi ^{2}(p)$-distribution with $p$ degrees of freedom, then 
\begin{flalign}
\label{app-eq-22}
\int_{\R^{p}}\int_{\R^{p}}\int_{\R^{p}}\rho_{0,t}({\hata_{0}},\ba|\bs)\rho_{z}(\bz)\|\bz\|^{2}_{2}\dd\ba\dd {\hata_{0}}\dd\bz=p;
\end{flalign}
Eq.(\ref{app-eq-23}) holds with the same analysis of (\ref{app-eq-22}), since ${\hata_{0}}\sim\calN(\bm{0},\bI)$, then $\|{\hata_{0}}\|_{2}^{2}\sim\chi ^{2}(p)$, which implies
\[
\int_{\R^{p}}\hatpi_{0}({\hata_{0}}|\bs)\left\|\hata_{0}\right\|_{2}^{2}\dd {\hata_{0}}=p;
\]
Eq.(\ref{app-eq-24}) holds since we use the fact: $\|\langle \bm{\alpha}+\bm{\beta}\rangle\|_{2}^{2}\leq 2\|\bm{\alpha}\|_{2}^{2}+2\|\bm{\beta}\|_{2}^{2}$.
\end{proof}

\section{Details and Discussions for multimodal Experiments}

\label{app-sec-multi-goal}

In this section, we present all the implementation details and the plots of both 2D and 3D Visualization.
Then we provide additional discussions for empirical results of the task of the multimodal environment in Section \ref{diff-me-po-re}.

\begin{figure*}[t]
    \centering
    {\includegraphics[width=3cm,height=3cm]{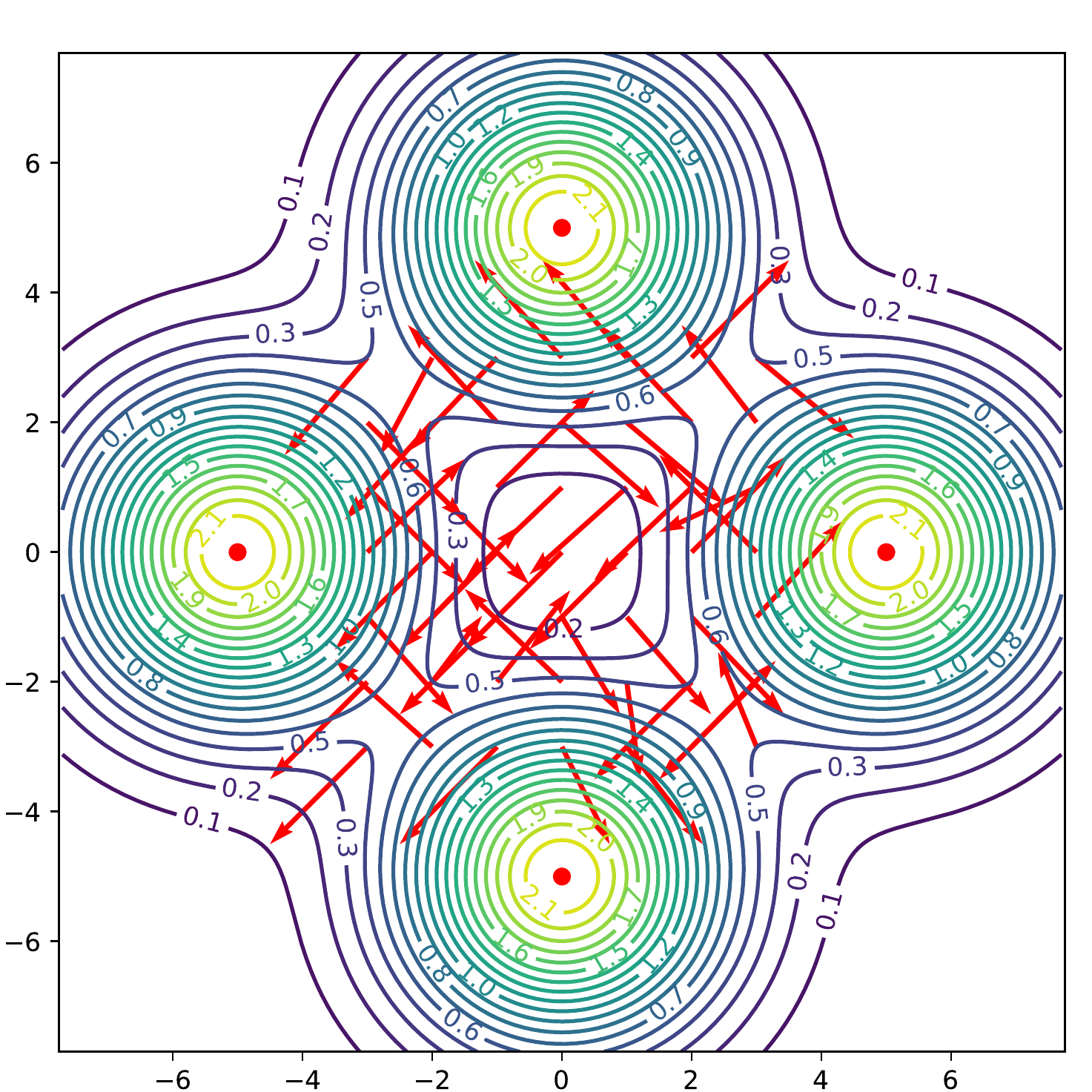}}
        {\includegraphics[width=3cm,height=3cm]{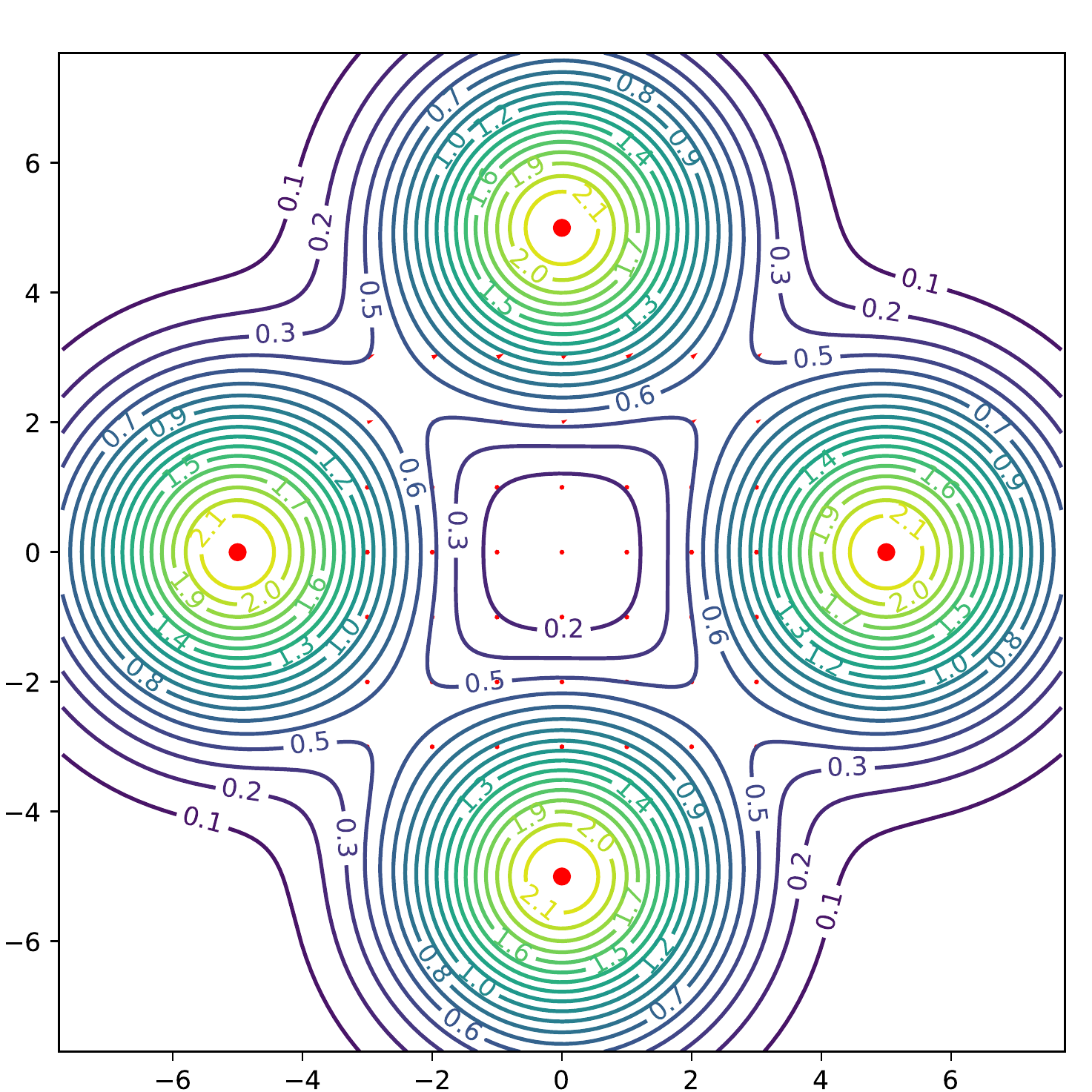}}
            {\includegraphics[width=3cm,height=3cm]{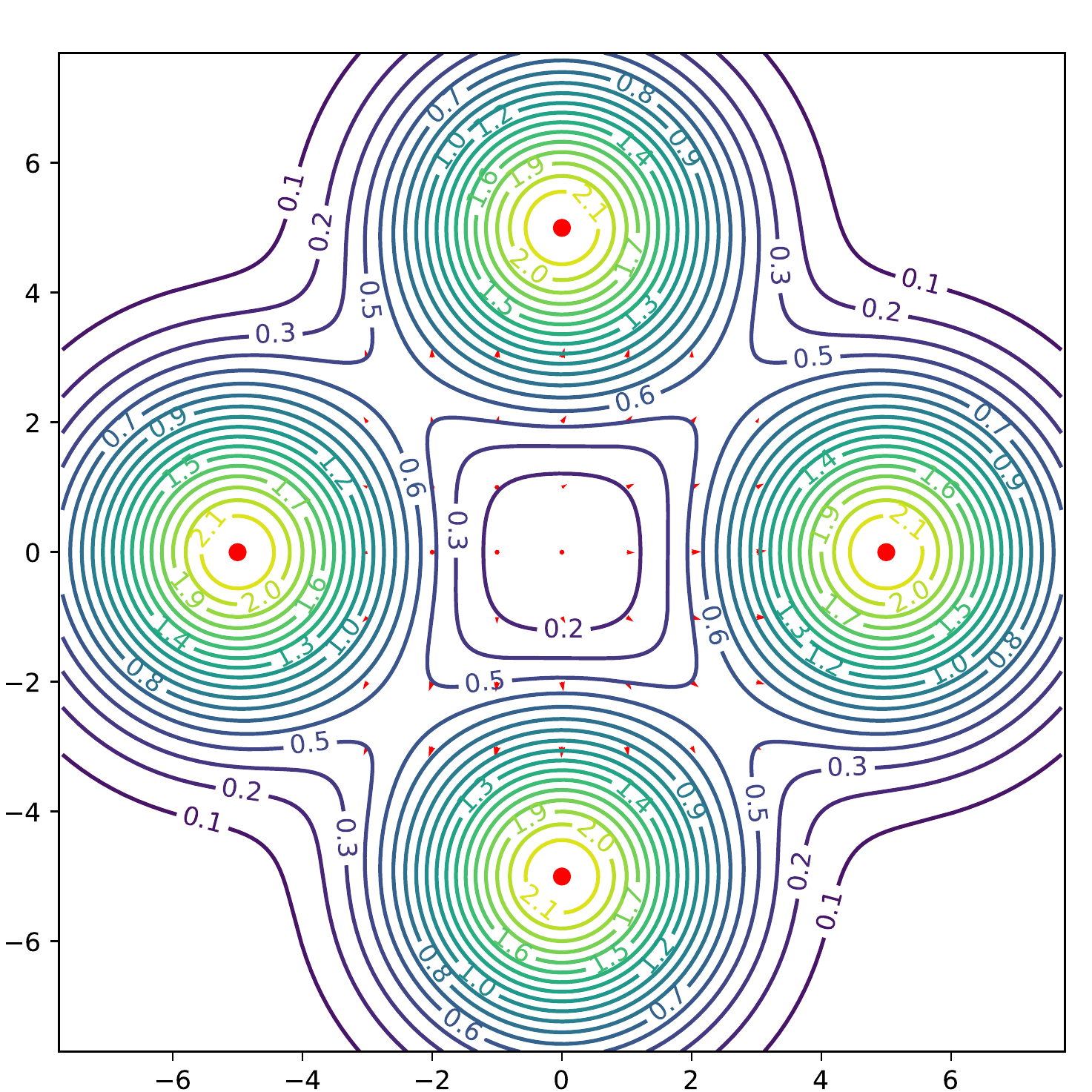}}
                {\includegraphics[width=3cm,height=3cm]{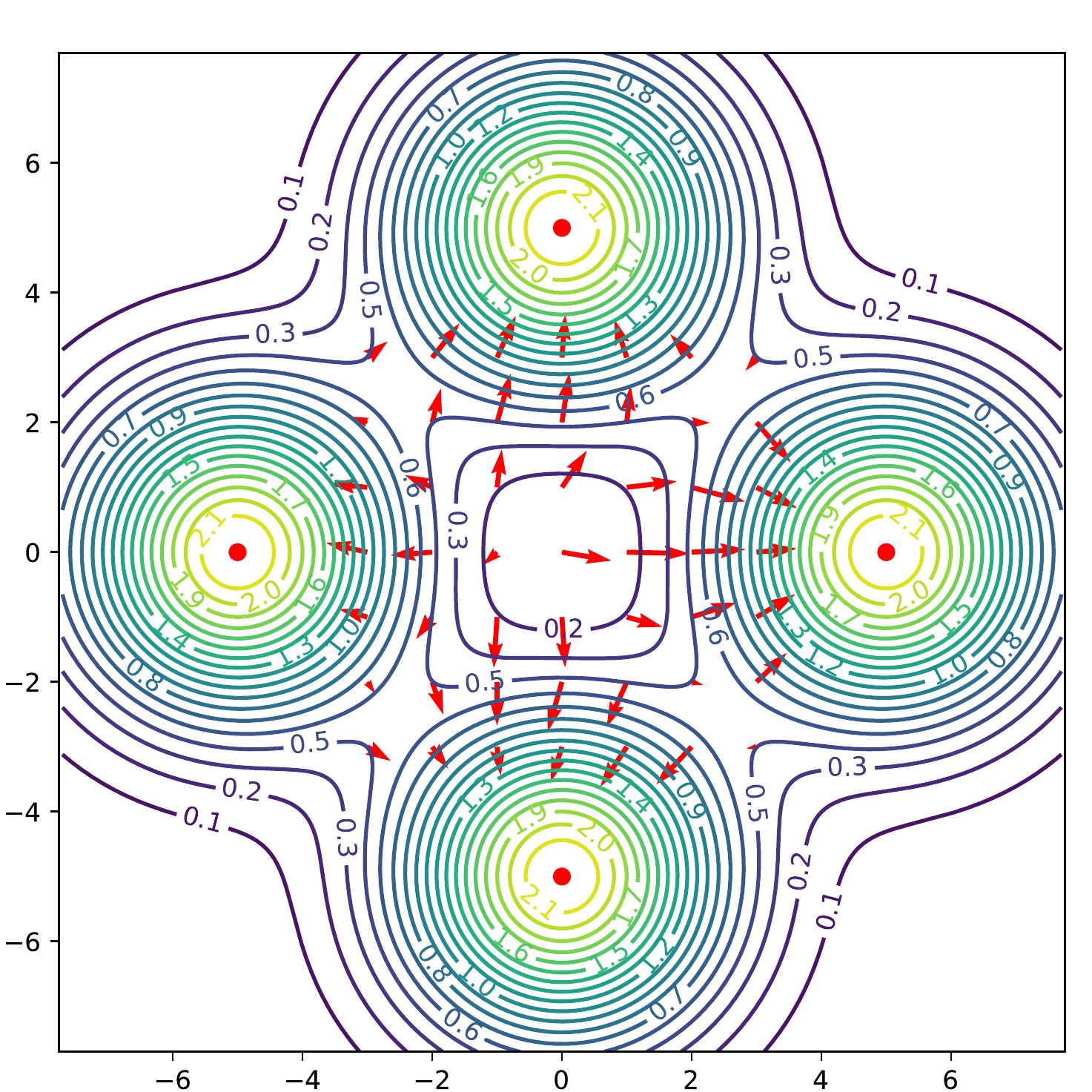}}
                    {\includegraphics[width=3cm,height=3cm]{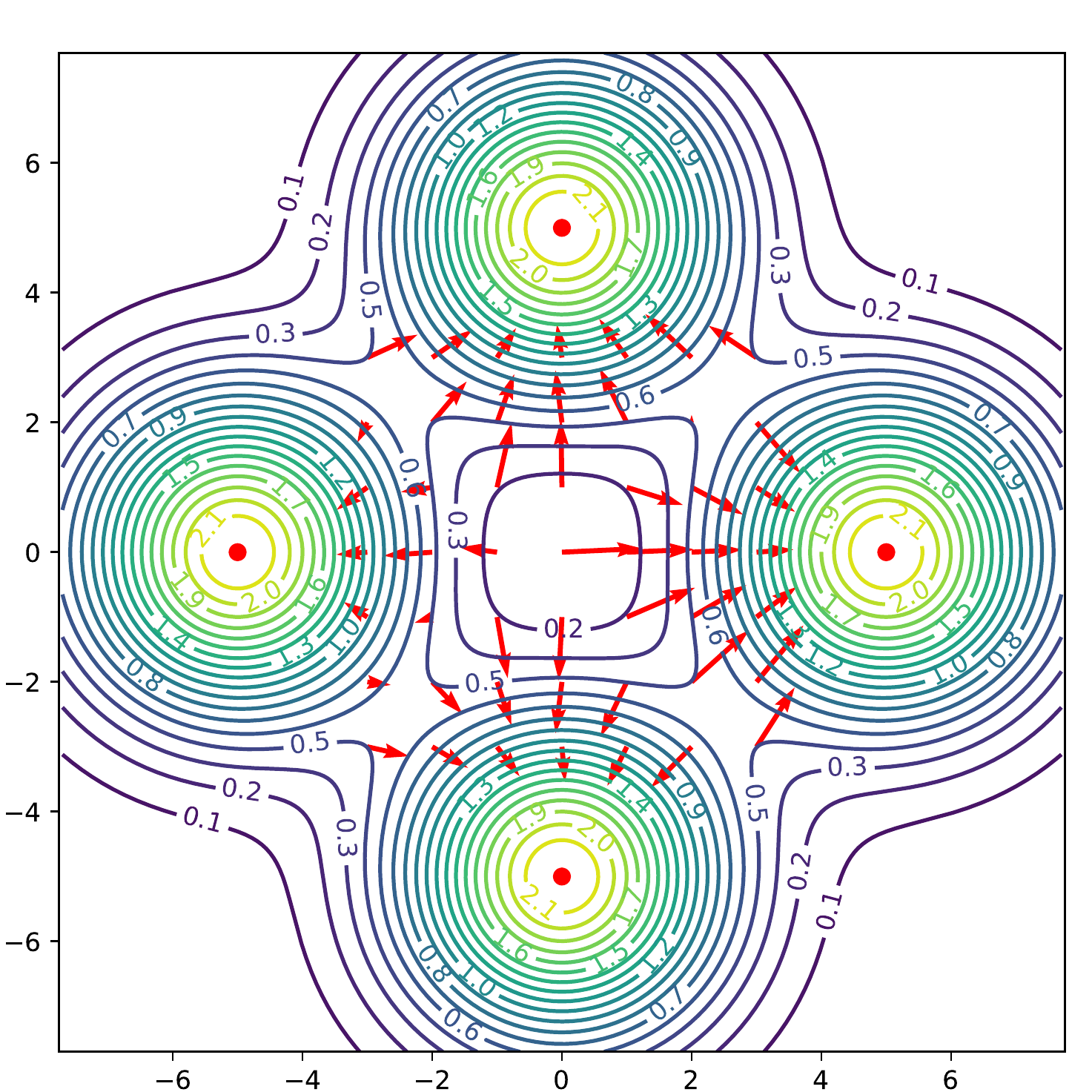}}
 \subfigure[1E$3$ iterations]
    {\includegraphics[width=3cm,height=3cm]{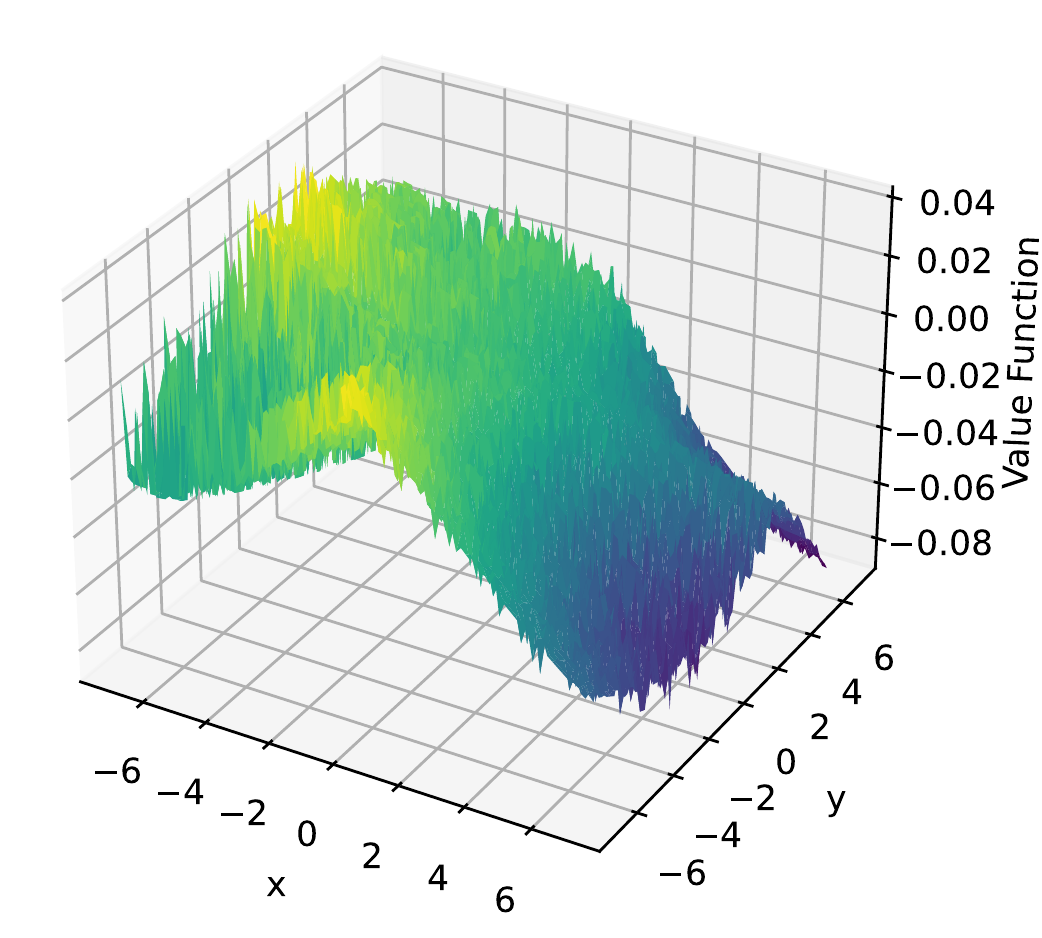}}
         \subfigure[2E$3$ iterations]
        {\includegraphics[width=3cm,height=3cm]{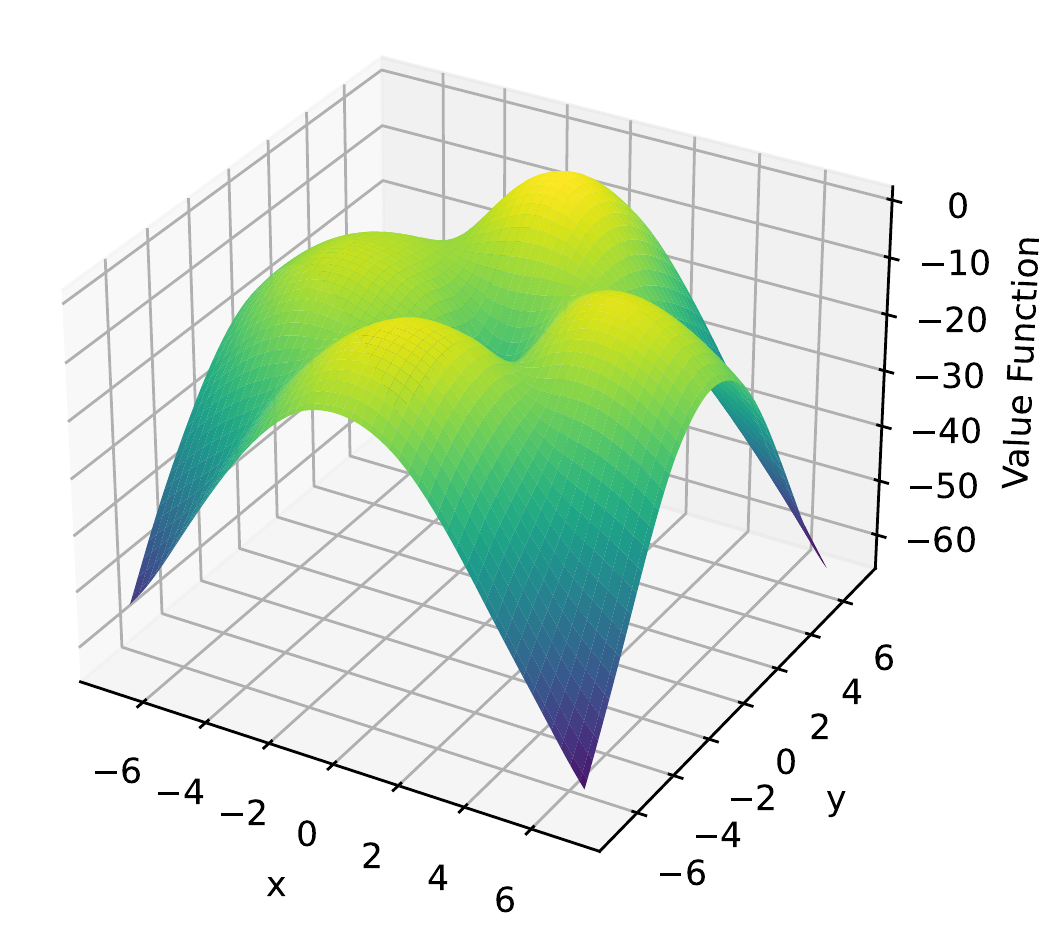}}
             \subfigure[3E$3$ iterations]
            {\includegraphics[width=3cm,height=3cm]{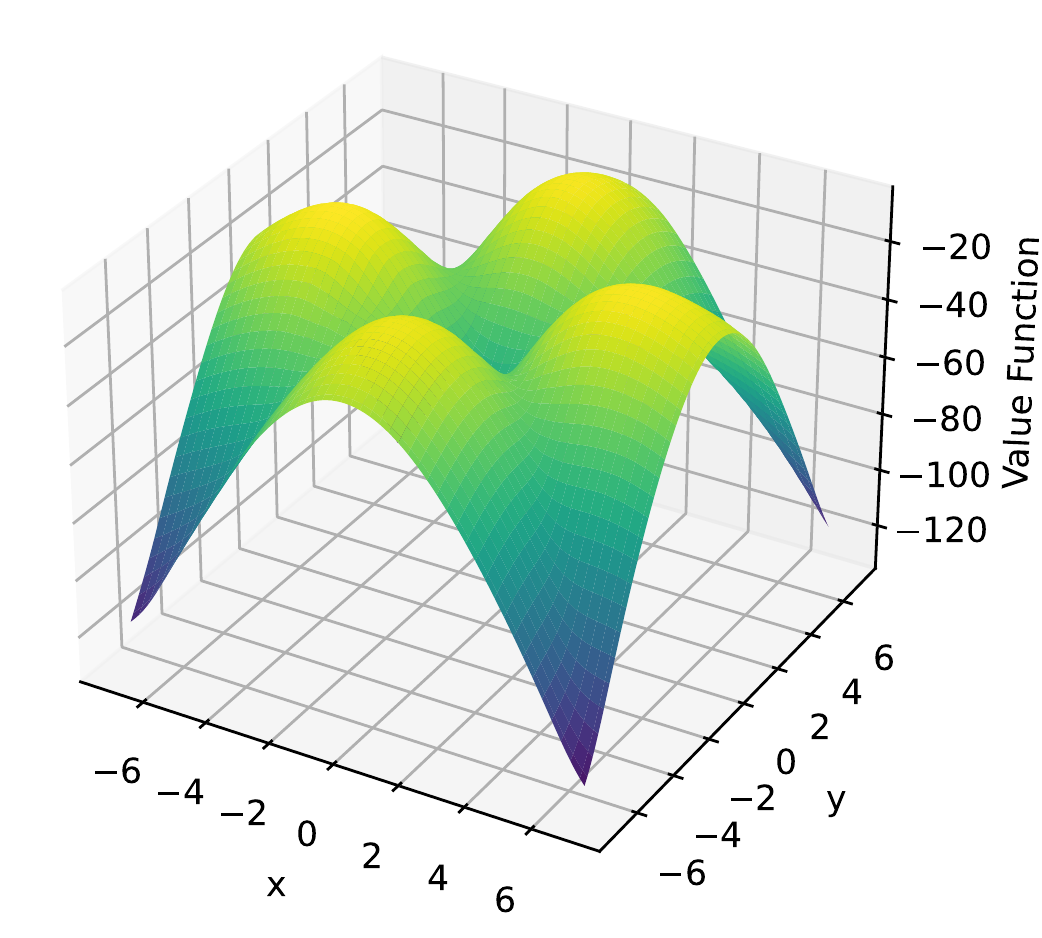}}
                 \subfigure[4E$3$ iterations]
                {\includegraphics[width=3cm,height=3cm]{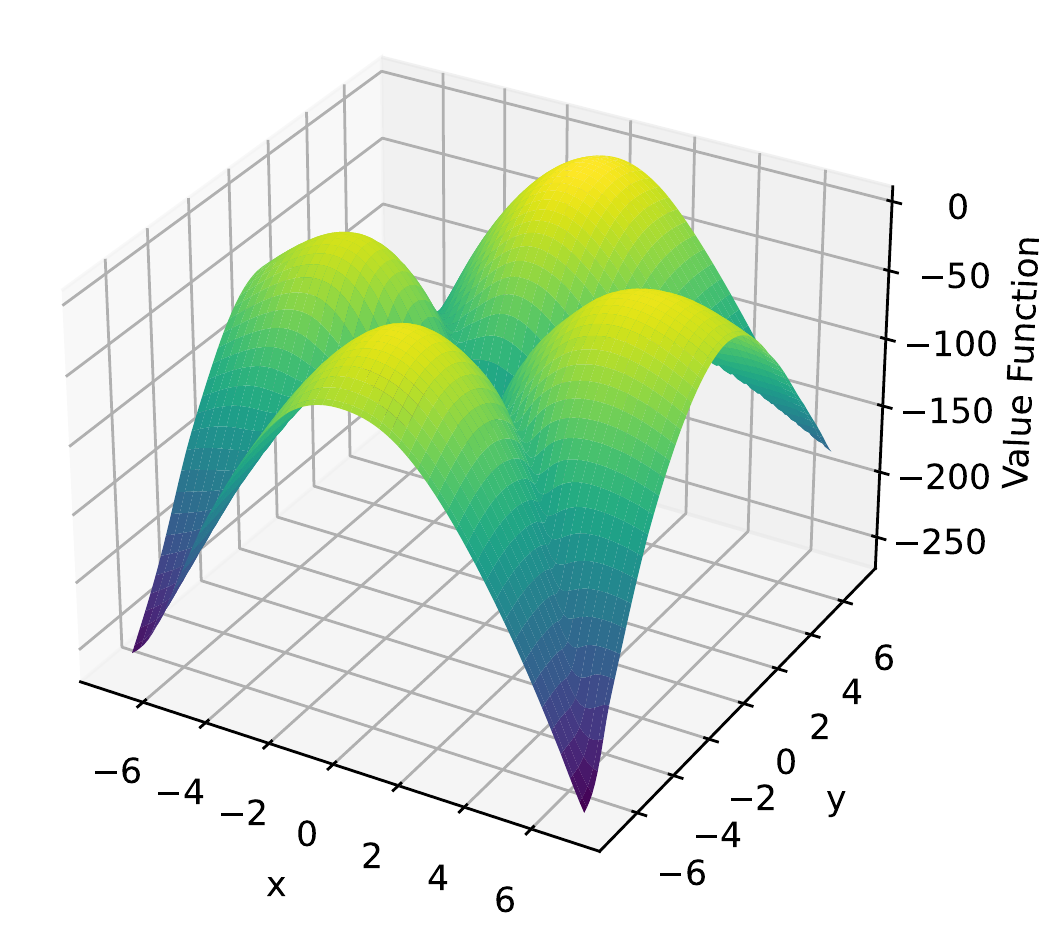}}
                     \subfigure[5E$3$ iterations]
                    {\includegraphics[width=3cm,height=3cm]{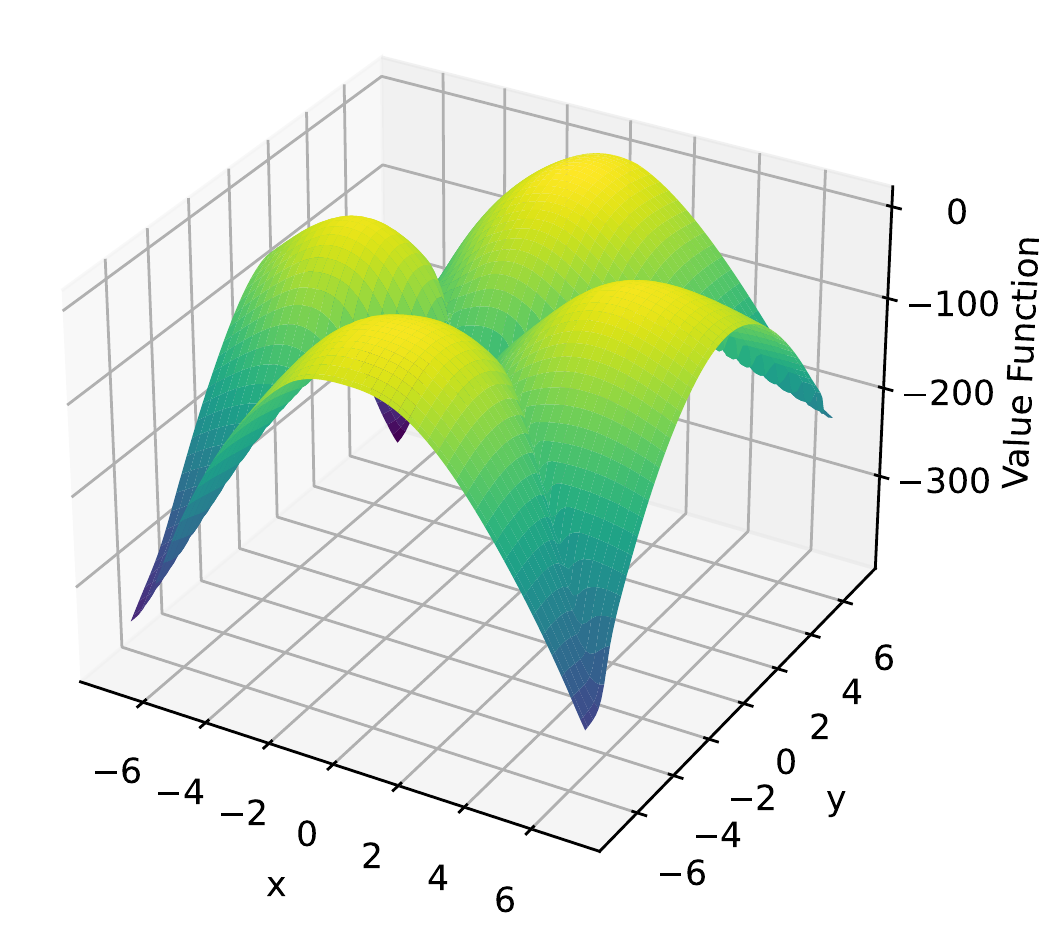}}
 {\includegraphics[width=3cm,height=3cm]{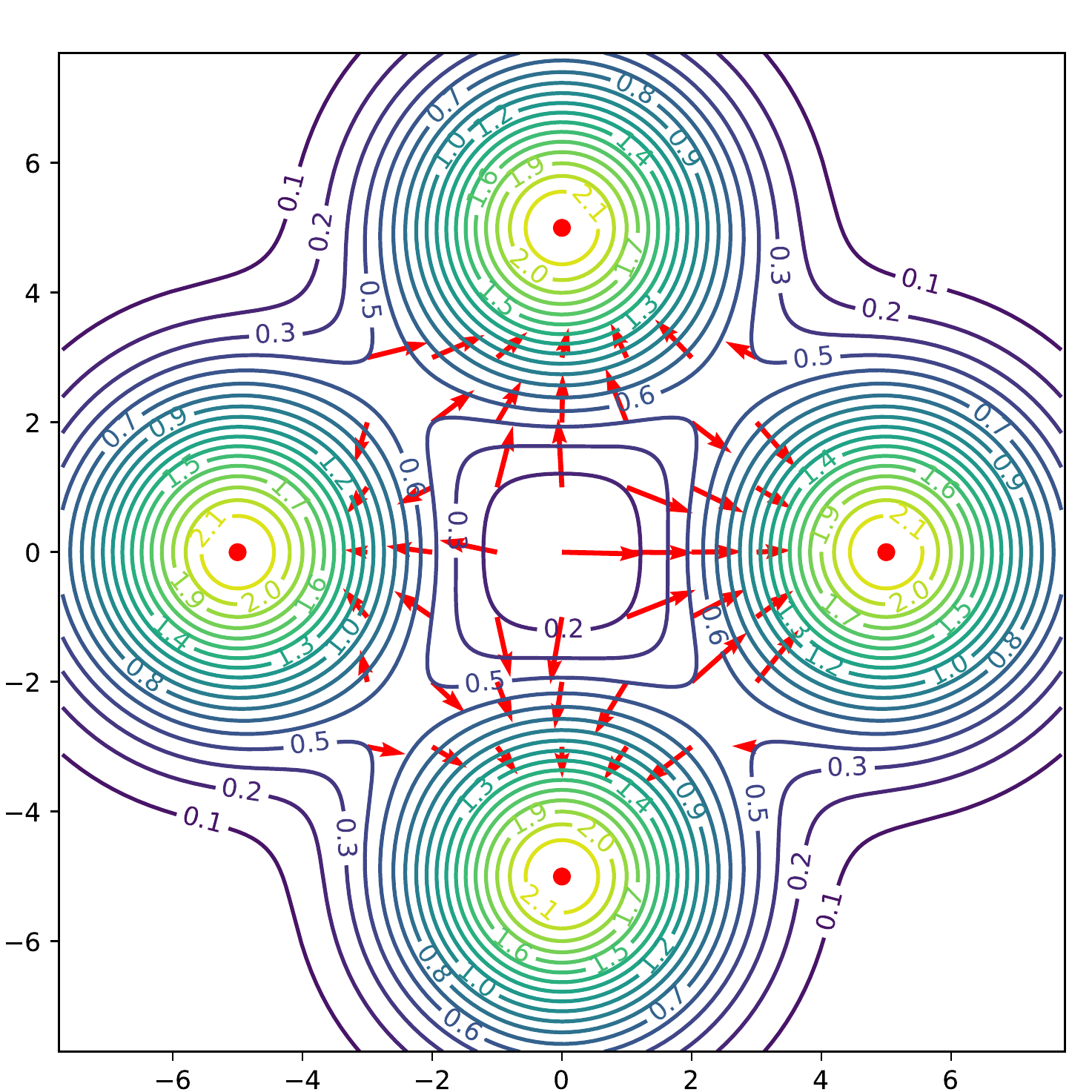}}
        {\includegraphics[width=3cm,height=3cm]{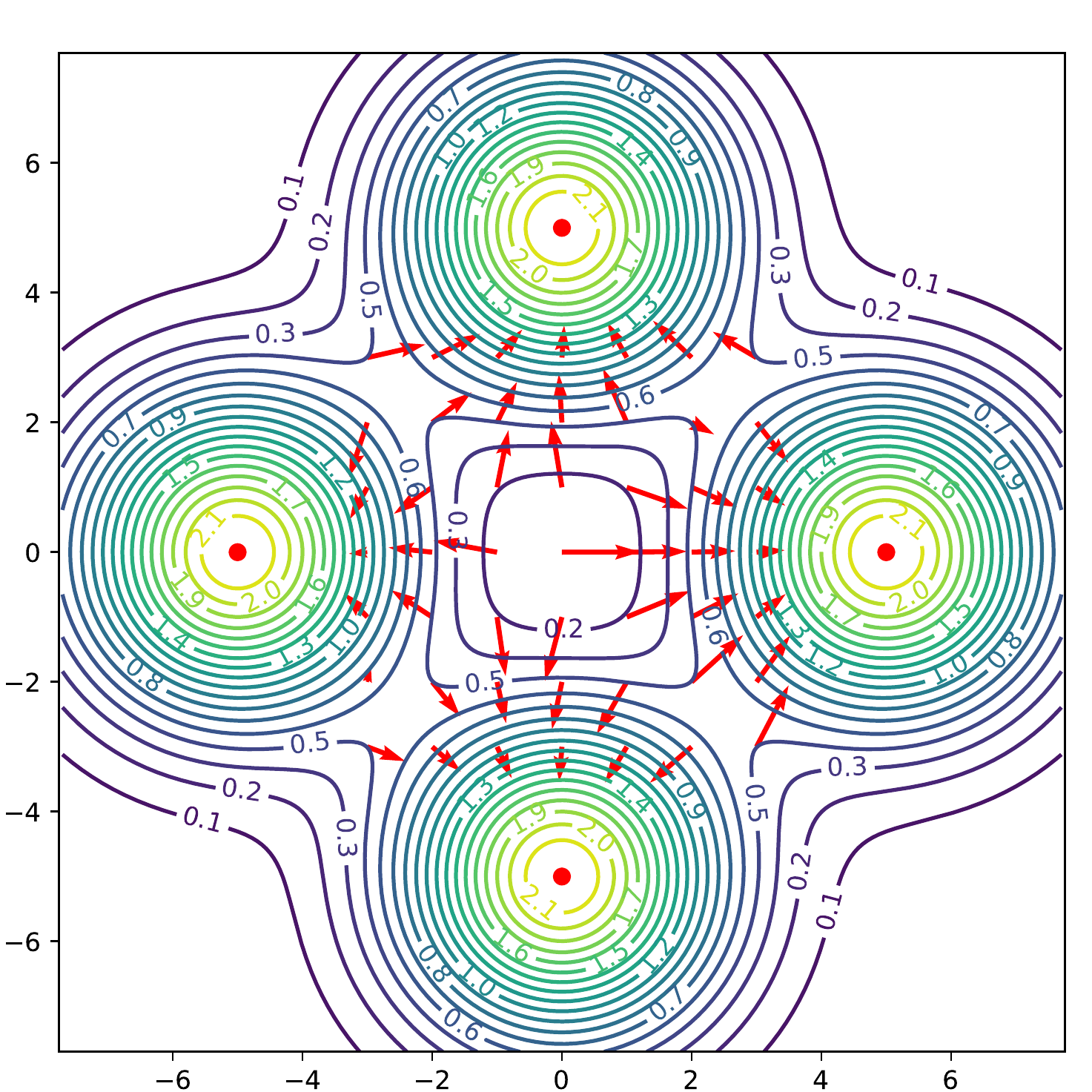}}
            {\includegraphics[width=3cm,height=3cm]{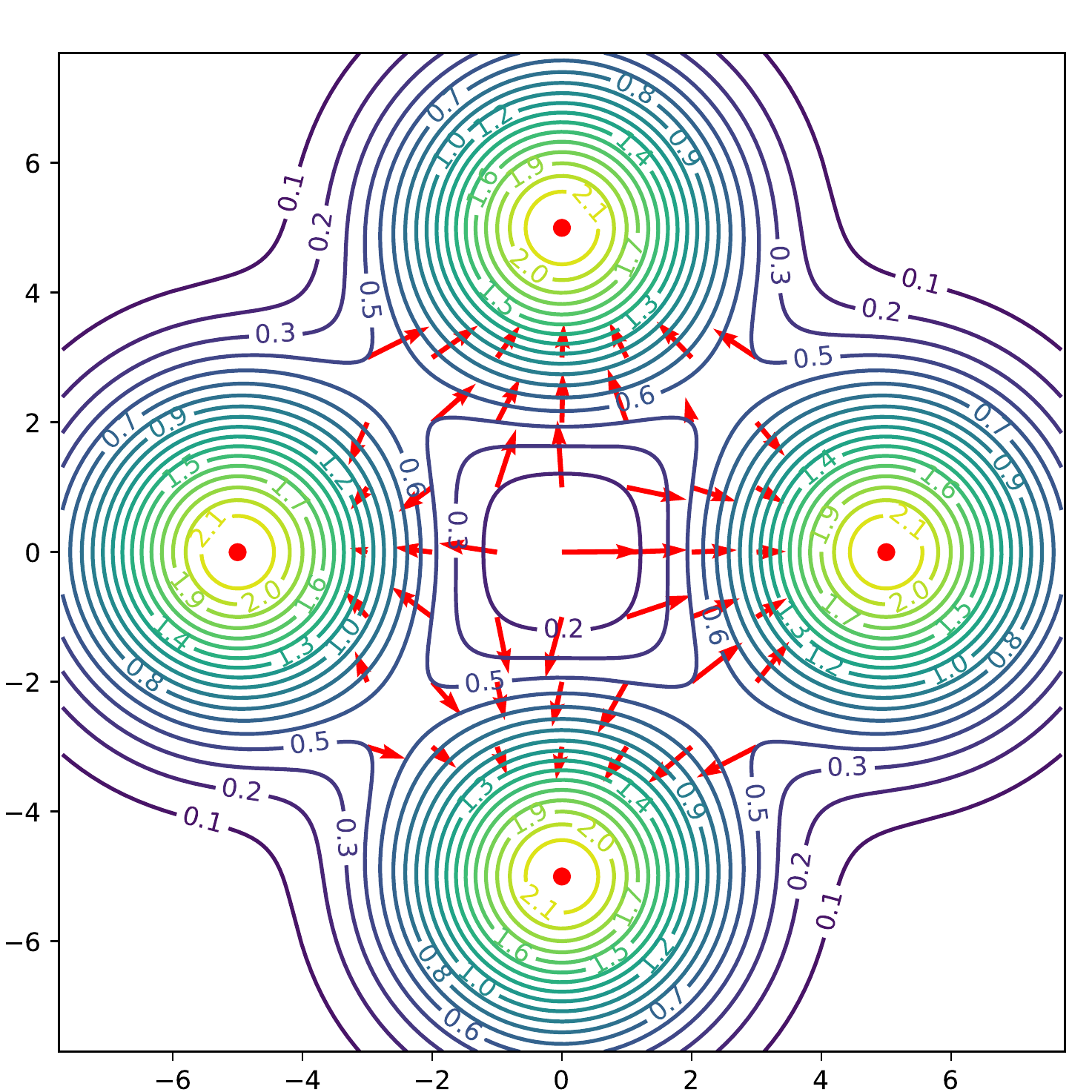}}
                {\includegraphics[width=3cm,height=3cm]{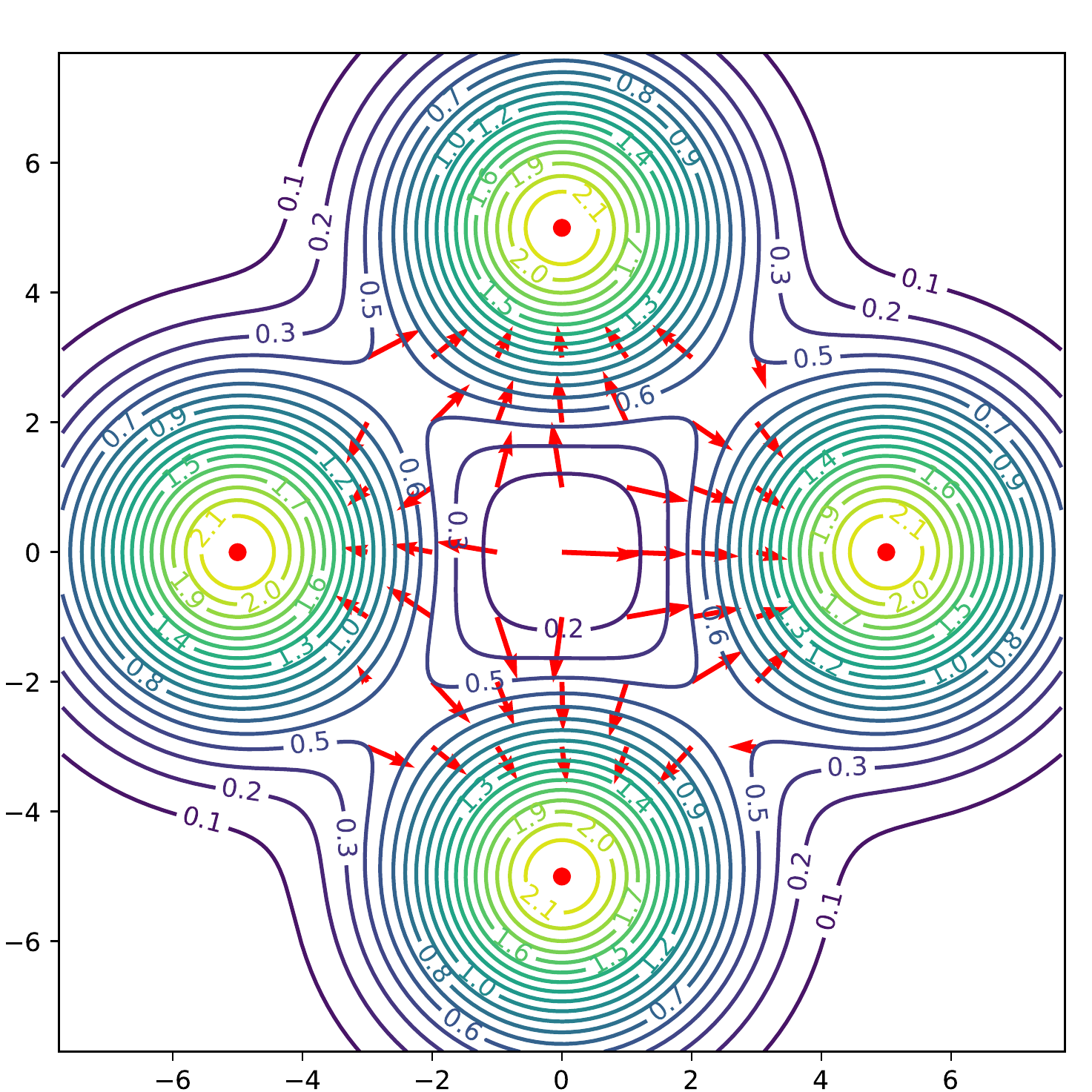}}
                    {\includegraphics[width=3cm,height=3cm]{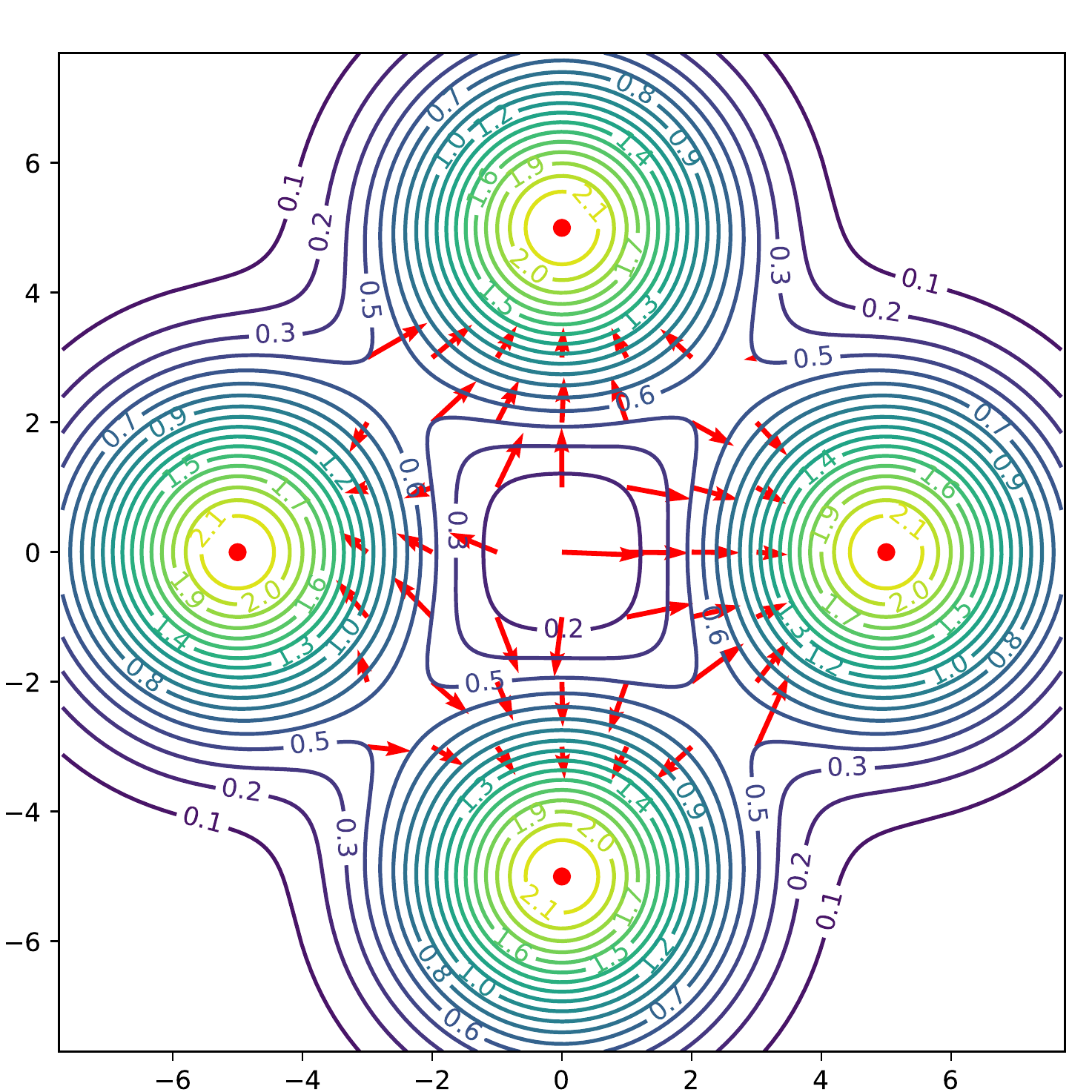}}
 \subfigure[6E$3$ iterations]
    {\includegraphics[width=3cm,height=3cm]{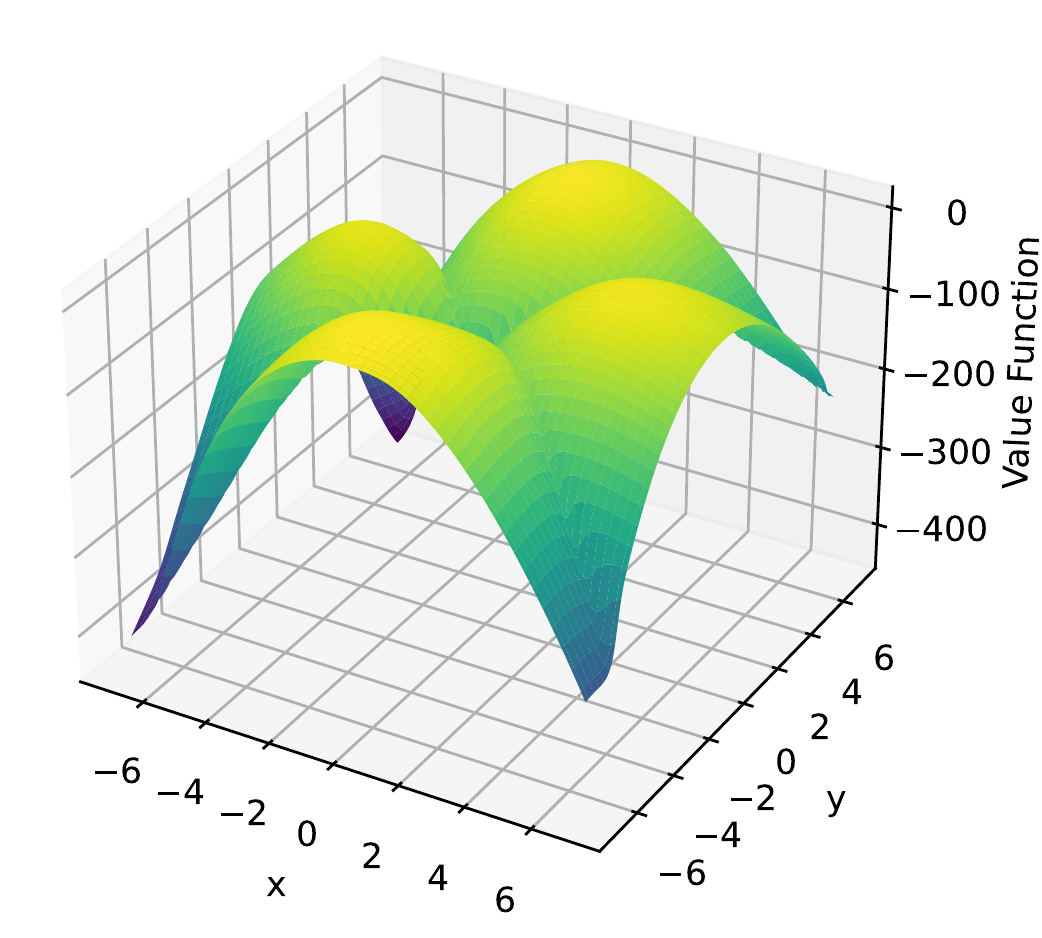}}
         \subfigure[7E$3$ iterations]
        {\includegraphics[width=3cm,height=3cm]{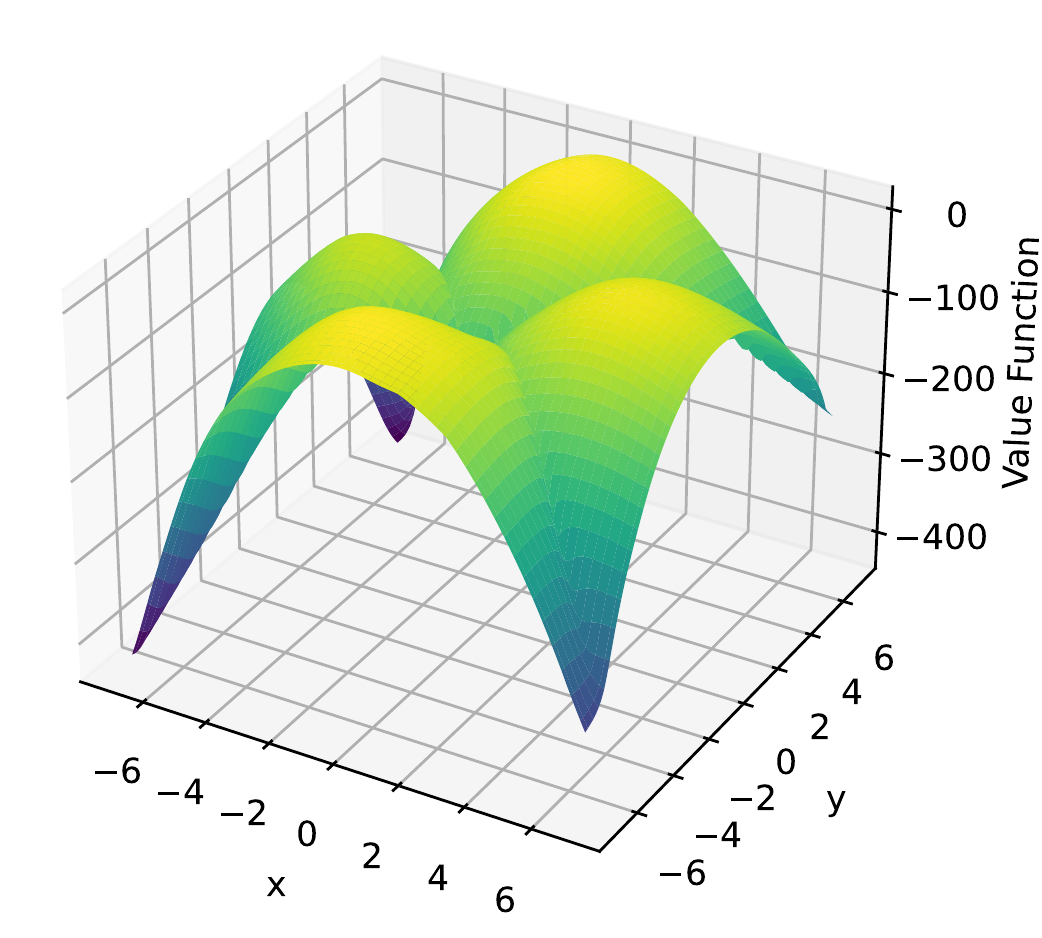}}
             \subfigure[8E$3$ iterations]
            {\includegraphics[width=3cm,height=3cm]{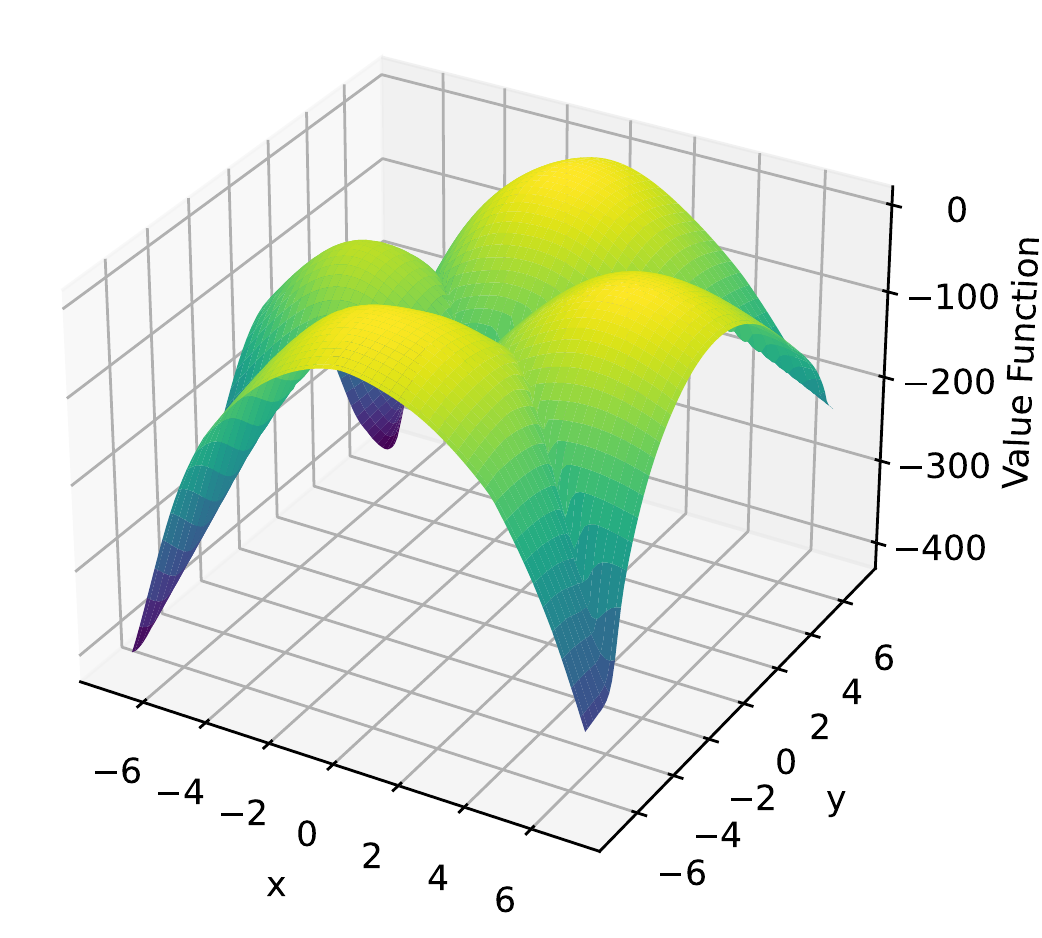}}
                 \subfigure[9E$3$ iterations]
                {\includegraphics[width=3cm,height=3cm]{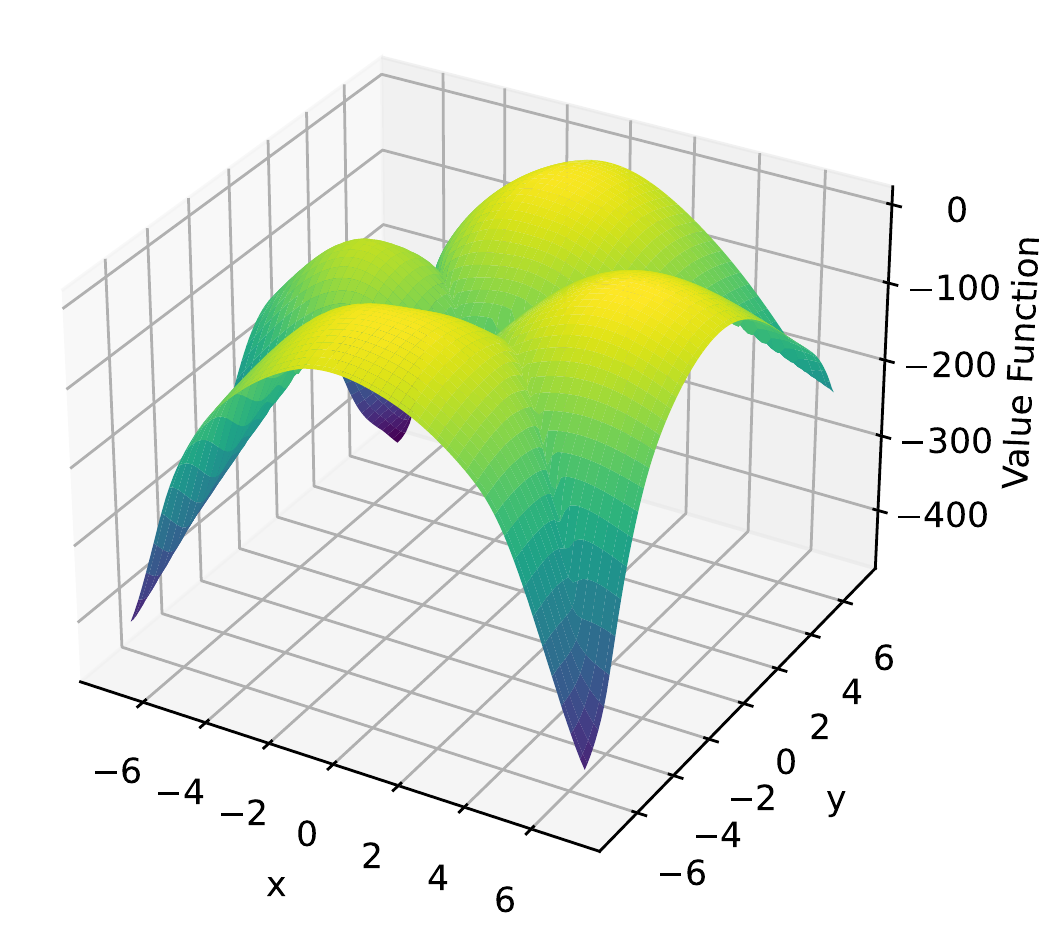}}
                     \subfigure[10E$3$ iterations]
                    {\includegraphics[width=3cm,height=3cm]{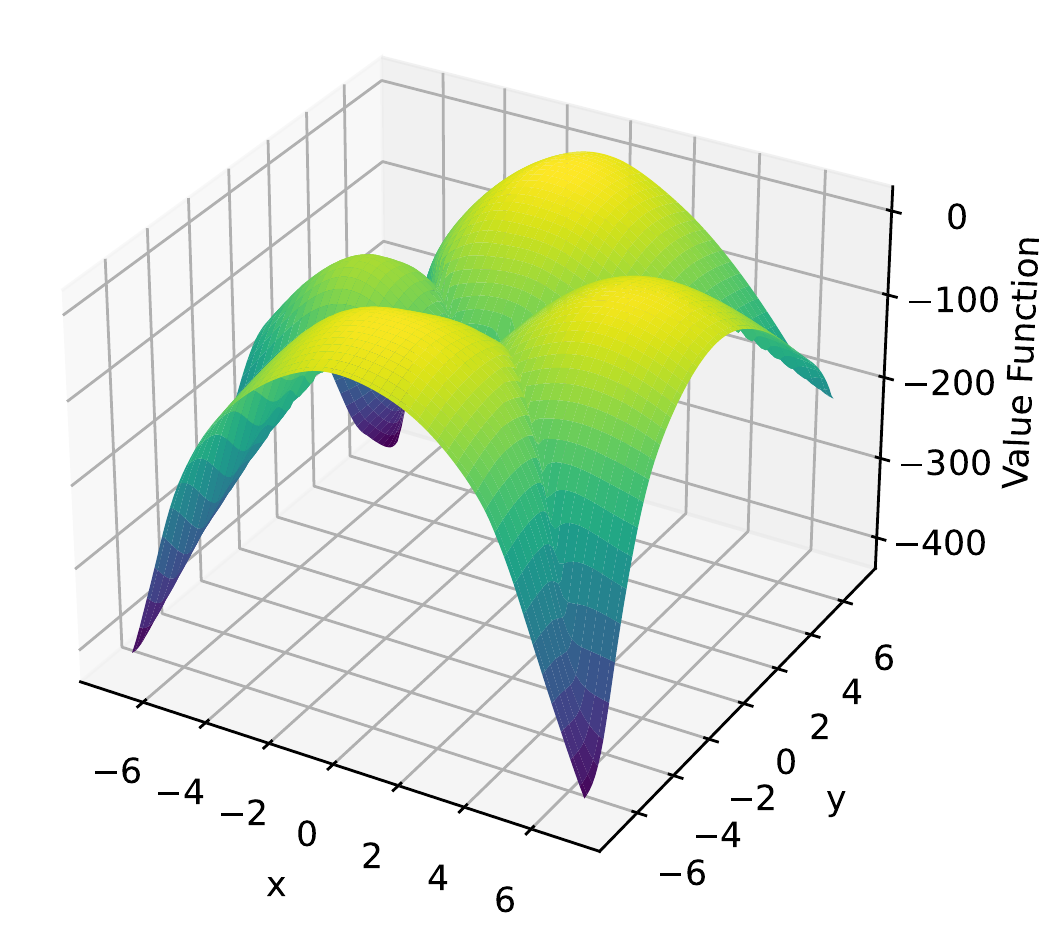}}
    \caption
    {Policy representation comparison of diffusion policy with different iterations.
    %and the optimal policy is to go to one of the goal positions randomly.
    %, and we have shown the solid red lines as the final action learned by different algorithms with 1000 iterations.
    }
     \label{app-fig-diffusion-policy}
\end{figure*}

\subsection{Multimodal Environment }

In this section, we clarify the task and reward of the multimodal environment.

%We design a simple “multi-goal” environment according to the \emph{Didactic Example} \citep{haarnoja2017reinforcement}, in which the agent is a 2D point mass tries to reach one of four symmetrically placed goals.
%The goal positions are symmetrically distributed at the four points $(0,5)$, $(0,-5)$, $(5,0)$ and $(-5,0)$. 
%The optimal policy is to go to one of the four-goal positions randomly, %we have shown the contour of the reward curve in Figure \ref{fig:comparsion-multi-goal}.
%and a reasonable policy should be able to take actions uniformly to those four goal positions with the same probability due to the symmetry of the reward curve, which characters the capacity of exploration of a policy to understand the environment. 

\subsubsection{Task} 
We design a simple “multi-goal” environment according to the \emph{Didactic Example} \citep{haarnoja2017reinforcement}, in which the agent is a 2D point mass on the $7\times7$ plane, and the agent tries to reach one of four points $(0,5)$, $(0,-5)$, $(5,0)$ and $(-5,0)$ symmetrically placed goals.

\begin{figure*}[t]
    \centering
    {\includegraphics[width=3cm,height=3cm]{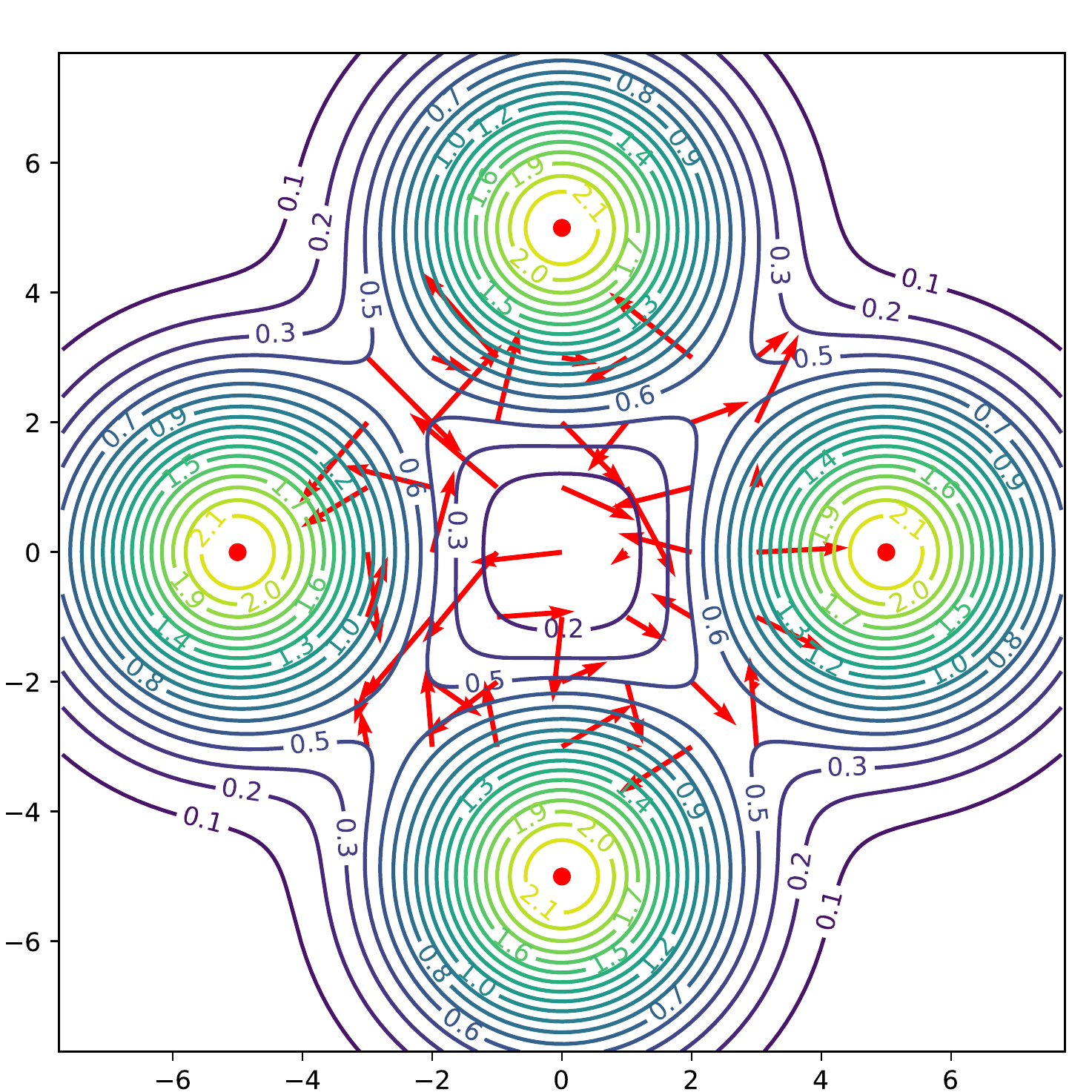}}
        {\includegraphics[width=3cm,height=3cm]{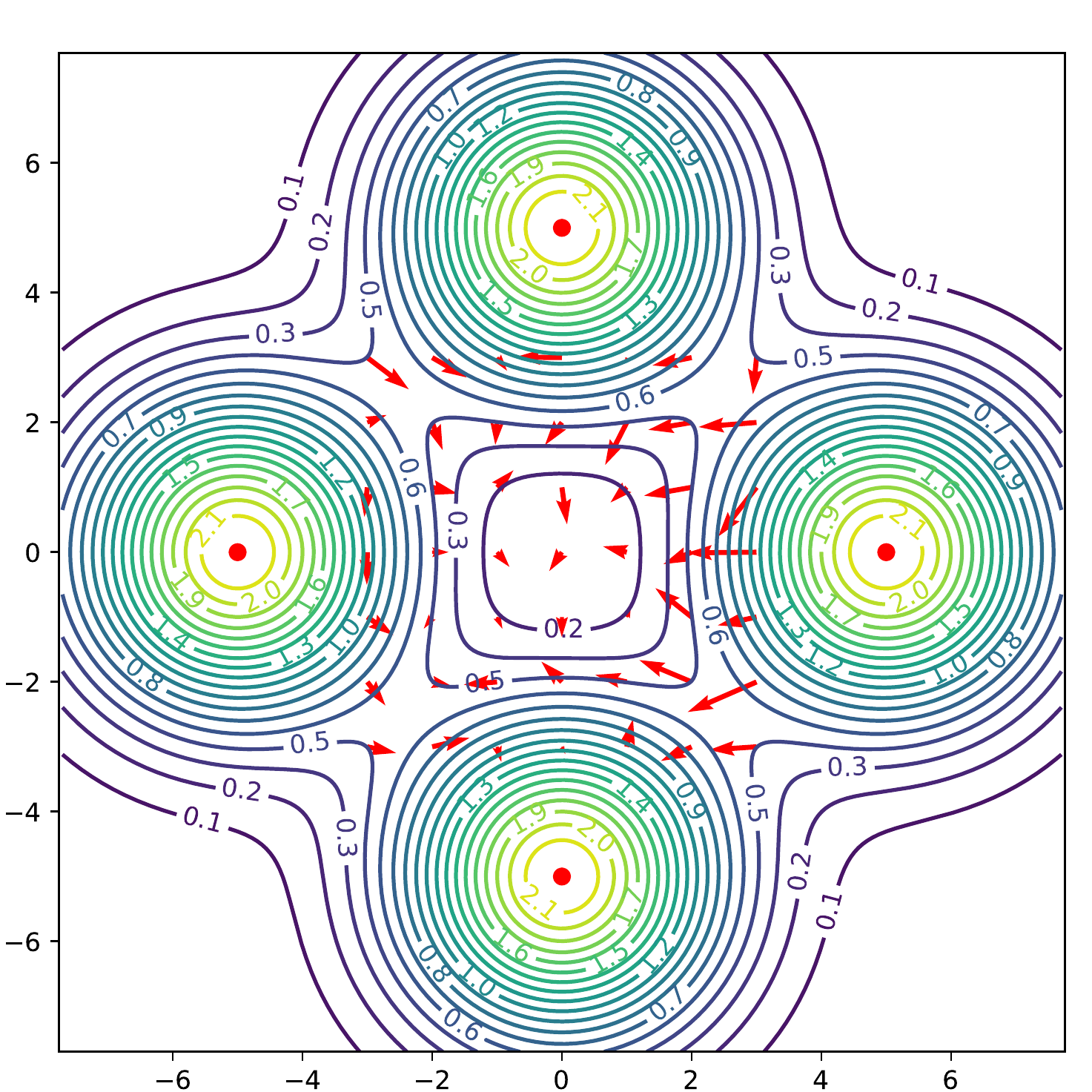}}
            {\includegraphics[width=3cm,height=3cm]{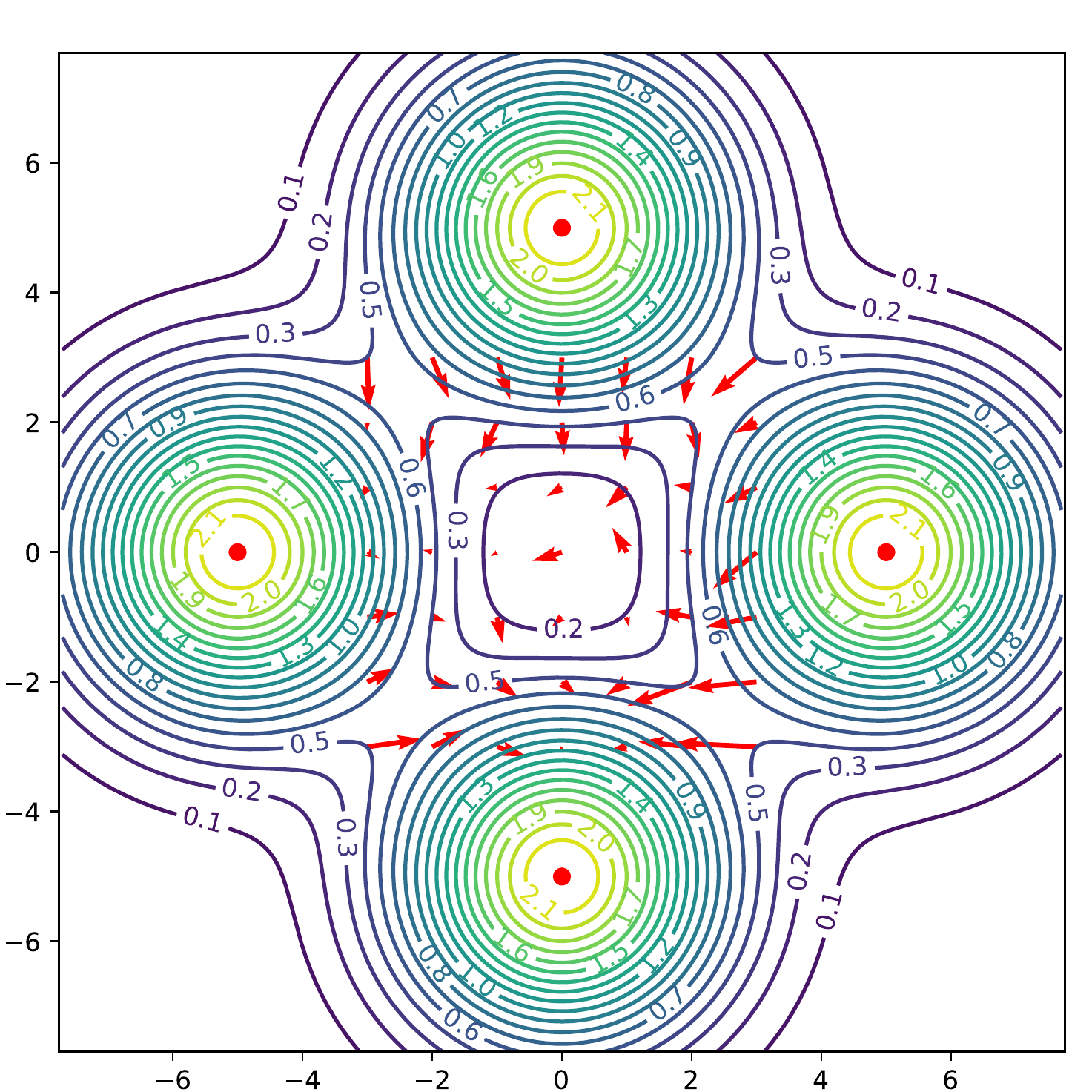}}
                {\includegraphics[width=3cm,height=3cm]{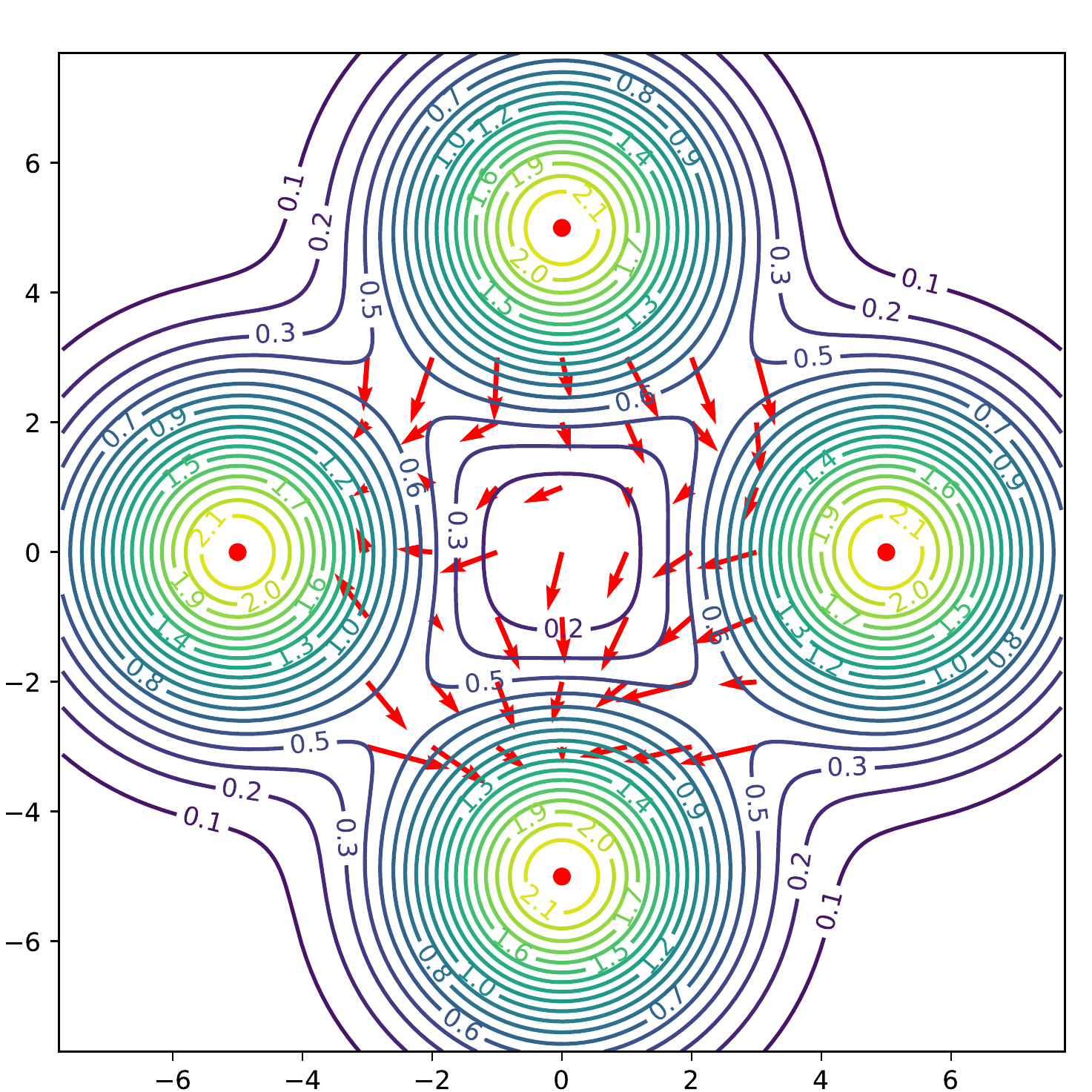}}
                    {\includegraphics[width=3cm,height=3cm]{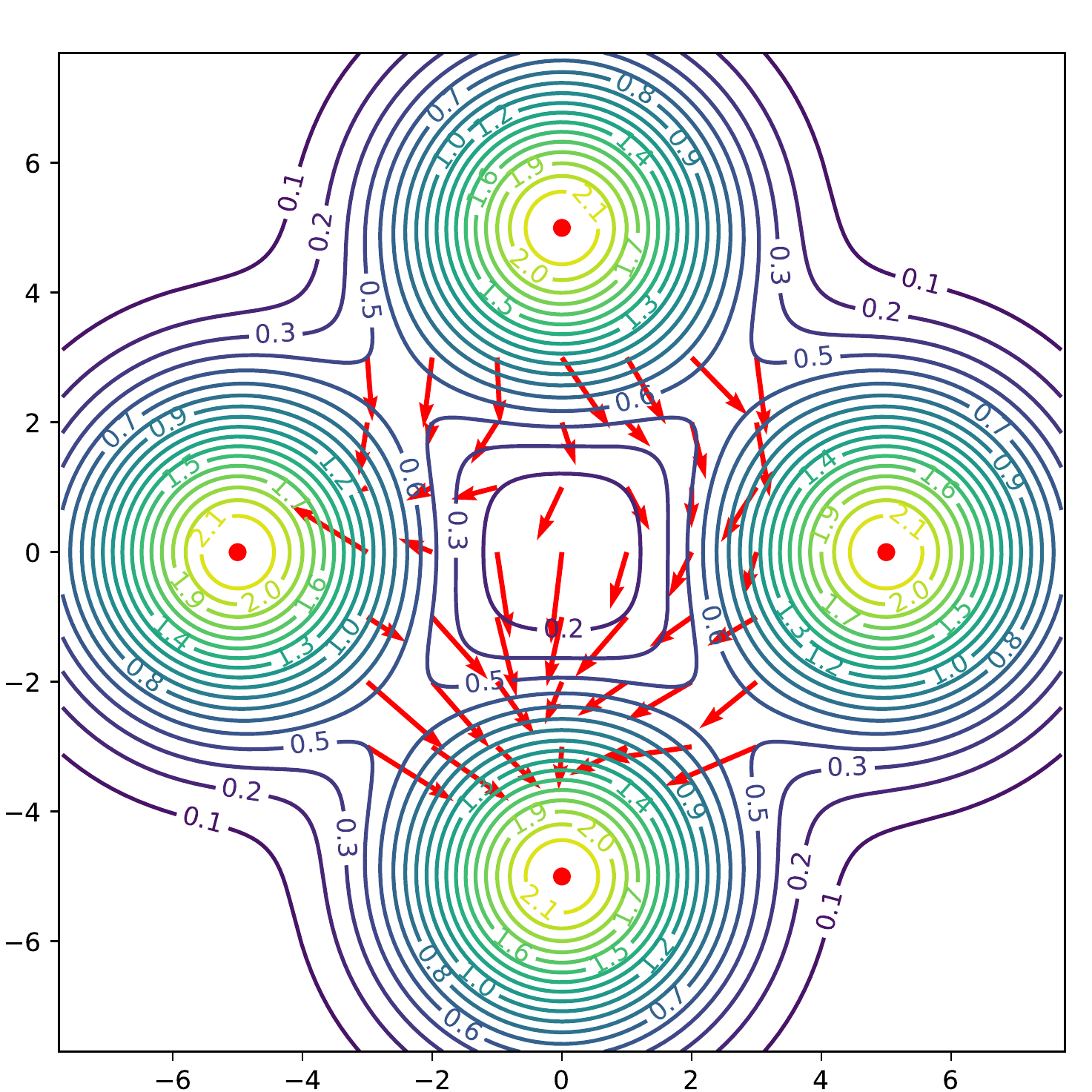}}
 \subfigure[1E$3$ iterations]
    {\includegraphics[width=3cm,height=3cm]{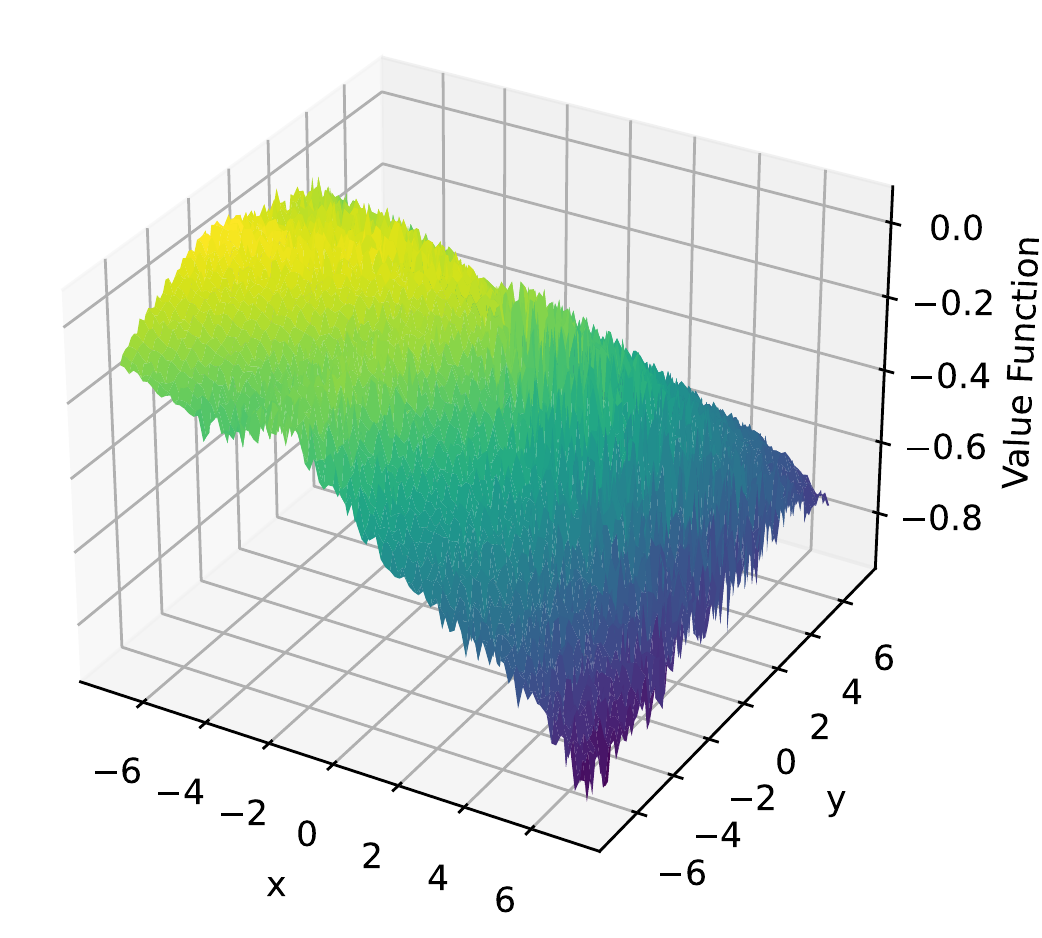}}
         \subfigure[2E$3$ iterations]
        {\includegraphics[width=3cm,height=3cm]{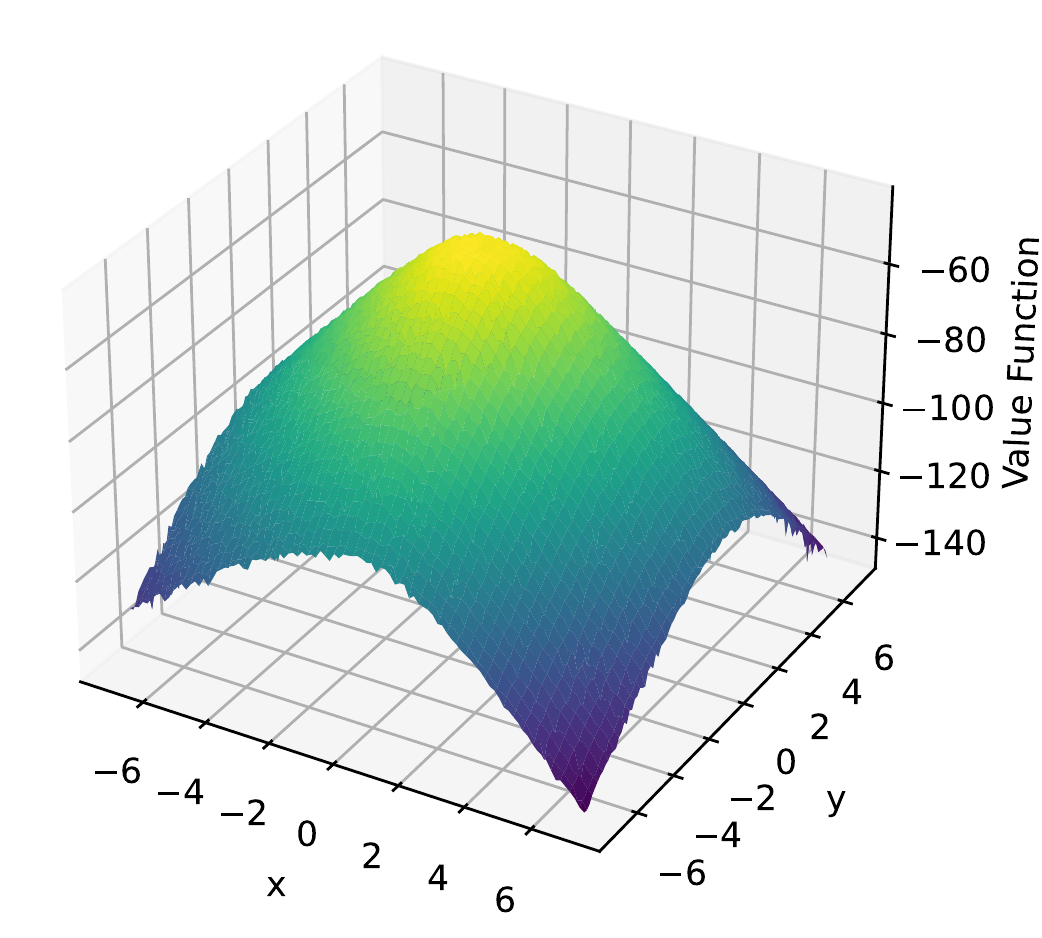}}
             \subfigure[3E$3$ iterations]
            {\includegraphics[width=3cm,height=3cm]{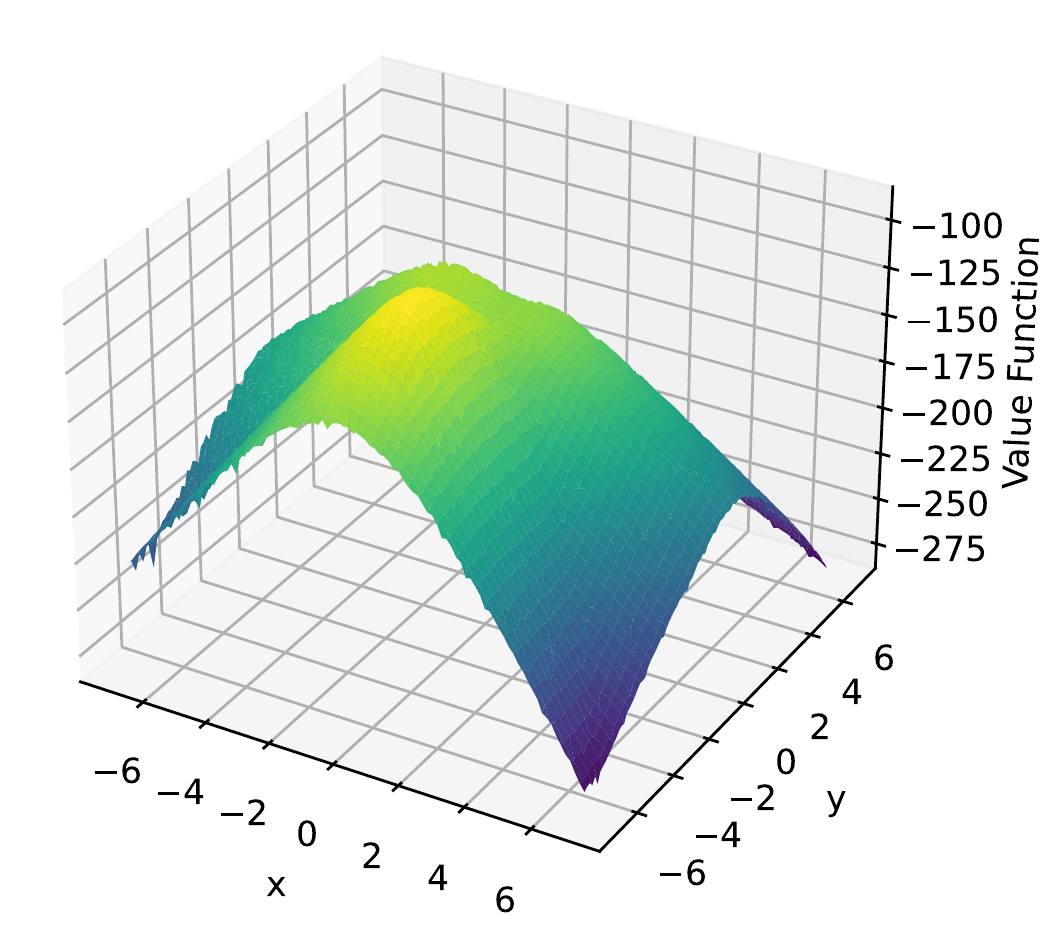}}
                 \subfigure[4E$3$ iterations]
                {\includegraphics[width=3cm,height=3cm]{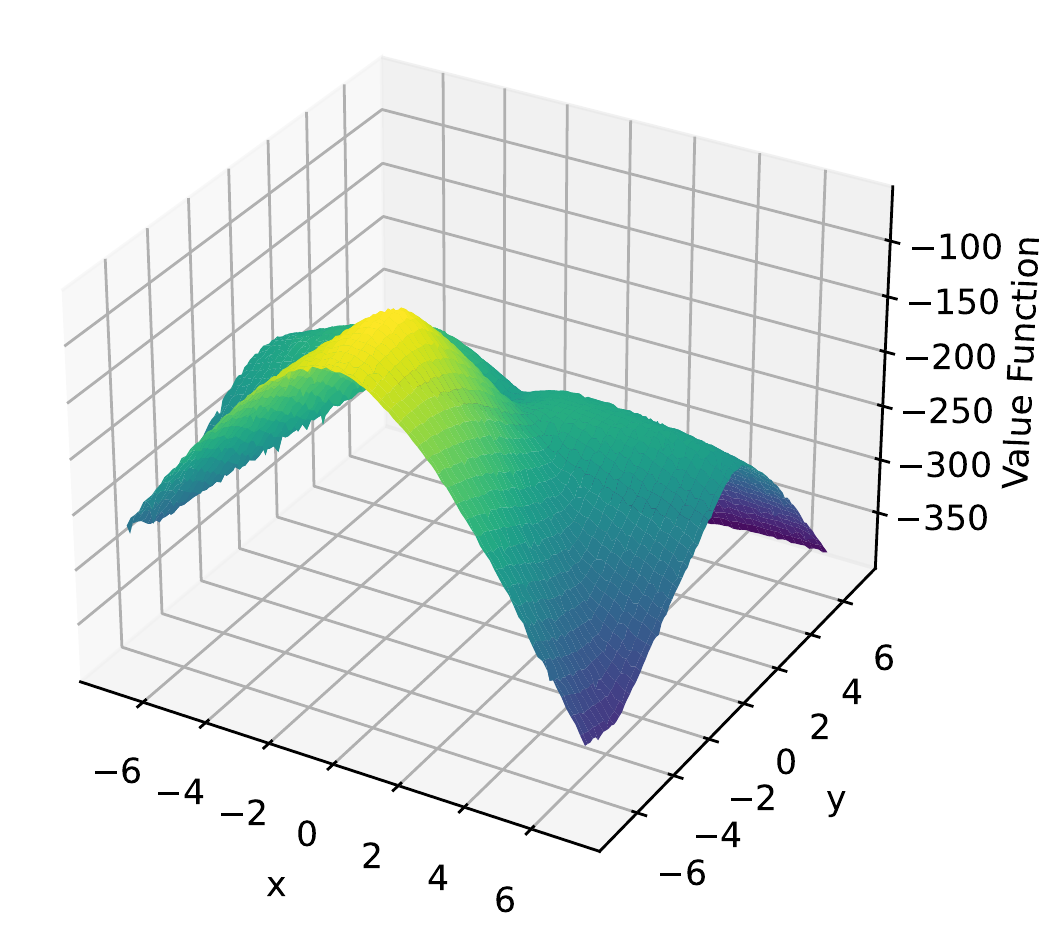}}
                     \subfigure[5E$3$ iterations]
                    {\includegraphics[width=3cm,height=3cm]{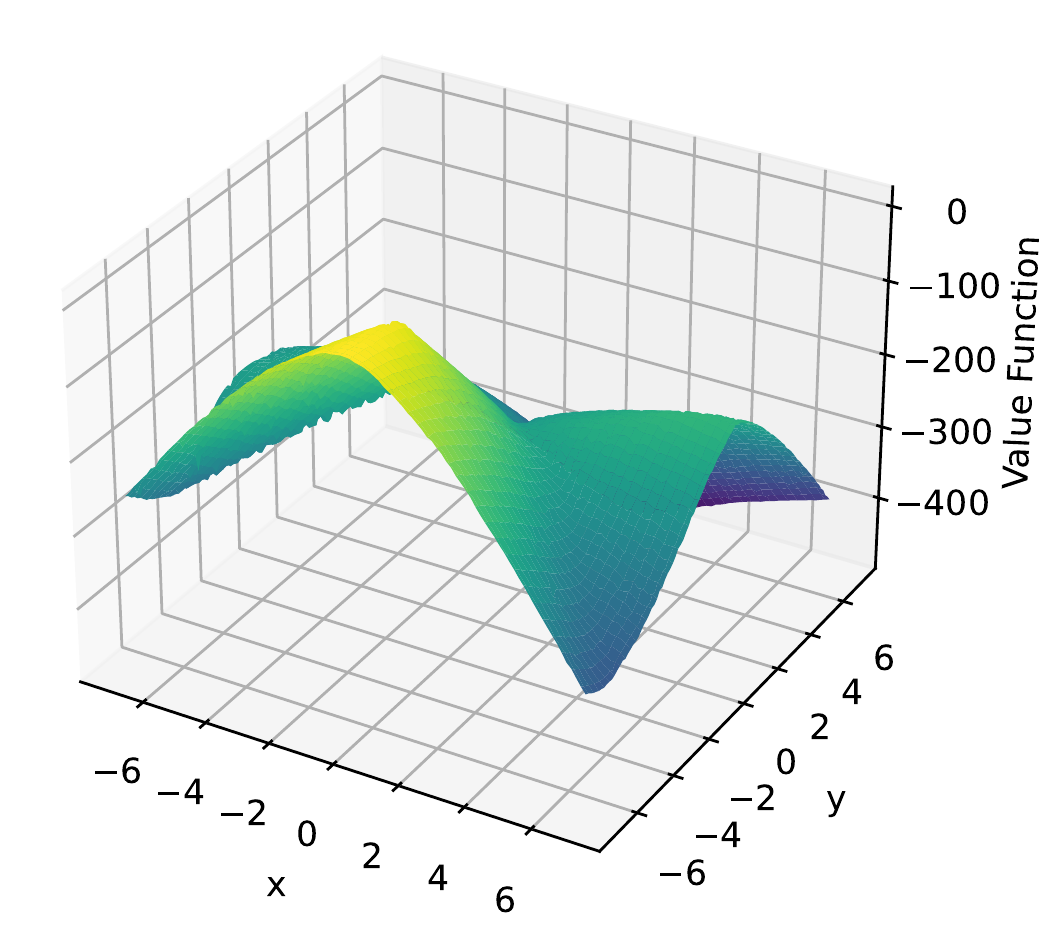}}
 {\includegraphics[width=3cm,height=3cm]{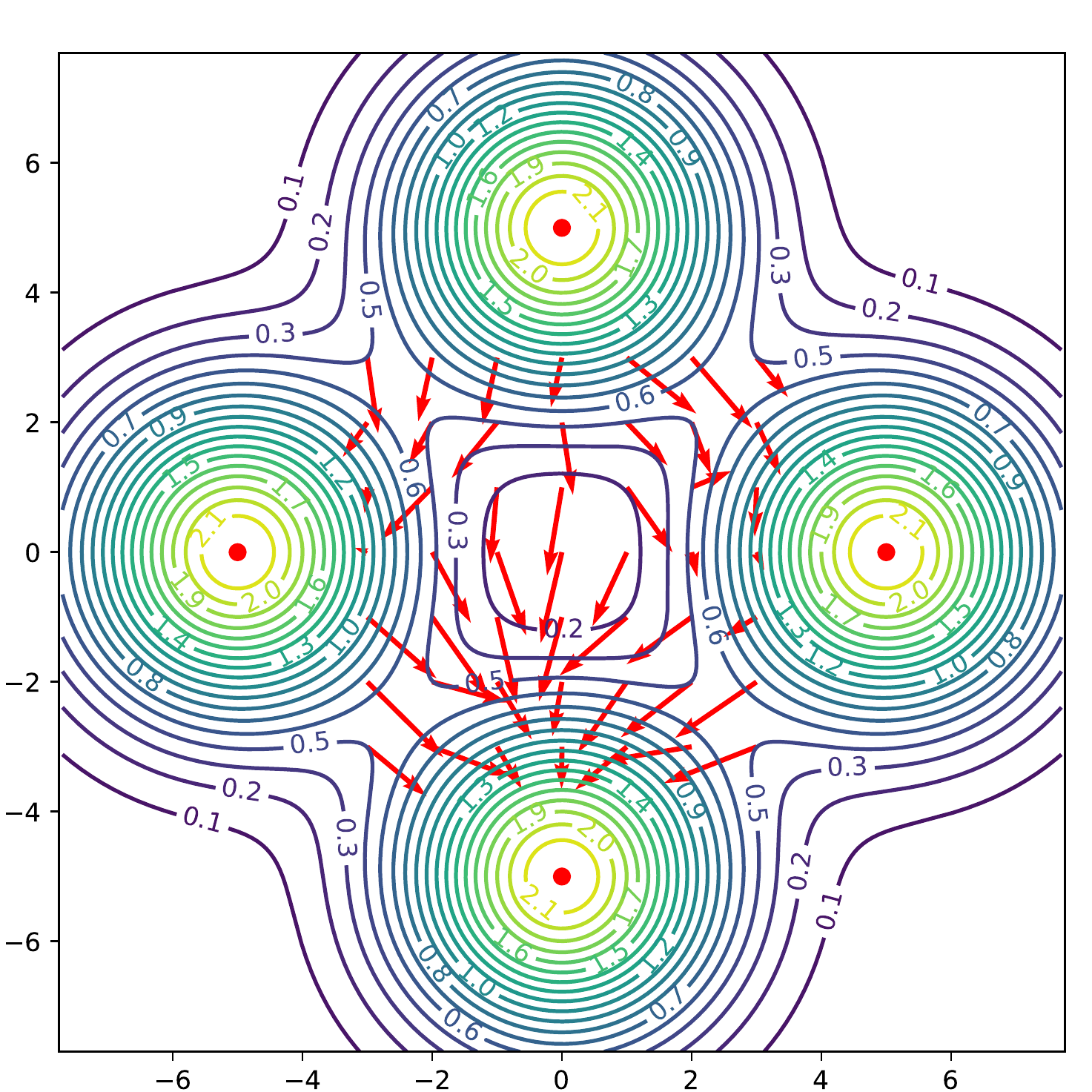}}
        {\includegraphics[width=3cm,height=3cm]{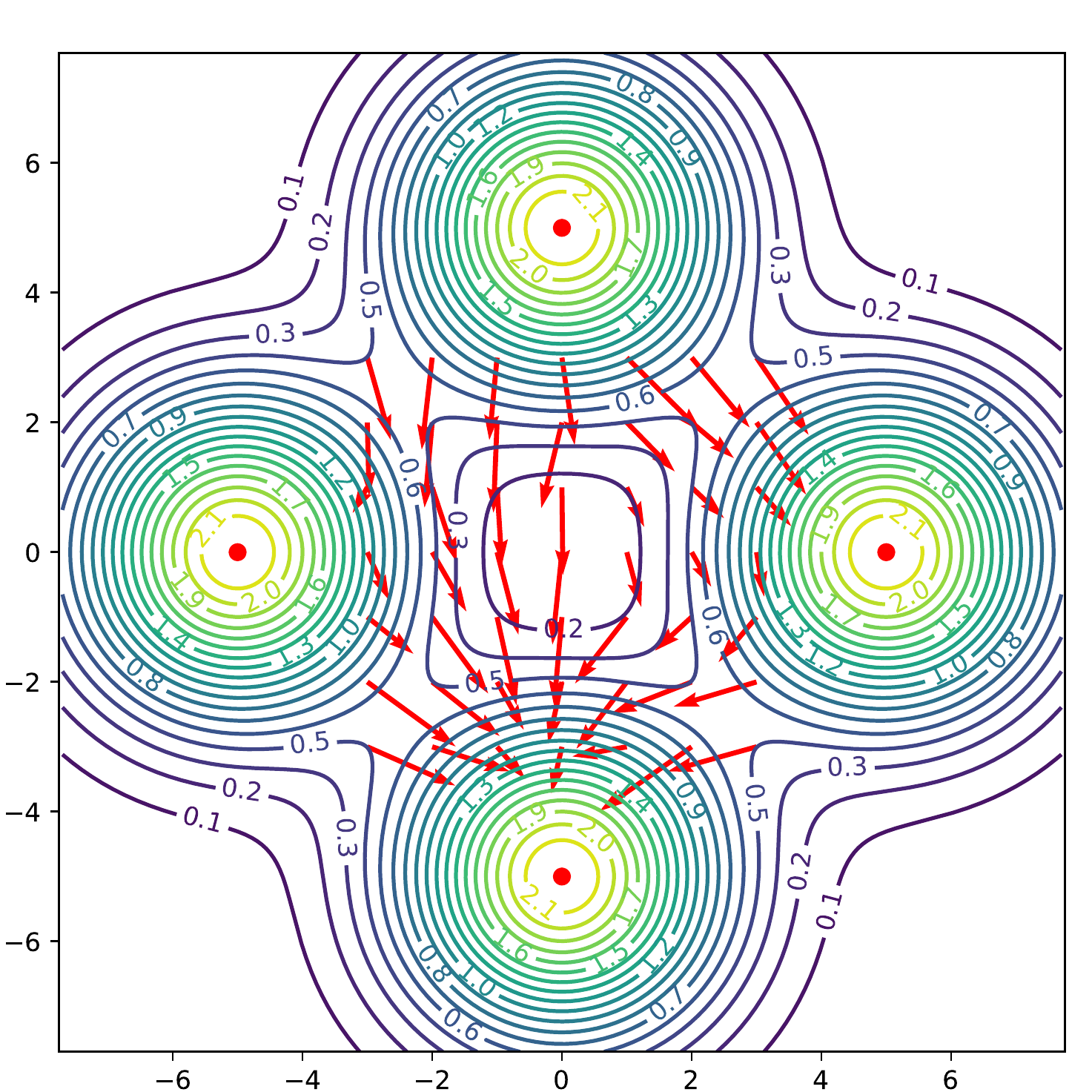}}
            {\includegraphics[width=3cm,height=3cm]{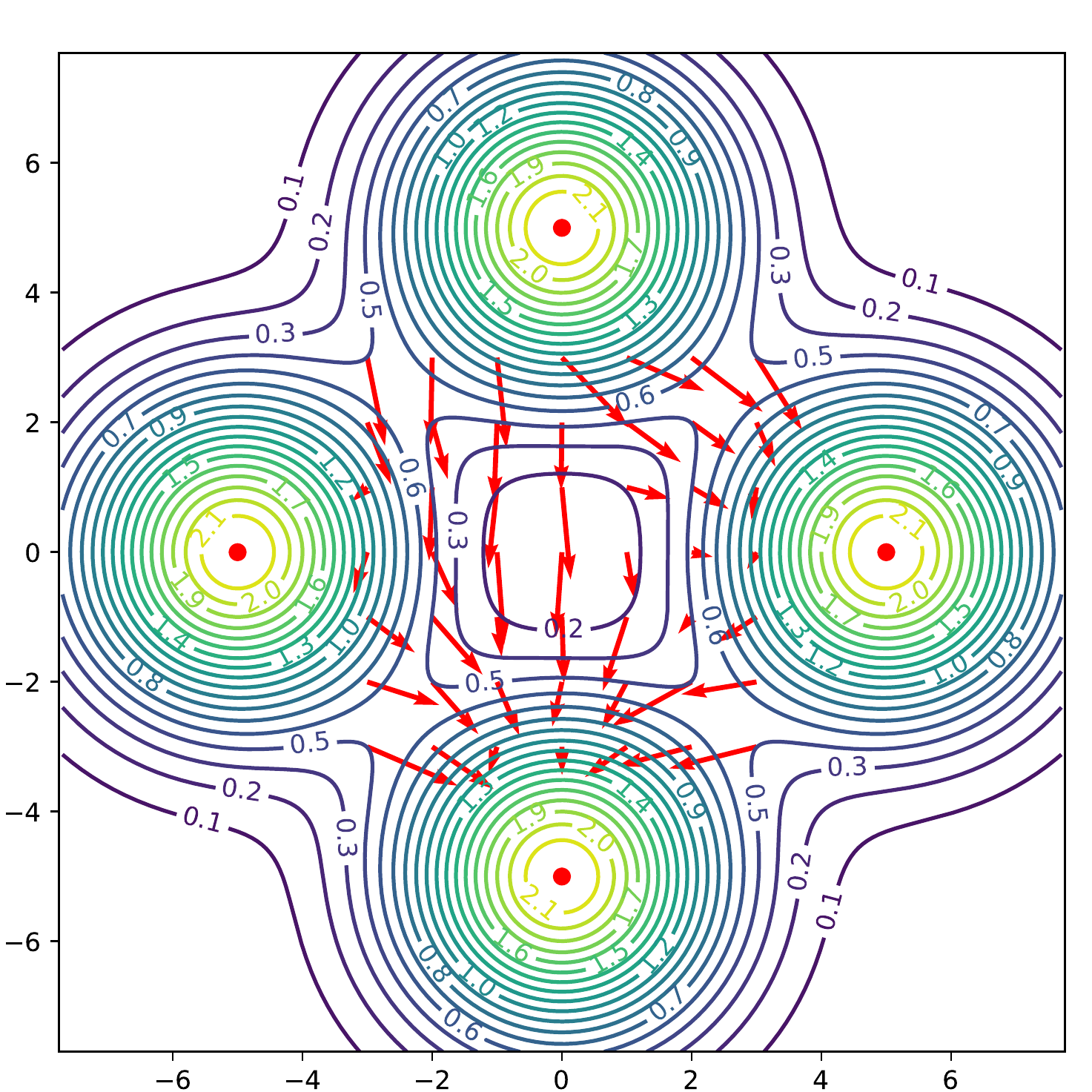}}
                {\includegraphics[width=3cm,height=3cm]{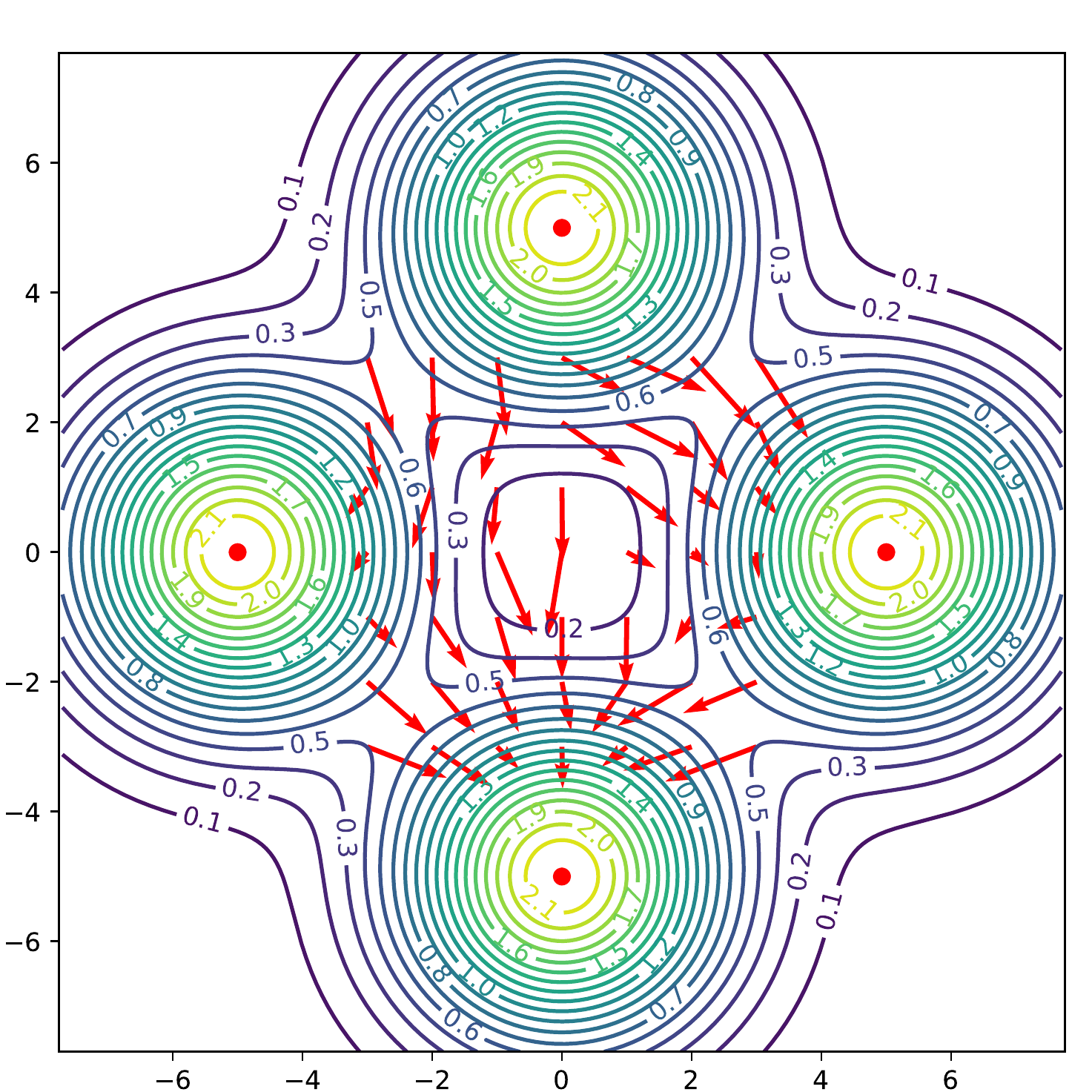}}
                    {\includegraphics[width=3cm,height=3cm]{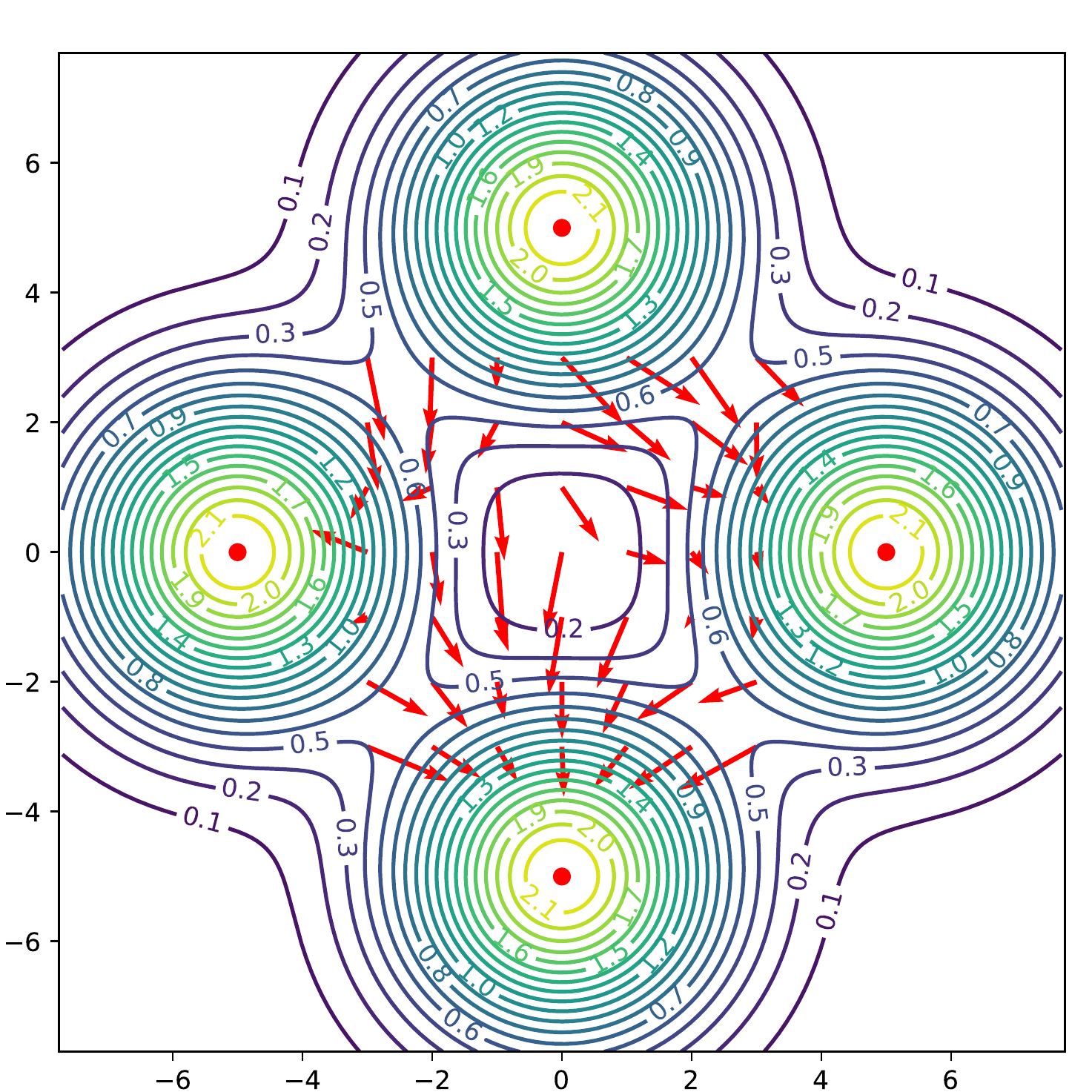}}
 \subfigure[6E$3$ iterations]
    {\includegraphics[width=3cm,height=3cm]{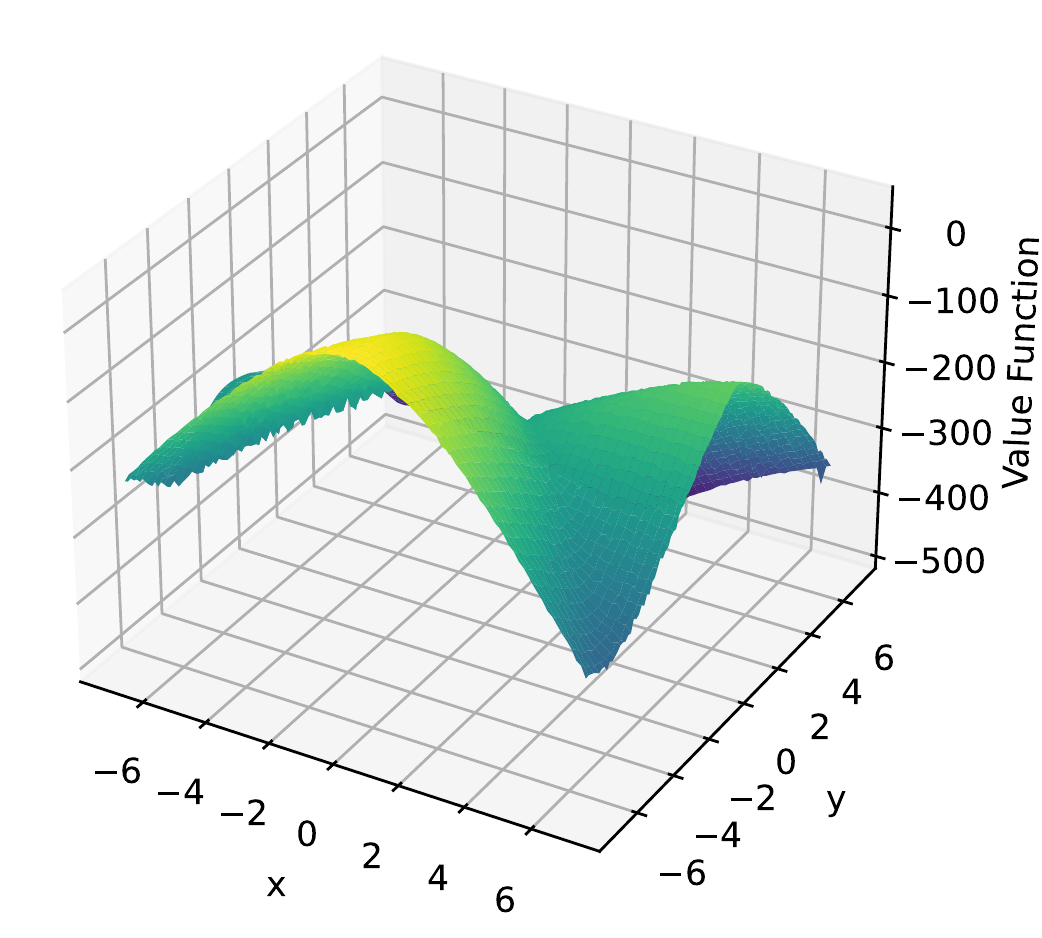}}
         \subfigure[7E$3$ iterations]
        {\includegraphics[width=3cm,height=3cm]{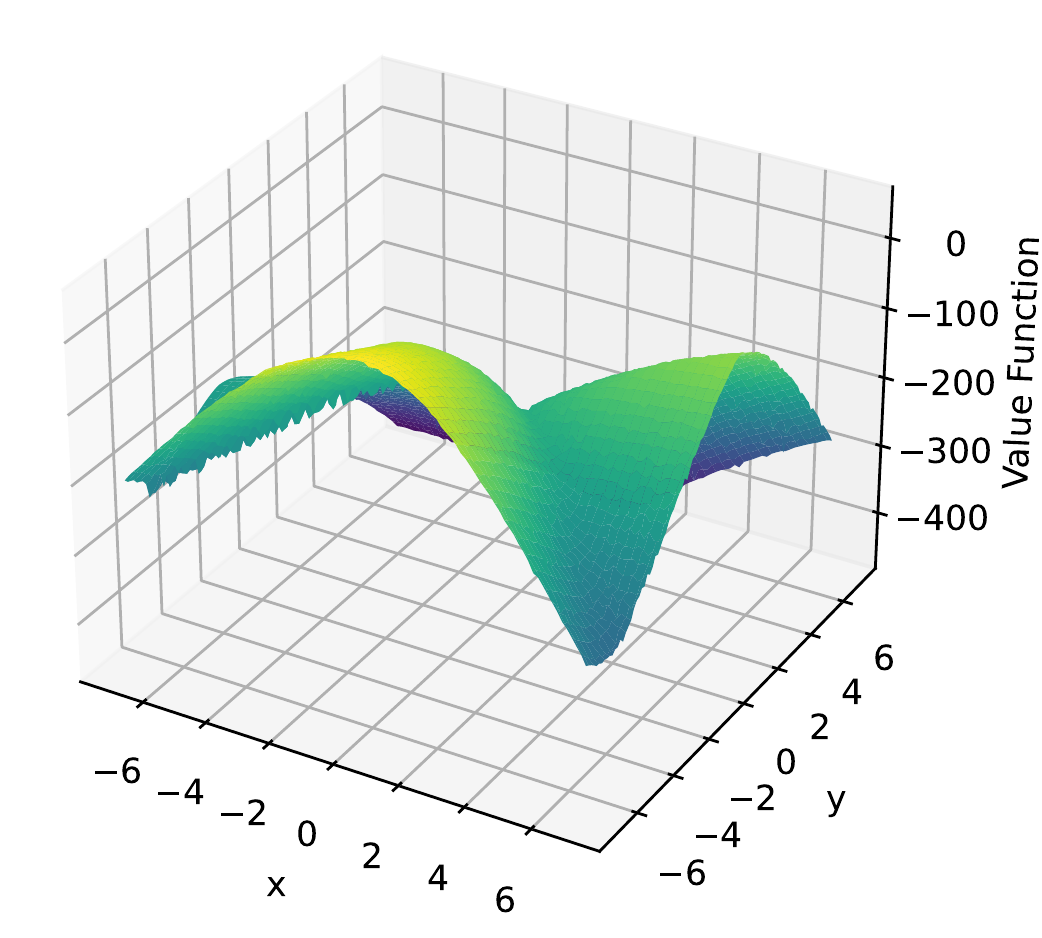}}
             \subfigure[8E$3$ iterations]
            {\includegraphics[width=3cm,height=3cm]{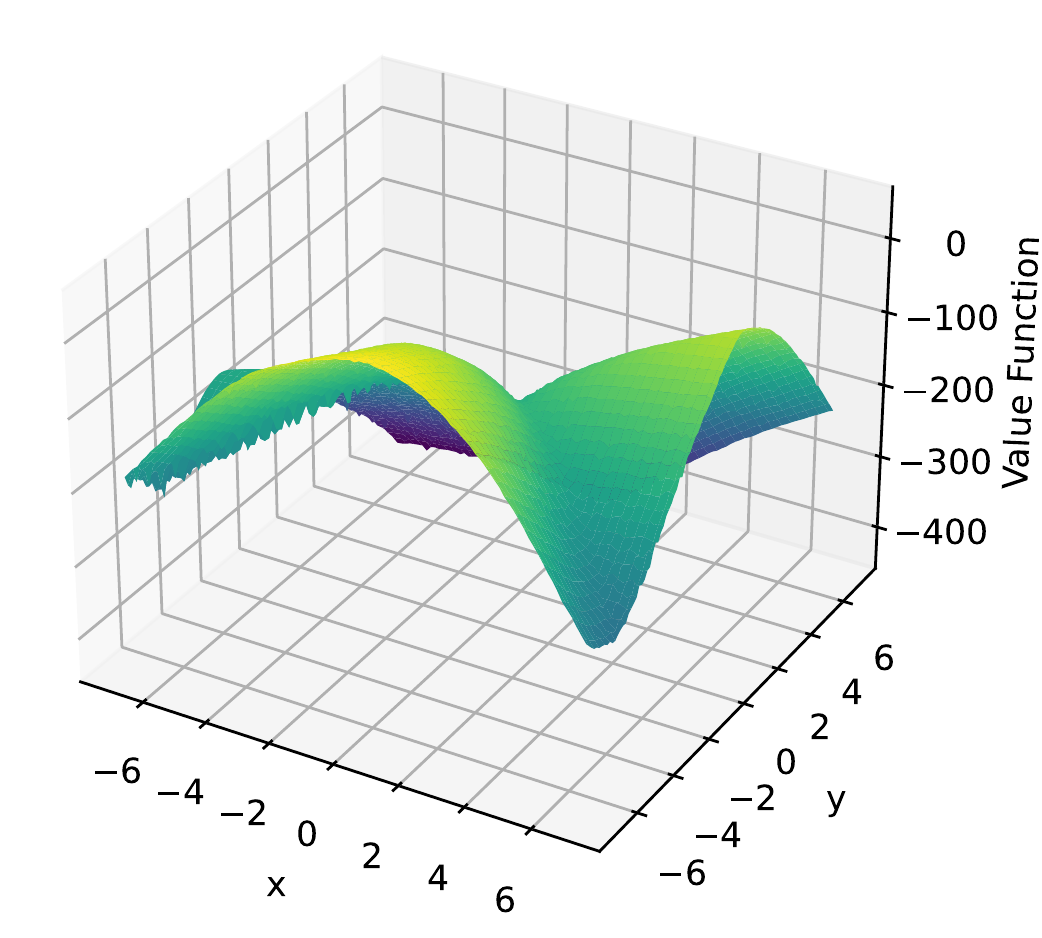}}
                 \subfigure[9E$3$ iterations]
                {\includegraphics[width=3cm,height=3cm]{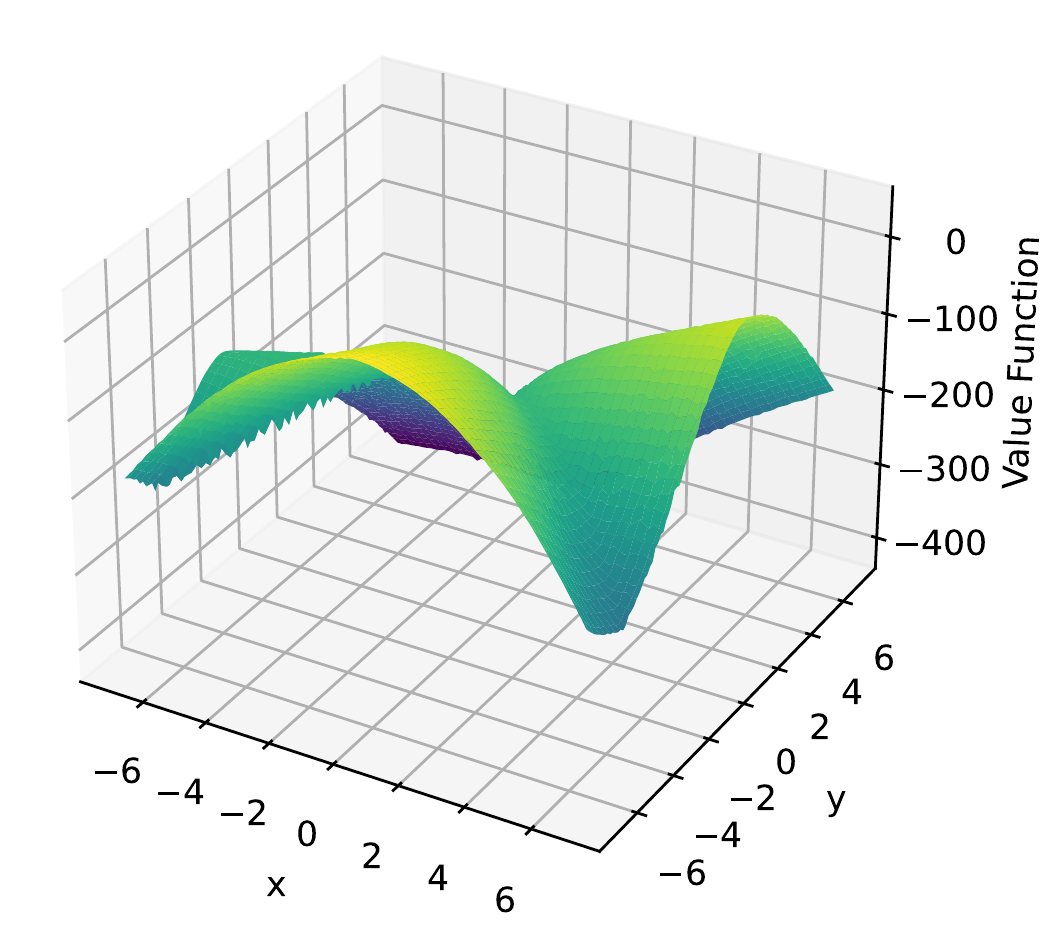}}
                     \subfigure[10E$3$ iterations]
                    {\includegraphics[width=3cm,height=3cm]{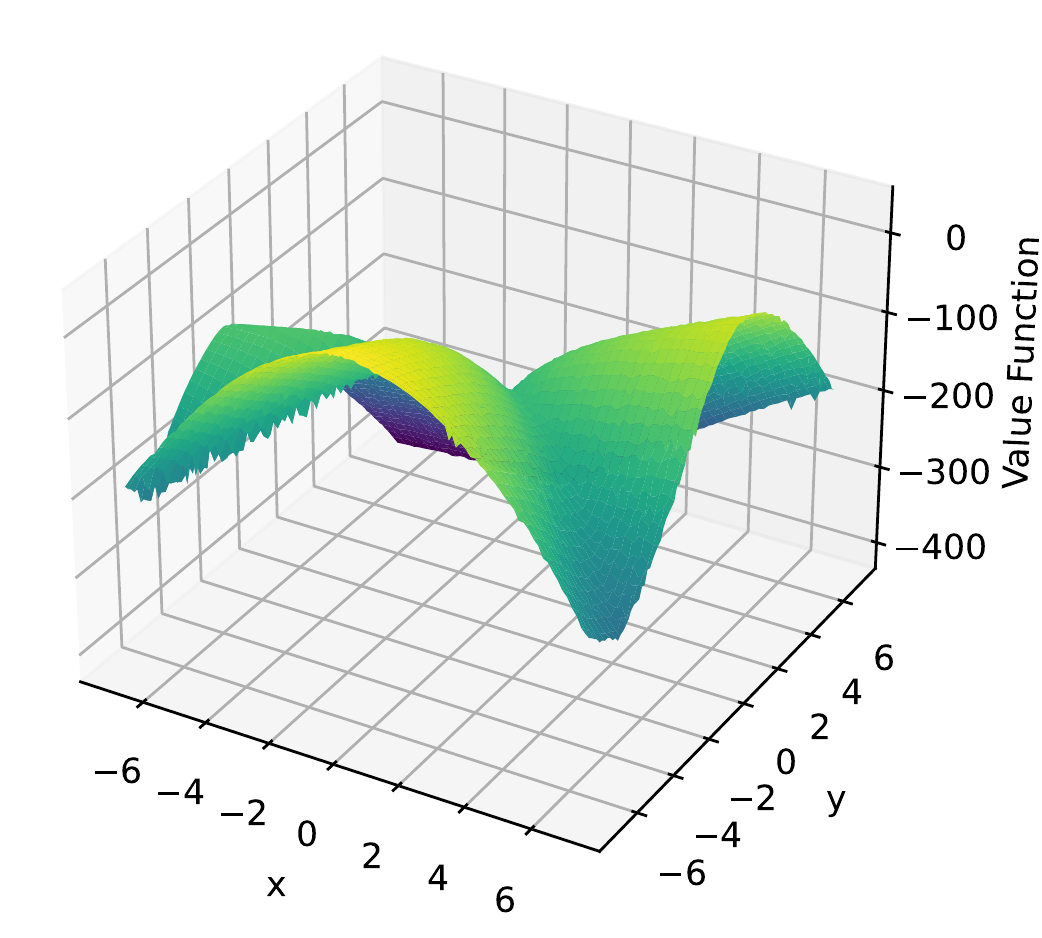}}
    \caption
    {Policy representation comparison of SAC with different iterations.
    %and the optimal policy is to go to one of the goal positions randomly.
    %, and we have shown the solid red lines as the final action learned by different algorithms with 1000 iterations.
    }
    \label{app-fig-sac-policy}
\end{figure*}

\subsubsection{Reward} The reward is defined according to the following three parts:
\[
R=r_{1}+r_{2}+r_{3},
\]
where 
\begin{itemize}
\item $r_{1}\propto -\|\ba\|_{2}^{2}$ if agent plays the action $\ba$;
\item $r_{2}=-\min\{\|(x,y)-\mathbf{target}\|^{2}_{2}\}$, $\mathbf{target}$ denotes one of the target points $(0,5)$, $(0,-5)$, $(5,0)$, and $(-5,0)$; 
\item if the agent reaches one of the targets among $\{(0,5),~(0,-5),~(5,0),~(-5,0)\}$, then it receives a reward $r_{3}=10$.
\end{itemize}
Since the goal positions are symmetrically distributed at the four points $(0,5)$, $(0,-5)$, $(5,0)$ and $(-5,0)$, 
a reasonable policy should be able to take actions uniformly to those four goal positions with the same probability, which characters the capacity of exploration of a policy to understand the environment. Furthermore, we know that the shape of the reward curve should be symmetrical with four equal peaks.

\subsection{Plots Details of Visualization}

\begin{figure*}[t]
    \centering
    {\includegraphics[width=3cm,height=3cm]{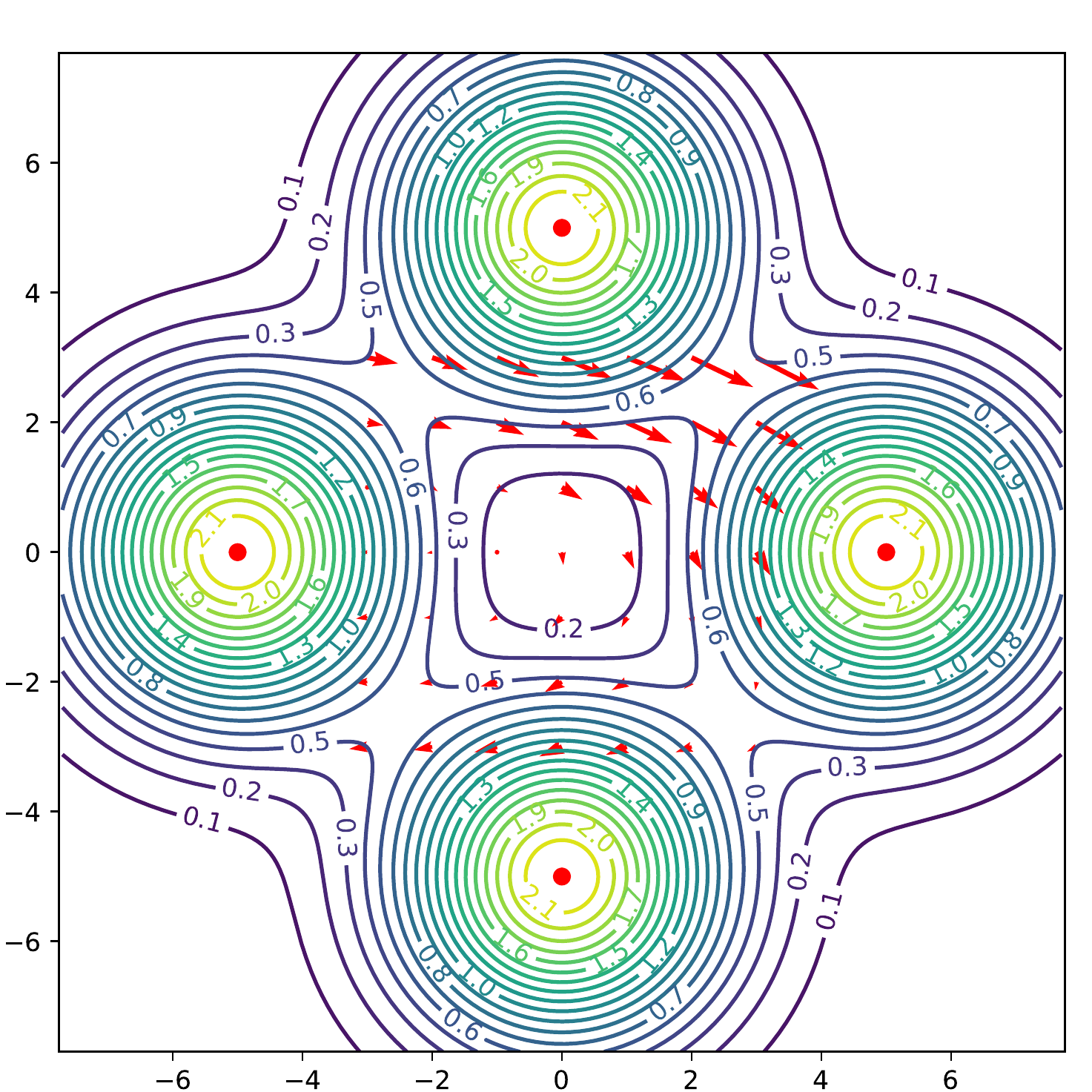}}
        {\includegraphics[width=3cm,height=3cm]{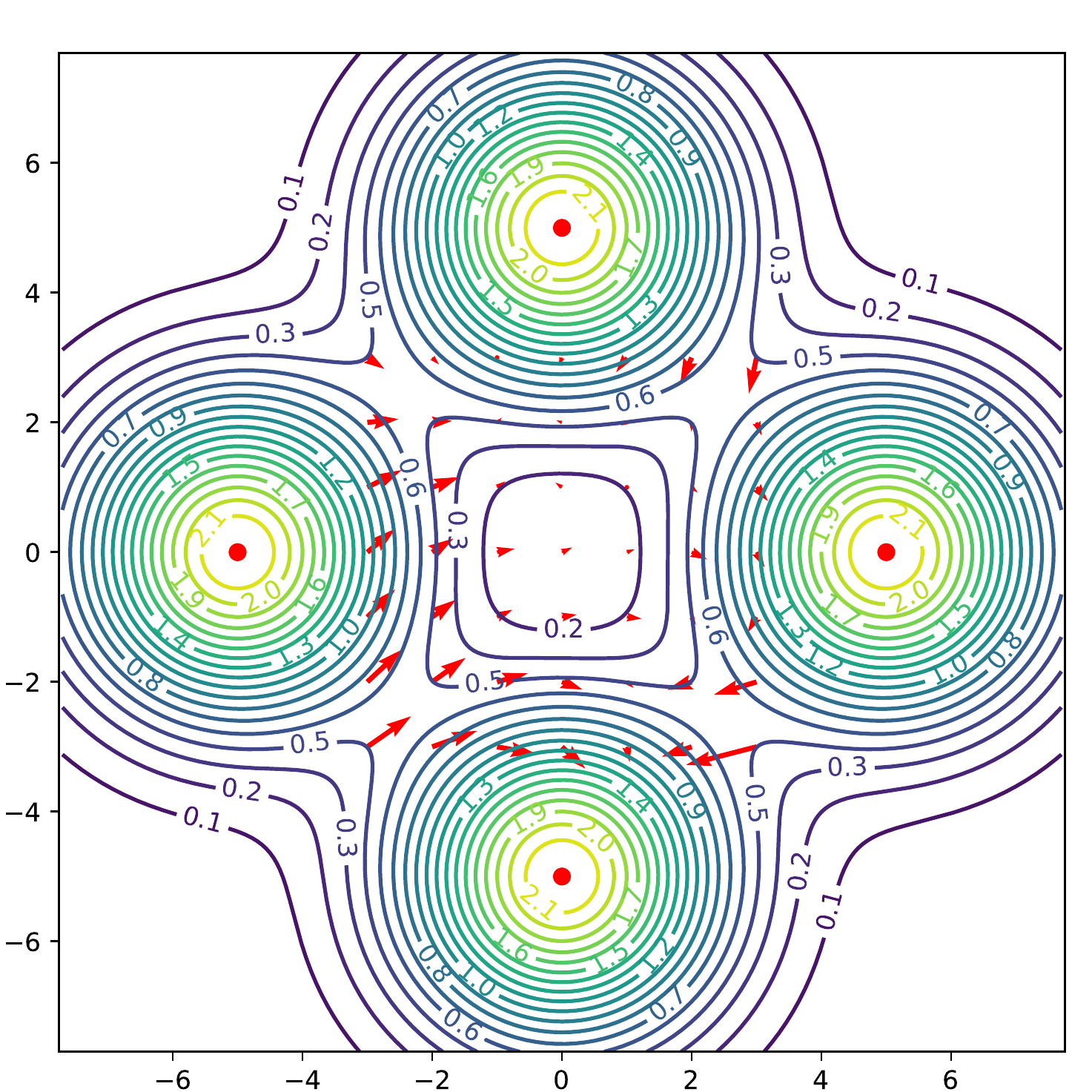}}
            {\includegraphics[width=3cm,height=3cm]{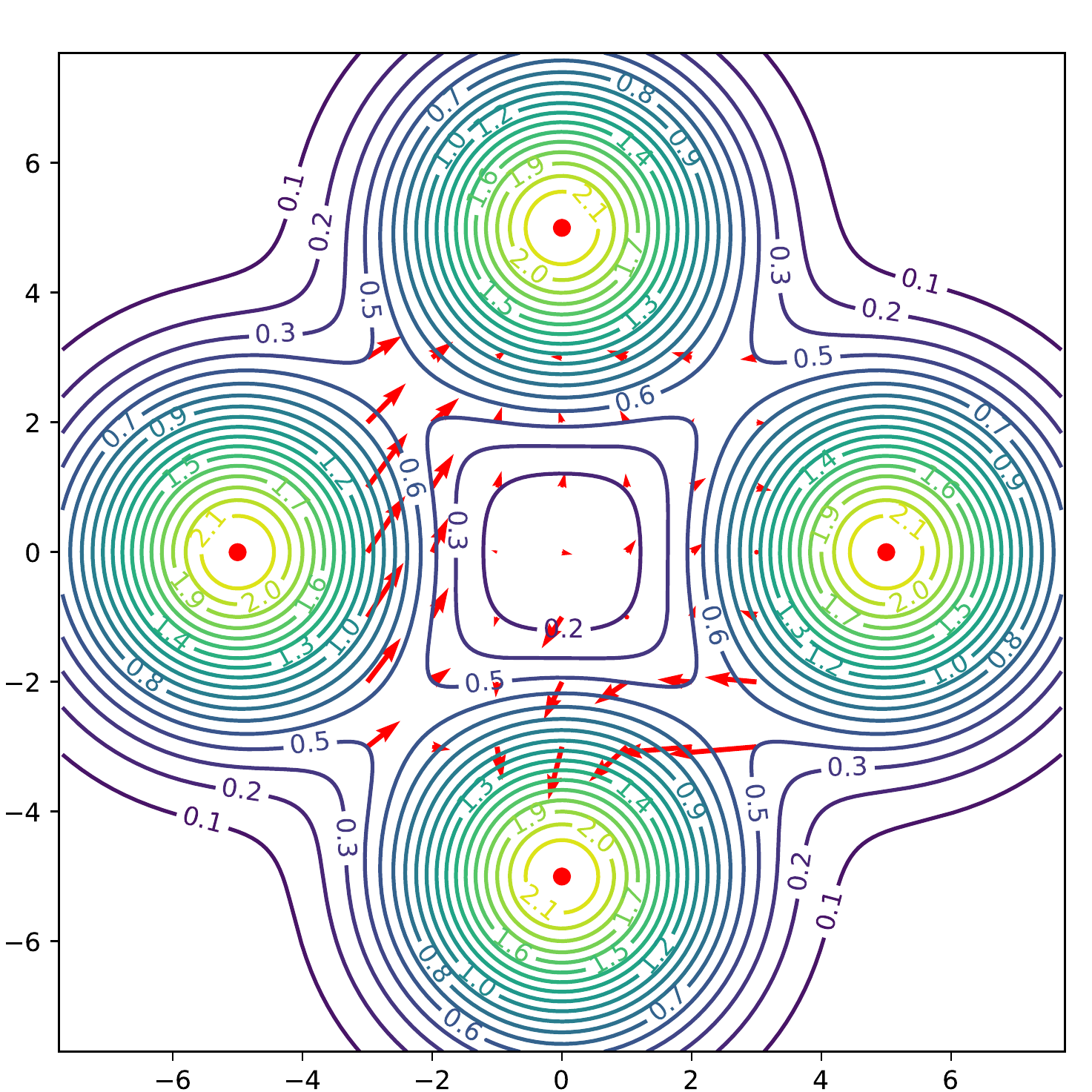}}
                {\includegraphics[width=3cm,height=3cm]{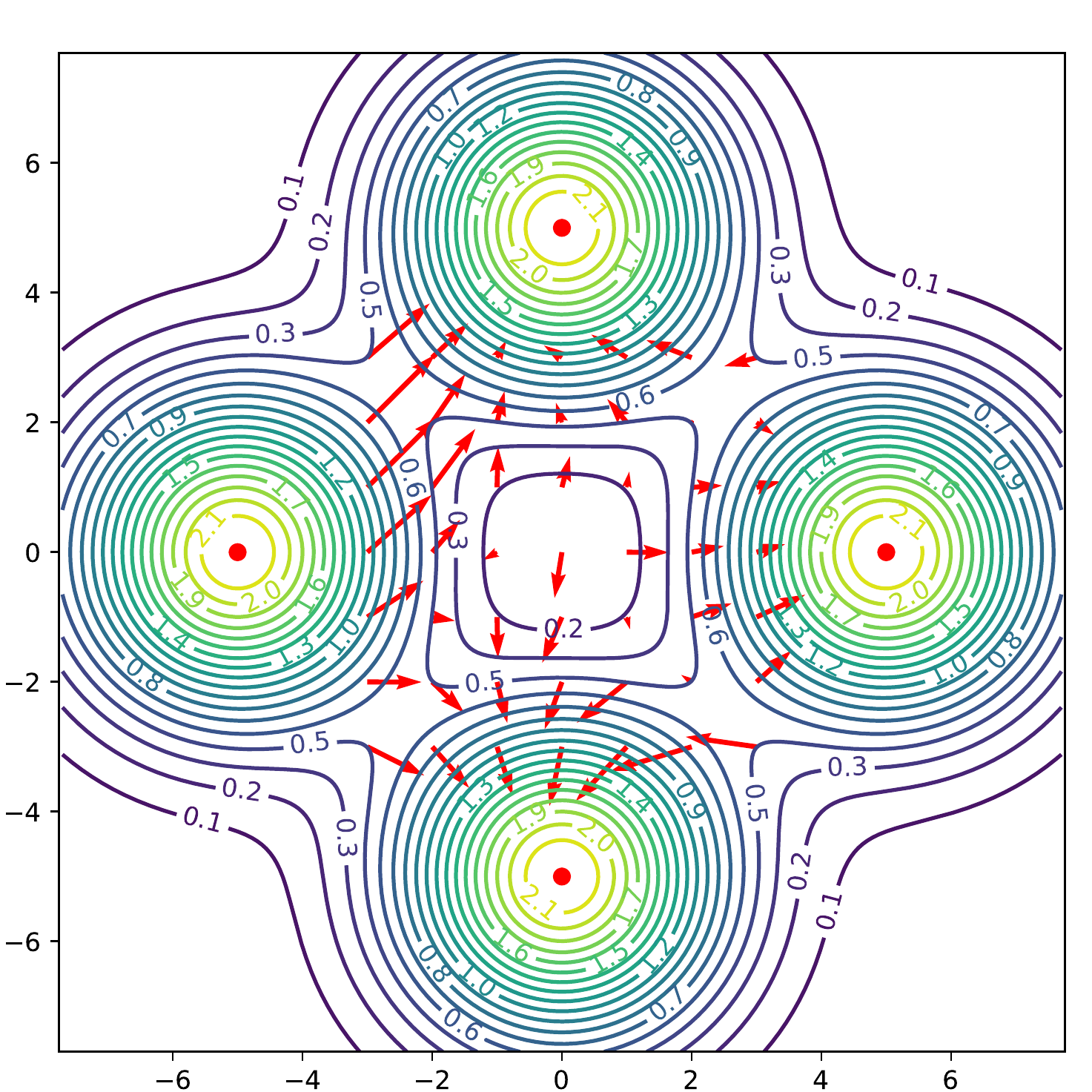}}
                    {\includegraphics[width=3cm,height=3cm]{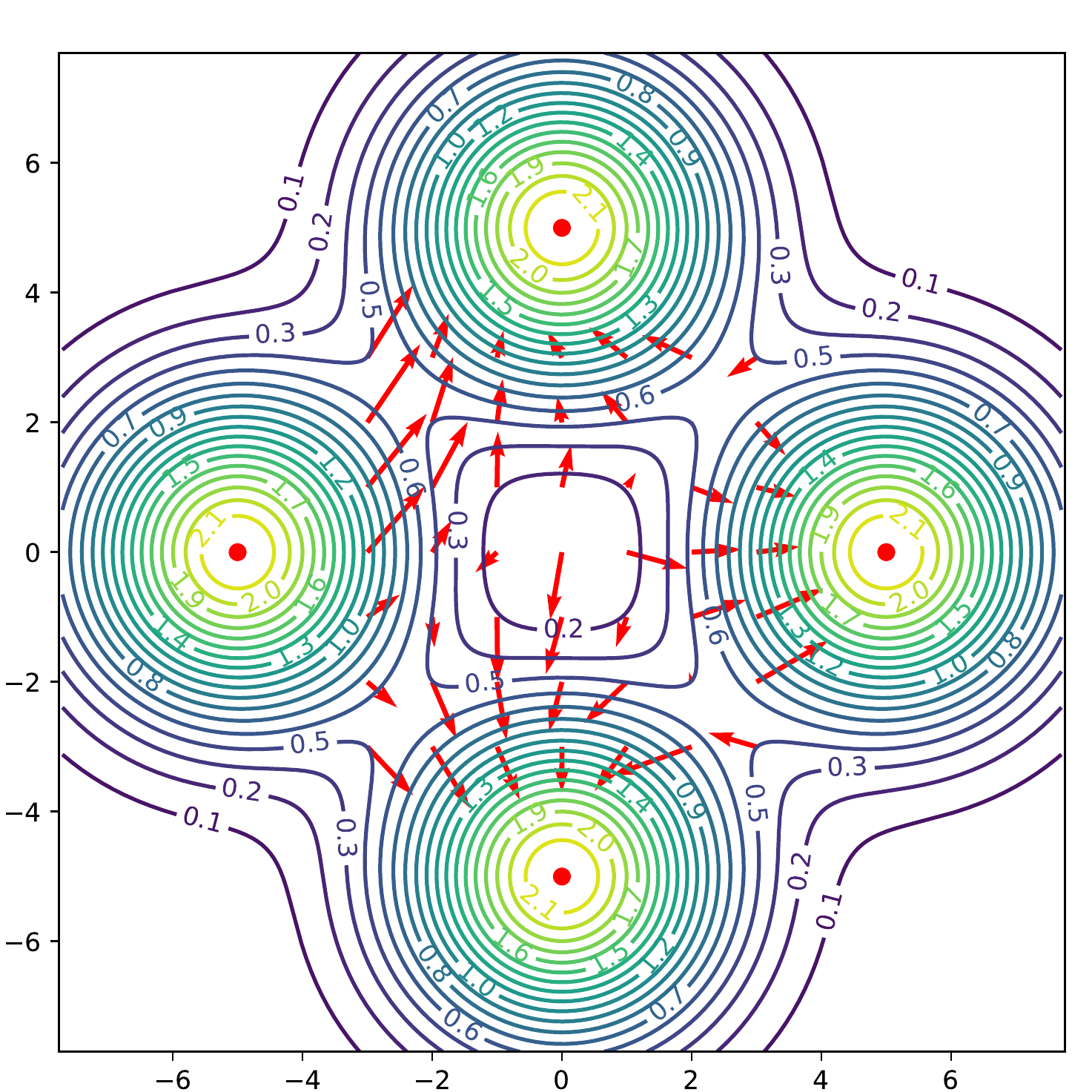}}
 \subfigure[1E$3$ iterations]
    {\includegraphics[width=3cm,height=3cm]{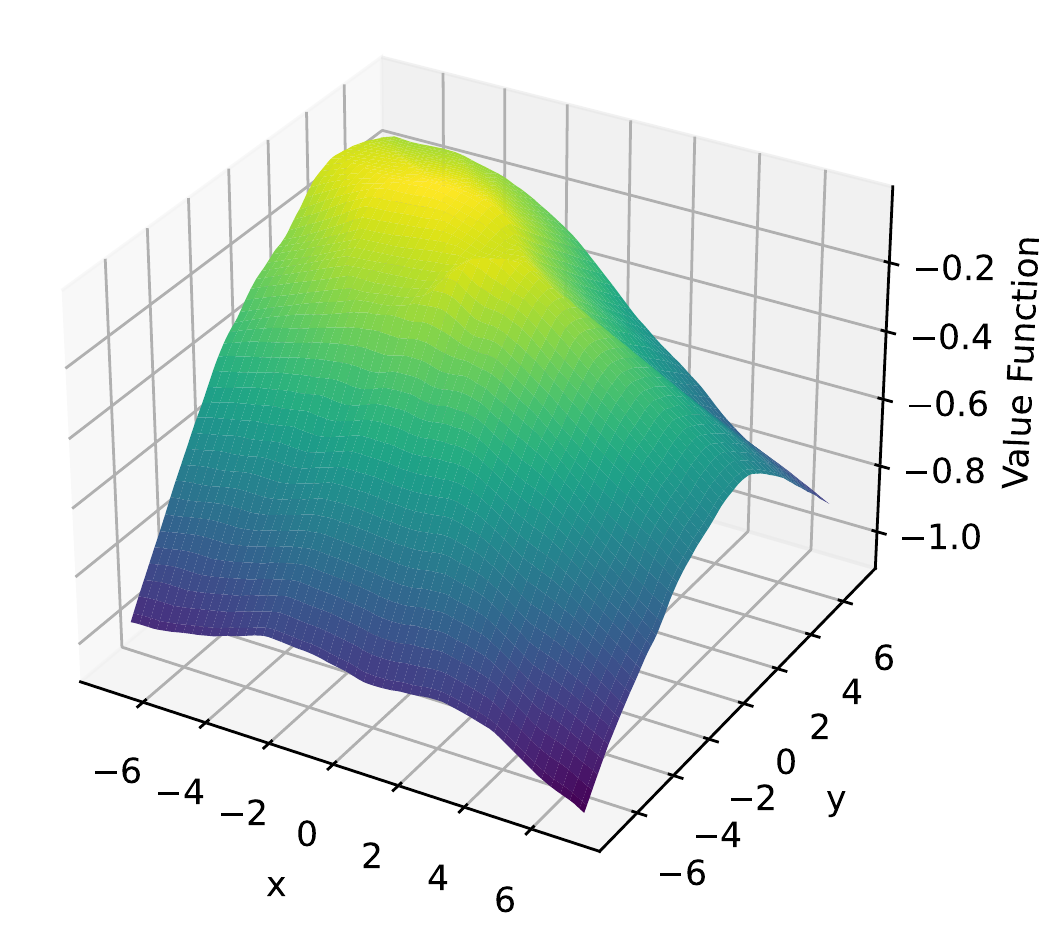}}
         \subfigure[2E$3$ iterations]
        {\includegraphics[width=3cm,height=3cm]{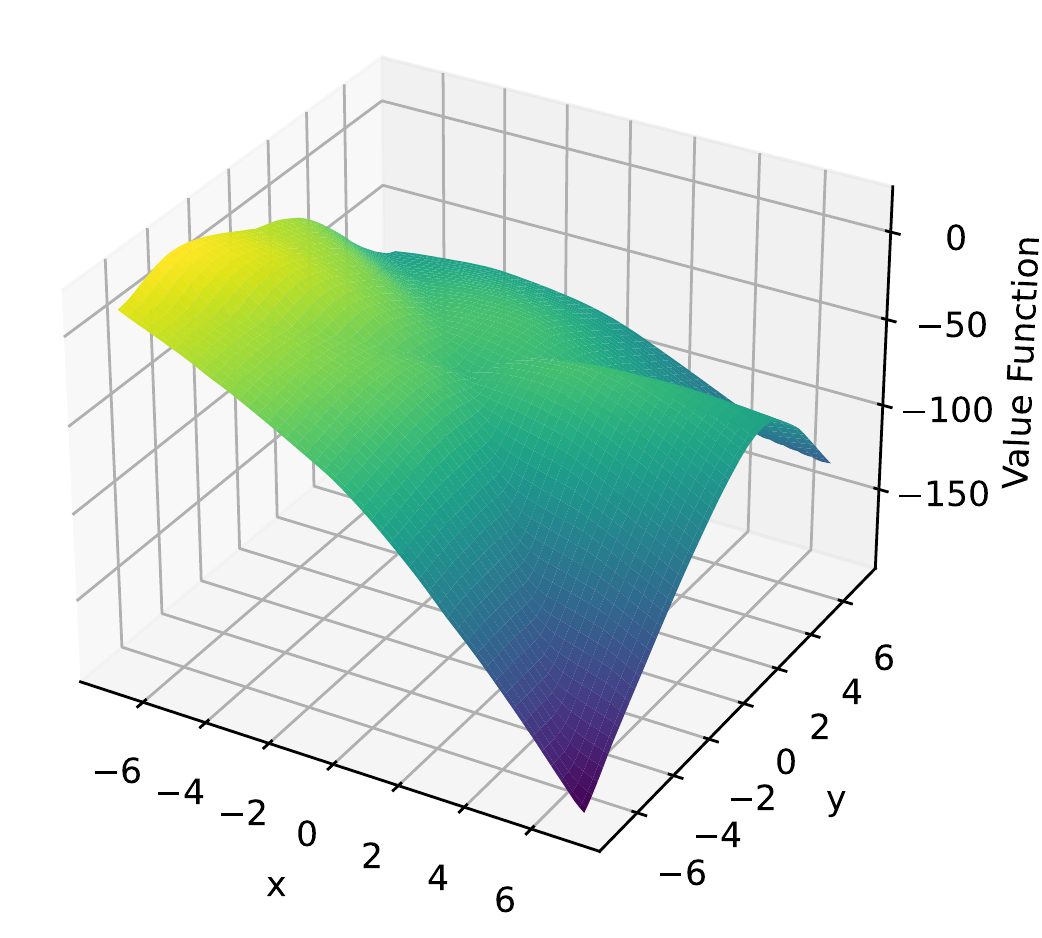}}
             \subfigure[3E$3$ iterations]
            {\includegraphics[width=3cm,height=3cm]{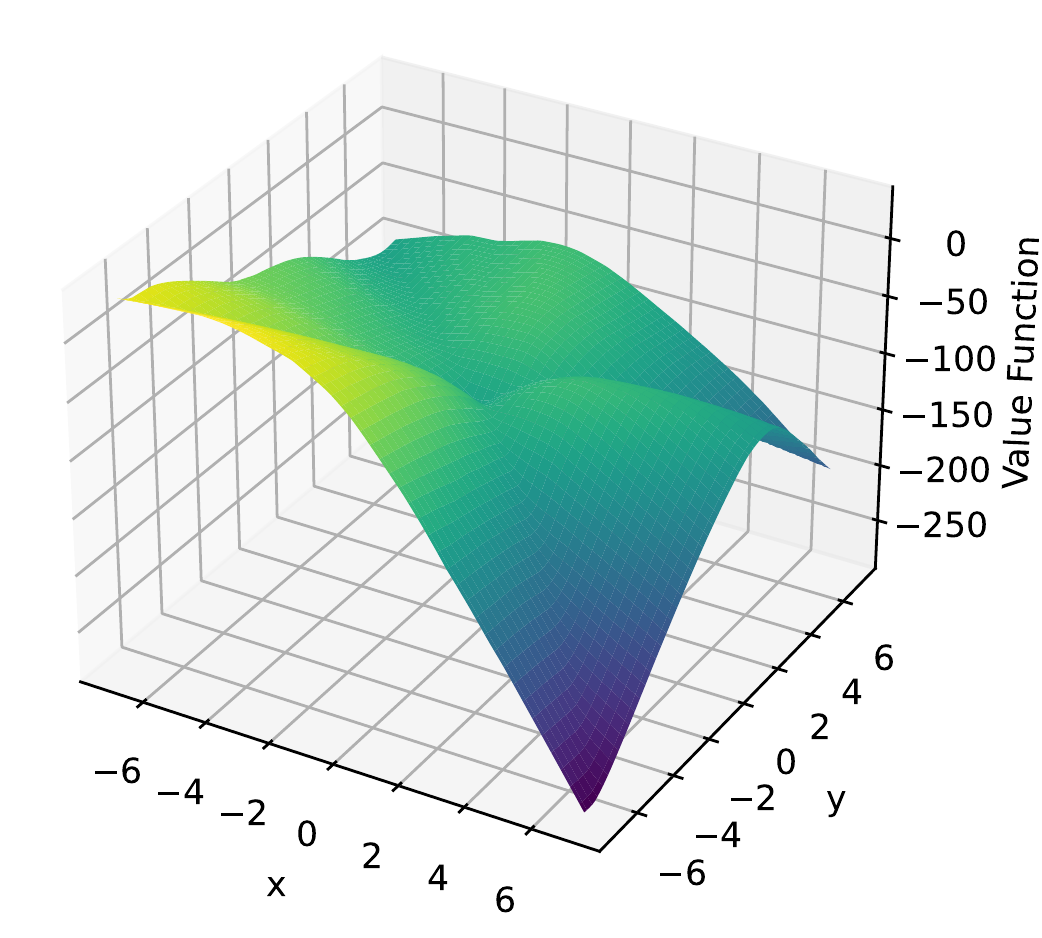}}
                 \subfigure[4E$3$ iterations]
                {\includegraphics[width=3cm,height=3cm]{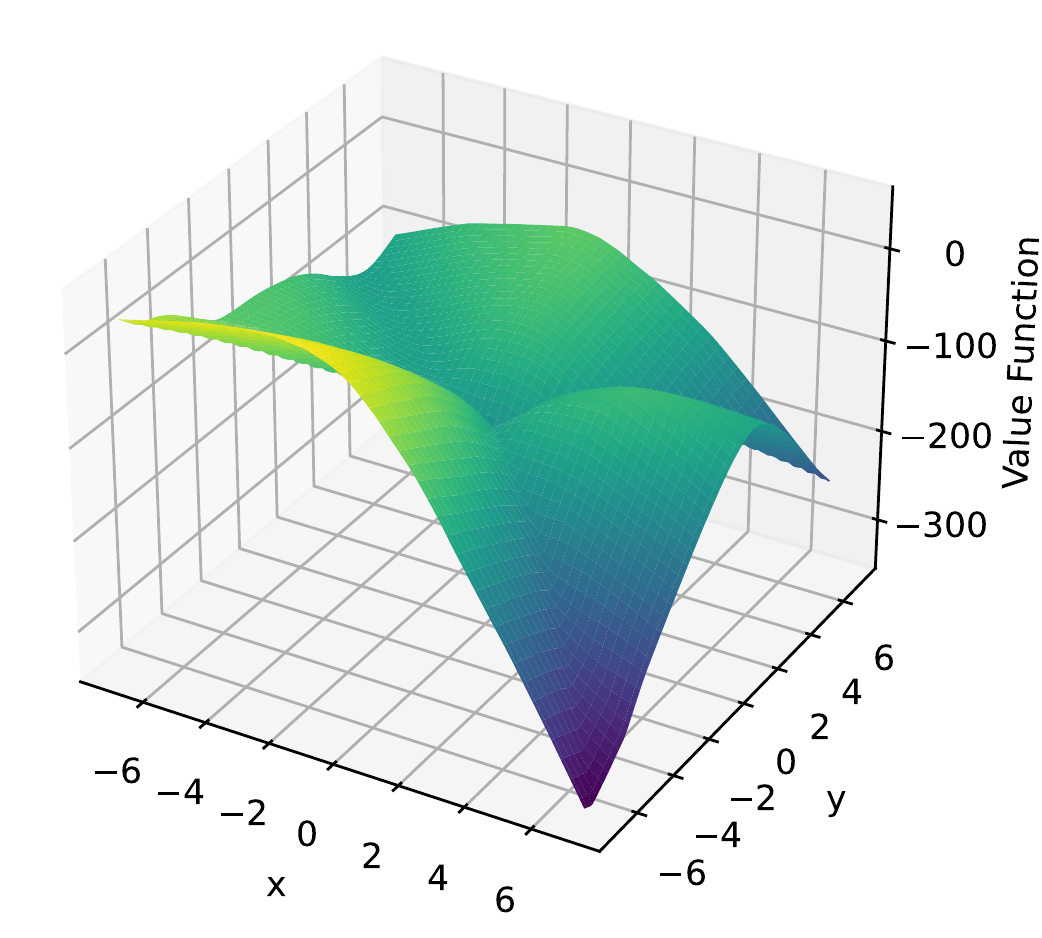}}
                     \subfigure[5E$3$ iterations]
                    {\includegraphics[width=3cm,height=3cm]{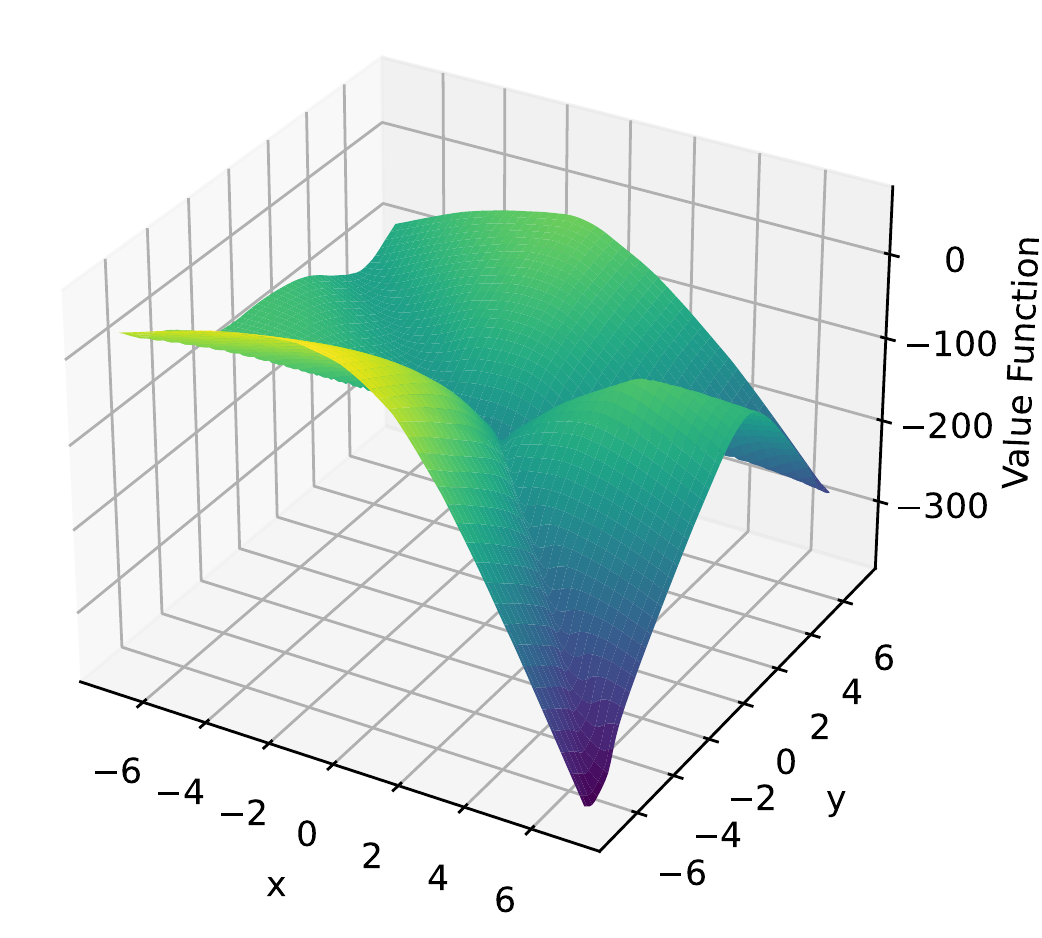}}
 {\includegraphics[width=3cm,height=3cm]{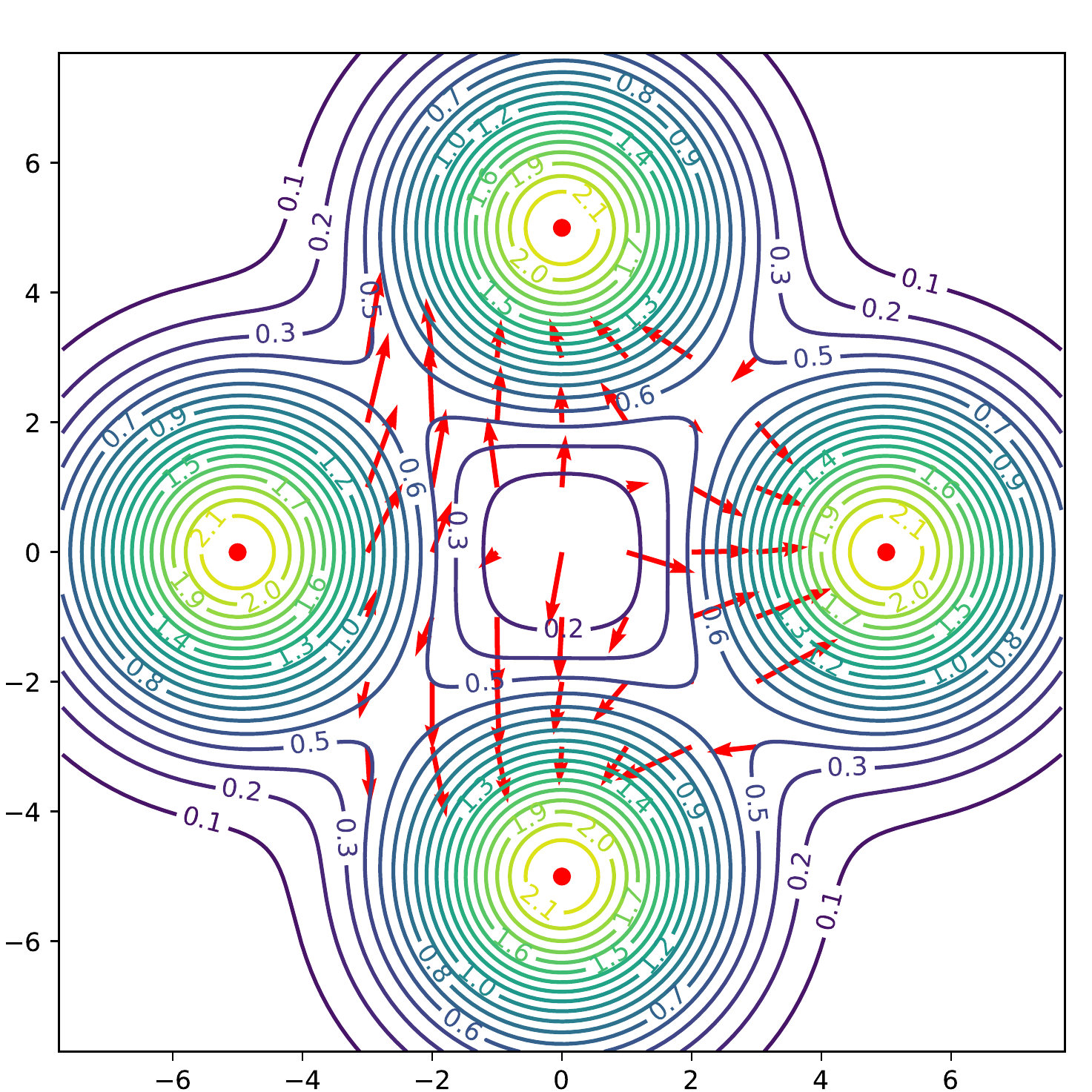}}
        {\includegraphics[width=3cm,height=3cm]{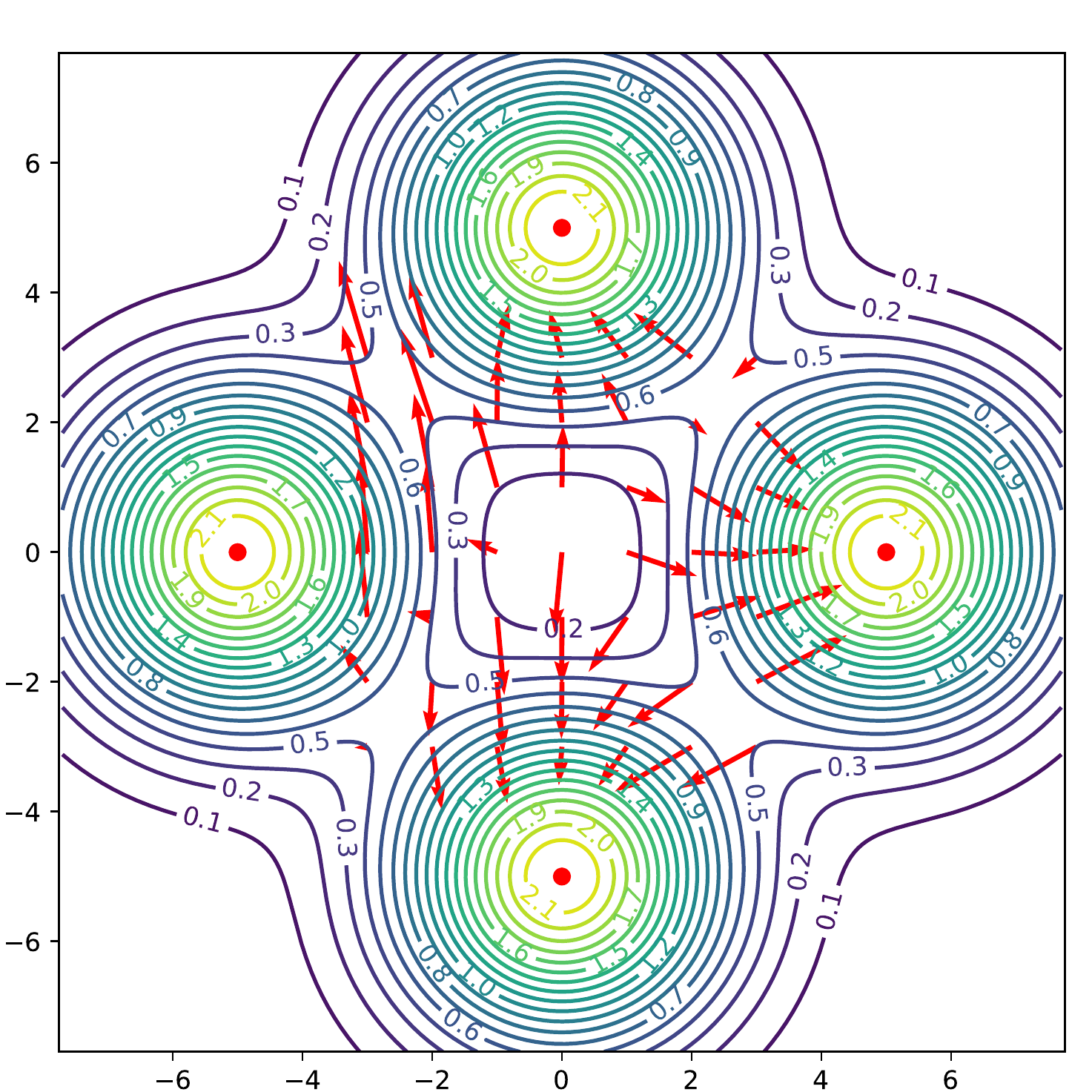}}
            {\includegraphics[width=3cm,height=3cm]{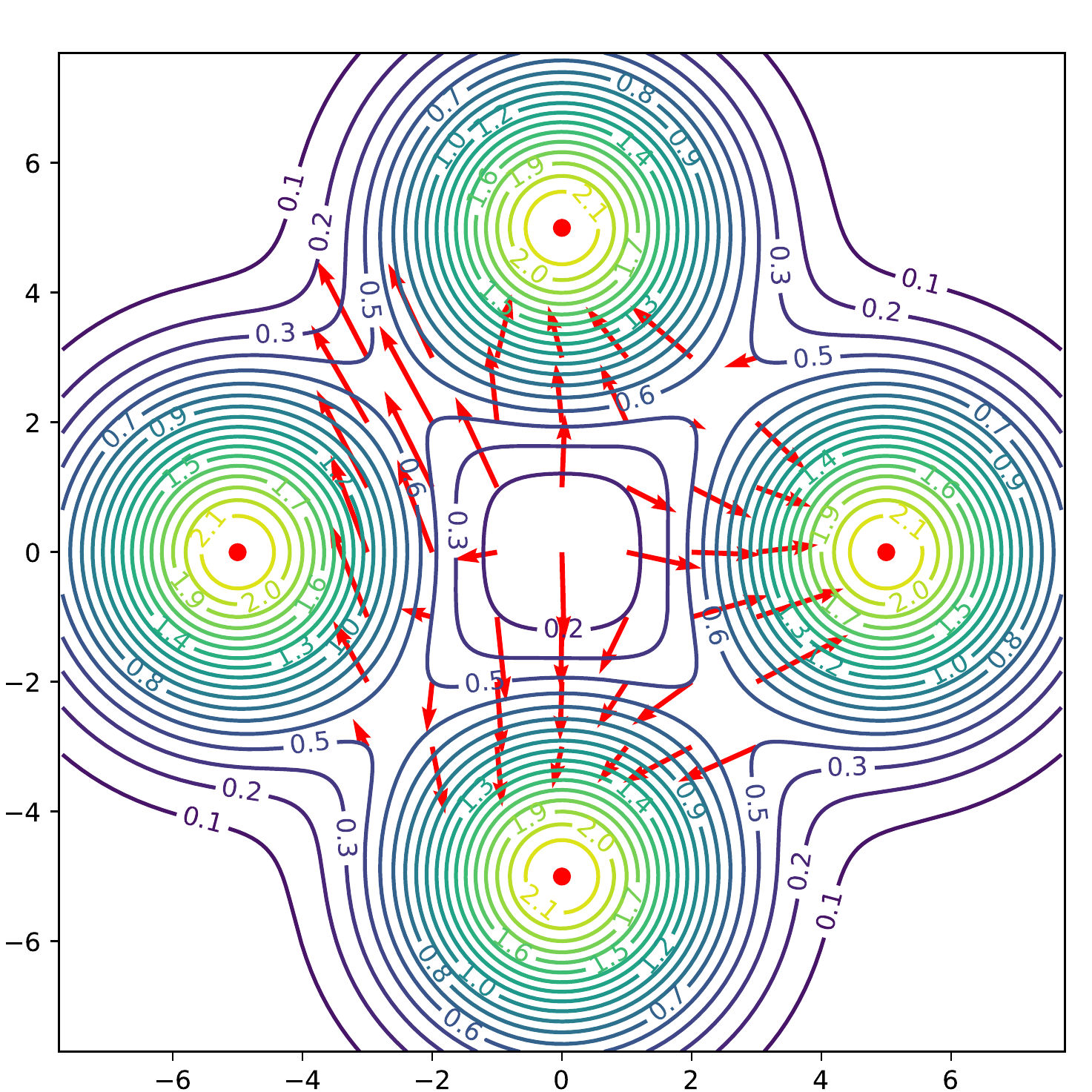}}
                {\includegraphics[width=3cm,height=3cm]{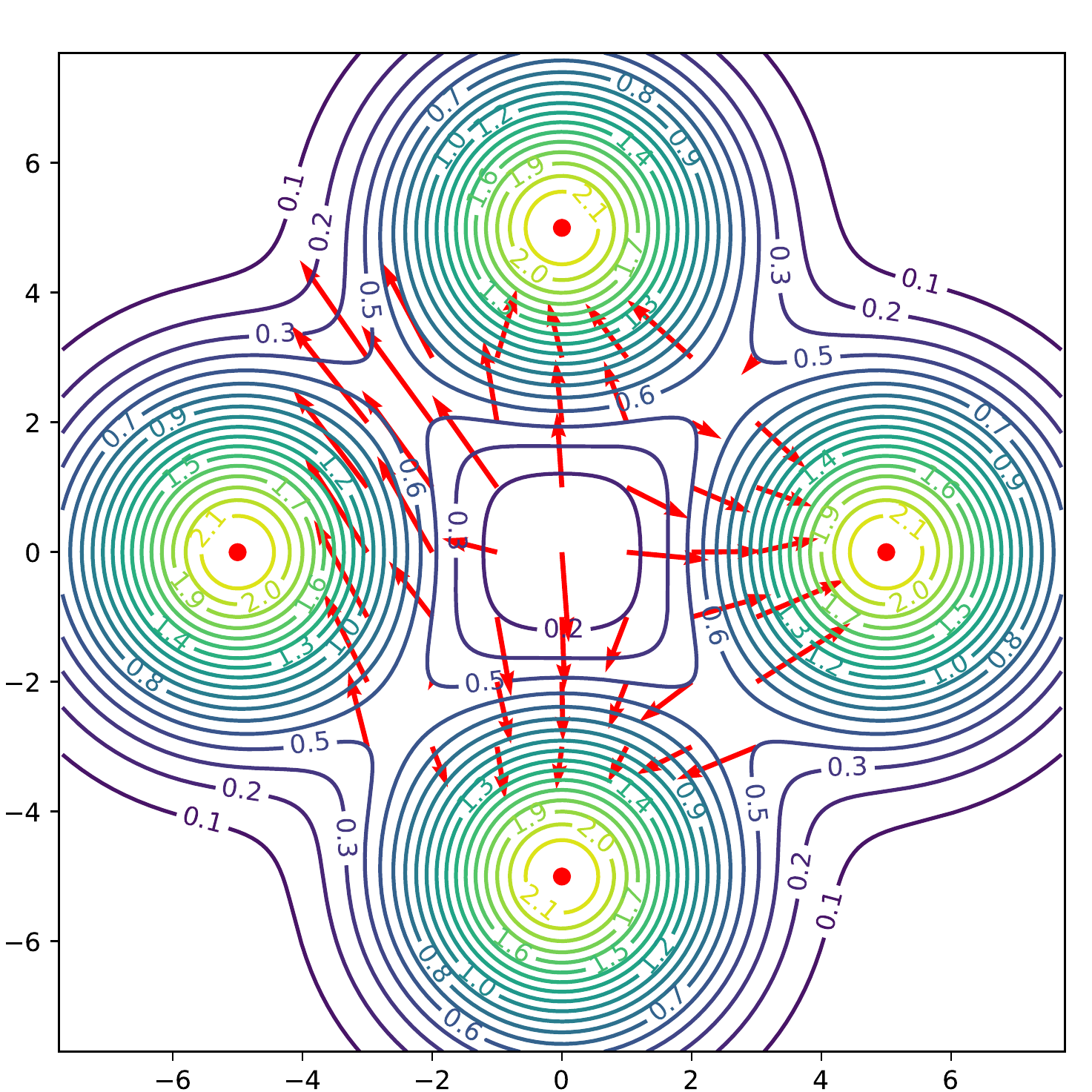}}
                    {\includegraphics[width=3cm,height=3cm]{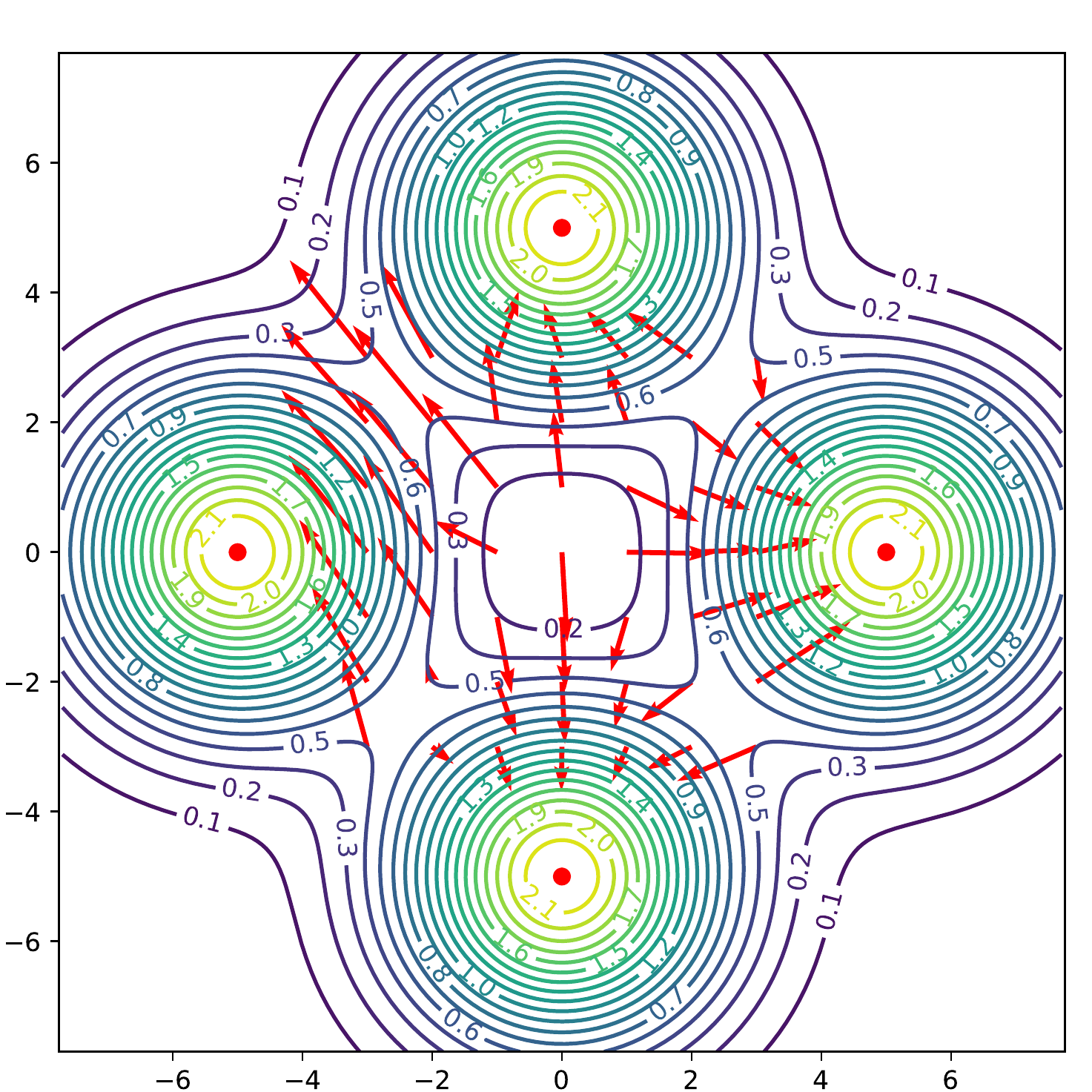}}
 \subfigure[6E$3$ iterations]
    {\includegraphics[width=3cm,height=3cm]{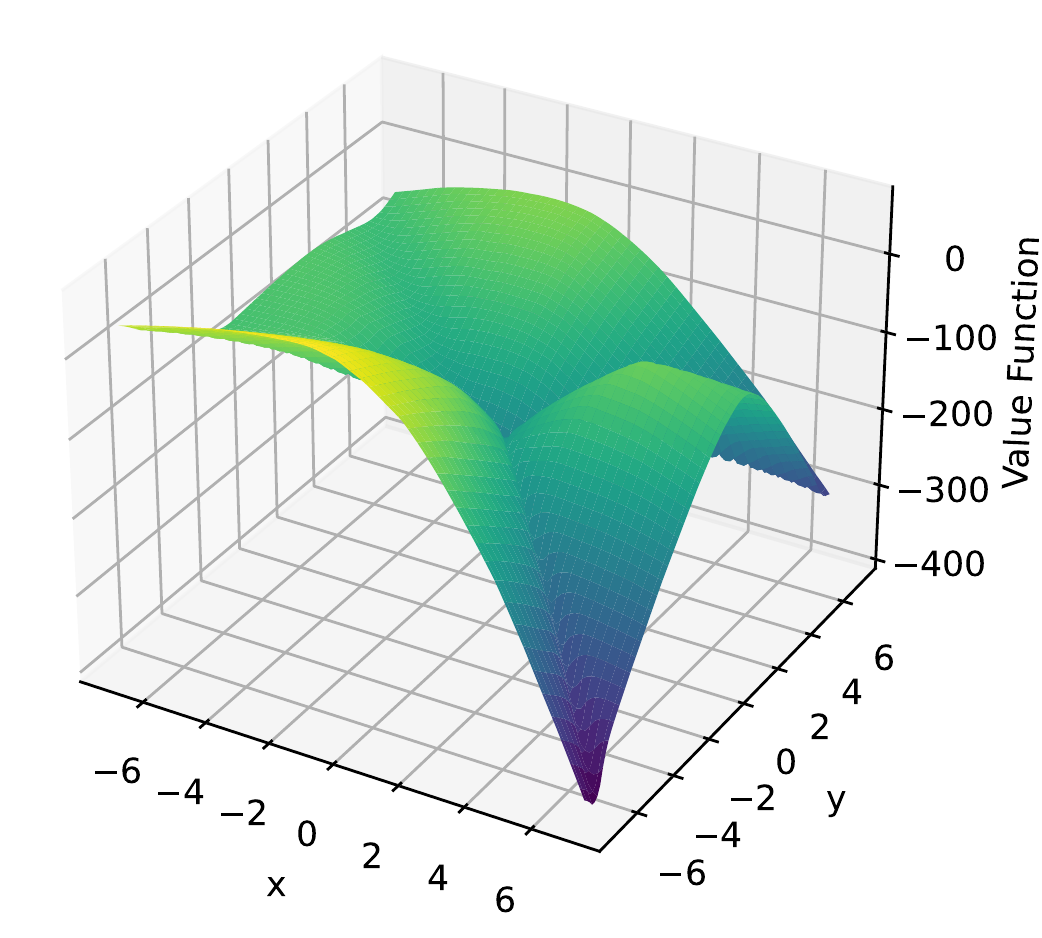}}
         \subfigure[7E$3$ iterations]
        {\includegraphics[width=3cm,height=3cm]{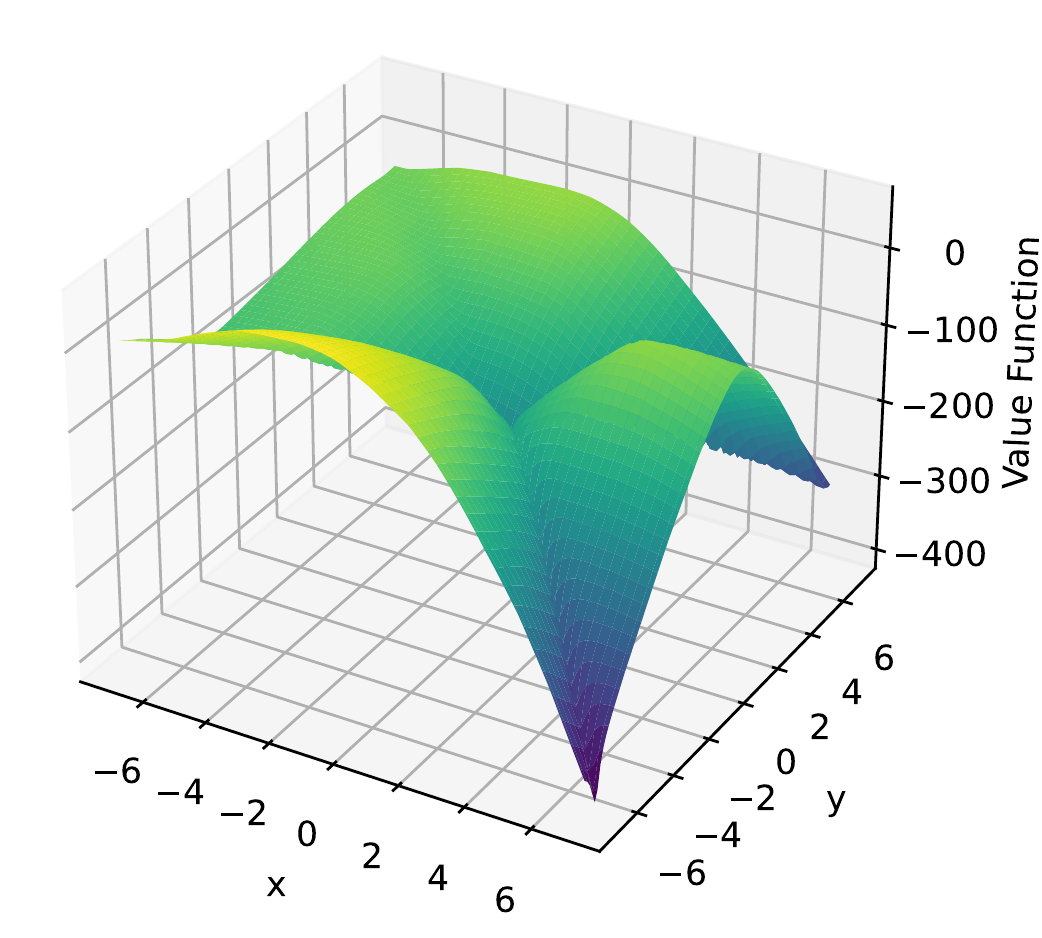}}
             \subfigure[8E$3$ iterations]
            {\includegraphics[width=3cm,height=3cm]{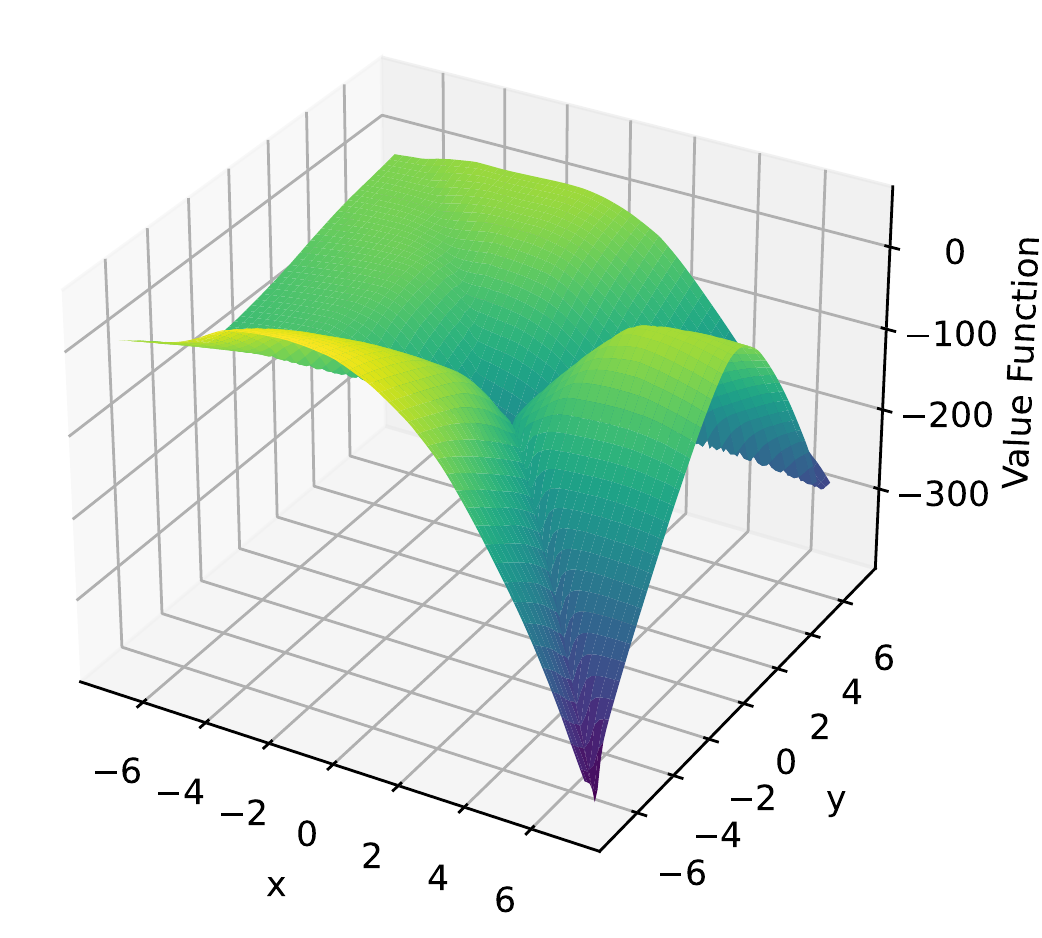}}
                 \subfigure[9E$3$ iterations]
                {\includegraphics[width=3cm,height=3cm]{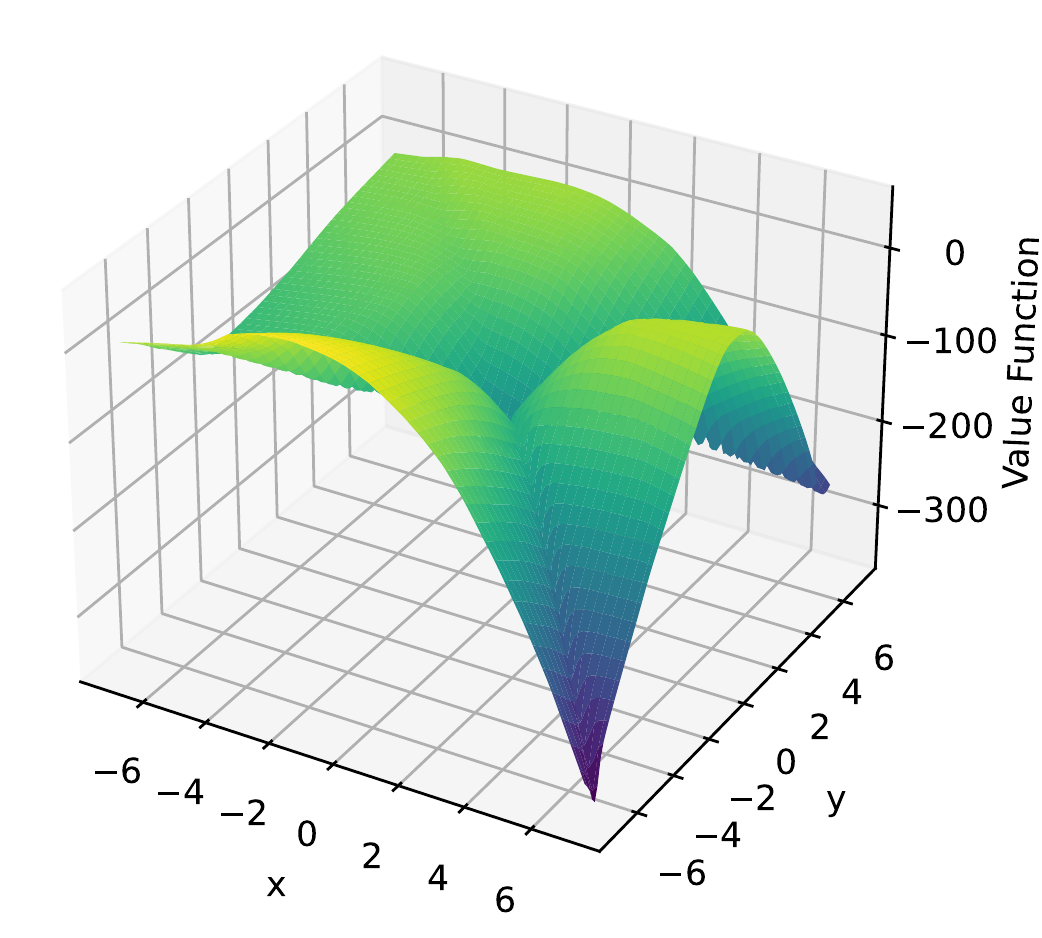}}
                     \subfigure[10E$3$ iterations]
                    {\includegraphics[width=3cm,height=3cm]{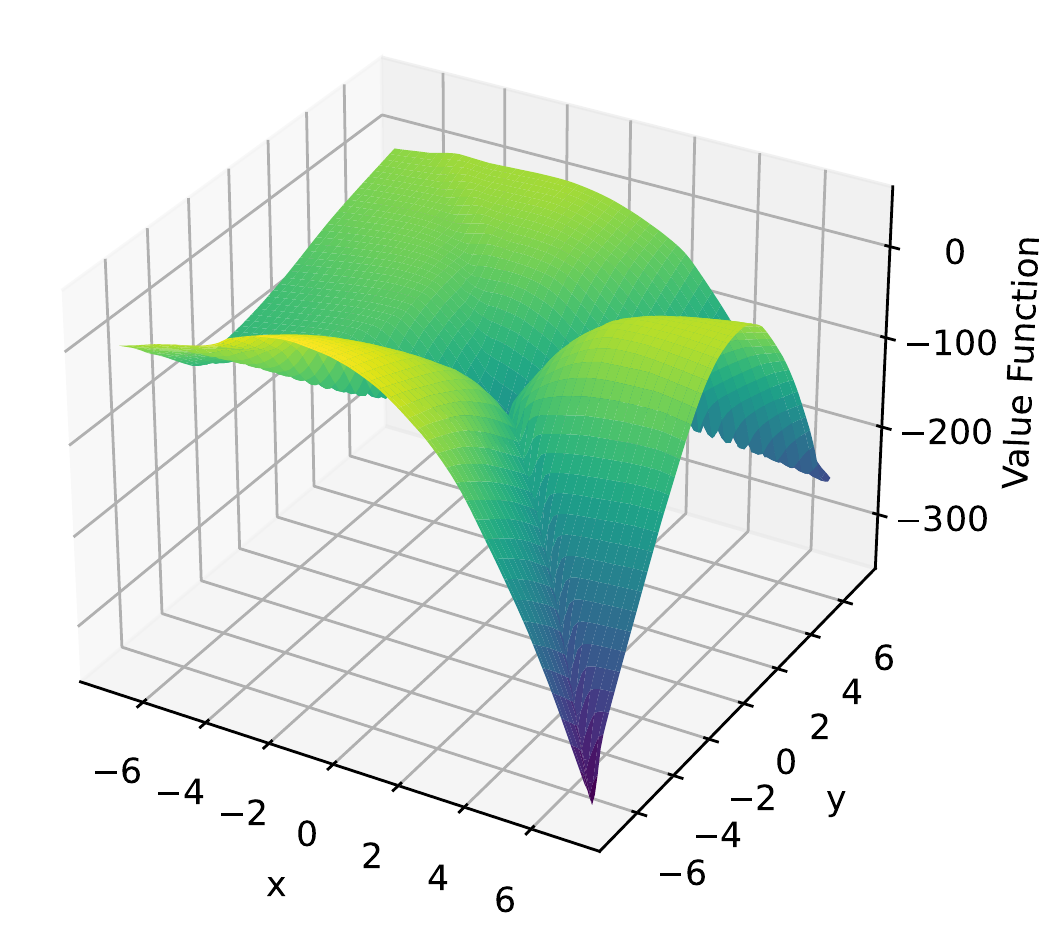}}
    \caption
    {Policy representation comparison of TD3 with different iterations.
    %and the optimal policy is to go to one of the goal positions randomly.
    %, and we have shown the solid red lines as the final action learned by different algorithms with 1000 iterations.
    }
    \label{app-fig-td3-policy}
\end{figure*}

This section presents all the details of the 2D and 3D visualization for the multi-goal task. At the end of this section, we present the shape of the reward curve.

\subsubsection{2D Visualization} For the 2D visualization, the red arrowheads denote actions learned by the corresponding RL algorithms, where each action starts at one of the totals of $7\times7=49$ points (corresponding to all the states) with horizontal and vertical coordinates ranges among $\{-3, -2, -1,0,1,2,3\}\times \{-3, -2, -1,0,1,2,3\}$. The length of the red arrowheads denotes the length of the action vector, and the direction of the red arrowheads denotes the direction of actions.
This is to say; for each figure, we plot all the actions starting from the same coordinate points.

\subsubsection{3D Visualization}
For the 3D visualization, we provide a decomposition of the the region $[-7,7]\times[-7,7]$ into $100\times100=10000$ points, each point $(x,y)\in [-7,7]\times[-7,7]$ denotes a state.
For each state $(x,y)$, a corresponding action is learned by its corresponding RL algorithms, denoted as $\ba$.
Then according to the critic neural network, we obtain the state-action value function $Q$ value of the corresponding point $((x,y),\ba)$. The 3D visualization shows the state-action $Q$ (for PPO, is value function $V$) with respect to the states.

\subsubsection{Shape of Reward Curve}

Since the shape of the reward curve is symmetrical with four equal peaks, the 2D visualization presents the distribution of actions toward those four equal peaks.
A good algorithm should take actions with a uniform distribution toward those four points $(0,5)$, $(0,-5)$, $(5,0)$, and $(-5,0)$ on the  2D visualization.
The 3D visualization presents the learned shape according to the algorithm during the learning process. A good algorithm should fit the symmetrical reward shape with four equal peaks. 
A multimodal policy distribution is efficient for exploration, which may lead an agent to learn a good policy and perform better.
Thus, both 2D and 3D visualizations character the algorithm's capacity to represent the multimodal policy distribution.

\begin{figure*}[t]
    \centering
    {\includegraphics[width=3cm,height=3cm]{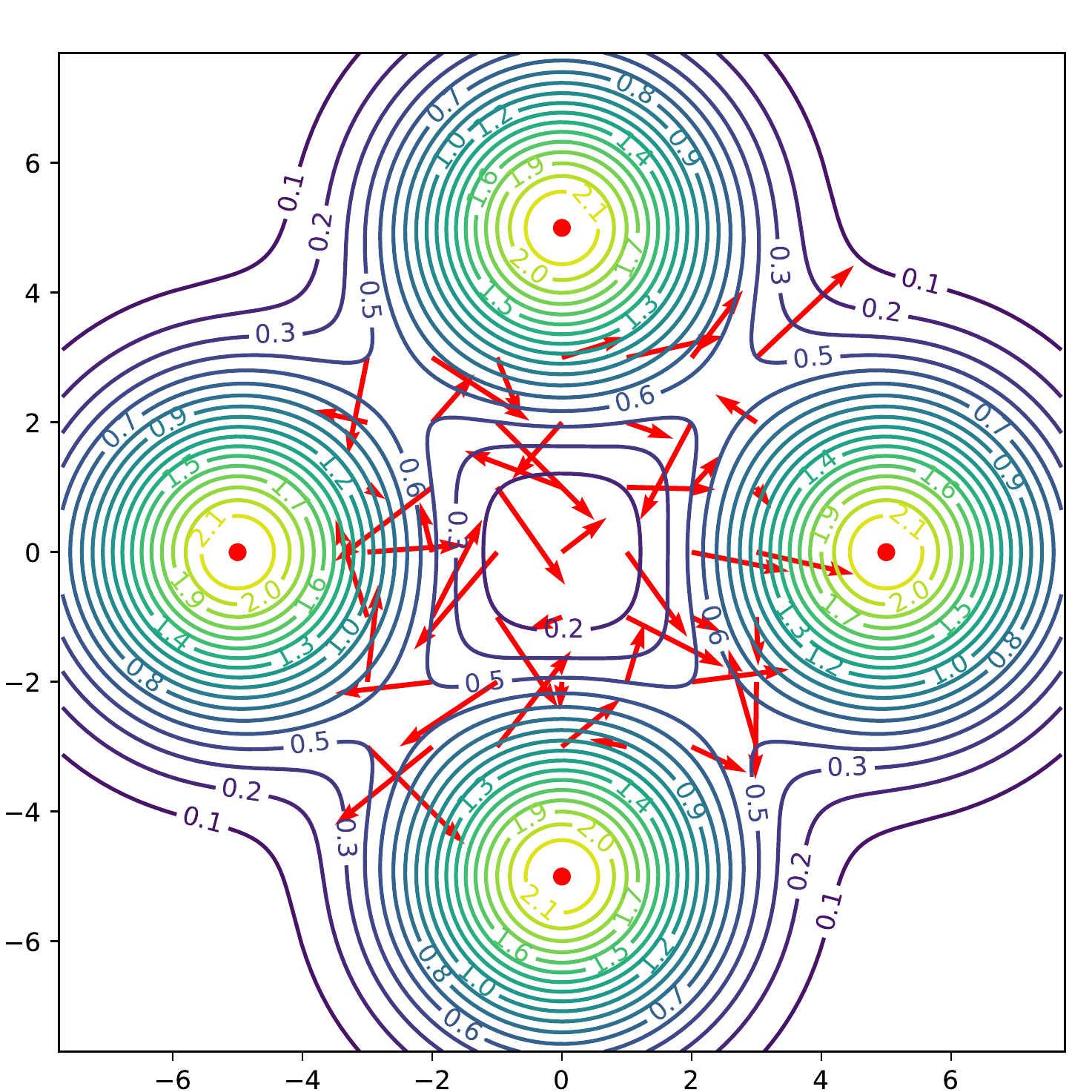}}
        {\includegraphics[width=3cm,height=3cm]{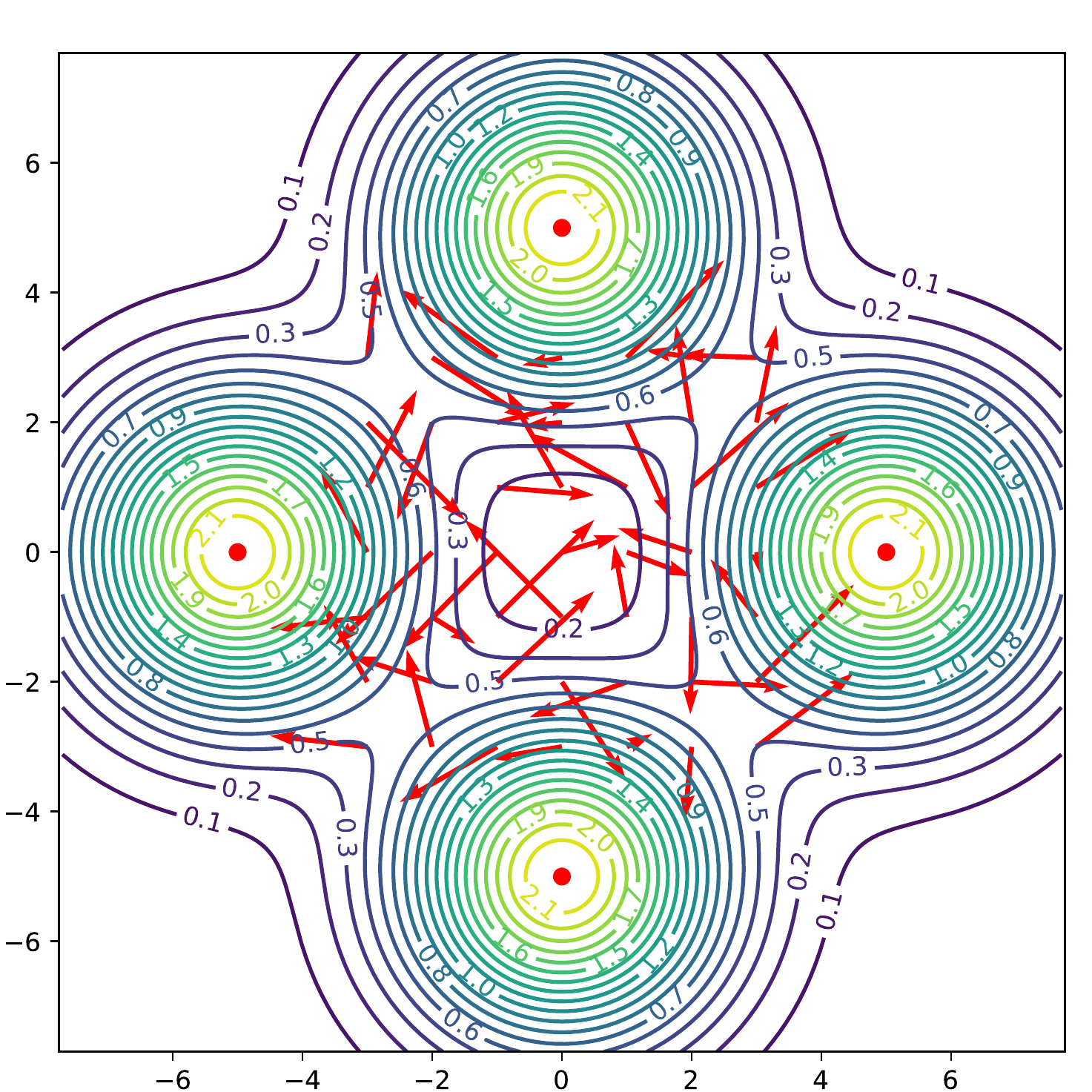}}
            {\includegraphics[width=3cm,height=3cm]{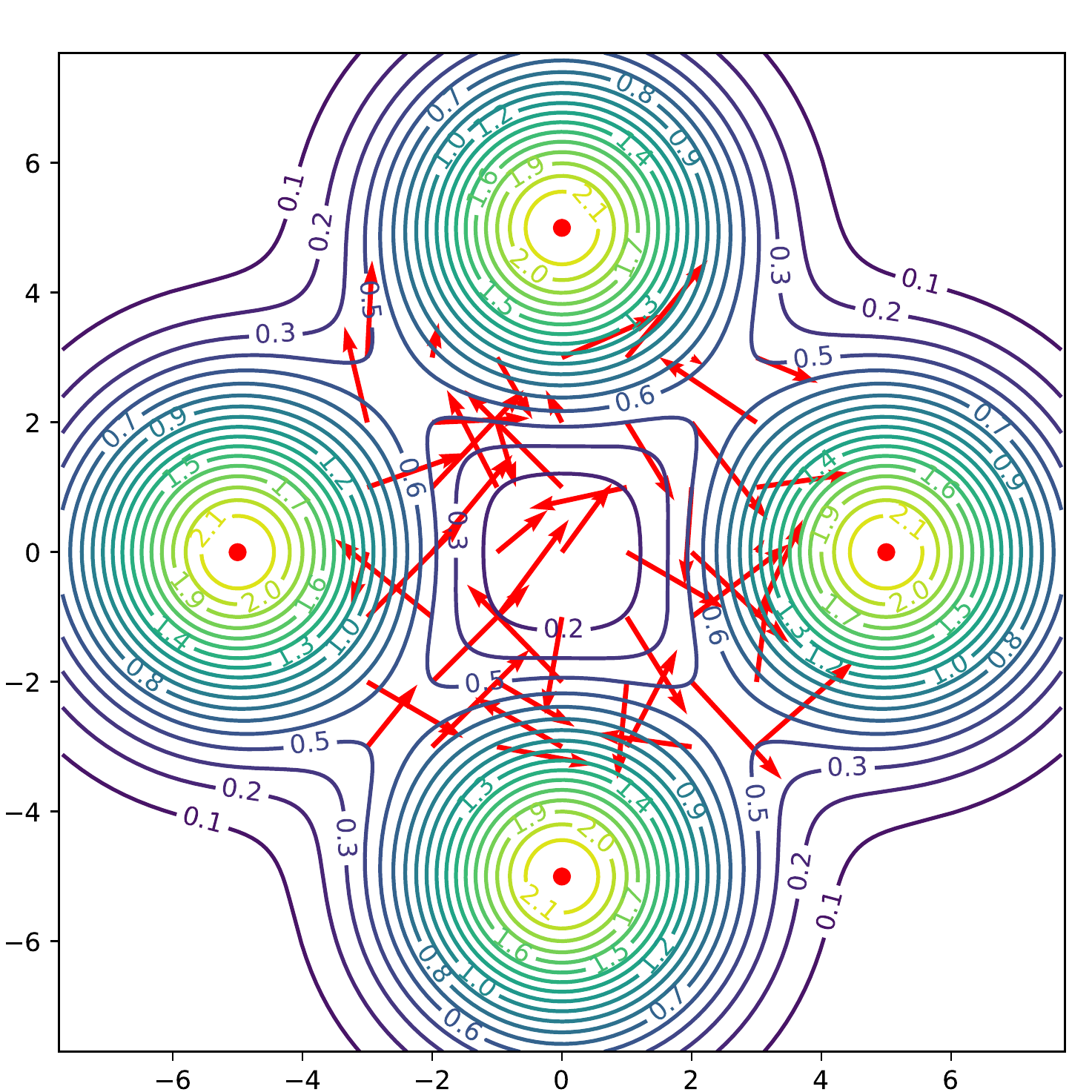}}
                {\includegraphics[width=3cm,height=3cm]{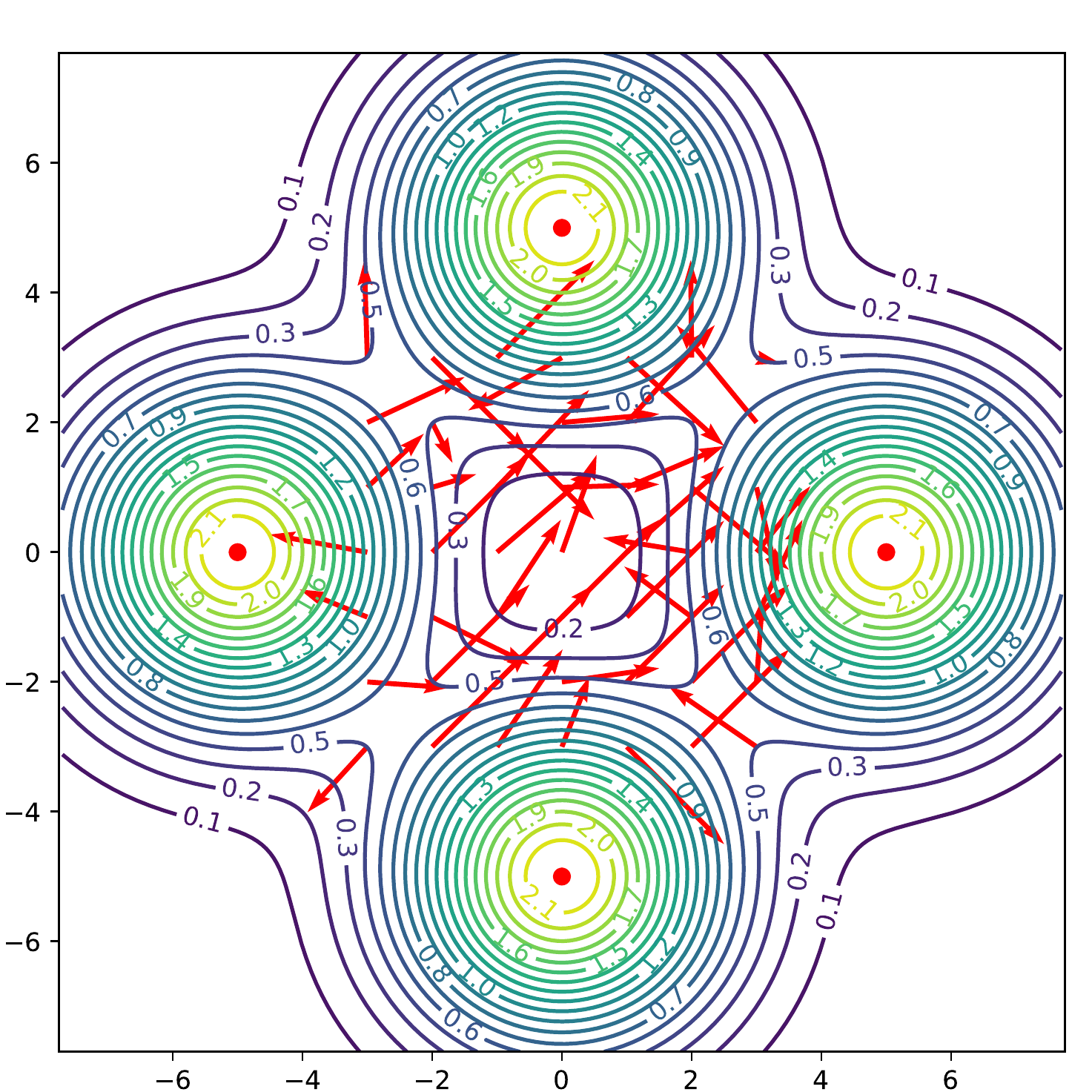}}
                    {\includegraphics[width=3cm,height=3cm]{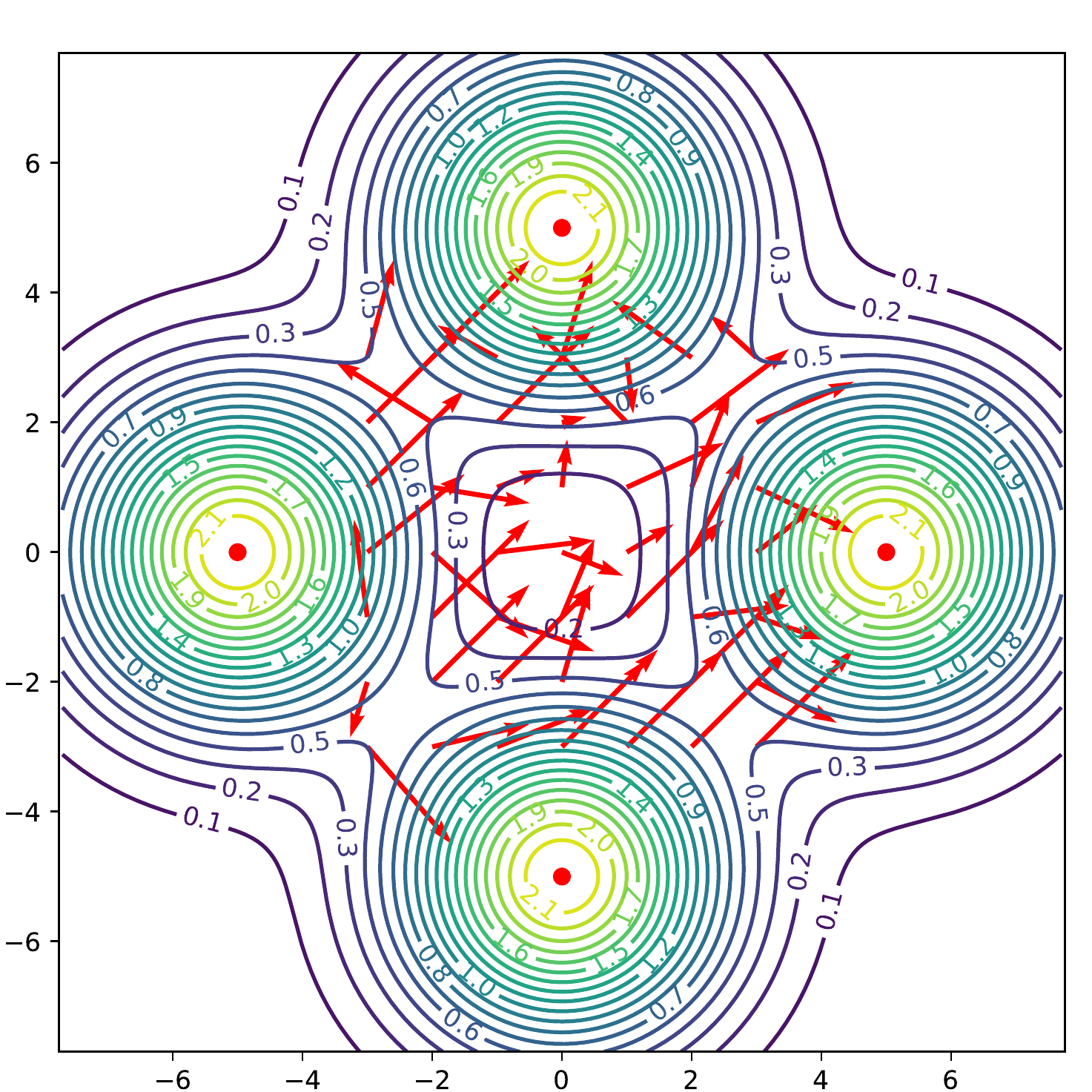}}
 \subfigure[1E$3$ iterations]
    {\includegraphics[width=3cm,height=3cm]{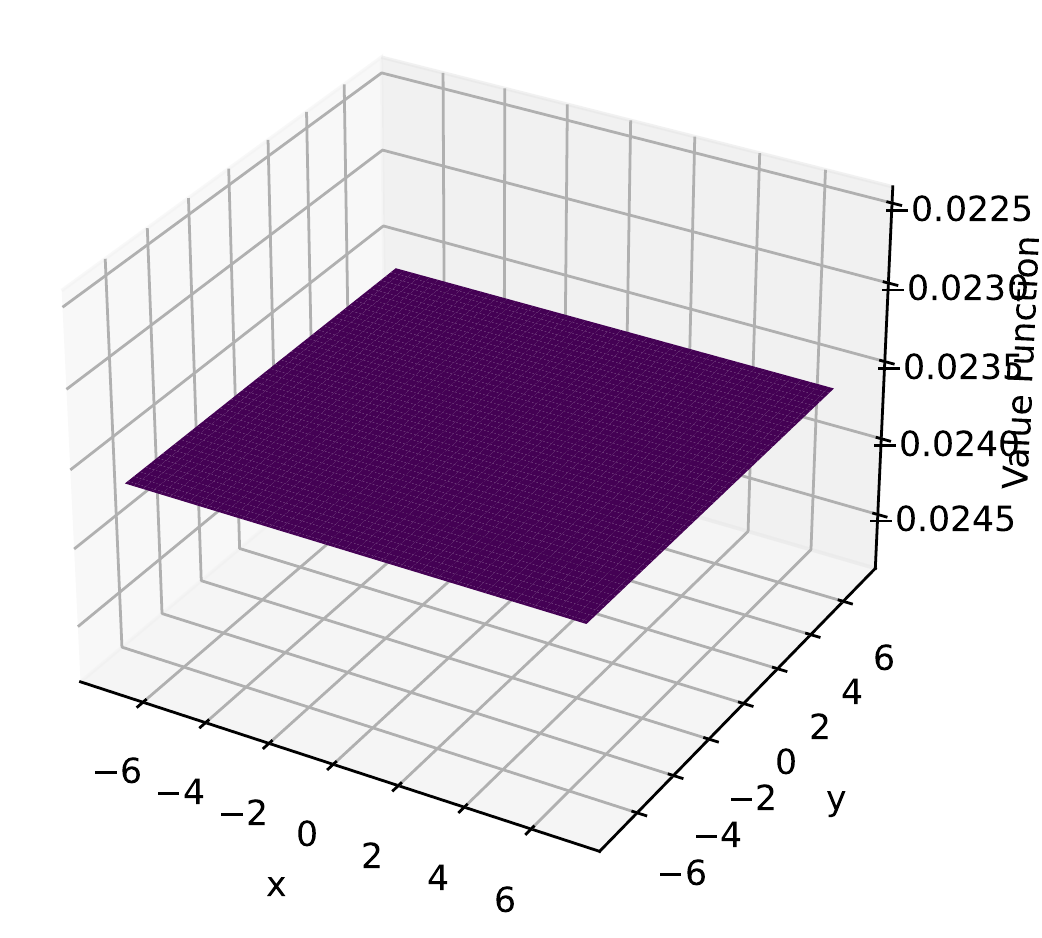}}
         \subfigure[2E$3$ iterations]
        {\includegraphics[width=3cm,height=3cm]{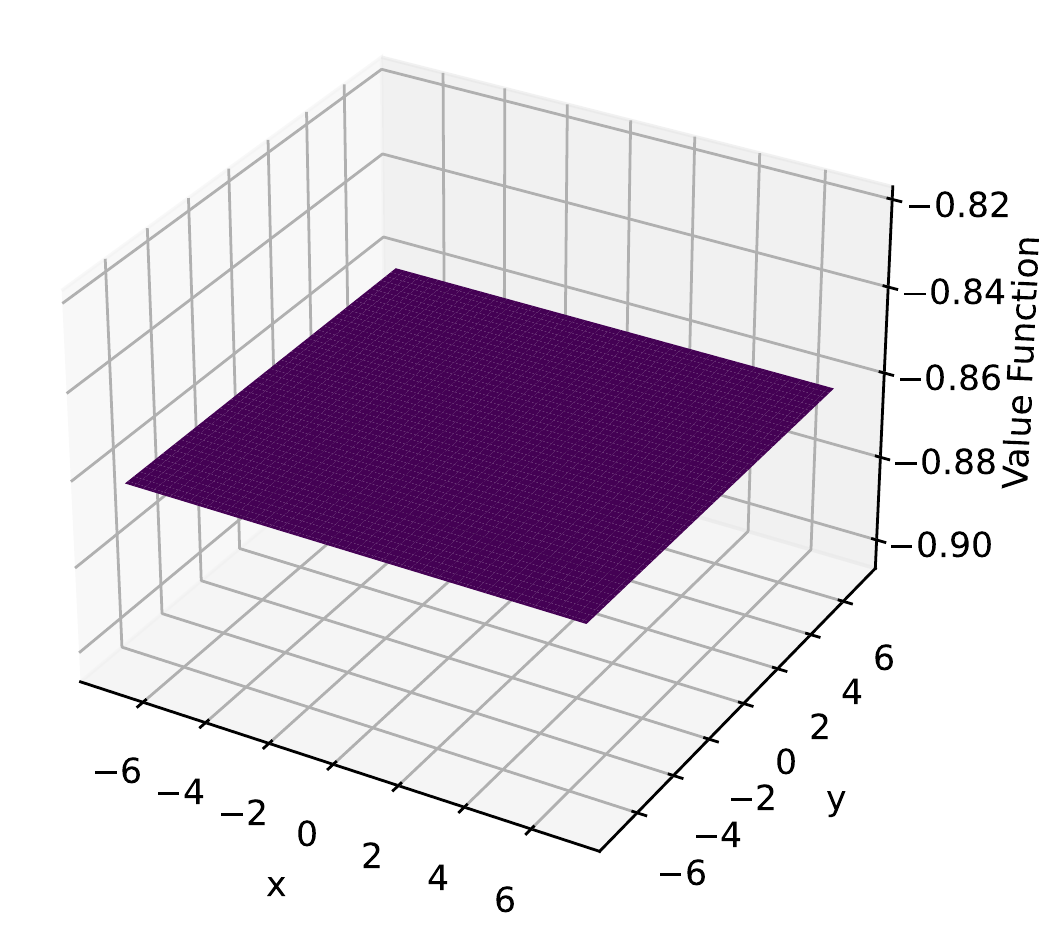}}
             \subfigure[3E$3$ iterations]
            {\includegraphics[width=3cm,height=3cm]{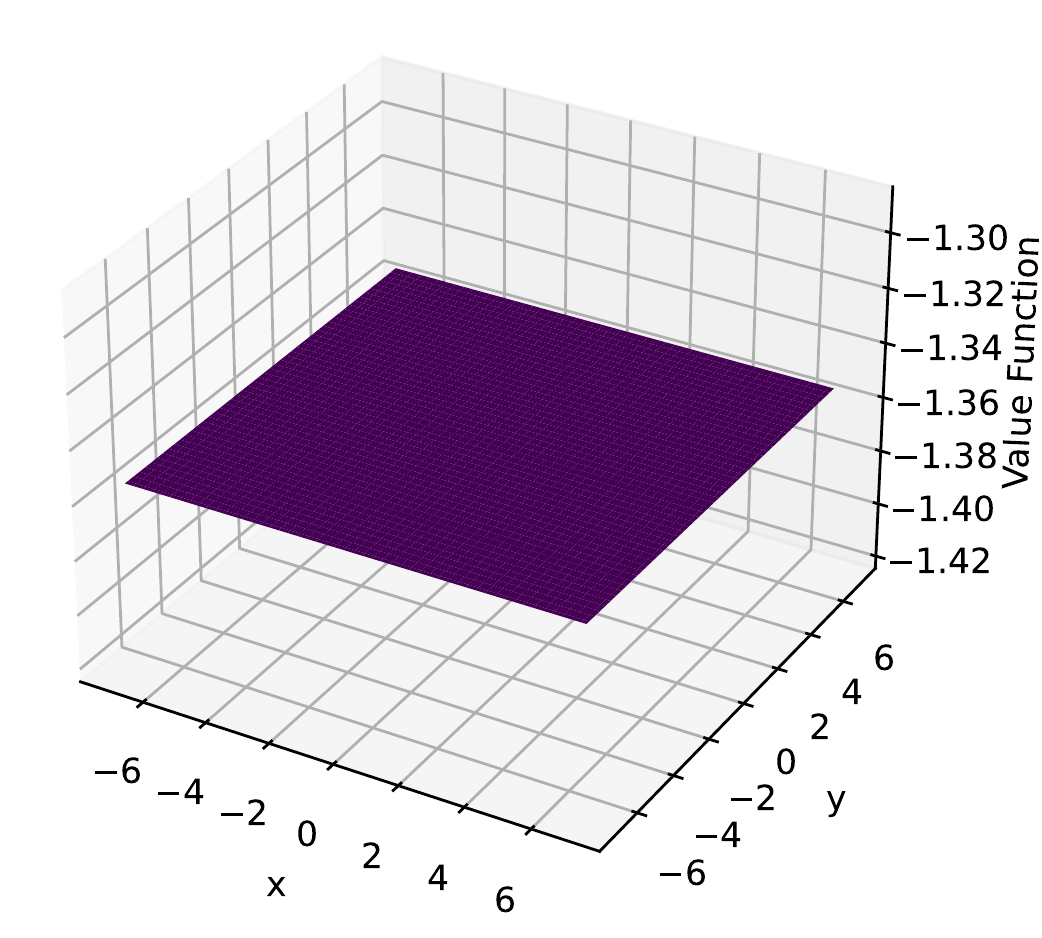}}
                 \subfigure[4E$3$ iterations]
                {\includegraphics[width=3cm,height=3cm]{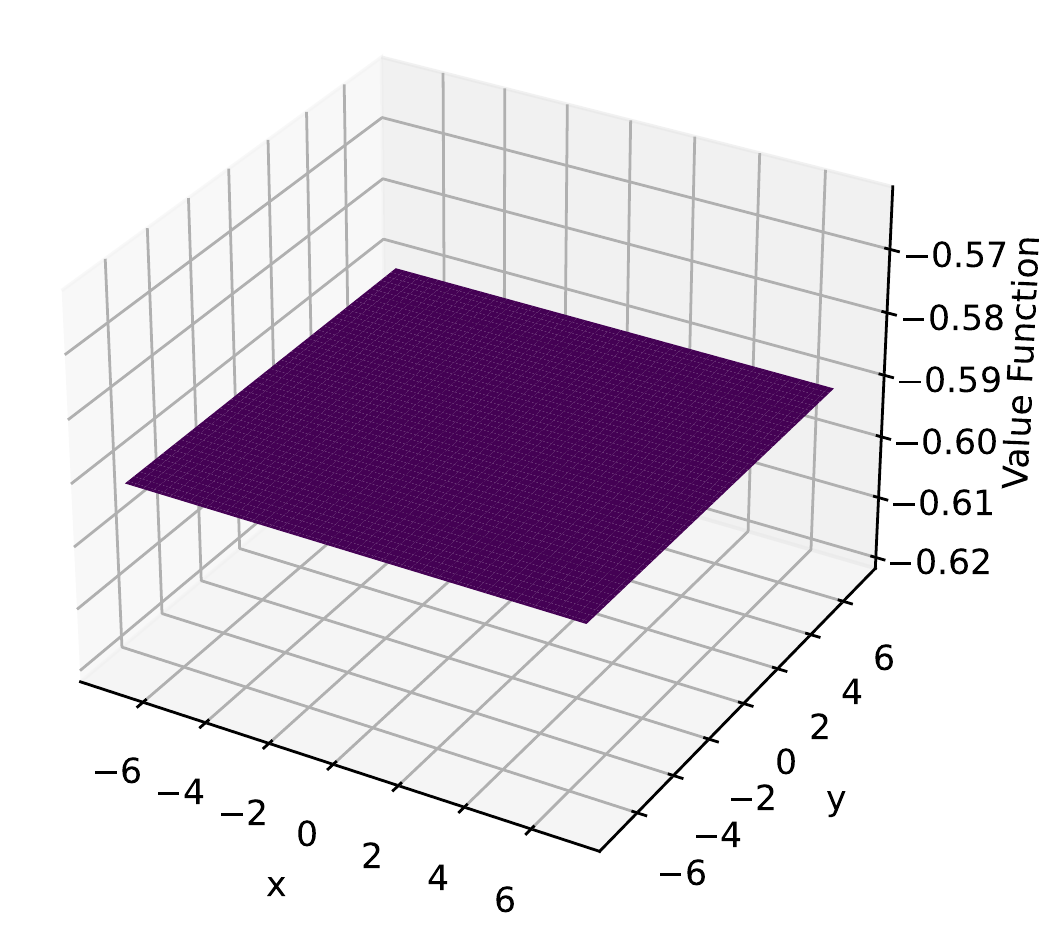}}
                     \subfigure[5E$3$ iterations]
                    {\includegraphics[width=3cm,height=3cm]{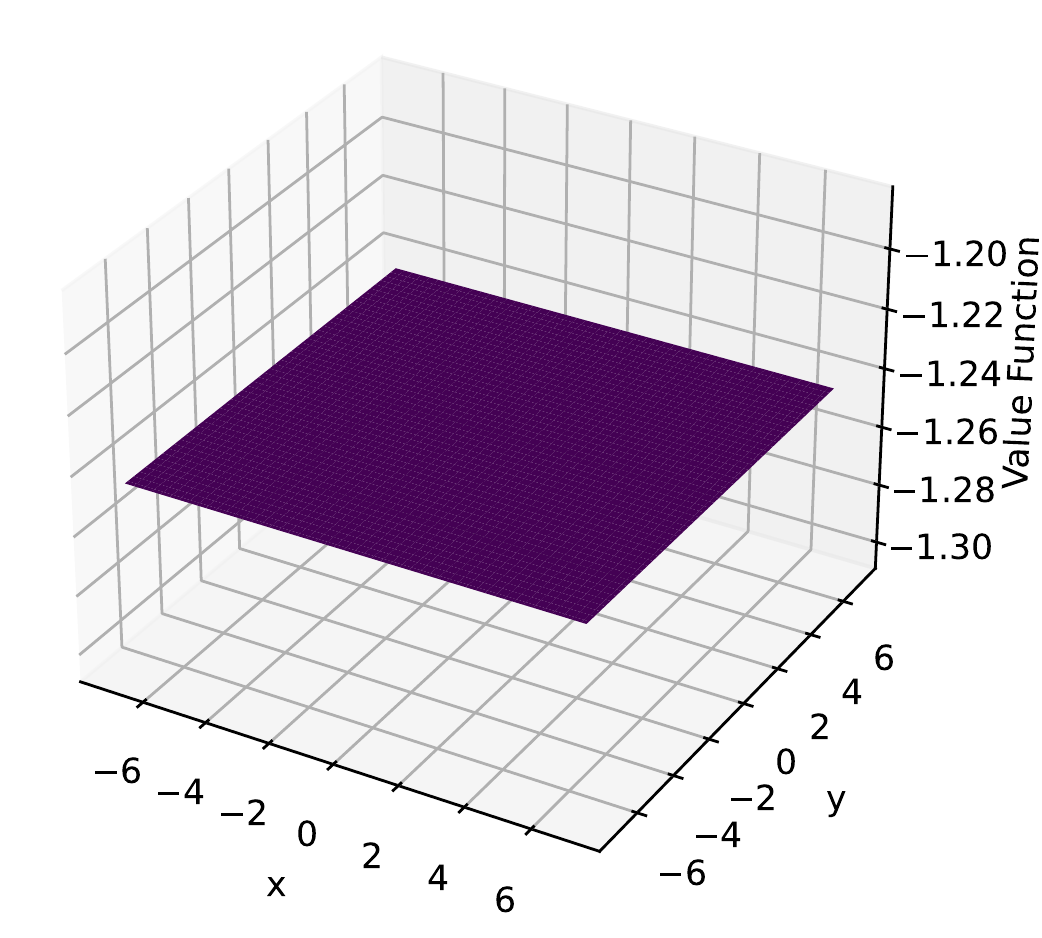}}
 {\includegraphics[width=3cm,height=3cm]{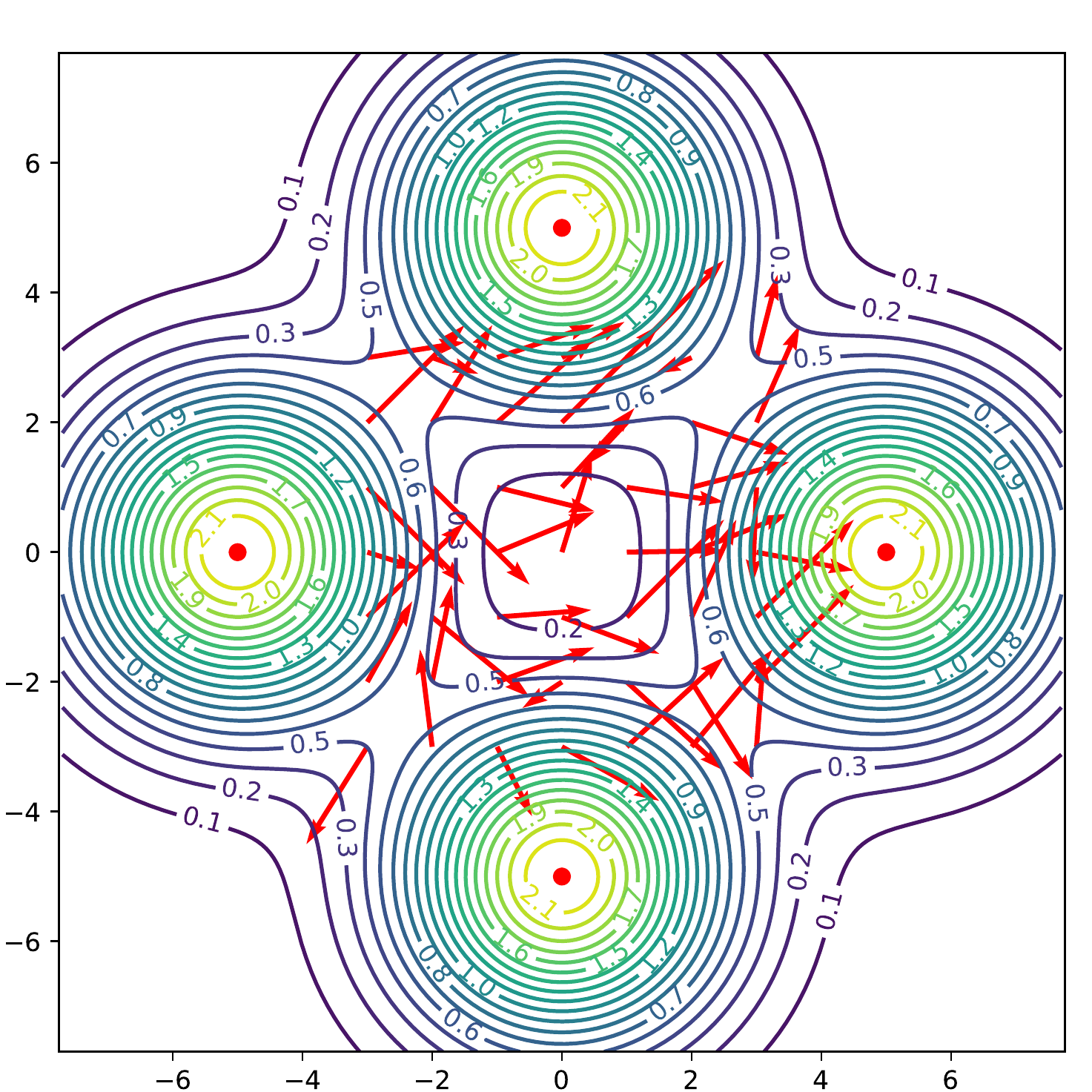}}
        {\includegraphics[width=3cm,height=3cm]{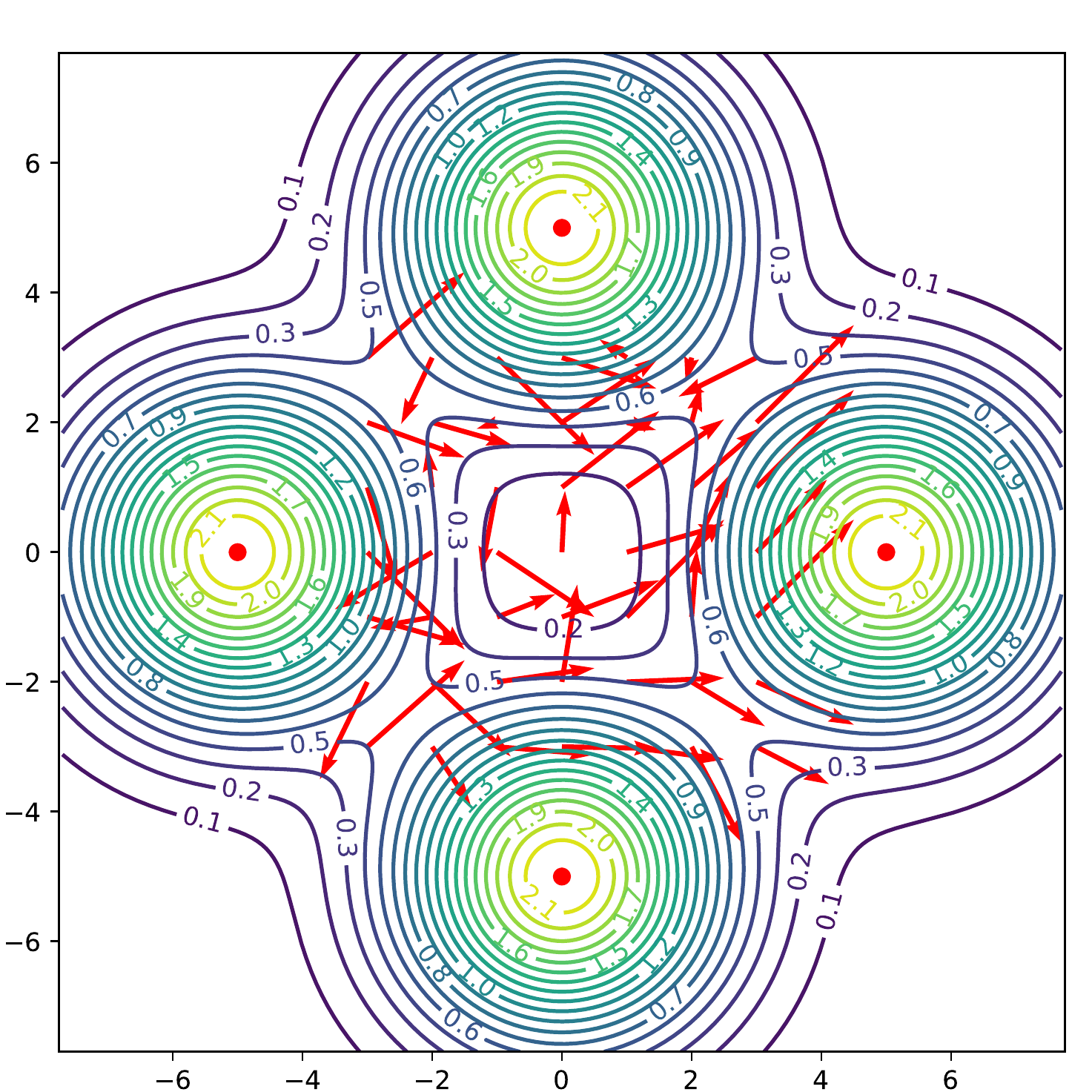}}
            {\includegraphics[width=3cm,height=3cm]{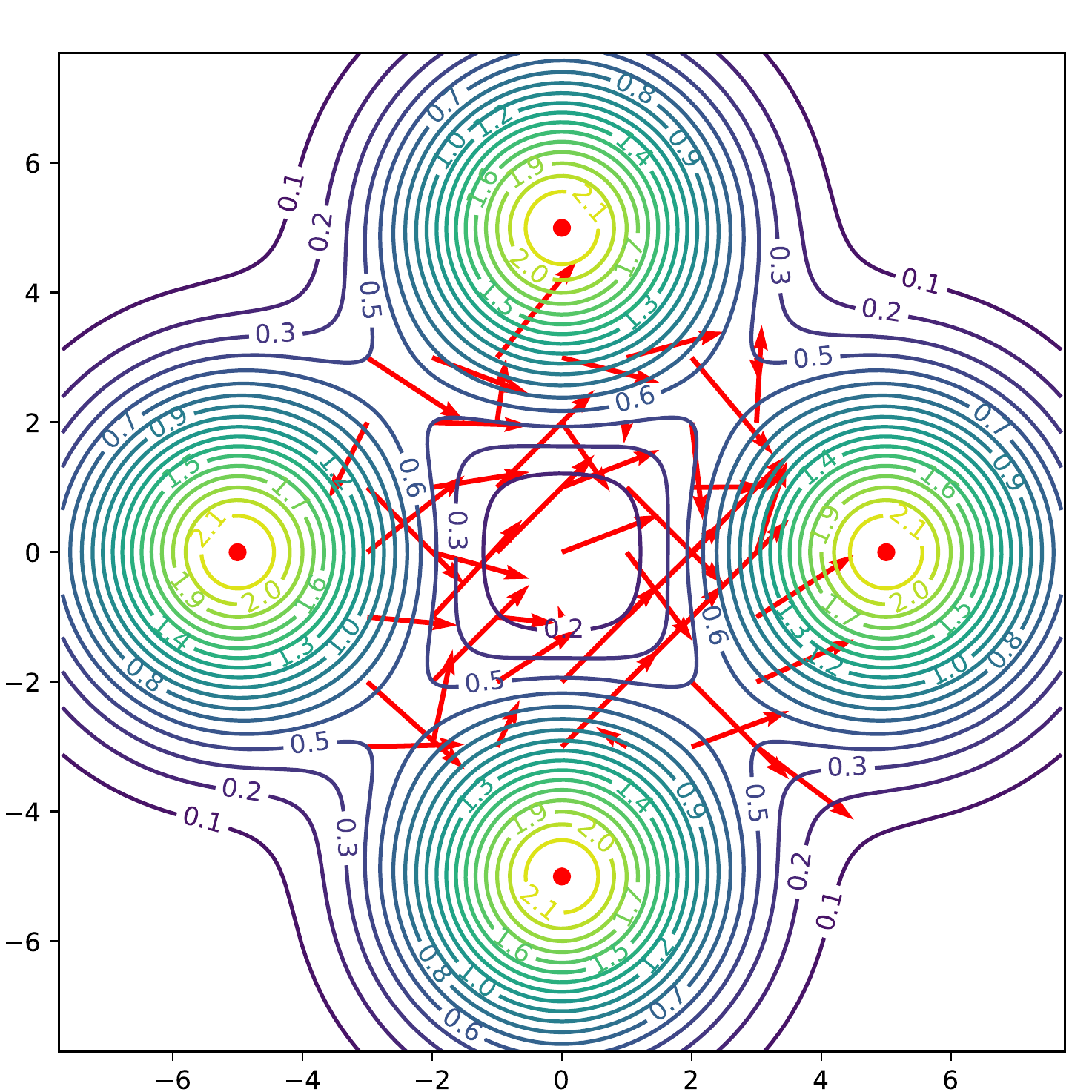}}
                {\includegraphics[width=3cm,height=3cm]{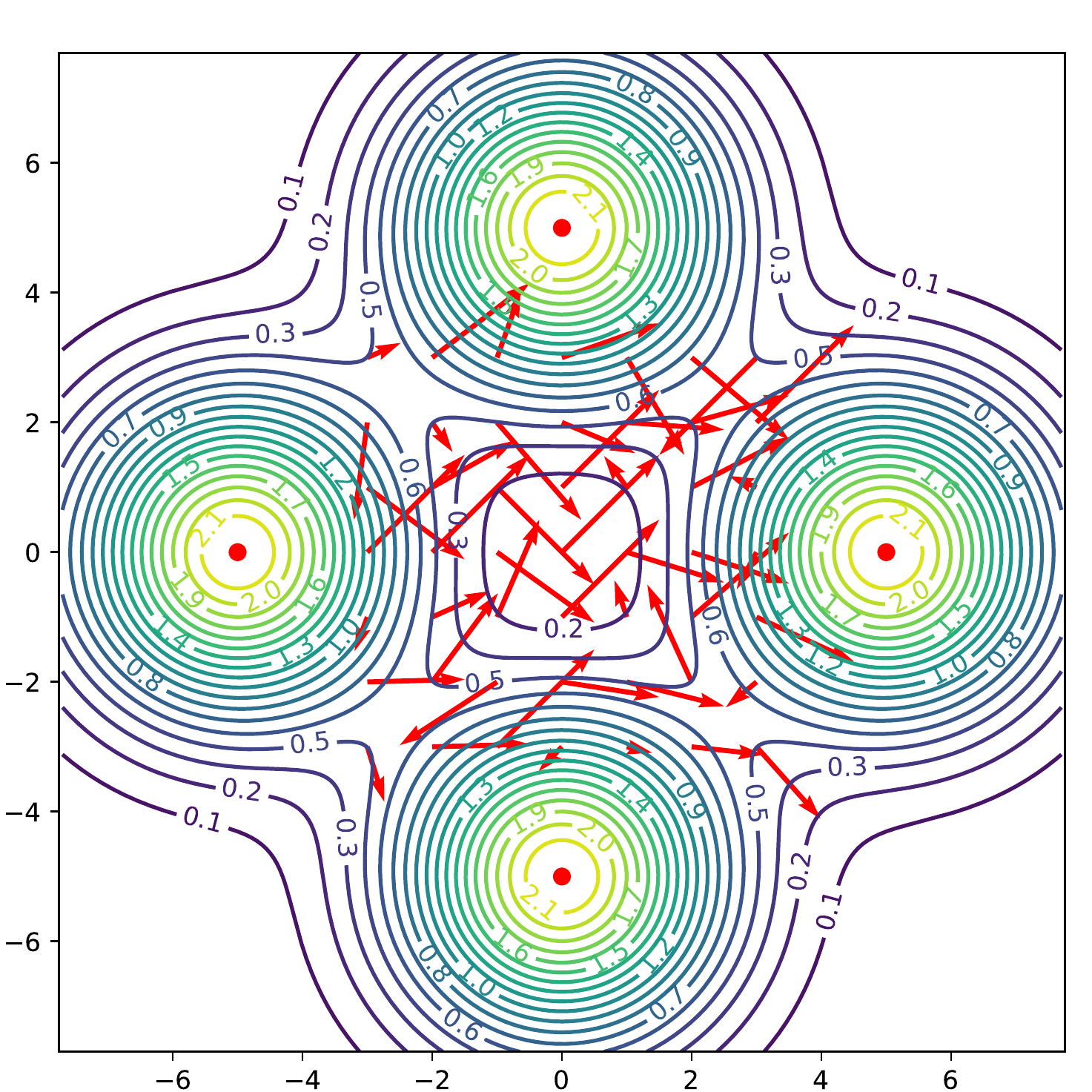}}
                    {\includegraphics[width=3cm,height=3cm]{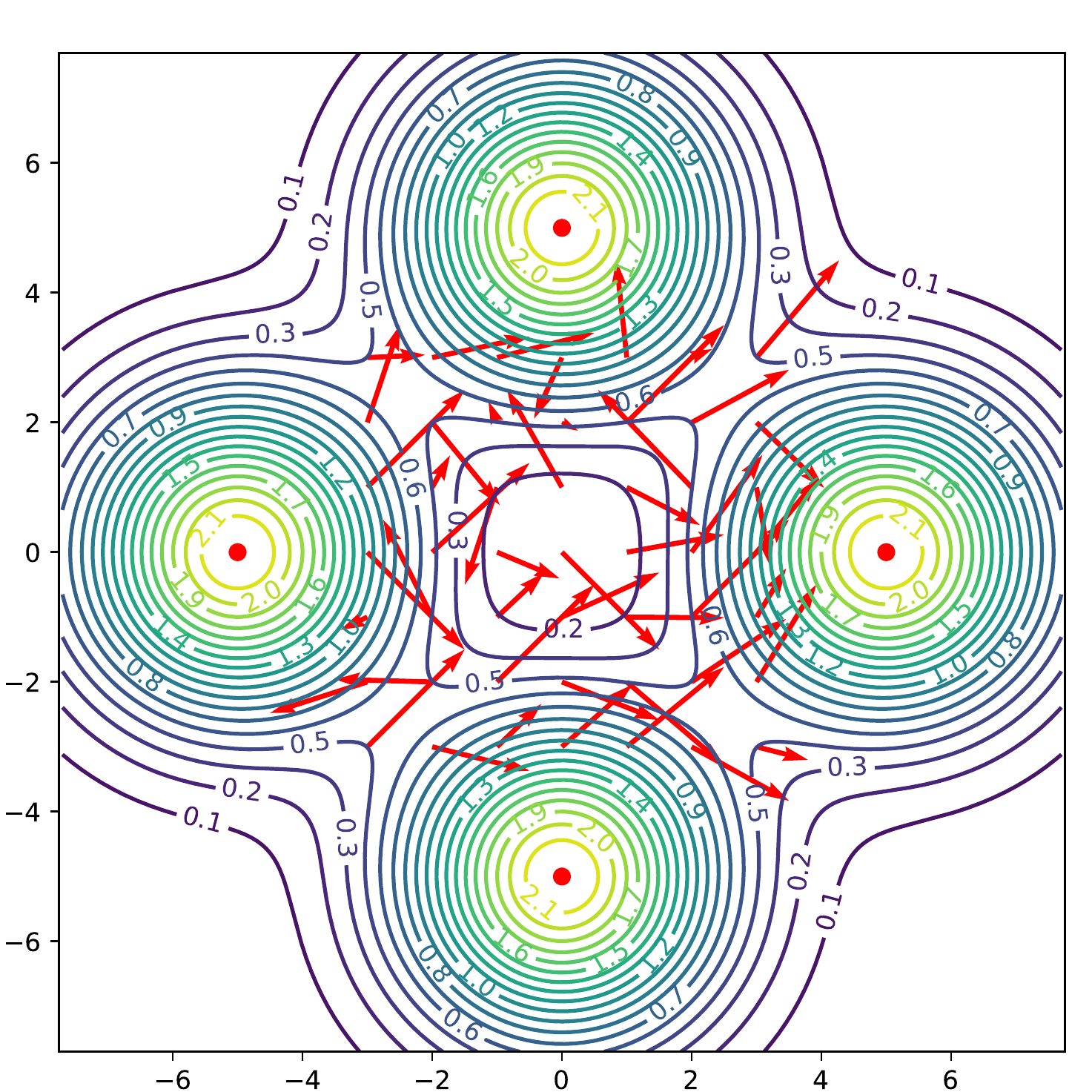}}
 \subfigure[6E$3$ iterations]
    {\includegraphics[width=3cm,height=3cm]{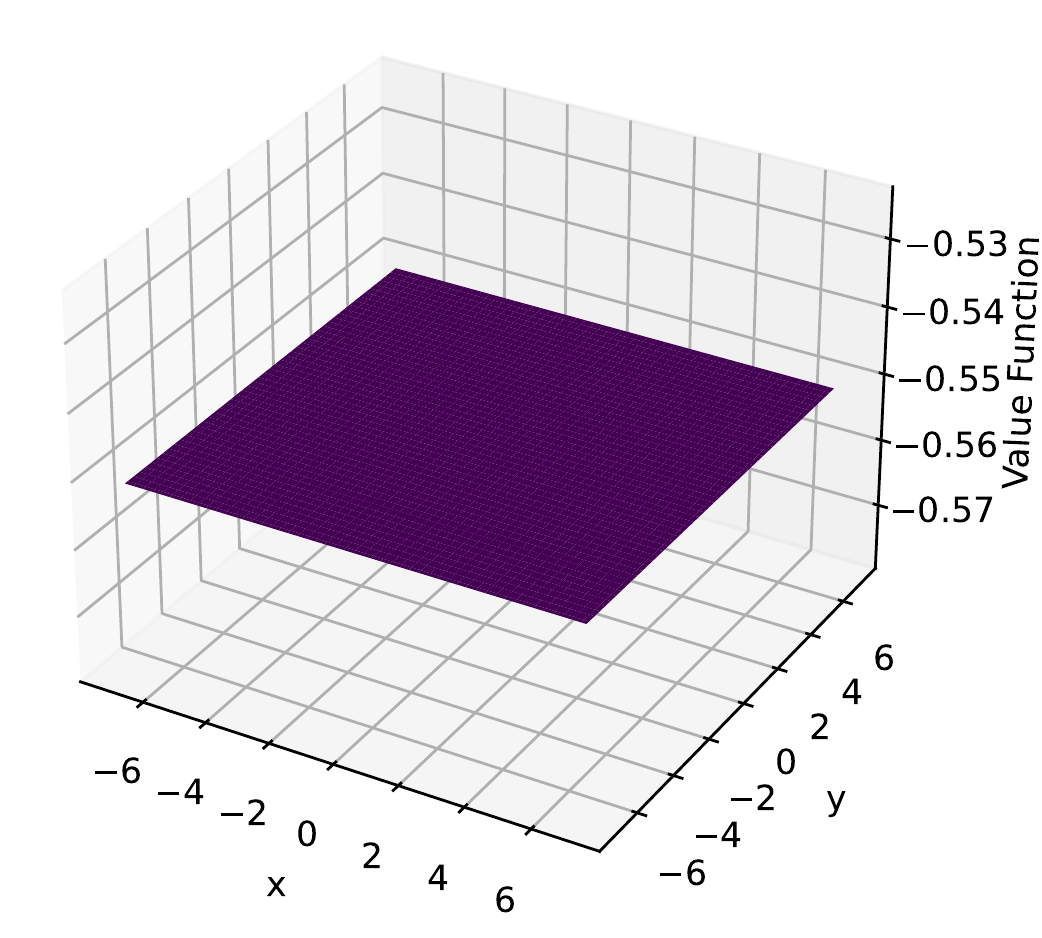}}
         \subfigure[7E$3$ iterations]
        {\includegraphics[width=3cm,height=3cm]{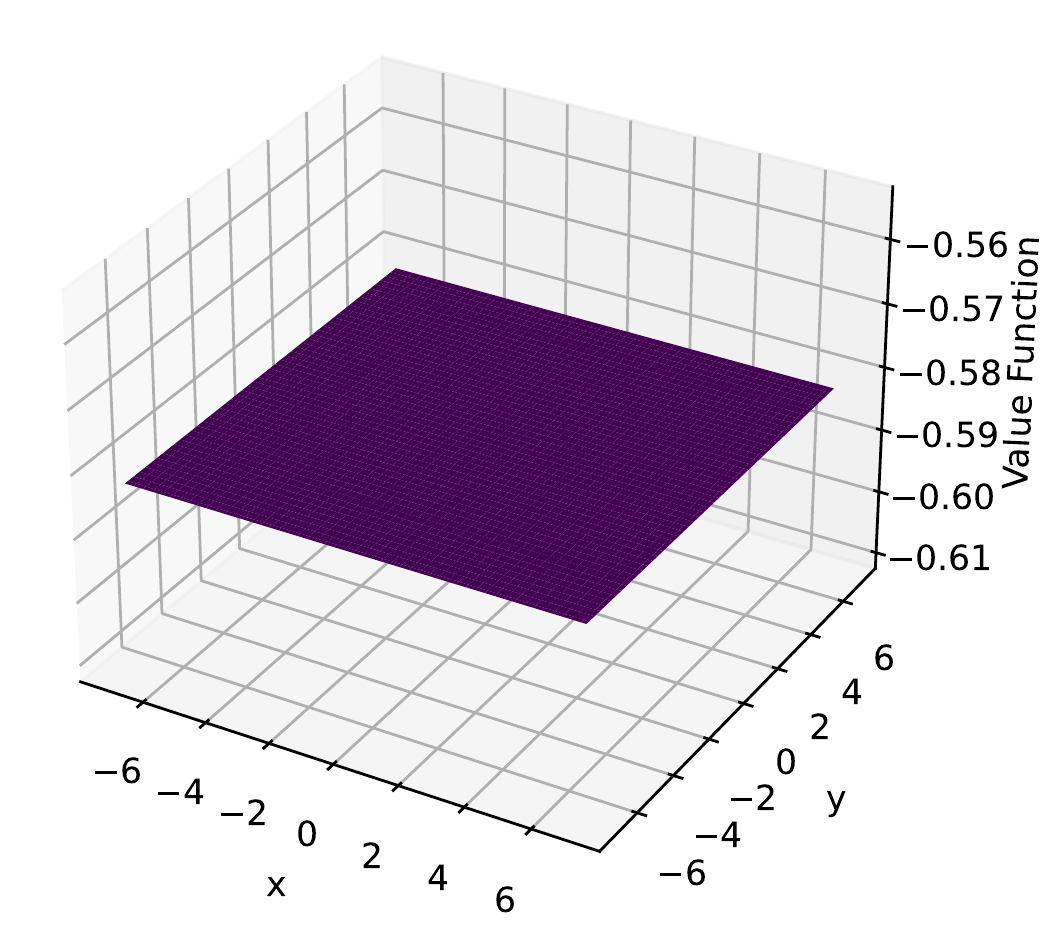}}
             \subfigure[8E$3$ iterations]
            {\includegraphics[width=3cm,height=3cm]{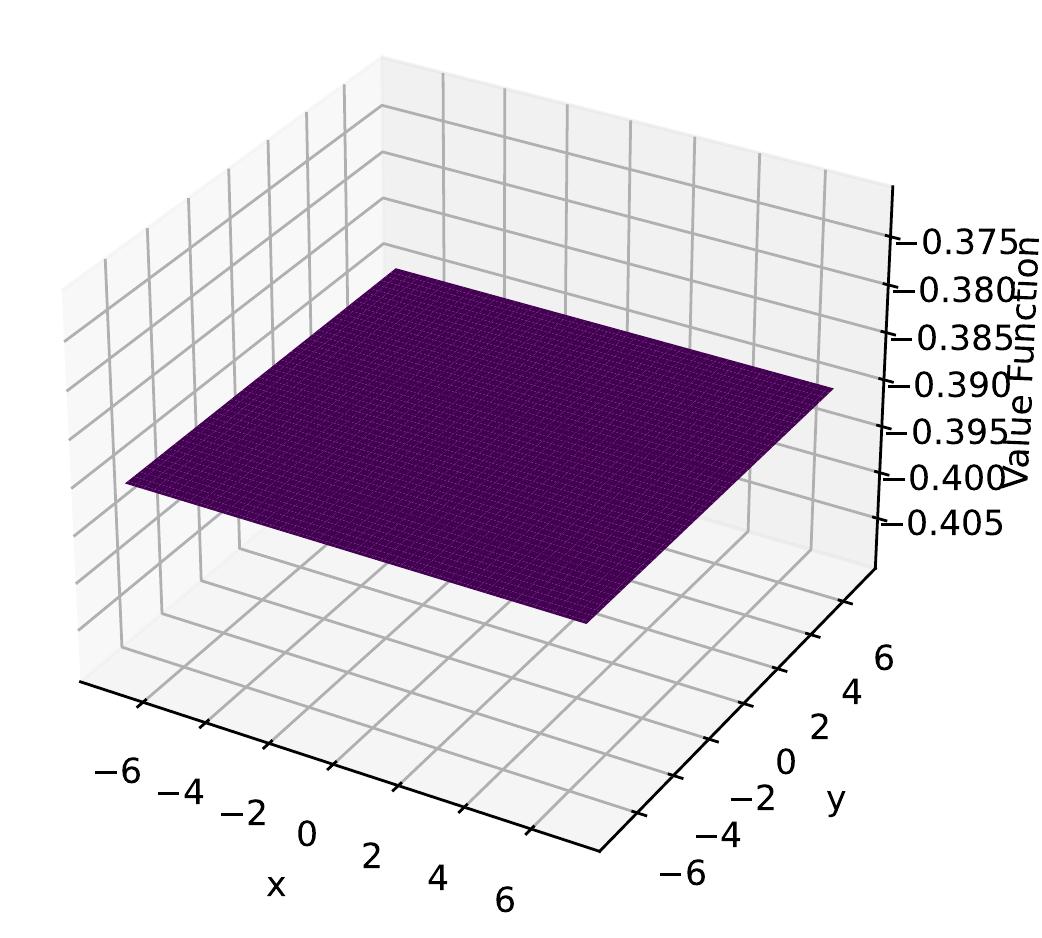}}
                 \subfigure[9E$3$ iterations]
                {\includegraphics[width=3cm,height=3cm]{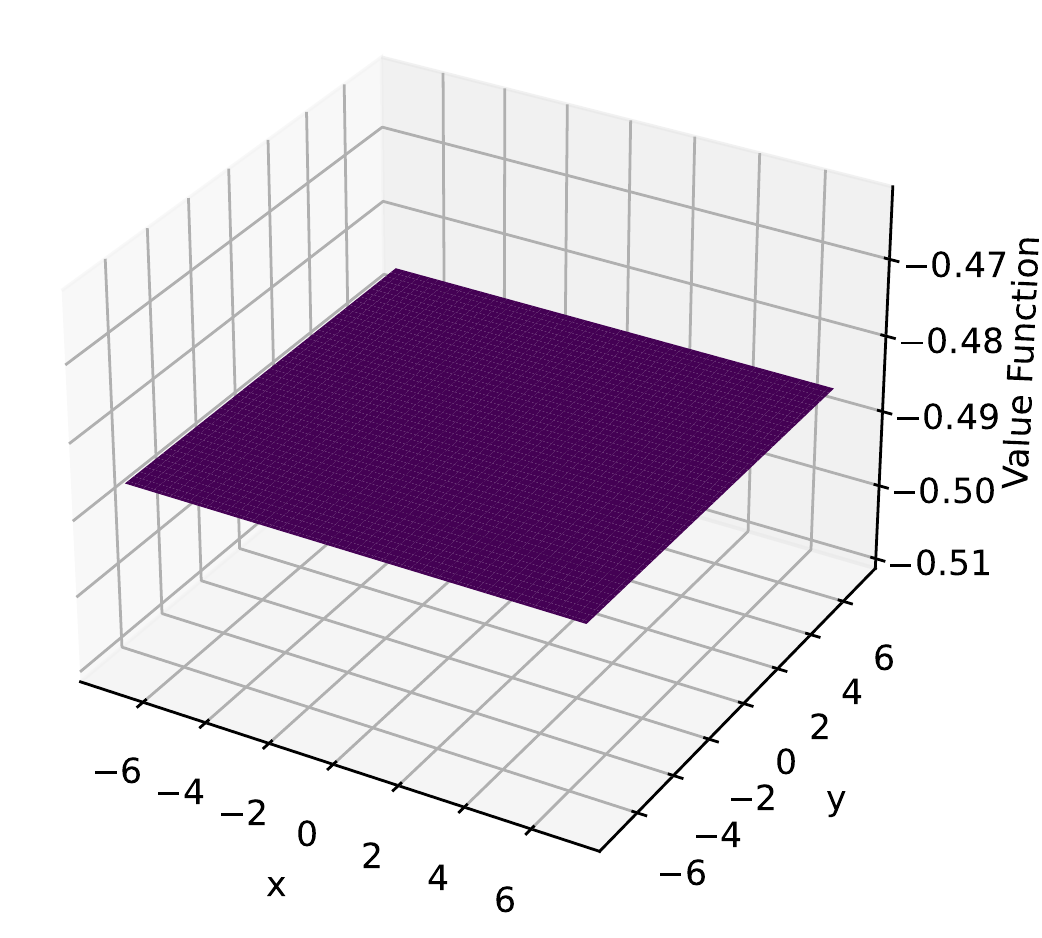}}
                     \subfigure[10E$3$ iterations]
                    {\includegraphics[width=3cm,height=3cm]{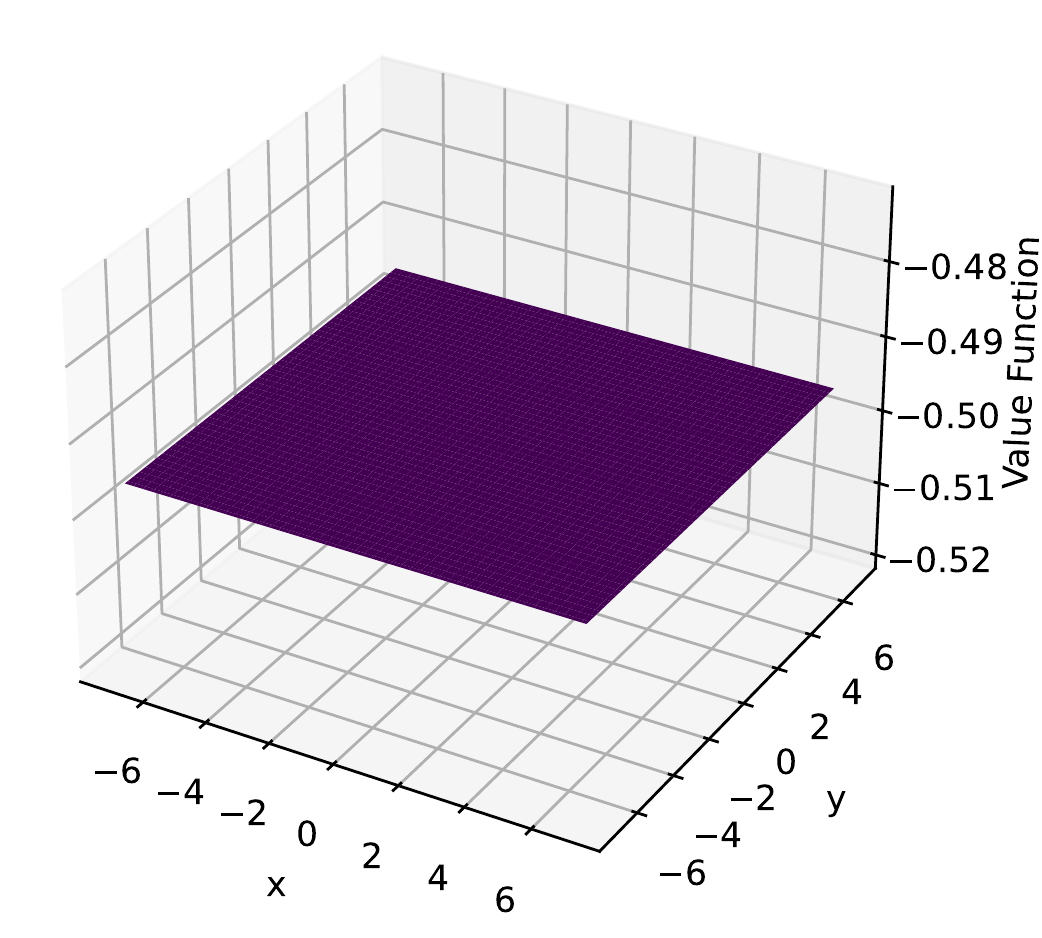}}
    \caption
    {Policy representation comparison of PPO with different iterations.
    %and the optimal policy is to go to one of the goal positions randomly.
    %, and we have shown the solid red lines as the final action learned by different algorithms with 1000 iterations.
    }
     \label{app-fig-ppo-policy}
\end{figure*}

\subsection{Results Report}

%
%In this section, we present the diffusion model is powerful to represent a complex policy distribution by two following aspects: 
%1.) fitting a multimodal policy distribution is efficient for exploration;
%2.) and empirical verification with a numerical experiment.

We have shown all the results in Figure \ref{app-fig-diffusion-policy} (for diffusion policy), \ref{app-fig-sac-policy} (for SAC), \ref{app-fig-td3-policy} (for TD3) and \ref{app-fig-ppo-policy} (for PPO), where we train the policy with a total $10000$ iterations, and show the 2D and 3D visualization every $1000$ iteration.

Figure \ref{app-fig-diffusion-policy} shows that the diffusion policy accurately captures a multimodal distribution landscape of reward, while from Figure \ref{app-fig-sac-policy},  \ref{app-fig-td3-policy}, and \ref{app-fig-ppo-policy}, we know that both SAC, TD3, and PPO are not well suited to capture such multimodality. 
Comparing  Figure \ref{app-fig-diffusion-policy} to Figure \ref{app-fig-sac-policy} and \ref{app-fig-td3-policy}, we know that although SAC and TD3 share a similar best reward performance, where both diffusion policy and SAC and TD3 keep the highest reward around $-20$, diffusion policy matches the real environment and performance shape.

From Figure \ref{app-fig-ppo-policy}, we also find PPO always runs around at the initial value, and it does not improve the reward performance, which implies PPO fails to fit multimodality. It does not learn any information about multimodality.

From the distributions of action directions and lengths, we also know the diffusion policy keeps a more gradual and steady action size than the SAC, TD3, and PPO to learn the multimodal reward performance.  Thus, the diffusion model is a powerful policy representation that leads to a more sufficient exploration and better performance, which is our motivation to consider representing policy via the diffusion model.

\clearpage

\section{Additional Experiments}
\label{app-sec-experiment-dipo}

In this section, we provide additional details about the experiments, including Hyper-parameters of all the algorithms; additional tricks for implementation of DIPO; details and additional reports for state-visiting; and ablation study on MLP and VAE.

The Python code for our implementation of DIPO is provided along with this submission in the supplementary material. 
SAC: \url{https://github.com/toshikwa/soft-actor-critic.pytorch}
PPO: \url{https://github.com/ikostrikov/pytorch-a2c-ppo-acktr-gail}
TD3: \url{https://github.com/sfujim/TD3}, which were official code library.

\subsection{Hyper-parameters for MuJoCo}

\label{app-sec:hyper-parameters4mujoco}

Common Hyper-parameters:

\begin{table}[h]
    \centering
    %\begin{adjustbox}{width={\textwidth}}
    \begin{tabular}{l l l l l}
    \toprule
       Hyperparameter & DIPO & SAC & TD3 & PPO \\
       \midrule
       No. of hidden layers & 2 & 2 & 2 & 2 \\
       No. of hidden nodes  & 256 & 256 & 256 & 256 \\
       Activation & mish & relu & relu & tanh \\
       Batch size & 256 & 256 & 256 & 256 \\
       Discount for reward $\gamma$ & 0.99 & 0.99 & 0.99 & 0.99 \\
       Target smoothing coefficient $\tau$ & 0.005 & 0.005 & 0.005 & 0.005 \\
       Learning rate for actor & $3\times 10^{-4}$ & $3\times 10^{-4}$ & $3\times 10^{-4}$ & $7\times 10^{-4}$ \\
       Learning rate for critic & $3\times 10^{-4}$ & $3\times 10^{-4}$ & $3\times 10^{-4}$ & $7\times 10^{-4}$ \\
       Actor Critic grad norm & 2 & N/A & N/A & 0.5  \\
       Memeroy size & $1\times 10^{6}$ & $1\times 10^{6}$ & $1\times 10^{6}$ & $1\times 10^{6}$ \\
       Entropy coefficient & N/A & 0.2 & N/A & 0.01 \\
       Value loss coefficient & N/A & N/A & N/A & 0.5 \\
       Exploration noise  & N/A & N/A & $\mathcal{N}(0,0.1)$ & N/A \\
       Policy noise  & N/A & N/A & $\mathcal{N}(0,0.2)$ & N/A \\
       Noise clip & N/A & N/A & 0.5 & N/A \\
       Use gae & N/A & N/A &N/A  &True  \\
       \bottomrule
    \end{tabular}
    %\end{adjustbox}
    \caption{Hyper-parameters for algorithms.}
    \label{tab:hyperparameters-algorithm}
\end{table}

Additional Hyper-parameters of DIPO:

\begin{table}[h]
    \centering
    \begin{adjustbox}{width={\textwidth}}
    \begin{tabular}{l l l l l l}
    \toprule
       Hyperparameter & Hopper-v3 & Walker2d-v3 & Ant-v3 & HalfCheetah-v3 & Humanoid-v3 \\
       \midrule
       Learning rate for action & 0.03 & 0.03 & 0.03 & 0.03 & 0.03  \\
       Actor Critic grad norm & 1 & 2 & 0.8 & 2 & 2  \\
       Action grad norm ratio & 0.3 & 0.08 & 0.1 & 0.08 & 0.1  \\
       Action gradient steps & 20 & 20 & 20 & 40 & 20 \\
       Diffusion inference timesteps & 100 & 100 & 100 & 100 & 100  \\
       Diffusion beta schedule & cosine & cosine & cosine & cosine & cosine  \\
       Update actor target every & 1 & 1 & 1 & 2 & 1  \\
       \bottomrule
    \end{tabular}
    \end{adjustbox}
    \caption{Hyper-parameters of DIPO.}
    \label{tab:hyperparameters-mujoco}
\end{table}

\subsection{Additional Tricks for Implementation of DIPO}

We have provided the additional details for the Algorithm \ref{app-algo-diffusion-free-based-rl}.

\subsubsection{Double Q-learning for Estimating $Q$-Value}

We consider the double Q-learning \citep{hasselt2010double} to update the $Q$ value. 
We consider the two critic networks $Q_{\bpsi_{1}}$, $Q_{\bpsi_{2}}$, two target networks $Q_{\bpsi^{'}_{1}}$,  $Q_{\bpsi^{'}_{2}}$.
Let Bellman residual be as follows,
\begin{flalign}
\nonumber
\calL_{\mathrm{Q}}(\bpsi)=\E_{(\bs_{t},\ba_{t},\bs_{t+1},\ba_{t+1})}\left[
\left\|\left(r(\bs_{t+1}|\bs_t,\ba_t)+\gamma\min_{i=1,2} Q_{\bpsi_{i}^{'}}(\bs_{t+1},\ba_{t+1})\right)-Q_{\bpsi}(\bs_t,\ba_t)\right\|^2\right].
\end{flalign}
Then, we update $\bpsi_{i}$ as follows, for $i\in\{1,2\}$
\[
\bpsi_{i}\gets\bpsi_{i}-\eta\bnabla \calL_{\mathrm{Q}}(\bpsi_{i}).
\]

Furthermore, we consider the following soft update rule for $\bpsi^{'}_{i}$ as follows,
\[
\bpsi^{'}_{i}\gets\rho\bpsi^{'}_{i}+(1-\rho)\bpsi_{i}.
\]

Finally, for the action gradient step, we consider the following update rule: replacing each action $\ba_{t}\in\calD_{\mathrm{env}}$ as follows 
         \[\ba_{t}\gets\ba_{t}+\eta_{a}\bnabla_{\ba}\left(\min_{i=1,2}\left\{Q_{\bpsi_{i}}(\bs_{t},\ba)\right\}\right)\bigg|_{\ba=\ba_t}.\]

\subsubsection{Critic and 	Diffusion Model}

We use a four-layer feedforward neural network of 256 hidden nodes, with activation function Mish \citep{misra2019mish} between each layer, to design the two critic networks $Q_{\bpsi_{1}}$, $Q_{\bpsi_{2}}$ two target networks $Q_{\bpsi^{'}_{1}}$,  $Q_{\bpsi^{'}_{2}}$, and the noise term $\bepsilon_{\bphi}$.
We consider gradient normalization for critic and $\bepsilon_{\bphi}$ to stabilize the training process.

For each reverse time $k\in[K]$, we consider the sinusoidal positional encoding \citep{vaswani2017attention} to encode each $k\in[K]$ into a 32-dimensional vector.

\begin{figure*}[t!]
    \centering
   % \subfigure[$10^5$ iterations]
    {\includegraphics[width=1.4cm,height=1.4cm]{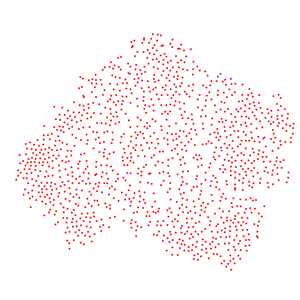}}
     %\subfigure[$2\times10^5$ iterations]
   {\includegraphics[width=1.4cm,height=1.4cm]{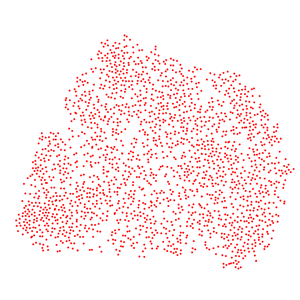}}
     % \subfigure[$3\times10^5$ iterations]
{\includegraphics[width=1.4cm,height=1.4cm]{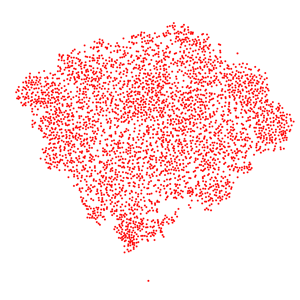}}
     % \subfigure[$4\times10^5$ iterations]
    {\includegraphics[width=1.4cm,height=1.4cm]{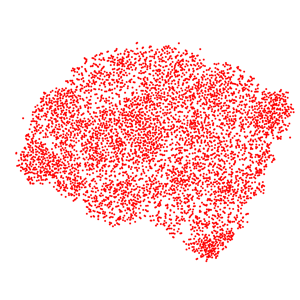}}
     %\subfigure[$5\times10^5$ iterations]
{\includegraphics[width=1.4cm,height=1.4cm]{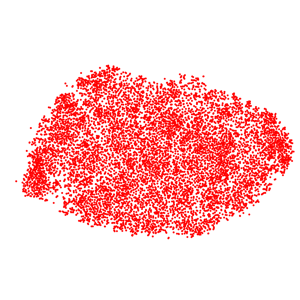}}
   % \subfigure[$10^5$ iterations]
    {\includegraphics[width=1.4cm,height=1.4cm]{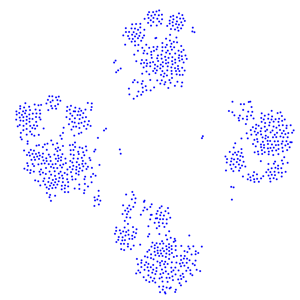}}
    % \subfigure[$2\times10^5$ iterations]
   {\includegraphics[width=1.4cm,height=1.4cm]{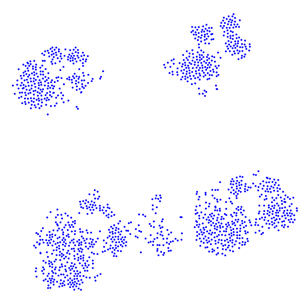}}
    %  \subfigure[$3\times10^5$ iterations]
{\includegraphics[width=1.4cm,height=1.4cm]{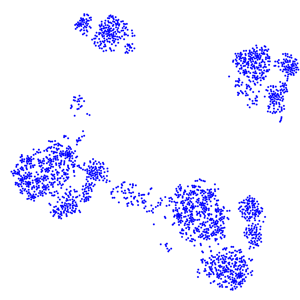}}
     % \subfigure[$4\times10^5$ iterations]
    {\includegraphics[width=1.4cm,height=1.4cm]{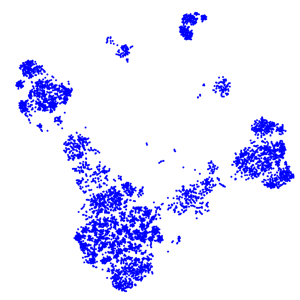}}
   %  \subfigure[$5\times10^5$ iterations]
{\includegraphics[width=1.4cm,height=1.4cm]{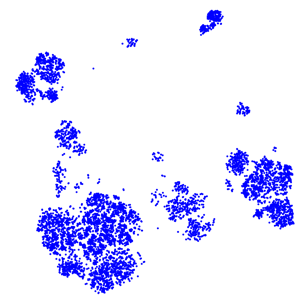}}
 % \subfigure[$10^5$ iterations]
    {\includegraphics[width=1.4cm,height=1.4cm]{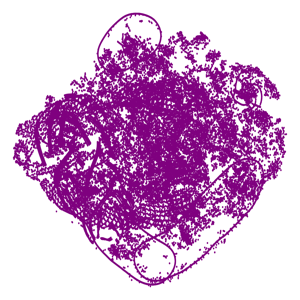}}
%     \subfigure[$2\times10^5$ iterations]
   {\includegraphics[width=1.4cm,height=1.4cm]{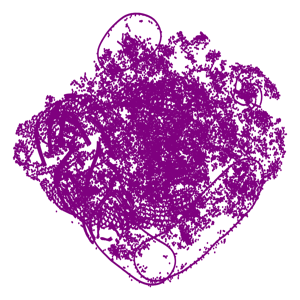}}
%      \subfigure[$3\times10^5$ iterations]
{\includegraphics[width=1.4cm,height=1.4cm]{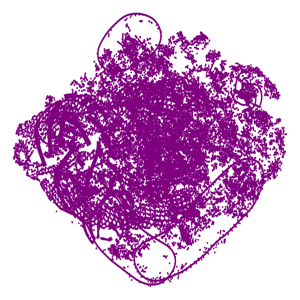}}
%      \subfigure[$4\times10^5$ iterations]
{\includegraphics[width=1.4cm,height=1.4cm]{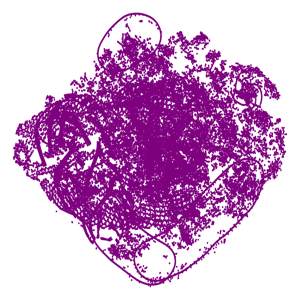}}
%     \subfigure[$5\times10^5$ iterations]
{\includegraphics[width=1.4cm,height=1.4cm]{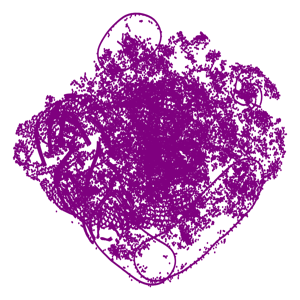}}
    {\includegraphics[width=1.4cm,height=1.4cm]{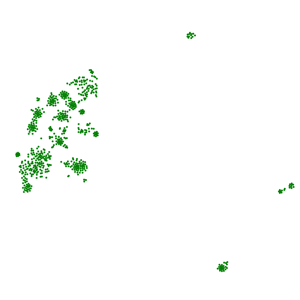}}
   {\includegraphics[width=1.4cm,height=1.4cm]{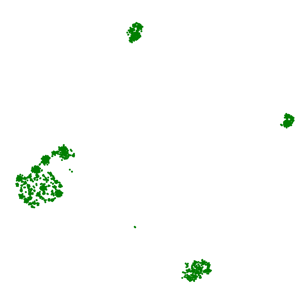}}
{\includegraphics[width=1.4cm,height=1.4cm]{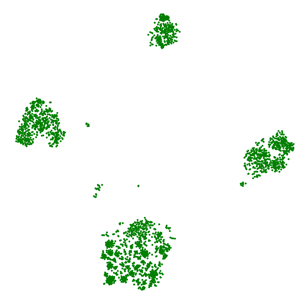}}
    {\includegraphics[width=1.4cm,height=1.4cm]{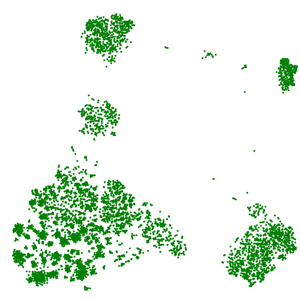}}
{\includegraphics[width=1.4cm,height=1.4cm]{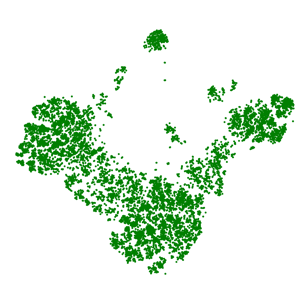}}
    {\includegraphics[width=9cm,height=0.4cm]{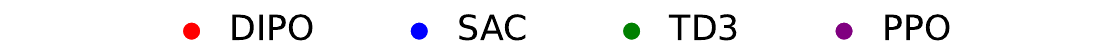}}
    \caption
    {State-visiting distribution of Humanoid-v3, where states get dimension reduction by t-SNE. The points with different colors represent the states visited by the policy with the style. The distance between points represents the difference between states.
    %and the optimal policy is to go to one of the goal positions randomly.
    %, and we have shown the solid red lines as the final action learned by different algorithms with 1000 iterations.
    }
        \label{state-Humanoid}
\end{figure*}

\begin{figure*}[t!]
    \centering
   % \subfigure[$10^5$ iterations]
    {\includegraphics[width=1.4cm,height=1.4cm]{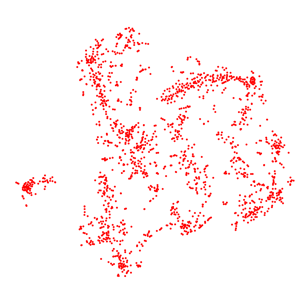}}
     %\subfigure[$2\times10^5$ iterations]
   {\includegraphics[width=1.4cm,height=1.4cm]{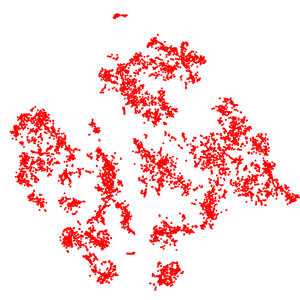}}
     % \subfigure[$3\times10^5$ iterations]
{\includegraphics[width=1.4cm,height=1.4cm]{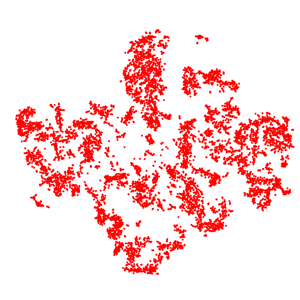}}
     % \subfigure[$4\times10^5$ iterations]
    {\includegraphics[width=1.4cm,height=1.4cm]{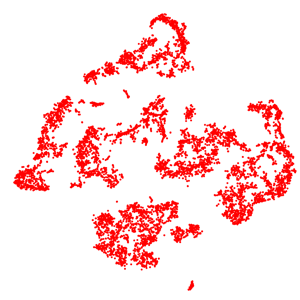}}
     %\subfigure[$5\times10^5$ iterations]
{\includegraphics[width=1.4cm,height=1.4cm]{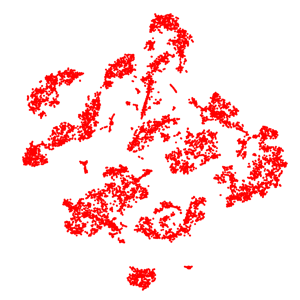}}
   % \subfigure[$10^5$ iterations]
    {\includegraphics[width=1.4cm,height=1.4cm]{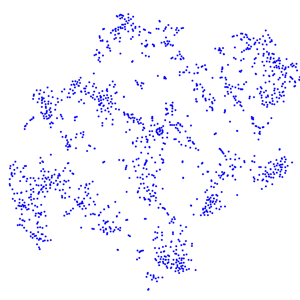}}
    % \subfigure[$2\times10^5$ iterations]
   {\includegraphics[width=1.4cm,height=1.4cm]{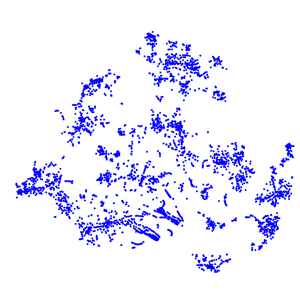}}
    %  \subfigure[$3\times10^5$ iterations]
{\includegraphics[width=1.4cm,height=1.4cm]{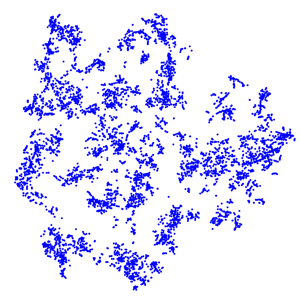}}
     % \subfigure[$4\times10^5$ iterations]
    {\includegraphics[width=1.4cm,height=1.4cm]{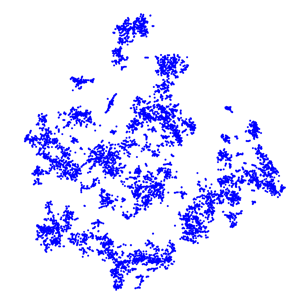}}
   %  \subfigure[$5\times10^5$ iterations]
{\includegraphics[width=1.4cm,height=1.4cm]{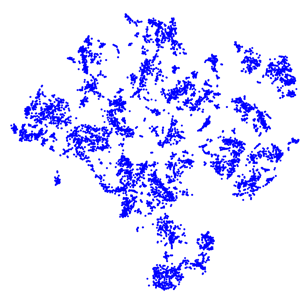}}
 % \subfigure[$10^5$ iterations]
    {\includegraphics[width=1.4cm,height=1.4cm]{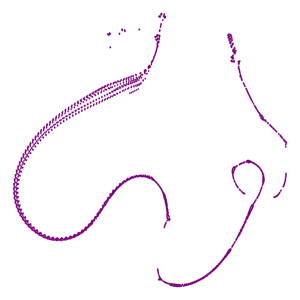}}
%     \subfigure[$2\times10^5$ iterations]
   {\includegraphics[width=1.4cm,height=1.4cm]{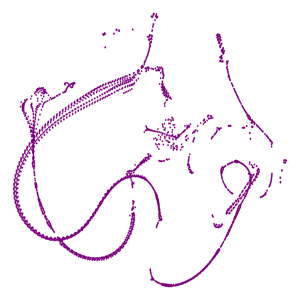}}
%      \subfigure[$3\times10^5$ iterations]
{\includegraphics[width=1.4cm,height=1.4cm]{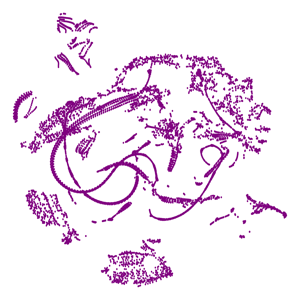}}
%      \subfigure[$4\times10^5$ iterations]
{\includegraphics[width=1.4cm,height=1.4cm]{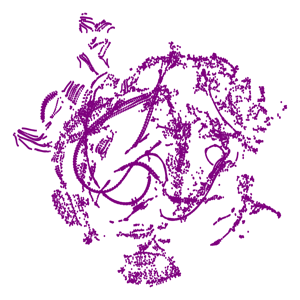}}
%     \subfigure[$5\times10^5$ iterations]
{\includegraphics[width=1.4cm,height=1.4cm]{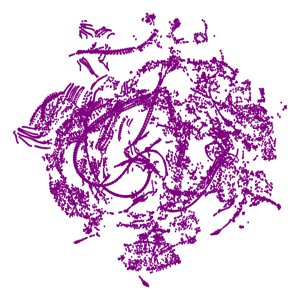}}
    {\includegraphics[width=1.4cm,height=1.4cm]{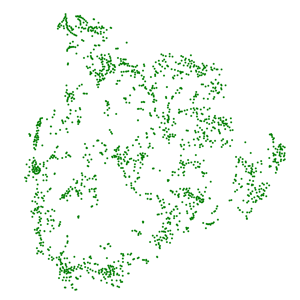}}
   {\includegraphics[width=1.4cm,height=1.4cm]{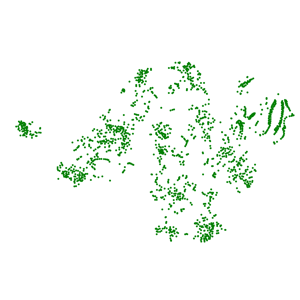}}
{\includegraphics[width=1.4cm,height=1.4cm]{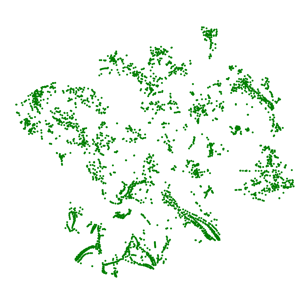}}
    {\includegraphics[width=1.4cm,height=1.4cm]{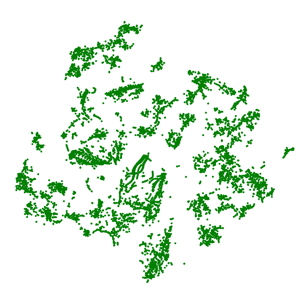}}
{\includegraphics[width=1.4cm,height=1.4cm]{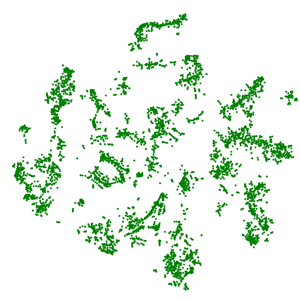}}
    {\includegraphics[width=9cm,height=0.4cm]{states/dot_label.pdf}}
    \caption
    {State-visiting distribution of Walker2d-v3, where states get dimension reduction by t-SNE. The points with different colors represent the states visited by the policy with the style. The distance between points represents the difference between states.
    %and the optimal policy is to go to one of the goal positions randomly.
    %, and we have shown the solid red lines as the final action learned by different algorithms with 1000 iterations.
    }
    \label{state-Walker2d}
\end{figure*}
\begin{figure*}[t!]
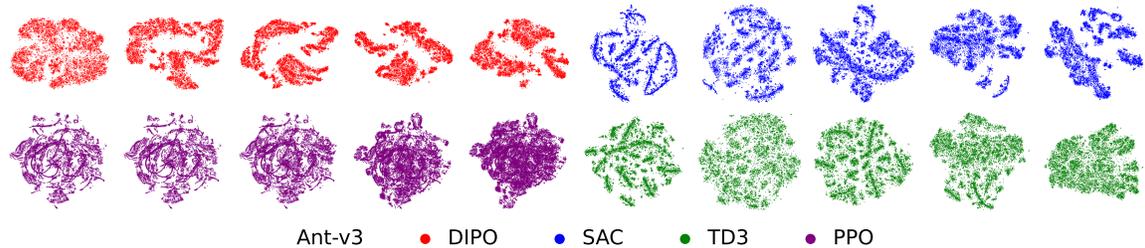

    \centering
   % \subfigure[$10^5$ iterations]
    {\includegraphics[width=1.4cm,height=1.4cm]{states/Ant-v3/DPL-step=100000}}
     %\subfigure[$2\times10^5$ iterations]
   {\includegraphics[width=1.4cm,height=1.4cm]{states/Ant-v3/DPL-step=200000}}
     % \subfigure[$3\times10^5$ iterations]
{\includegraphics[width=1.4cm,height=1.4cm]{states/Ant-v3/DPL-step=300000}}
     % \subfigure[$4\times10^5$ iterations]
    {\includegraphics[width=1.4cm,height=1.4cm]{states/Ant-v3/DPL-step=400000}}
     %\subfigure[$5\times10^5$ iterations]
{\includegraphics[width=1.4cm,height=1.4cm]{states/Ant-v3/DPL-step=500000}}
   % \subfigure[$10^5$ iterations]
    {\includegraphics[width=1.4cm,height=1.4cm]{states/Ant-v3/SAC-step=100000}}
    % \subfigure[$2\times10^5$ iterations]
   {\includegraphics[width=1.4cm,height=1.4cm]{states/Ant-v3/SAC-step=200000}}
    %  \subfigure[$3\times10^5$ iterations]
{\includegraphics[width=1.4cm,height=1.4cm]{states/Ant-v3/SAC-step=300000}}
     % \subfigure[$4\times10^5$ iterations]
    {\includegraphics[width=1.4cm,height=1.4cm]{states/Ant-v3/SAC-step=400000}}
   %  \subfigure[$5\times10^5$ iterations]
{\includegraphics[width=1.4cm,height=1.4cm]{states/Ant-v3/SAC-step=500000}}
 % \subfigure[$10^5$ iterations]
    {\includegraphics[width=1.4cm,height=1.4cm]{states/PPO_state/Ant-v3/step=100000}}
%     \subfigure[$2\times10^5$ iterations]
   {\includegraphics[width=1.4cm,height=1.4cm]{states/PPO_state/Ant-v3/step=200000}}
%      \subfigure[$3\times10^5$ iterations]
{\includegraphics[width=1.4cm,height=1.4cm]{states/PPO_state/Ant-v3/step=300000}}
%      \subfigure[$4\times10^5$ iterations]
{\includegraphics[width=1.4cm,height=1.4cm]{states/PPO_state/Ant-v3/step=400000}}
%     \subfigure[$5\times10^5$ iterations]
{\includegraphics[width=1.4cm,height=1.4cm]{states/PPO_state/Ant-v3/step=500000}}
% \subfigure[$2$E$5$ iterations]
    {\includegraphics[width=1.4cm,height=1.4cm]{states/Ant-v3/TD3-step=100000}}
    % \subfigure[$4$E$5$  iterations]
   {\includegraphics[width=1.4cm,height=1.4cm]{states/Ant-v3/TD3-step=200000}}
      %\subfigure[$6$E$5$  iterations]
{\includegraphics[width=1.4cm,height=1.4cm]{states/Ant-v3/TD3-step=300000}}
     % \subfigure[$8$E$5$  iterations]
    {\includegraphics[width=1.4cm,height=1.4cm]{states/Ant-v3/TD3-step=400000}}
     %\subfigure[E6 iterations]
{\includegraphics[width=1.4cm,height=1.4cm]{states/Ant-v3/TD3-step=500000}}
{\includegraphics[width=9cm,height=0.38cm]{states/Ant-v3_dot_label}}
   \caption
  {State-visiting visualization by each algorithm on the Ant-v3 task, where states get dimension reduction by t-SNE. The points with different colors represent the states visited by the policy with the style. The distance between points represents the difference between states. }
 \label{app-state-ant}
\end{figure*}

\subsection{Details and Additional Reports for State-Visiting}

\label{sec-app-state-visiting}

In this section, we provide more details for Section \ref{ex-comparative-eva-ill}, including the implementation details (see Appendix \ref{app-Implementation-details4-2D-state}), more comparisons and more insights for the empirical results. We provide the main discussions cover the following three observations:
\begin{itemize}
\item poor exploration results in poor initial reward performance;
\item good final reward performance along with dense state-visiting;
\item a counterexample: PPO violates the above two observations.
\end{itemize}

\subsubsection{Implementation Details for 2D State-Visiting}
\label{app-Implementation-details4-2D-state}

We save the parameters for each algorithm during the training for each 1E5 iteration.
Then we run the model with an episode with ten random seeds to compare fairly; those ten random seeds are the same among different algorithms.
Thus, we collect a state set with ten episodes for each algorithm.
Finally, we convert high-dimensional state data into two-dimensional state data by t-SNE \citep{van2008visualizing}, and we show the visualization according to the open implementation
\url{https://scikit-learn.org/stable/auto_examples/manifold/plot_t_sne_perplexity.html}
where we set the parameters as follows,
\[\texttt{perpexity}=50, \texttt{early\_exaggeration}=12,\texttt{random\_state}=33.\]

We have shown all the results in Figure \ref{state-Humanoid} (for Humanoid); Figure \ref{state-Walker2d} (for Walker2d); Figure \ref{app-state-ant} (for Ant); Figure \ref{app-state-HalfCheetah} (for HalfCheetah); and Figure \ref{state-hopper} (for Hopper), where we polt the result after each E5 iterations.

\begin{figure*}[t!]
    \centering
   % \subfigure[$10^5$ iterations]
    {\includegraphics[width=1.4cm,height=1.4cm]{states/HalfCheetah-v3/DPL-step=100000}}
     %\subfigure[$2\times10^5$ iterations]
   {\includegraphics[width=1.4cm,height=1.4cm]{states/HalfCheetah-v3/DPL-step=200000}}
     % \subfigure[$3\times10^5$ iterations]
{\includegraphics[width=1.4cm,height=1.4cm]{states/HalfCheetah-v3/DPL-step=300000}}
     % \subfigure[$4\times10^5$ iterations]
    {\includegraphics[width=1.4cm,height=1.4cm]{states/HalfCheetah-v3/DPL-step=400000}}
     %\subfigure[$5\times10^5$ iterations]
{\includegraphics[width=1.4cm,height=1.4cm]{states/HalfCheetah-v3/DPL-step=500000}}
   % \subfigure[$10^5$ iterations]
    {\includegraphics[width=1.4cm,height=1.4cm]{states/HalfCheetah-v3/SAC-step=100000}}
    % \subfigure[$2\times10^5$ iterations]
   {\includegraphics[width=1.4cm,height=1.4cm]{states/HalfCheetah-v3/SAC-step=200000}}
    %  \subfigure[$3\times10^5$ iterations]
{\includegraphics[width=1.4cm,height=1.4cm]{states/HalfCheetah-v3/SAC-step=300000}}
     % \subfigure[$4\times10^5$ iterations]
    {\includegraphics[width=1.4cm,height=1.4cm]{states/HalfCheetah-v3/SAC-step=400000}}
   %  \subfigure[$5\times10^5$ iterations]
{\includegraphics[width=1.4cm,height=1.4cm]{states/HalfCheetah-v3/SAC-step=500000}}
 % \subfigure[$10^5$ iterations]
    {\includegraphics[width=1.4cm,height=1.4cm]{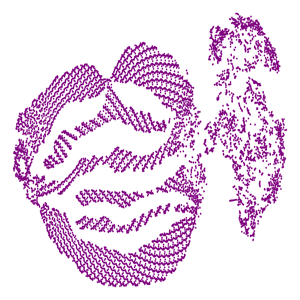}}
%     \subfigure[$2\times10^5$ iterations]
   {\includegraphics[width=1.4cm,height=1.4cm]{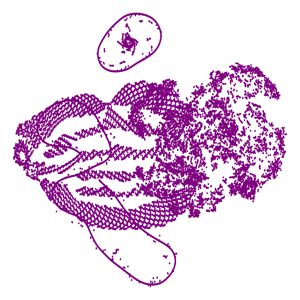}}
%      \subfigure[$3\times10^5$ iterations]
{\includegraphics[width=1.4cm,height=1.4cm]{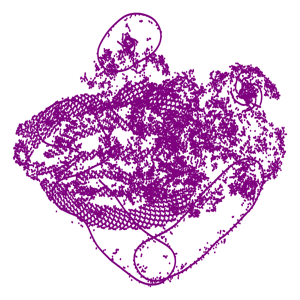}}
%      \subfigure[$4\times10^5$ iterations]
{\includegraphics[width=1.4cm,height=1.4cm]{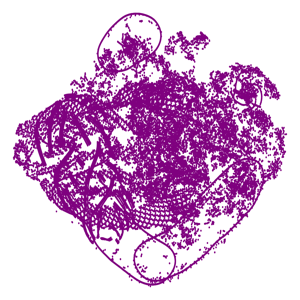}}
%     \subfigure[$5\times10^5$ iterations]
{\includegraphics[width=1.4cm,height=1.4cm]{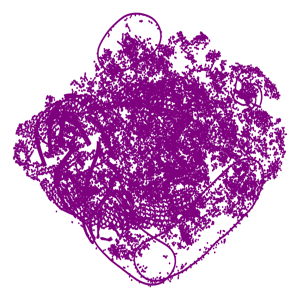}}
% \subfigure[$2$E$5$ iterations]
    {\includegraphics[width=1.4cm,height=1.4cm]{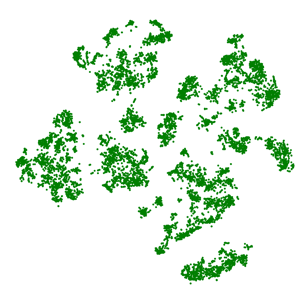}}
    % \subfigure[$4$E$5$  iterations]
   {\includegraphics[width=1.4cm,height=1.4cm]{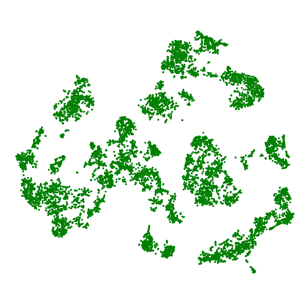}}
      %\subfigure[$6$E$5$  iterations]
{\includegraphics[width=1.4cm,height=1.4cm]{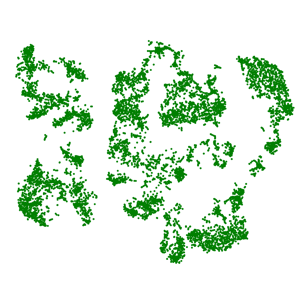}}
     % \subfigure[$8$E$5$  iterations]
    {\includegraphics[width=1.4cm,height=1.4cm]{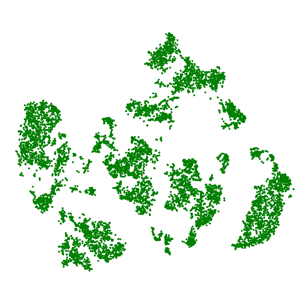}}
     %\subfigure[E6 iterations]
{\includegraphics[width=1.4cm,height=1.4cm]{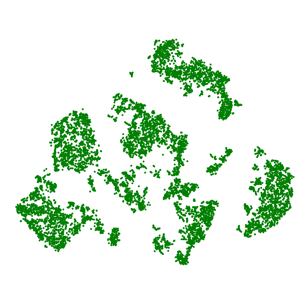}}
    {\includegraphics[width=9cm,height=0.38cm]{states/dot_label.pdf}}
    \caption
    {The state-visiting visualization by each algorithm on the HalfCheetah-v3 task, where states get dimension reduction by t-SNE. The points with different colors represent the states visited by the policy with the style. The distance between points represents the difference between states.}
     \label{app-state-HalfCheetah}
\end{figure*}
\begin{figure*}[t!]
    \centering
   % \subfigure[$10^5$ iterations]
    {\includegraphics[width=1.4cm,height=1.4cm]{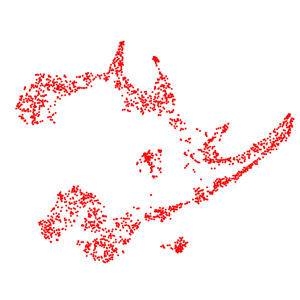}}
     %\subfigure[$2\times10^5$ iterations]
   {\includegraphics[width=1.4cm,height=1.4cm]{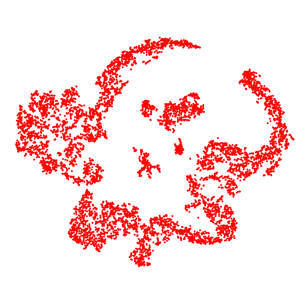}}
     % \subfigure[$3\times10^5$ iterations]
{\includegraphics[width=1.4cm,height=1.4cm]{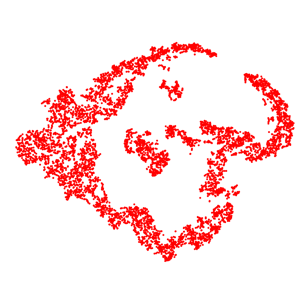}}
     % \subfigure[$4\times10^5$ iterations]
    {\includegraphics[width=1.4cm,height=1.4cm]{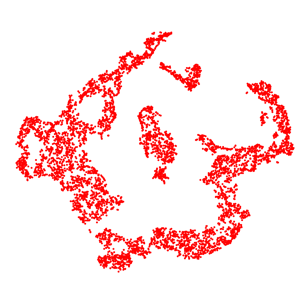}}
     %\subfigure[$5\times10^5$ iterations]
{\includegraphics[width=1.4cm,height=1.4cm]{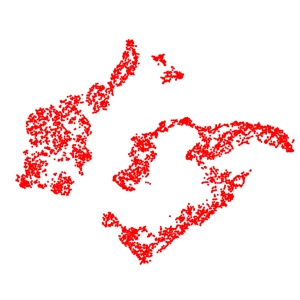}}
   % \subfigure[$10^5$ iterations]
    {\includegraphics[width=1.4cm,height=1.4cm]{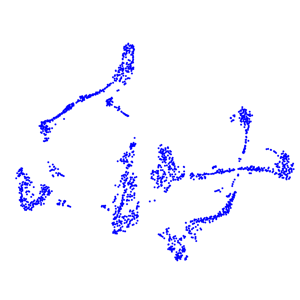}}
    % \subfigure[$2\times10^5$ iterations]
   {\includegraphics[width=1.4cm,height=1.4cm]{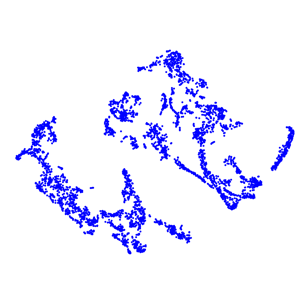}}
    %  \subfigure[$3\times10^5$ iterations]
{\includegraphics[width=1.4cm,height=1.4cm]{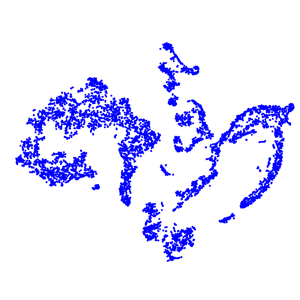}}
     % \subfigure[$4\times10^5$ iterations]
    {\includegraphics[width=1.4cm,height=1.4cm]{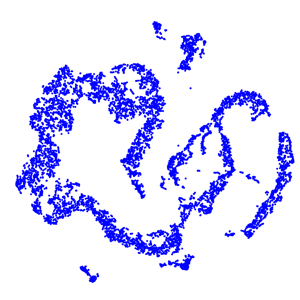}}
   %  \subfigure[$5\times10^5$ iterations]
{\includegraphics[width=1.4cm,height=1.4cm]{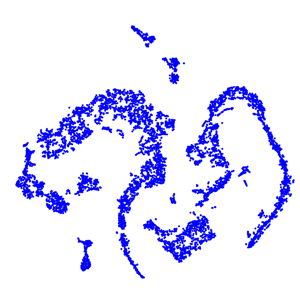}}
 % \subfigure[$10^5$ iterations]
    {\includegraphics[width=1.4cm,height=1.4cm]{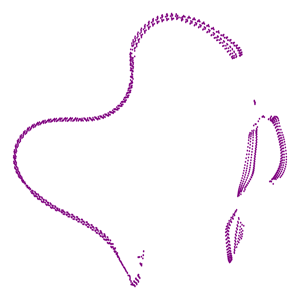}}
%     \subfigure[$2\times10^5$ iterations]
   {\includegraphics[width=1.4cm,height=1.4cm]{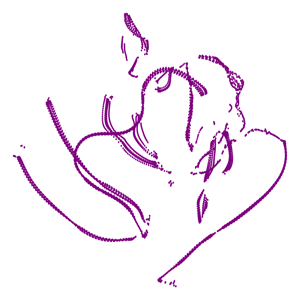}}
%      \subfigure[$3\times10^5$ iterations]
{\includegraphics[width=1.4cm,height=1.4cm]{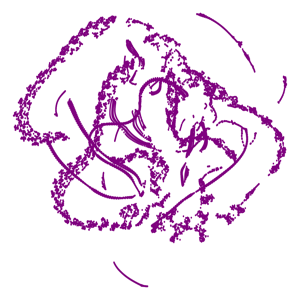}}
%      \subfigure[$4\times10^5$ iterations]
{\includegraphics[width=1.4cm,height=1.4cm]{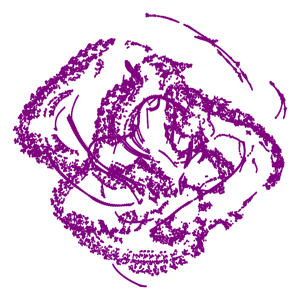}}
%     \subfigure[$5\times10^5$ iterations]
{\includegraphics[width=1.4cm,height=1.4cm]{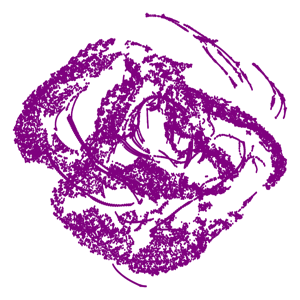}}
    {\includegraphics[width=1.4cm,height=1.4cm]{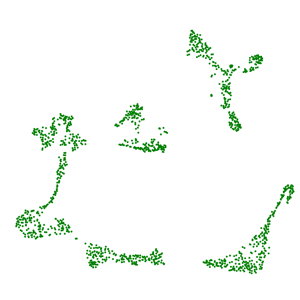}}
   {\includegraphics[width=1.4cm,height=1.4cm]{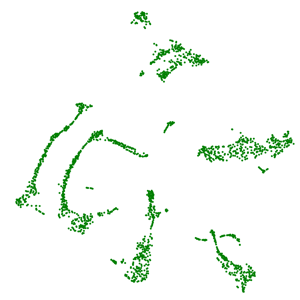}}
{\includegraphics[width=1.4cm,height=1.4cm]{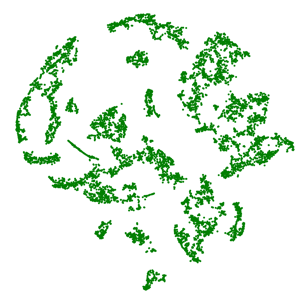}}
    {\includegraphics[width=1.4cm,height=1.4cm]{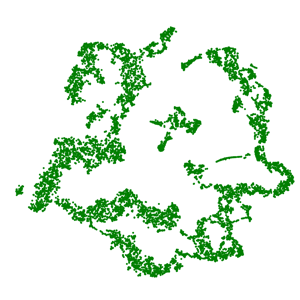}}
{\includegraphics[width=1.4cm,height=1.4cm]{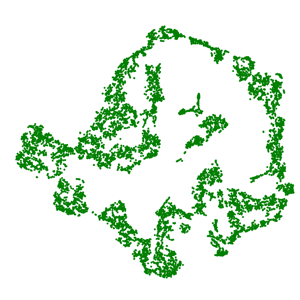}}
    {\includegraphics[width=9cm,height=0.4cm]{states/dot_label.pdf}}
    \caption
    {State-visiting distribution of Hopper-v3, where states get dimension reduction by t-SNE. The points with different colors represent the states visited by the policy with the style. The distance between points represents the difference between states.
    }
    \label{state-hopper}
\end{figure*}

\subsubsection{Observation 1: Poor Exploration Result in Poor Initial Reward Performance}

From Figure \ref{fig:comparsion-mujoco}, we know TD3 and PPO reach a worse initial reward performance than DIPO and SAC for the Hopper task, which coincides with the results appear in Figure \ref{state-hopper}. 
At the initial interaction, TD3 and PPO explore within a very sparse state-visiting region, which decays the reward performance. 
Such an empirical result also appears in the Walker2d task for PPO (see Figure \ref{state-Walker2d}), Humanoid task for TD3 and SAC (see Figure \ref{state-Humanoid}), where a spare state-visiting is always accompanied by a worse initial reward performance.
Those empirical results once again confirm a common sense: poor exploration results in poor initial reward performance.

Conversely, from Figure \ref{fig:comparsion-mujoco}, we know DIPO and SAC obtain a better initial reward performance for the Hopper task, and Figure \ref{state-hopper} shows that DIPO and SAC explore a wider range of state-visiting that covers than TD3 and PPO. That implies that a wide state visit leads to better initial reward performance.
Such an empirical result also appears in the Walker2d task for DIPO, SAC, and TD3 (see Figure \ref{state-Walker2d}), Humanoid task for DIPO (see Figure \ref{state-Humanoid}), where the agent runs with a wider range state-visiting, which is helpful to the agent obtains a better initial reward performance.

In summary, poor exploration could make the agent make a poor decision and cause a poor initial reward performance.
While if the agent explores a wider range of regions to visit more states, which is helpful for the agent to understand the environment and could lead to better initial reward performance.

\subsubsection{Observation 2: Good Final Reward Performance along with Dense State-Visiting}

From Figure \ref{state-hopper}, we know DIPO, SAC, and TD3 achieve a more dense state-visiting for the Hopper task at the final iterations. 
Such an empirical result also appears in the Walker2d and Humanoid tasks for DIPO, SAC, and TD3 (see Figure \ref{state-Humanoid} and \ref{state-Walker2d}).
This is a reasonable result since after sufficient training, the agent identifies and avoids the "bad" states, and plays actions to transfer to "good" states. 
Besides, this observation is also consistent with the result that appears in Figure \ref{fig:comparsion-mujoco}, the better algorithm (e.g., the proposed DIPO) usually visits a more narrow and dense state region at the final iterations. On the contrary, PPO shows an aimless exploration among the Ant-v3 task (see Figure \ref{state-ant}) and HalfCheetah (see Figure \ref{state-HalfCheetah}), which provides a partial explanation for why PPO is not so good in the Ant-v3 and HalfCheetah task.
This is a natural result for RL since a better algorithm should keep a better exploration at the beginning and a more sufficient exploitation at the final iterations.

\begin{figure*}[t!]
    \centering
    \subfigure[Ant-v3]
    {\includegraphics[width=5.1cm,height=3.5cm]{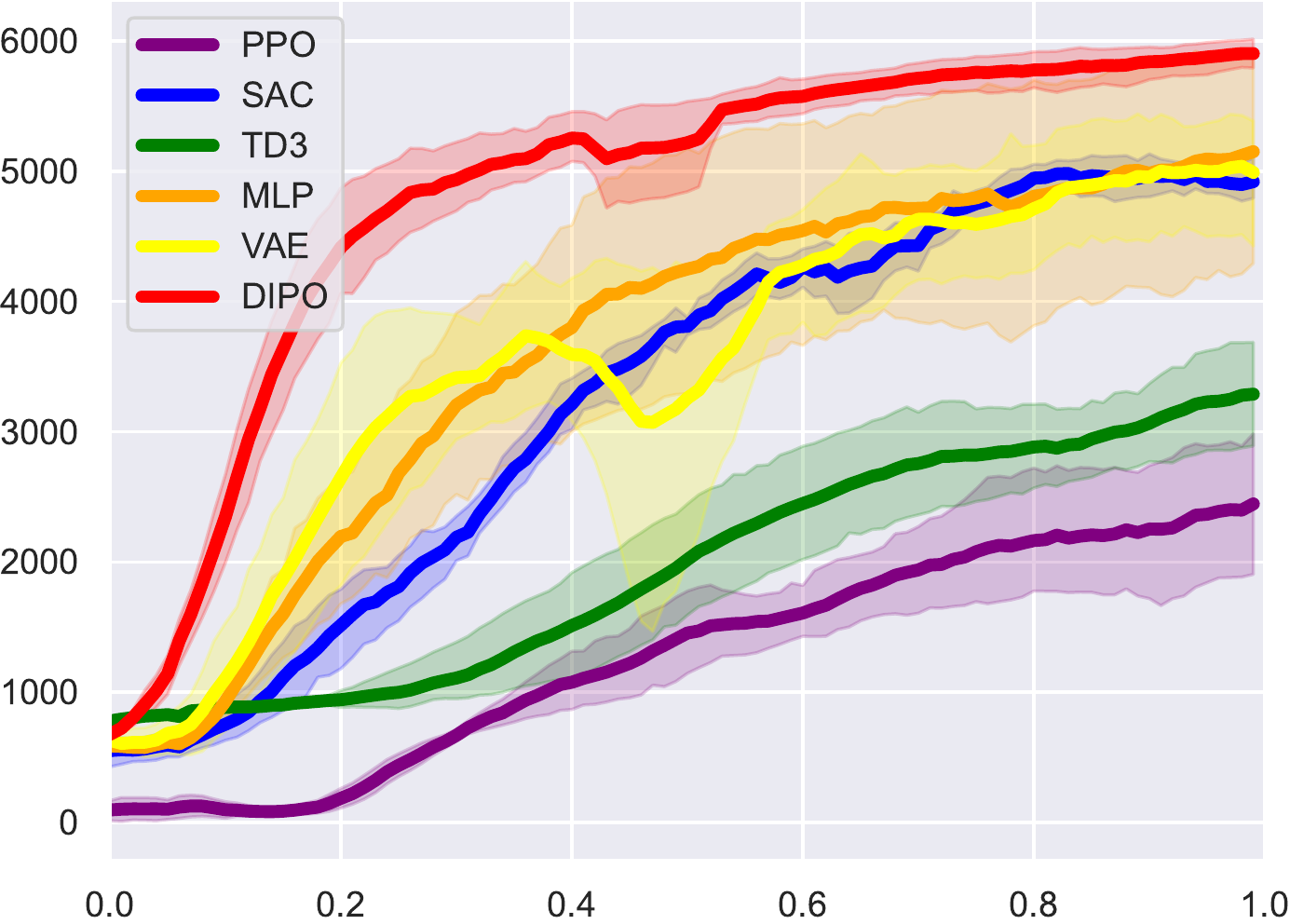}}
      \subfigure[HalfCheetah-v3]
     {\includegraphics[width=5.1cm,height=3.5cm]{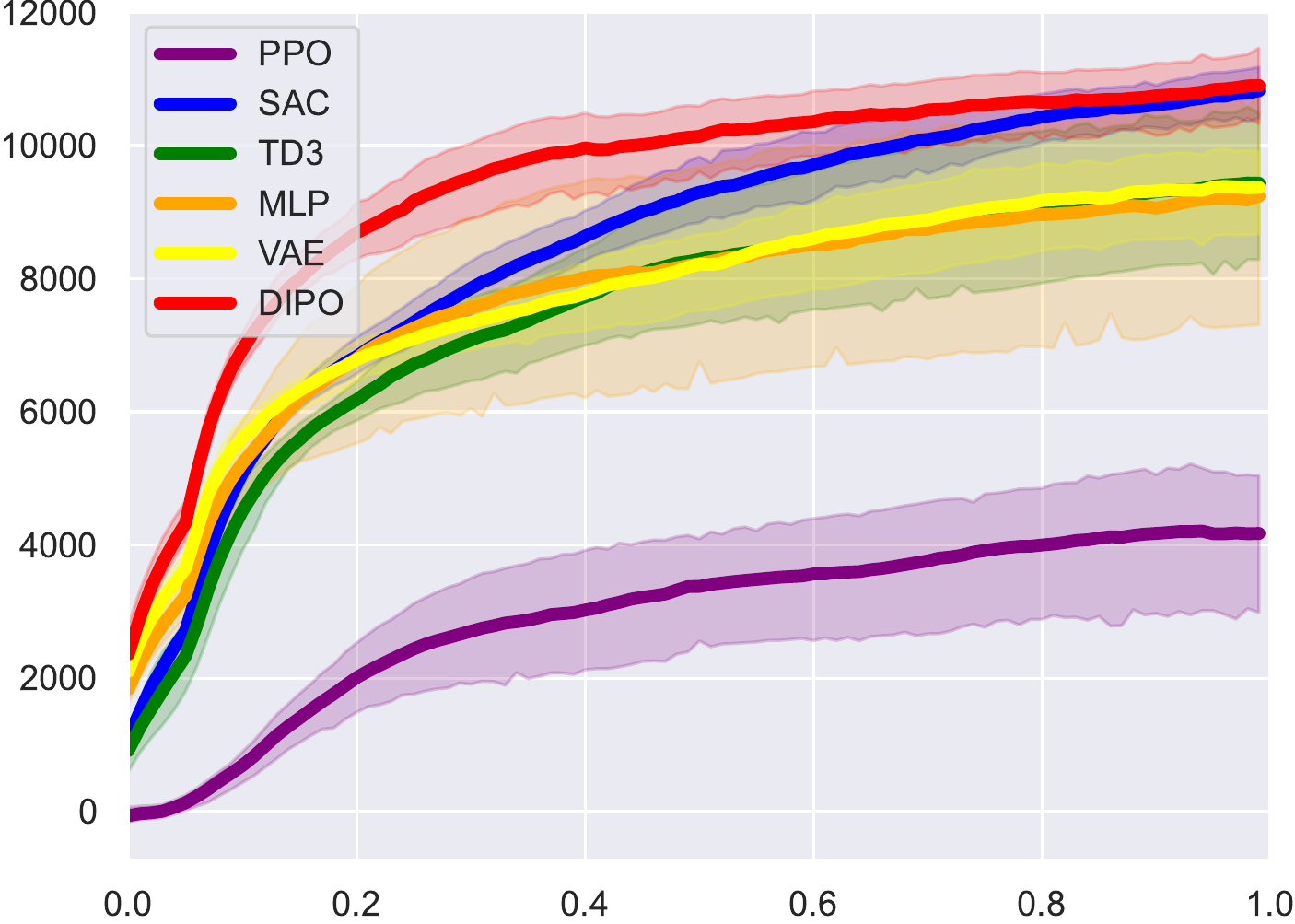}}\\
       \subfigure[Hopper-v3]
      {\includegraphics[width=5.1cm,height=3.5cm]{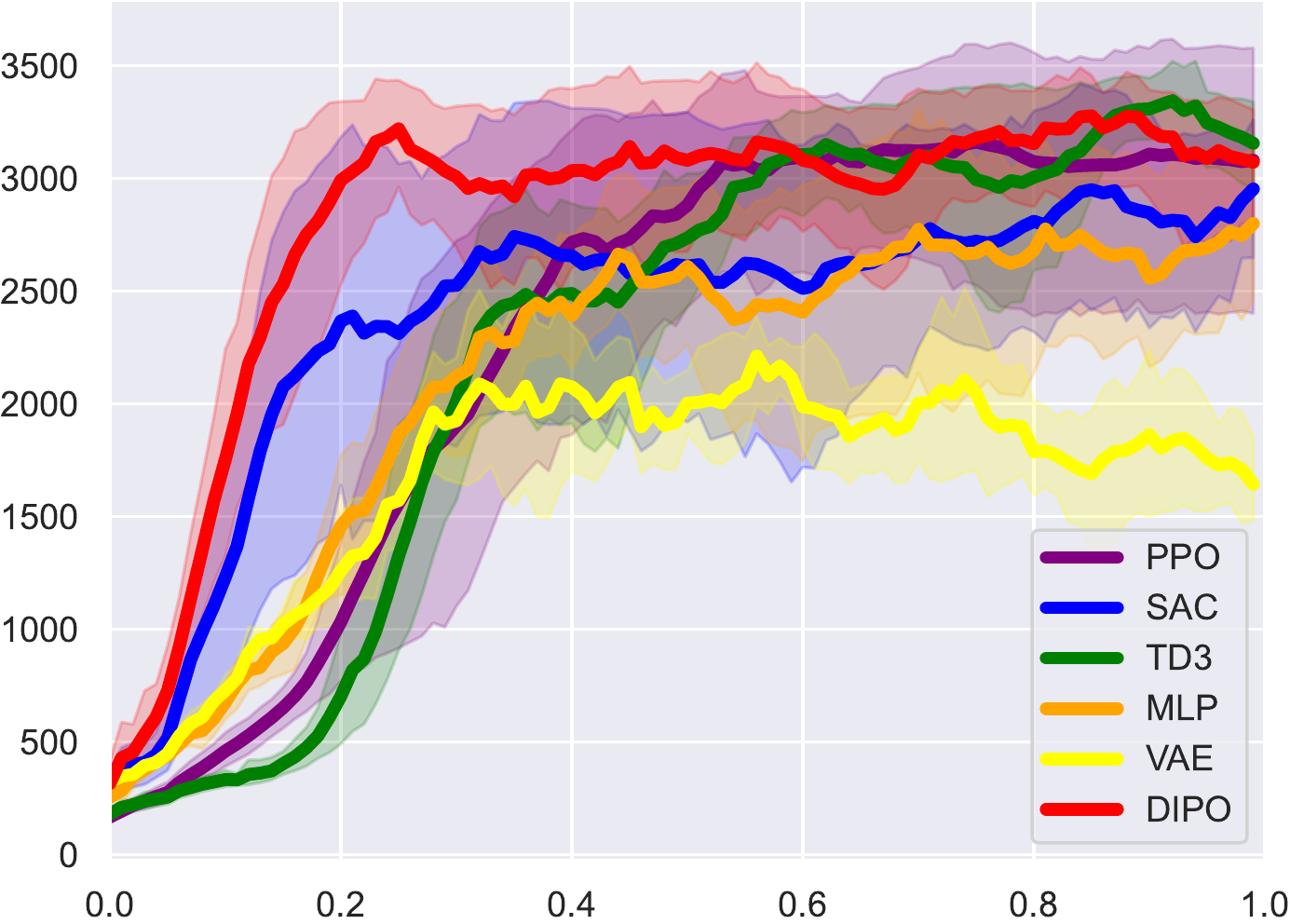}}
        \subfigure[Humanoid-v3]
       {\includegraphics[width=5.1cm,height=3.5cm]{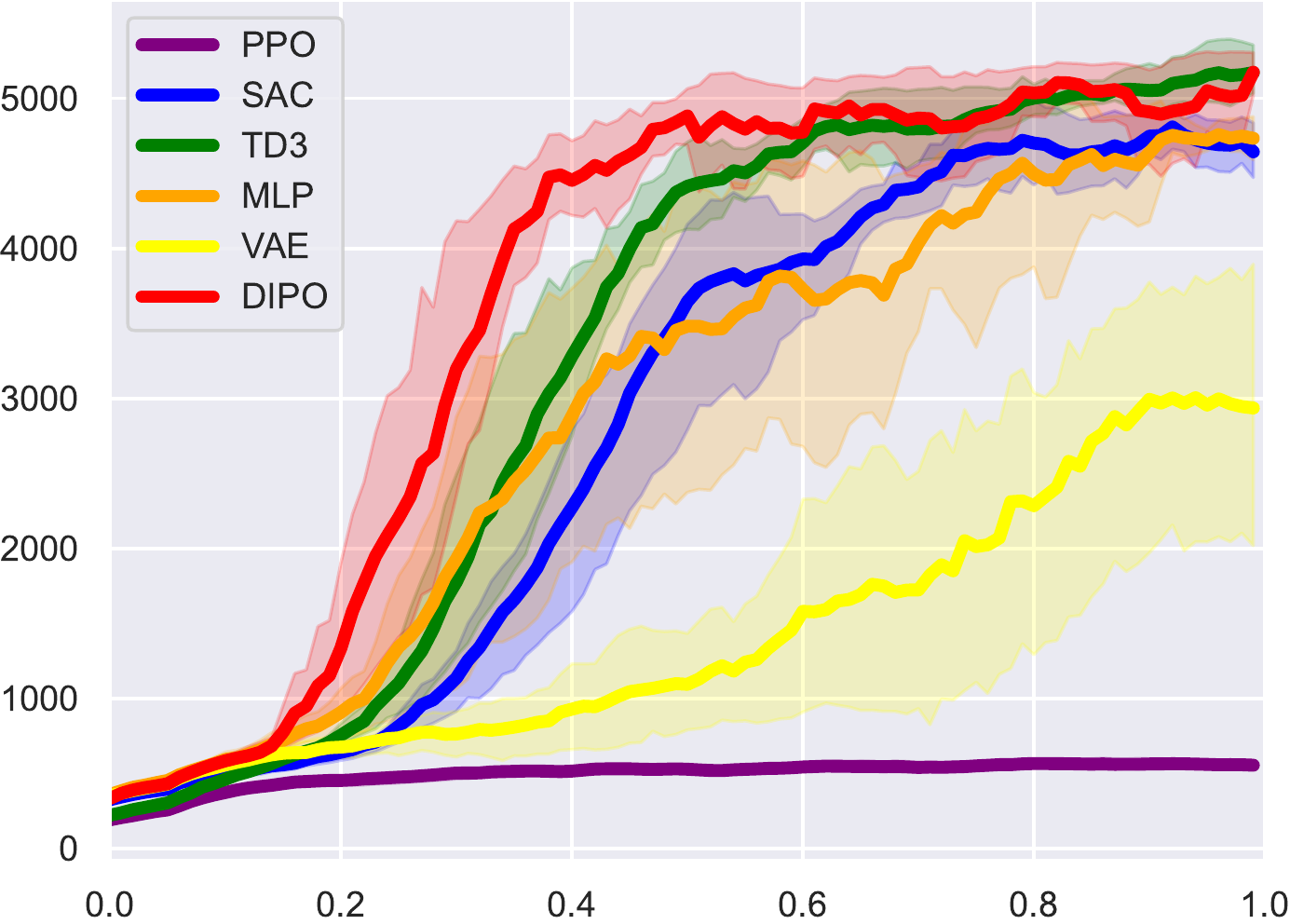}}
         \subfigure[Walker2d-v3]
        {\includegraphics[width=5.1cm,height=3.5cm]{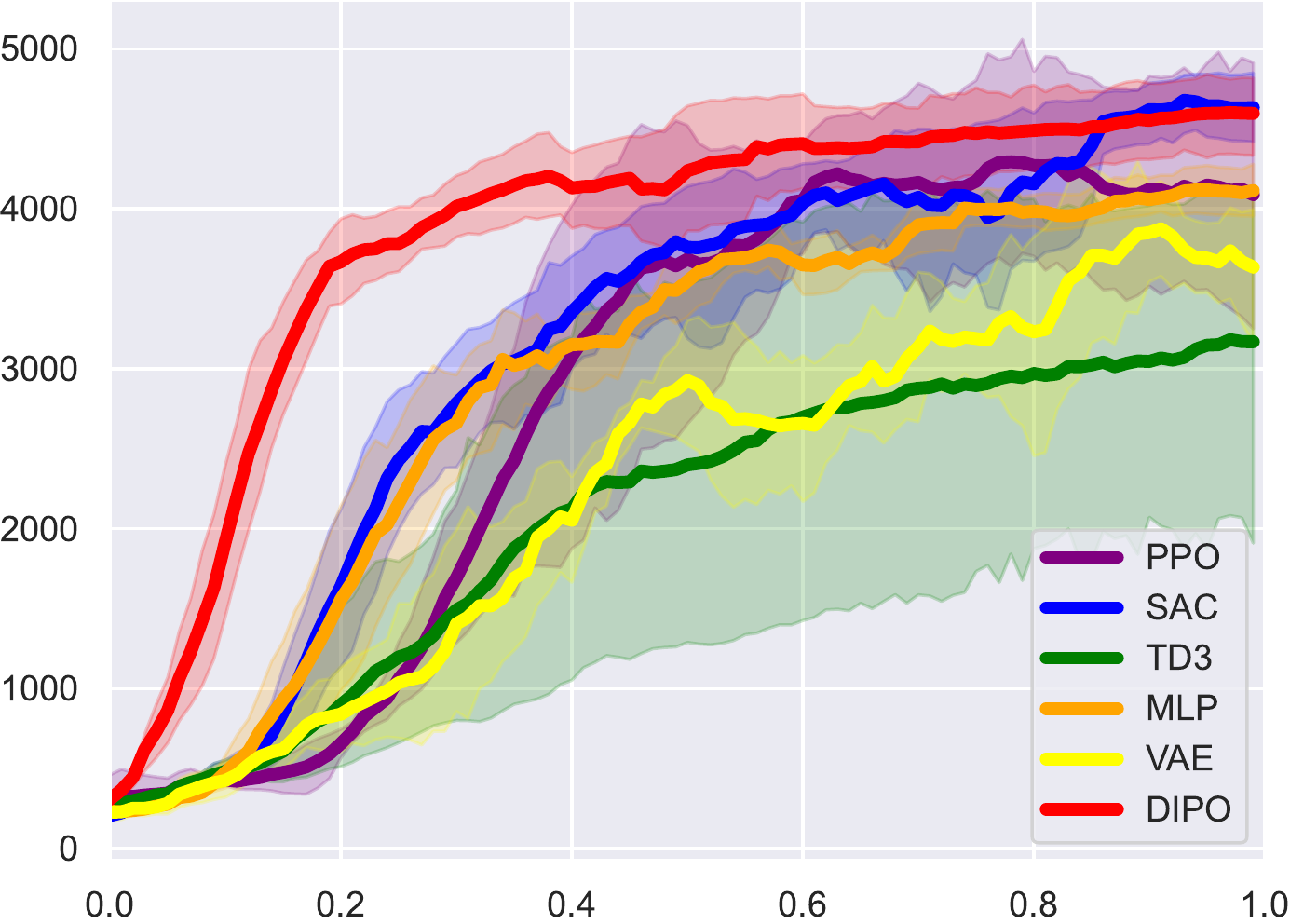}}
              \caption
    { Average performances on MuJoCo Gym environments with $\pm$ std shaded, where the horizontal axis of coordinate denotes the iterations $(\times 10^6)$, the plots smoothed with a window of 10.
    %and the optimal policy is to go to one of the goal positions randomly.
    %, and we have shown the solid red lines as the final action learned by different algorithms with 1000 iterations.
    }
    \label{app-Ablation-Study}
\end{figure*}

\subsubsection{Observation 3: PPO Violates above Two Observations}

From all of those 5 tasks (see Figure \ref{state-Humanoid} to  \ref{state-hopper}), we also find PPO violates the common sense of RL, where PPO usual with a narrow state-visiting at the beginning and wide state-visiting at the final iteration.
For example, from Figure \ref{fig:comparsion-mujoco} and \ref{state-hopper}, we know PPO achieves an asymptotic reward performance as DIPO for the Hopper-v3, while the state-visiting distribution of PPO is fundamentally different from DIPO.  
DIPO shows a wide state-visiting region gradually turns into a narrow state-visiting region, 
while PPO shows a narrow state-visiting region gradually turns into a wide state-visiting region.
We show the fair visualization with t-SNE by the same setting for all of those 5 tasks, the abnormal empirical results show that PPO may find some new views different from DIPO/TD3/SAC to understand the environment.

\subsection{Ablation Study on MLP and VAE}

A fundamental question is why must we consider the diffusion model to learn a policy distribution. 
In fact, Both VAE and MLP are widely used to learn distribution in machine learning, can we replace the diffusion model with VAE and MLP in DIPO? 
In this section, we further analyze the empirical reward performance among DIPO, MLP, and VAE.

We show the answer in Figure \ref{dipo-vae-mlp} and Figure \ref{app-Ablation-Study}, where the VAE (or MLP) is the result we replace the diffusion policy of DIPO (see Figure \ref{fig:framework}) with VAE (or MLP), i.e.,
we consider VAE (or MLP)+action gradient (\ref{def:improve-action}) for the tasks.

Results of Figure \ref{app-Ablation-Study} show that the diffusion model achieves the best reward performance among all 5 tasks.
This implies the diffusion model is an expressive and flexible family to model a distribution, which is also consistent with the field of the generative model.

Additionally, from the results of Figure \ref{app-Ablation-Study} we know MLP with action gradient also performs well among all 5 tasks, which implies the action gradient is a very promising way to improve reward performance.
For example, Humanoid-v3 is the most challenging task among Mujoco tasks, MLP achieves a final reward performance near the PPO, SAC, DIPO, and TD3.
We all know that these algorithms (PPO, SAC, DIPO, and TD3) are meticulously constructed mathematically, while MLP with action gradient is a simple model, but it achieves so good reward performance, which is a direction worth further in-depth research to search simple but efficient RL algorithm.

\end{document}